\documentclass{article}
\usepackage[preprint]{neurips_2026}

\usepackage{microtype}
\usepackage{graphicx}
\usepackage{booktabs}
\usepackage{hyperref}

\usepackage{amsmath}
\usepackage{amssymb}
\usepackage{mathtools}
\usepackage{amsthm}
\usepackage{booktabs,array,longtable}
\usepackage{geometry}
\usepackage{float}
\usepackage[capitalize,noabbrev]{cleveref}

\usepackage{tikz}
  \usepackage{pgfplots}
  \usepgfplotslibrary{groupplots}
  \pgfplotsset{compat=1.17}
\usetikzlibrary{arrows.meta}
\usetikzlibrary{backgrounds}
\usepgfplotslibrary{patchplots}
\usepgfplotslibrary{fillbetween}
\pgfplotsset{%
    layers/standard/.define layer set={%
        background,axis background,axis grid,axis ticks,axis lines,axis tick labels,pre main,main,axis descriptions,axis foreground%
    }{
        grid style={/pgfplots/on layer=axis grid},%
        tick style={/pgfplots/on layer=axis ticks},%
        axis line style={/pgfplots/on layer=axis lines},%
        label style={/pgfplots/on layer=axis descriptions},%
        legend style={/pgfplots/on layer=axis descriptions},%
        title style={/pgfplots/on layer=axis descriptions},%
        colorbar style={/pgfplots/on layer=axis descriptions},%
        ticklabel style={/pgfplots/on layer=axis tick labels},%
        axis background@ style={/pgfplots/on layer=axis background},%
        3d box foreground style={/pgfplots/on layer=axis foreground},%
    },
}
\usepackage{times}

\def\1{\bm{1}}
\newcommand{\train}{\mathcal{D}}

\def\rve{{\mathbf{e}}}

\def\rvu{{\mathbf{i}}}

\def\rvs{{\mathbf{s}}}

\def\rvu{{\mathbf{u}}}
\def\rvv{{\mathbf{v}}}
\def\rvw{{\mathbf{w}}}
\def\rvx{{\mathbf{x}}}
\def\rvy{{\mathbf{y}}}
\def\rvz{{\mathbf{z}}}

\def\sR{{\mathbb{R}}}

\def\sZ{{\mathbb{Z}}}

\newcommand{\Ls}{\mathcal{L}}

\newcommand{\KL}{\mathrm{KL}}

\DeclareMathOperator*{\argmin}{arg\,min}

\renewcommand{\KL}{\mathrm{KL}}
\title{PAC-Bayesian Bounds for Learning Partially Observed Stochastic Linear Time-Invariant State-Space Systems with Inputs and Sub-Gaussian Noise}

\author{
Mihaly Petreczky \\
Univ. Lille, CNRS, Centrale Lille, UMR 9189
CRIStAL, F-59000 Lille,\\
\And
Mohamad Al Ahab \\
Department of Electronic Systems, Aalborg University, Denmark \\
\And
John Leth \\
Department of Electronic Systems, Aalborg University, Denmark \\
}

\newtheorem{theorem}{Theorem}[section]

\newtheorem{lemma}[theorem]{Lemma}
\newtheorem{corollary}[theorem]{Corollary}

\newtheorem{definition}[theorem]{Definition}
\newtheorem{assumption}[theorem]{Assumption}

\newtheorem{remark}[theorem]{Remark}

\newcommand{\e}{\mathbf{e}}

\newcommand{\x}{\mathbf{x}}
\newcommand{\y}{\mathbf{y}}
\newcommand{\z}{\mathbf{z}}
\newcommand{\w}{\mathbf{w}}
\newcommand{\bv}{\mathbf{v}}
\newcommand{\reals}{\mathbb{R}}

\newcommand{\bP}{\mathbf{P}}
\newcommand{\bE}{\mathbf{E}}
\newcommand{\F}{\mathcal{F}}
\newcommand{\B}{\mathcal{B}_\Theta}

\usepackage{pgfplots}
\pgfplotsset{compat=1.17}
\begin{document}
\maketitle 

\begin{abstract}
In this paper we derive a Probably Approximately Correct (PAC)-Bayesian error bound for partially observed linear time-invariant (LTI) stochastic dynamical systems in state-space form with inputs.  Such bounds
are widespread in machine learning, and they are useful for characterizing the predictive power of models learned from  
finitely many data points.  
The bound derived in this paper relates the expectation of prediction errors with the prediction error generated by the model on the data used for learning. In addition, we show that it can also be used to derive bounds for the parameter estimation error.
In turn, this allows us to provide finite-sample error bounds  for the prediction error and parameter estimation error for
a wide class of system identification algorithms. 
\end{abstract}

\section{Introduction}


PAC and PAC-Bayesian bounds are fundamental tools for analyzing learning algorithms. They provide generalization guarantees that relate the \emph{true error} (performance on unseen data) to the \emph{empirical error} (performance on training data), by bounding the \emph{generalization gap}, the difference between true and empirical error, in a way that is independent of the specific learning algorithm. In this framework, learning is viewed as selecting a model from a data-dependent posterior distribution over hypotheses, which is a modification of a prior based on the observed data. Different learning algorithms correspond to different choices of priors, posteriors, and selection mechanisms (e.g., sampling, averaging, or maximizing the posterior). PAC-Bayesian bounds then provide a uniform upper bound on the average generalization gap with respect to any posterior, thus offering generalization guarantees for a broad class of learning and model selection methods.

This generality makes PAC-Bayesian approaches especially useful for analyzing complex models, including neural networks, where they can yield non-vacuous generalization bounds \cite{Dziugaite2017}.
Although PAC and PAC-Bayesian theory is well developed for i.i.d.\ settings \cite{shalev2014understanding,alquier2021userfriendly,guedj2019primer}, there are comparatively few results for dynamical systems. 
\begin{color}{black}
This paper presents PAC-Bayesian bounds for learning discrete-time stochastic linear time-invariant (LTI) state-space systems with sub-Gaussian noise and partial observations (the full states is not observed) from a single time-series. 
LTI systems are widely used in control theory to model physical systems. Learning such models is known as \emph{system identification} in control theory. The core problem of the former is the same as that of machine learning: learning predictive models from data.
The main difference is that system system identification prioritizes model classes useful for control design. For this reason, it often focuses on parameter estimation error over generalization error, as the former is more relevant for control. 
Learning theory of LTI systems is also relevant for machine learning, since LTI systems are special cases of such popular architectures as recurrent neural networks (RNNs) or selective state-space models (SSMs). 
   A detailed discussion on machine learning and control, the relationship between the learning problem of this paper and classical learning theory, and possible extensions of the current work to more general settings (e.g., RNNs) is presented  in Appendix \ref{extensions}.
\par
   \textbf{Contribution:}
  We derive PAC-Bayesian error bounds for  both the generalization gap and the parameter estimation error for learning LTI models, rendering
  these bounds are relevant for both system identification and machine learning.
    Following standard practice in system identification, we view LTI systems as predictors that generate predictions for the current output based on past inputs and outputs.
    Our main contribution is a PAC-Bayesian bound showing that, with high probability (at least $1-\delta$), for any posterior distribution over predictors, the average generalization gap is $O(\frac{(\ln N)^r}{\sqrt{N}})$, where $N$ is the sample size, $r=0$ for bounded noise, $r=1$ for Lipschitz loss, and $r=3/2$ for quadratic loss. The bound depends on the Kullback--Leibler (KL) divergence between prior and posterior distributions and on $\ln(1/\delta)$. We subsequently derive generalization gap and parameter estimation error bounds for concrete learning algorithms (e.g., maximum likelihood, regularized prediction-error methods (PEM), and Bayesian methods).
\end{color}

    \textbf{Related work:}
    PAC and PAC-Bayesian bounds for autoregressive models with bounded signals or Lipschitz loss have been studied in \cite{campi2002finite,vidyasagar2006learning,alquier2012pred,alquier2013prediction,MASSUCCI2022110532,METAKALARD2026100373}. 
    In contrast, our results address state-space models with partially observed states and  with inputs, which are more general than autoregressive models, and we  allow for unbounded (sub-Gaussian) noise, and handle quadratic loss. We also build on techniques from \cite{alquier2012pred,alquier2013prediction}. 
\par
In contrast to the literature on finite-sample parameter estimation bounds for specific least-squares based learning
algorithms for LTI systems, e.g., \cite{simchowitz2019learning,simchowitz2021statistical,oymak2021revisiting,lale2020logarithmic,foster2020learning,NEURIPS2018_d6288499,Ziemann1,Ziemann2,SarkarRD21}, this papers presents bounds on both the parameter estimation and the generalization gap for a wide-range of algorithms, see Remark \ref{rem:finite-sample}, Appendix \ref{app:indent} for a detailed comparison.
\par
PAC bounds for recurrent neural networks, of which LTI state-space representations are a subclass, have been developed using VC dimension \cite{KOIRAN199863,sontag1998learning}, Rademacher complexity \cite{KOIRAN199863,sontag1998learning,pmlr-v108-chen20d}, and PAC-Bayesian approaches \cite{Dziugaite2017}. However, existing results assume noiseless models, fixed prediction horizons, i.i.d.\ training data samples, and bounded signals. In contrast, this paper considers noisy models, infinite-horizon prediction errors, a single observed time series, and unbounded signals with sub-Gaussian noise.
\par
In \cite{eringisRenyi}, bounds based on R\'enyi divergence were proposed, but the resulting probabilistic guarantees scale as $O(1/\delta)$ with the confidence level, rather than the standard $O(\ln(1/\delta))$, and it applies only to quadratic losses as opposed to general Lipschitz losses, and it provides parameter estimation bounds only in terms of the empirical error. 
In contrast, the bound of this paper depends on the confidence level $O(\ln(1/\delta))$ and it provides an absolute bound 
on the parameter estimation error. 
The works \cite{eringis2023pacbayesRNN,eringis2023pacbayesian} provide PAC-Bayesian generalization gap bounds for RNNs and LTI systems with bounded noise, respectively. In contrast, the present paper addresses unbounded (sub-Gaussian) noise and also derives parameter estimation error bounds. Even for bounded signals, our results refine the constants used in \cite{eringis2023pacbayesRNN,eringis2023pacbayesian} (see Remark~\ref{rem:deividas}, Appendix~\ref{app:indent}). 
PAC-Bayesian bounds for autonomous LTI systems without inputs were considered in \cite{CDC21paper} but those bounds did not converge to zero as $N \rightarrow \infty$.


\textbf{Paper Outline:}
  We start by defining the problem formulation in Section \ref{sect:Problem_formulation},
  followed by the main results in Section \ref{sec:main_results} and discussion of the bounds in Section \ref{sec:conclusions}.
  \textcolor{black}{The technical proofs and numerical examples are presented in the appendices.}
  

\section{Problem formulation}\label{sect:Problem_formulation}
\subsection{ Notation and terminology} 
We occasionally use $\triangleq$ to denote ''defined by''. 
Let $\F$ denote a $\sigma$-algebra on the set $\Omega$ and $\bP$ be a probability measure on $\F$. Unless otherwise stated, all probabilistic considerations will be with respect to the probability space $(\Omega,\F,\bP)$, and we let $\bE[\z]$ denote expectation of the stochastic variable $\z$.
We use bold face letters to indicate stochastic signal vectors and processes relative to $(\Omega,\F,\bP)$. 
Each euclidean space is associated with the topology generated by the 2-norm $\|\cdot\|_2$, and the Borel $\sigma$-algebra generated by the open sets. The induced matrix 2-norm is also denoted $\|\cdot\|_2$ 
\textcolor{black}{and $\|\cdot\|_{\infty}$ denote the infinity norm.}
\par
Let $B_{\Theta}$ be the $\sigma$-algebra of Lebesgue-measurable subsets of $\Theta \subset \mathbb{R}^{n_\theta}$, and $m$ the Lebesgue measure. For a probability density $\rho$ on $(\Theta, B_{\Theta}, m)$ and a measurable, integrable function $g:\Theta \to \mathbb{R}$, define
\(
E_{\theta \sim \rho}[g(\theta)] \triangleq \int_{\Theta} \rho(\theta) g(\theta) dm(\theta)
\)
as the expectation of $g$ under $\rho$. For a density $\pi$ on $\Theta$, let $\mathcal{M}_{\pi}$ denote the set of all densities on $\Theta$ 
whose probability measure  is absolutely continuous w.r.t. the probability measure of $\pi$. 
\par
Let $\mathrm{vol}(A) = \int_{A} dm(\theta)$ denote the Lebesgue measure of $A \subseteq \mathbb{R}^{n_\theta}$. The uniform distribution on a measurable set $A \in B_\Theta$ is $\rho(\theta) = \frac{\chi_A(\theta)}{\mathrm{vol}(A)}$, where $\chi_A$ is the indicator function. For $A = \Theta$, this is the uniform distribution on $\Theta$.
\par
The phrase \emph{probability over data} refers to probability under $\bP$ on $(\Omega, \F)$. Specifically, for measurable functions $\{X_i, Y_i\}_{i \in I}$ and a binary relation $\diamond$, we say $\forall i \in I: X_i \diamond Y_i$ holds with probability at least $1-\delta$ over data if the set $F = \{\omega \in \Omega : \forall i \in I,\, X_i(\omega) \diamond Y_i(\omega)\}$ satisfies $\bP(F) \ge 1-\delta$.

\textcolor{black}{A complete list of symbols can be found in Appendix \ref{sect:symbols}.}
\subsection{\textcolor{black}{Problem formulation}}
The problem formulation follows \cite{eringis2023pacbayesRNN,eringis2023pacbayesian,eringisRenyi}.  
\par
Let us fix stationary stochastic processes $\y(t)\in \reals^{n_\y}$, and $\rvu(t)\in  \reals^{n_\rvu}$, that share a time axis $t\in\sZ$, that is, for any $t\in\sZ$, $\rvy(t):\Omega\to\sR^{n_\rvy};\omega\mapsto\rvy(t)(\omega)$, and $\rvu(t):\Omega\to\sR^{n_\rvu};\omega\mapsto\rvu(t)(\omega)$
are random vectors on $(\Omega,\F,\bP)$:
$\y$ represents outputs (labels), and $\rvu$ represents inputs (features).
The goal is to learn models from the finite sample $\train=\{\rvy(t)(\omega),\rvu(t)(\omega)\}_{t=0}^{N-1}$, $\omega \in \Omega$, such that the learned model
predicts  $\rvy$  from past measurements of $(\rvy,\rvu)$ sufficiently well.

\textbf{Predictor class:}
To simplify notation, let us define a stochastic process $\w(t)\in \sR^{n_\w}$ by 
\begin{description}
    \item $\w(t)=\begin{bmatrix} \y^T(t), & \rvu^T(t) \end{bmatrix}^T$ when $n_\w=n_\y+n_\rvu$ and both past $\y$ and $\rvu$ are used for prediction,
    \item $\w(t)=\rvu(t)$ when  $n_\w=n_\rvu$, when only past $\rvu$ is used for prediction.
\end{description}
Note that with a trivial extension of the results of this paper one can define $\w(t)$ to consist of only some of the components of $\y(t)$, extending the results to not only prediction but also filtering.

\begin{assumption}[Predictors]\label{as:parameterisation}
The class of predictors 
is the class of LTI systems parametrized by $\theta \in \Theta \subseteq \mathbb{R}^{n_{\theta}} $ which generate 
predictions $\hat{\y}_\theta(t|t_0)$ of $\y(t)$ based on $\{\w(s)\}_{s=t_0}^{t}$
\begin{equation}\label{eq:predictor}
    \begin{aligned}
        \hat{\x}(t+1)=\hat{A}_\theta\hat{\x}(t)+\hat{B}_\theta\w(t), ~ \hat{\x}(t_0)=0, \quad \hat{\y}_\theta(t|t_0)=\hat{C}_\theta\hat{\x}(t)+\hat{D}_\theta\w(t), \quad t \ge t_0,
    \end{aligned}
\end{equation}
  where $\hat{A}_\theta$,$\hat{B}_\theta,\hat{C}_\theta,\hat{D}_\theta$ 
	are $\hat{n}\times\hat{n}$, $\hat{n}\times n_\w$,
	$n_\y\times \hat{n}$ and $n_\y\times n_\w$ matrices
	respectively, they are are continuous in $\theta$,
  and 
  $\hat{A}_\theta$ is Schur (all its eigenvalues are inside the unit disk).
  Moreover, if $\w(t)=[\y^T(t)~\rvu^T(t)]^T$, then  $\hat{D}_\theta=[0~\hat{D}_{\rvu,\theta}]$ for some $n_\y \times n_\rvu$ matrix $\hat{D}_{\rvu,\theta}$. 
  \footnote{The latter assumption is necessary, as it guarantees that $\hat{D}_\theta\w(t)$ depends only on $\rvu(t)$ and hence it does not depend on $\rvy(t)$, since otherwise we would be using $\rvy(t)$ to predict $\rvy(t)$, which is not meaningful.}
    Moreover, there exist constants
     $M,C  >0$, $\gamma \in (0,1)$ such that 
   for all $\theta \in \Theta$, $\|\hat{A}_\theta^k\|_2 \le M \gamma^k$, $\max\{\|\hat{B}_\theta\|_2,\|\hat{C}_\theta\|_2,\|\hat{D}_\theta\|_2\} \le C$.
   \textcolor{black}{In addition, we assume that the Markov parameteres of \eqref{eq:predictor} are not all zero, i.e., either $\hat{D}_{\rvu,\theta} \neq 0$ or $\hat{C}_\theta \hat{A}_\theta^k \hat{B}_\theta \neq 0$ for some $k \ge 0$.}
\end{assumption}
Intuitively, our assumption expresses that the predictors are stable LTIs and we have a uniform upper bound on the matrices and on the spectral radius of $\hat{A}_\theta$. These assumptions are standard in system identification, see Remark \ref{rem:parameter:methods}, Appendix \ref{app:sect:Problem_formulation}.
Note that the predictors are started at the zero initial state at inital time $t_0$. The latter represents the time instant at which the predictor was applied to generate predictions, or when learning started. Most of the time, we will set $t_0=0$ to simplify notation.

\textbf{Data generator:}
Intuitively, we would like 
the data  $\y$ to be generated by an LTI system 
\begin{align}
    \begin{split}
		&\hat{\rvs}_{g}(t+1)=A_0\hat{\rvs}_{g}(t)+B_0\rvu(t)+K_0\e^s(t), \quad \y(t) = C_0\hat{\rvs}_{g}(t)+D_0\rvu(t)+\e^s(t), \quad t \in \sZ,
    \end{split}     
		\label{eq:T0sys1}
\end{align}
where $A_0,B_0,K_0,C_0,D_0$ are matrices of suitable dimension, $A_0$ and $A_0-K_0C_0$ are both Schur, and $\e^s(t)$ is a zero-mean i.i.d. sub-Gaussian process. The Schur conditions ensure stationarity of $\rvx$ and that $\e^s(t)$ is the innovation process (see \cite[Chapter 4, p.~87]{LjungBook}), i.e., the difference between $\y(t)$ and its best linear predictor based on past inputs and outputs. Thus, \eqref{eq:T0sys1} is in steady-state Kalman filter form.
Unlike predictors, the data generator \eqref{eq:T0sys1} is assumed to operate in steady state, i.e., it has run for infinite time prior to prediction. This standard assumption in system identification (e.g., \cite{LjungBook}) reflects that the physical system has been subject to noise and inputs long before learning begins. This subtle difference has significant theoretical implications.

Assuming the data is generated by a system of the same form as the predictors is a realizability assumption, as \eqref{eq:T0sys1} yields the optimal predictor.
\begin{equation}
    \begin{split}
		&\hat{\rvs}_{g}(t+1 \mid t_0)\!\!=\!\!\hat{A}_0\hat{\rvs}_{g}(t)+\hat{B}_0\w(t), ~ \hat{\rvs}_g(t_0 | t_0)=0, \quad  
    \hat{\y}(t \mid t_0) \!\!=\!\! \hat{C}_0\hat{\rvs}_{g}(t \mid t_0)+\hat{D}_0\w(t), 
    \end{split}     
		\label{eq:T0sys2}
\end{equation}
where $t \ge t_0$,
 $\hat{A}_0=A_0,\hat{B}_0=B_0,\hat{C}_0=C_0$, $\hat{D}_0=D_0$
if $n_w=n_u$ and 
$\hat{A}_0=A_0-K_0C_0$, $\hat{B}_0=\begin{bmatrix} K_0 & B_0-K_0D_0 \end{bmatrix}$,
$\hat{C}_0=C_0$, $\hat{D}_0=\begin{bmatrix} 0 & D_0 \end{bmatrix}$ if $n_w=n_u+n_y$. 
When $\w(s)=[\rvy^T(s),\; \rvu^T(s) ]^T$, 
the predictor \eqref{eq:T0sys2} corresponds 
to stationary Kalman-filter and it is an optimal predictor of $\y(t)$ in the mean square error sense.
  Formally, instead of assuming the existence of \eqref{eq:T0sys1}, for the data generator  we use the following assumption, which is essentially equivalent to the existence of \eqref{eq:T0sys1}.
\begin{assumption}\label{as:generator}
	Let $\y(t)$ and $\mathbf{u}(t)$ be generated by an autonomous stochastic LTI system
    \begin{align}\label{eq:generator}
			\rvx(t+1) &=A_g\rvx(t)+K_g\rve_g(t), \quad
			\begin{bmatrix}\rvy(t)\\\rvu(t)\end{bmatrix}=C_g\rvx(t)+\rve_g(t)
		\end{align}
	with $A_g \in \mathbb{R}^{n \times n},K_g \in \mathbb{R}^{n \times n_\rve},C_g \in \mathbb{R}^{n_\rve \times n}$, for some $n>0$, and $n_\rve=n_\rvy+n_\rvu\geq2$, $A_g$ and $A_g-K_gC_g$ are Schur. Moreover, 
		$\rvx(t)$ has finite variance, $\rve_g(t)$ is independent of $\rvx(t)$ for all $t \in \mathbb{Z}$, 
		and $\rve_g$ is square integrable, i.i.d process,
    $\|\rve_g(t)\|_2$ is sub-Gaussian and the sub-Gaussian norm of $\|\rve_g(t)\|_2$
    \cite[Definition 2.6.4]{vershynin2026highdimensional} is denoted by $\|\rve_g\|_{\psi_2}$.
\end{assumption}
If $\e_g(t)$ is sub-Gaussian, then $\|\e_g(t)\|_2$ is also sub-Gaussian, see
\cite[Lemma 1]{jin2019short}.
Under Assumption \ref{as:generator}, 
$\rvy$ and $\rvu$ are jointly stationary and square integrable processes
\cite{CainesBook}.
%
{\color{black}
   Under mild 
   conditions, Assumption \ref{as:generator} is equivalent to the existence of  an LTI system
   \eqref{eq:T0sys1}, see Remark \ref{rem:equiv_gen:proof} in Appendix \ref{app:sect:Problem_formulation}. 
   Assumption \ref{as:generator}
   is more convenient, and it is often
   used in system identification \cite[Chapter 9]{Katayama:05}.
}

\textbf{Measure of prediction accuracy:} 
Following standard practice in system identification \cite{LjungBook}, we measure the predictive accuracy
using a \emph{loss function}
$\ell: \mathbb{R}^{n_y} \times \mathbb{R}^{n_y} \rightarrow [0,\infty)$
which is continuous and satisfies $\ell(y,y)=0$.
Typical choices are 
  $\ell(y,\hat{y})=\|y-\hat{y}\|_{2}$, to which we refer to as $\ell_2$ loss,
  and 
    \( \ell(y,\hat{y})=\|y-\hat{y}\|_2^2 \) to which we refer to as the \emph{quadratic loss}. 
\begin{assumption}[Qudaratic or Lipschitz loss functions]
\label{as:loss}
 The loss function $\ell$ is either quadratic or it is $L_{\ell}$-Lipschitz, i.e., for all $y_1,y_2,\hat{y}_1,\hat{y}_2 \in \mathbb{R}^{n_y}$,
\( |\ell(y_1,\hat{y}_1)-\ell(y_2,\hat{y}_2)| \le L_{\ell} (\|y_1-y_2\|_2+\|\hat{y}_1-\hat{y}_2\|_2). \)
\end{assumption}
\begin{color}{black}
   That is, we consider the loss to be Lipschitz or quadratic, which, for the bounded case is also Lipschitz. Lipschitzness of the loss is standard assumptions for PAC bounds. Intuitively, this allows us to reduce the case of a general loss 
   to the case when the loss is the norm of the prediction error. The quadratic case is considered due to its popularity in system identification.
   For more details on the role of losses in the proof of the results see Appendix \ref{loss:assumption}
  \end{color}

We define the \emph{true error} as the limit
\begin{align}
\mathcal{L}(\theta) &\triangleq\lim_{t_0 \rightarrow -\infty} \bE[\ell(\y(t),\hat{\rvy}_\theta(t\mid t_0))]. \label{eq:GenLoss}
\end{align}  
The true error captures the long-term, steady-state prediction error obtained during the deployment of the predictor on an increasing number of past inputs.
\begin{lemma}\label{l:ihp} Let Assumptions \ref{as:generator}, \ref{as:parameterisation}, and \ref{as:loss} hold.
Then the limits 
	\( \hat{\rvy}_\theta(t)=\lim_{t_0 \rightarrow -\infty} \hat{\rvy}_\theta(t\mid t_0) \)
  and $\lim_{t \rightarrow \infty} (\hat{\rvy}_{\theta}(t)-\hat{\rvy}_{\theta}(t \mid 0))=0$
	exist in the mean-square sense for all $t$ and $\theta$, the process $\hat{\rvy}_\theta(t)$ is stationary, and 
	\(	\mathcal{L}(\theta)=\bE[\ell(\y(t),\hat{\rvy}_\theta(t) )]=\lim_{t \rightarrow \infty} \bE[\ell(\y(t),\hat{\rvy}_\theta(t\mid 0))]\). 
\end{lemma}
The proof of Lemma \ref{l:ihp} is presented in Appendix \ref{app:sect:Problem_formulation}. 
Intuitively, $\hat{\rvy}_{\theta}(t)$
is the output predicted by the LTI system corresponding to
$\theta$, if the latter is started from zero at $-\infty$, and as $t$ grows, $\hat{\rvy}_{\theta}(t)$
and $\hat{\rvy}_{\theta}(t \mid 0)$ converge to each other.

\textbf{The Learning Problem} is to 
find a parameter $\theta_{\star}$  from sampled data $\train$ of the sequence of random variables $\{\rvy(t),\rvu(t)\}_{t=0}^{N-1}$ which minimizes the true error, i.e., $\theta_{\star}\in\arg\min \mathcal{L}(\theta)$.
Note that we do not assume the knowledge of the matrices $A_g,K_g,C_g$ and noise process $\e_g$ when learning. 

Since the true error is unknown, a standard practice is to use the \emph{empirical error} 
\begin{align}
    \hat{\mathcal{L}}_N(\theta)=\frac{1}{N}\sum_{t=0}^{N-1}\ell(\rvy(t),\hat{\rvy}_\theta(t|0))
\end{align}
as a proxy for the true error, and try to find a 
model $\hat{\theta}$ for which the empirical error 
$\hat{\mathcal{L}}_N(\hat{\theta})$ is small. 
While doing so, we would like to keep the \emph{generalization gap} $\mathcal{L}(\theta)-\hat{\Ls}_N(\theta)$, i.e., the difference between a model's performance on the training data (empirical error) and its expected performance (true error) on unseen test data, small.  The PAC-Bayesian approach provides a framework to derive such guarantees.

\subsection{Constants used in PAC-bounds }\label{sec:asm}
In order to present the main result,  we will need several constants, many of which have an interesting system-theoretic interpretation.

   
\begin{definition}[Convolution and $\ell_1$ norm]
  Let $\alpha=\{\alpha(k)\}_{k=0}^{\infty}$ be a sequence of matrices.
  Define the $\ell_1$ norm $\|\alpha\|_{\ell_1}$ and 
  the \emph{weakly dependent mixing coefficient (WDMC)} $\theta_{\infty}(\alpha)$ of $\alpha$ as
  \[ \|\alpha\|_{\ell_1}=\sum_{k=0}^{\infty} \|\alpha(k)\|_2, \quad  \theta_{\infty}(\alpha)=\sum_{k=0}^{\infty} k \|\alpha(k)\|_2. \]
  If $\|\alpha\|_{\ell_1} < \infty$ and $\mathbf{r}$ is a stationary square integrable stochastic process, then define the \emph{convolution $(\alpha \star \mathbf{r})$ of $\alpha$ with $\mathbf{r}$}  as the stochastic process
  \( (\alpha \star \mathbf{r})(t)=\sum_{k=0}^{\infty} \alpha(k) \mathbf{r}(t-k) \). 
  By \cite[Theorem 1.1]{CainesBook}, $(\alpha \star \mathbf{r})$ is well-defined and stationary. 
\end{definition}
  \textcolor{black}{Note that the matrices  $\alpha=\{\alpha(k)\}_{k=0}^{\infty}$ are deterministic quantities, but their convolution
  $(\alpha \star \mathbf{r})$ is a stochastic process.}
  Intuitively, the $\ell_1$ norm of a sequence is the extension of the standard norm on the space of
  absolutely summable scalar sequences to the case of matrix valued ones. The convolution 
  $(\alpha \star \mathbf{r})$ represents the response of a linear filter with impulse response $\alpha$
  to the noise $\mathbf{r}$, and $\|\alpha\|_{\ell_1}$ represents the $\ell_1$ norm of the filter.
  If  $\mathbf{r}$ is bounded, i.e., $\|\mathbf{r}(t)\|_2 < C$, then 
  the convolution  $(\alpha \star \mathbf{r})(t)$ is bounded, i.e., $\|(\alpha \star \mathbf{r})\|_2 < \|\alpha\|_{\ell_1} C$, and it is a  weakly dependent process, \textcolor{black}{in the sense of \cite{alquier2012pred};
  the corresponding mixing coefficients $\{\theta_{\infty,n}(1)\}_{n=1}^{\infty}$ which is defined in \cite{alquier2012pred}, and whose definition is
  recalled in Appendix \ref{mixing:def},} 
  are bounded by the $2\theta_{\infty}(\alpha)C$, \cite[Proposition 4.12]{alquier2012pred}.
   Intuitively, the smaller $\theta_{\infty}(\alpha)$ is,
  the closer $(\alpha \star \mathbf{r})(t)$ is to being i.i.d., if $\mathbf{r}$ is i.i.d. 
  If $\alpha$ is the sequence of Markov parameters of a stable LTI system, then both $\|\alpha\|_{\ell_1}$ and $\theta_{\infty}(\alpha)$ are finite and can be upper bounded in terms of bounds on the norms of the system matrices and the spectral radius of the system matrix $A$, see Remark \ref{norms:rem}.
 \begin{definition}[Auxiliary definitions]
	\label{def:constants}
  Let us consider the Markov-parameters $\alpha_g$ and 
    $\alpha_{\theta}$ of the data generator \eqref{eq:generator} and predictor 
    \eqref{eq:predictor} corresponding to
    $\theta \in \Theta$ respectively:
    \begin{equation}
    \label{alpha:lemma:alpha_def}
    \begin{aligned}
    & \alpha_{g}(k)=\left\{\begin{array}{rl}
            C_g A_g^{k-1}K_g & k > 0 \\
            I & k = 0
    \end{array}\right., \quad   \alpha_{\theta}(k)=\left\{\begin{array}{rl}
            \widehat{C}_{\theta}\widehat{A}_{\theta}^{k-1}\widehat{B}_{\theta} & k > 0 \\
           \widehat{D}_{\theta}   & k = 0
    \end{array}\right. \\
\end{aligned}
\end{equation}
 Let $\bar{\alpha}(k)=\sup_{\theta \in \Theta} \|\alpha_{\theta}(k)\|_2$.
 With notation as above
  \begin{align*}
       & G_e(\theta) \triangleq (1+\|\alpha_{\theta}\|_{\ell_1})\|\alpha_g\|_{\ell_1}, \quad   G_{e,1}(\theta) \triangleq \theta_{\infty}(\alpha_{\theta})\|\alpha_g\|_{\ell_1}+\theta_{\infty}(\alpha_g)(\|\alpha_{\theta}\|_{\ell_1}+1) \\
        & G_e \triangleq (1+\|\bar{\alpha}\|_{\ell_1})\|\alpha_g\|_{\ell_1}, \quad 
        G_{e,1} \triangleq \theta_{\infty}(\bar{\alpha})\|\alpha_g\|_{\ell_1}+\theta_{\infty}(\alpha_g)(\|\bar{\alpha}\|_{\ell_1}+1)
   \end{align*}
\end{definition}
For the definition of the upcoming bounds, we are going to use the quantities
$\|\alpha\|_{\ell_1}$, $\theta_{\infty}(\alpha)$ for 
$\alpha \in  \{\alpha_g, \bar{\alpha}\} \cup \{\alpha_\theta \mid \theta \in \Theta\}$ and
the constants $G_e(\theta), G_{e,1}(\theta), G_e, G_{e,1}$ defined above. 
Note that all these quantities are well-defined, see Lemma \ref{alpha:lemma1.5}. In addition, they are related to 
 the true outputs $\rvy(t)$ and predicted outputs $\hat{\rvy}_\theta(t)$, $\hat{\rvy}_\theta(t \mid t_0)$ as follows:
 $\begin{bmatrix}\rvy^T(t) &  \rvu^T(t)\end{bmatrix}^T=(\alpha_{g} \star \e_g)(t)$,
 $\hat{\rvy}_{\theta}(t)=(\alpha_{\theta} \star \rvw)(t)$, 
 $\hat{\rvy}_{\theta}(t \mid t_0)=(\alpha_{\theta, t-t_0} \star \rvw)(t)$.
The  formal relationship between the constants is provided in 
Lemma \ref{alpha:lemma2}, Appendix \ref{app:constants}. 
The interpretation of the various terms appearing in  Definition \ref{def:constants} is as follows
\par
\textbf{$\|\alpha_g\|_{\ell_1}$, $\theta_{\infty}(\alpha_g)$}
are the $\ell_1$  norm  and 
the mixing coefficients of Markov-parameters of the data
generator:  both quantities decrease with stability of the data generator,
see Remark \ref{norms:rem}.
\par
\textbf{$\|\alpha_{\theta}\|_{\ell_1}$} is the $\ell_1$ norm of the predictor
 \eqref{eq:predictor}.  
This term decreases with the spectral radius of 
the predictor, see Remark \ref{norms:rem}.  
\par
\textbf{$\theta_{\infty}(\alpha_{\theta})$  } 
describes the mixing properties of the output generated
by the predictor, and it also bounds the sum of differences between 
the infinite and finite past prediction $\sum_{t=0}^{N-1} \|\hat{\rvy}_{\theta}(t)-\hat{\rvy}_{\theta}(t\mid 0)\|$.
Intuitively, the smaller $\theta_{\infty}(\alpha_{\theta})$ is, the closer
$\hat{\rvy}_\theta(t \mid 0)$ is to $\hat{\rvy}_{\theta}(t)$.
This quantity decreases with the spectral radius of the predictor by Remark \ref{norms:rem}.
\par
\textcolor{black}{\textbf{$G_e(\theta)$} characterizes the gain of the system mapping $\e_g$ to
the prediction error $\y(t)-\hat{\y}_{\theta}(t)$ in the standard $\ell_{\infty}$ norm.}
\par
\textbf{$G_{e,1}(\theta)$} characterizes
mixing properties (short memory condition \cite{alquier2013prediction})  of the prediction
error $\y(t)-\hat{\y}_\theta(t)$. Intuitively, the smaller $G_{e,1}(\theta)$ is, the more the prediction error behaves like an i.i.d process. 
This constant is related to stability of the generator and the predictor,
the smaller the spectral radius of those systems are , the
smaller this constant is.  
\par
\textbf{$\|\bar{\alpha}\|_{\ell},\theta_{\infty}(\bar{\alpha}),G_{e},G_{e,1}$} are uniform upper bounds on 
$\|\alpha_{\theta}\|_{\ell_1}$, $\theta_{\infty}(\alpha_{\theta})$,
$G_e(\theta),G_{e,1}(\theta)$ respectively. 

\section{Main result}
\label{sec:main_results}
 We start by presenting the main result of the paper, which is the following theorem.
\begin{theorem}[Main result]
\label{thm:main}
    Let Assumptions \ref{as:generator}, \ref{as:loss} and \ref{as:parameterisation} hold.
    For any a prior \textcolor{black}{density} $\pi$ over $\Theta$, confidence level  $\delta\in(0,0.5)$, integer $\epsilon \in \{-1,1\}$,  and $N \ge K_{\delta}$, 
    the inequality
    \begin{equation}
        \label{eq:initial:tildern1} 
    \begin{aligned}
         \forall \rho & \in \mathcal{M}_\pi: \epsilon E_{\theta\sim\rho} \mathcal{L}(\theta)
         \leq   \epsilon E_{\theta\sim\rho}\hat{\mathcal{L}}_N(\theta) +   \\
          & \frac{1}{\lambda}\Big(\KL(\rho\|\pi)  +\ln\dfrac{C_{\delta}}{\delta} + \frac{1}{2}\sum_{i=1}^2
            \ln E_{\theta\sim\pi} \exp\left(\mathcal{C}_i(2\lambda,\theta,N,\delta,\epsilon)\right) \Big)
    \end{aligned}
    \end{equation}
    holds with probability at least $1-2\delta$ over data, with
    $\KL(\rho\|\pi)=E_{\theta\sim\rho}\ln\frac{\rho(\theta)}{\pi(\theta)}$ the Kullback–Leibler divergence
    and the constants $K_{\delta}$, 
    $\mathcal{C}_i(\lambda,\theta,N,\delta,\epsilon)$, $i=1,2$ and $C_\delta$ are defined as follows:

    \textbf{Bounded noise: }
       If 
       $\|\e_g(t)\|_{\infty} \le c_{\rve}$, then $C_{\delta}=1, K_{\delta}=2$, $\lambda > 0$, and 
    \begin{equation*}
    \begin{aligned}
         & \mathcal{C}_i(\lambda, \theta,N, \delta,\epsilon)  \triangleq 
         \mathcal{C}_{b,i}(\lambda,\theta,L(\theta),c_{\rve},N),  \quad 
         \textcolor{black}{\mathcal{C}_{b,i}(\lambda,\theta,L,c,N) \triangleq 
         \frac{\mathcal{C}_{m,b,i}(\lambda,\alpha_{\theta},G_e(\theta)+2G_{e,1}(\theta), L,c)}{N}} \\
         & L(\theta)   \triangleq \left\{\begin{array}{rl} 2 G_e(\theta)c_{\rve}\sqrt{n_{\rve}} & \mbox{ $\ell$ is quadratic loss} \\
              L_{\ell} &  \mbox{ $\ell$ is $L_\ell$-Lipschitz }
         \end{array}\right.   \\
         &   \mathcal{C}_{m,b,i}(\lambda,\alpha,X,L,\textcolor{black}{c}) 
         \triangleq \left\{ \begin{array}{rl}
            \lambda L\textcolor{black}{c}\sqrt{n_{\rve}}\theta_{\infty}(\alpha)\|\alpha_g\|_{\ell_1} & i=1 \\
             \lambda^2 \frac{L^2}{2} X (\sqrt{n_{\rve}}\textcolor{black}{c})^2 & i=2
          \end{array}\right. 
        \end{aligned}
      \end{equation*}          
          

   \textbf{Lipschitz loss:}
     If  
     $\ell$ is $L_{\ell}$-Lipschitz, then $C_{\delta}=1, K_{\delta}=2$, $0 <\lambda < \frac{N}{\textcolor{black}{32L_{\ell}}}$, and 
     \begin{equation*}
     \begin{aligned}
         & \mathcal{C}_i(\lambda,\theta,N,\delta,\epsilon) \triangleq  \mathcal{C}_{ub,i}(\lambda,\theta,\textcolor{black}{L_{\ell}},N), \\
         & \textcolor{black}{\mathcal{C}_{ub,i}(\lambda,\theta,L,N)\triangleq 
          \frac{1}{2} \mathcal{C}_{b,i}\left(2\lambda,\theta,L,\frac{2\ln(N)}{\|\alpha_g\|_{\ell_1}K},N\right) + \bar{\mathcal{C}}_{ub}(\lambda,L,N)} \\
            &  \bar{\mathcal{C}}_{ub}(\lambda,\textcolor{black}{L},N)  \triangleq \frac{\lambda^2 32 \textcolor{black}{L}^2 C_{u,1}}{N}+ \frac{4\lambda \textcolor{black}{L}  C_{u,2} \ln(N)}{N}
            =
             \frac{\lambda^2 32 \textcolor{black}{L}^2 C_{u,1}}{N}+ \frac{4\lambda \textcolor{black}{L}  C_{u,2} \ln(N)}{N}
     \end{aligned}
    \end{equation*}
    where the constants, $K$,$C_{u,j}$, $j=1,2$ are  defined as follows:
    \begin{equation}
      \label{lipschitz:bound1}
     \begin{aligned}
        &  C_{u,1} \triangleq \bar{\Psi}_{\e_g}(\|\bar{\alpha}\|_{\ell_1}\|\alpha_g\|_{\ell_1})+\bar{\Psi}_{\e_g}(\|\alpha_g\|_{\ell_1}),  \quad \bar{\Psi}_{\rve_g}(s)=e^{14^2s^2 \|\rve_g\|_{\psi_2}^2+3s \|\rve_g\|_{\psi_2}} \\
        & C_{u,2} \triangleq \bar{\Psi}_{\rve}(\|\bar{\alpha}\|_{\ell_1}\|\alpha_g\|_{\ell_1})
        \frac{\textcolor{black}{\|\bar{\alpha}\|_{\ell_1}}}{K}+
        \bar{\Psi}_{\e_g}(\|\alpha_g\|_{\ell_1}) \frac{\textcolor{black}{1}}{K},  \quad
        \quad K=\min\{1,\|\bar{\alpha}\|_{\ell_1}\}
   \end{aligned}
    \end{equation}

 \textbf{Quadratic loss:}   
   if $\ell$ is the quadratic loss, then $C_{\delta}=\textcolor{black}{\frac{3}{2}}$, 
   $K_{\delta}=3$,
   $0 < \lambda < \frac{N}{\textcolor{black}{32} L_{qd}}$,
  \begin{equation*}
  \begin{aligned}
     \mathcal{C}_i(\lambda,\theta,N,\delta,\epsilon)  \triangleq  
     \frac{\lambda L^2_{qd} \max\{0,\epsilon\}}{2N} +
       \mathcal{C}_{ub,i}\left(\lambda,\theta,L_{qd},
       N\right), \quad   L_{qd}=12G_e\|\e_g\|_{\psi_2}\sqrt{\ln\left(\frac{3N^2}{2\delta}\right)} 
  \end{aligned}
  \end{equation*}
\end{theorem}
  For $\epsilon=1$, \eqref{eq:initial:tildern1} bounds the average generalization gap $E_{\theta \sim \rho}(\mathcal{L}(\theta) - \hat{\mathcal{L}}_N(\theta))$; 
  for $\epsilon=-1$, it bounds $E_{\theta \sim \rho}(\hat{\mathcal{L}}_N(\theta) - \mathcal{L}(\theta))$. 
  The latter is useful for deriving oracle inequalities. By applying the union bound to both cases, we obtain a bound on $\left|E_{\theta \sim \rho}(\mathcal{L}(\theta) - \hat{\mathcal{L}}_N(\theta))\right|$.
  Note that except for the quadratic loss case with potentially unbounded data, 
  $\mathcal{C}_i(\lambda,\theta,N,\delta,\epsilon)$ is independent of $\delta$ and $\epsilon$.
  \textcolor{black}{For a detailed discussion on the statement of Theorem \ref{thm:main}  (bounded/unbounded data, Lipschitz/quadratic loss) see Appendix \ref{app:additional_intuition}.}
   

  \textbf{Intuition behind the proof:}
  The proof of Theorem \ref{thm:main} is in Appendix \ref{sect:pf:thm_main}. 
  The key idea is to decompose the generalization gap using an auxiliary empirical error 
    \( V_N(\theta) = \frac{1}{N}\sum_{t=0}^{N-1}\ell(\rvy(t),\hat{\rvy}_\theta(t)) \)
    where the predictor has run for infinite time (started at $t_0=-\infty$). Since $\bE[V_N(\theta)]=\mathcal{L}(\theta)$, we have 
    \(\epsilon(\mathcal{L}(\theta) - \hat{\mathcal{L}}_N(\theta)) = \epsilon (\mathcal{L}(\theta) - V_N(\theta)) + \epsilon (V_N(\theta) - \hat{\mathcal{L}}_N(\theta)) \).
    The constants $\exp(\mathcal{C}_1)$ and $\exp(\mathcal{C}_2)$ bound the moment generating functions of these two differences. The PAC-Bayesian bound then follows by standard concentration arguments \cite[Chapter 2]{alquier2021userfriendly}.
    
    \textbf{Bounded case:} $\mathcal{C}_1$ is obtained by direct calculation, and $\mathcal{C}_2$ via Rios's concentration inequality for weakly-dependent processes \cite[Prop.\ 4.2]{alquier2013prediction}, improving upon constants in \cite{eringis2023pacbayesRNN,eringis2023pacbayesian}.
    
    \textbf{Unbounded case:} We employ a truncation approach \cite{alquier2012pred}. Define truncated data $\bar{\rvy}(t), \bar{\rvu}(t)$ generated from truncated noise $\bar{\e}_g(t)=\e_g(t)\chi(\|\e_g(t)\| \le C)$ with $C=\frac{2\ln(N)}{K\|\alpha_g\|_{\ell_1}}$, where $K=\min\{1,\|\bar{\alpha}\|_{\ell_1}\}$. We bound the moment generating functions of the differences $\sup_{\theta \in \Theta} |\hat{\bar{\mathcal{L}}}_N(\theta) - \hat{\mathcal{L}}_N(\theta)|$ and $\sup_{\theta \in \Theta} |V_N(\theta) - \bar{V}_N(\theta)|$ by $\exp(\bar{\mathcal{C}}_{ub}(\lambda,L_{\ell},N))$. Combining with the bounded-case bounds yields PAC-Bayesian guarantees for unbounded data.
    
    \textbf{Quadratic case:} We show that with probability at least $1-\frac{2\delta}{3}$ over data, the prediction errors satisfy $\|\y(t)-\y_{\theta}(t)\|_2 \le C$ and $\|\y(t)-\y_{\theta}(t\mid 0)\|_2 \le C$ for $C=6G_e\|\e_g\|_{\psi_2}\sqrt{\ln(\frac{3N^2}{2\delta})}$. The empirical quadratic loss then coincides with the empirical error for the truncated Lipschitz loss $\ell_C(y,y')=\min\{\|y-y'\|_2^2, C^2\}$. The result follows by union bound and the Lipschitz case, noting that the true error difference between the quadratic loss and the truncated loss is $O(C^2/N)$.

\textbf{Interpretation and $O(\frac{\ln(N)^{r}}{N})$ rates:}
The interpretation of Theorem \ref{thm:main} is similar to other PAC-Bayesian bounds \cite{alquier2021userfriendly}:
we view learning algorithms as choices of parameters from a  data-dependent probability density on model parameters, 
  referred to as \emph{posterior}
 $\rho^{\train}$, for which 
 the right-hand side of \eqref{eq:initial:tildern1} is small for $\epsilon=1$.
One option is to choose $\rho^{\train}$ as the \emph{Gibbs posterior} 
which minimizes of the right-hand side of \eqref{eq:initial:tildern1} 
and which is given by \cite[Definition 2.1]{alquier2021userfriendly}:
\begin{equation*}
\begin{aligned}
\rho_{\text{Gibbs}} \triangleq  \mathrm{arg min}_{\rho \in \mathcal{M}_{\pi}}
\left(E_{\theta\sim\rho}\hat{\mathcal{L}}_N(\theta) + 
          \frac{1}{\lambda}  \KL(\rho\|\pi)  \right), \quad 
\rho_{\text{Gibbs}}(\theta)=\frac{\pi(\theta)\exp(-\lambda \hat{\mathcal{L}}_N(\theta))}{E_{\theta\sim\pi}\exp(-\lambda \hat{\mathcal{L}}_N(\theta))}, 
\end{aligned}
\end{equation*}
Once a posterior $\rho^{\train}$ is obtained, the learned model 
 $\theta_{\star}$
 can be chosen as follows:
 \par
 \textbf{Mean posterior (MEP):} 
 as the mean of the posterior
 $\rho^{\train}$, i.e. $\theta_{\star}=E_{\theta \sim \rho^{\train}} \theta$,
 \par
\textbf{Single draw}:
   draw $\theta_{\star}$ randomly from $\rho^{\train}$,  and
\par
\textbf{Maximum posterior (MAP)}:
   choose $\theta_{\star}$ as 
    $\theta_{\star}=\mathrm{arg max} \rho^{\train}(\theta)$.
    For $\rho^{\train}=\rho^{\mathrm{Gibbs}}$,
     it minimizes the empirical error regularized using the prior, i.e.,
    $\theta_{\star}=\argmin_{\theta\in\Theta}\left ( \hat{\mathcal{L}}_N(\theta) - \frac{1}{\lambda} \ln\pi(\theta) \right )$,

These choices represent well-known classes of system identification algorithms, such as regularized PEM, 
maximimum likelihood estimation for state-space models, and Bayesian methods, see Remarks \ref{rem:pem}, \ref{rem:fir}, and \ref{rem:ml} in Appendix \ref{app:indent}.



Theorem~\ref{thm:main} bounds the average generalization gap $E_{\theta \sim \rho} \left(\mathcal{L}(\theta) - \hat{\mathcal{L}}_N(\theta)\right)$ by a sum of two terms: the KL divergence $\frac{1}{\lambda}\KL(\rho \| \pi)$ between the posterior and the prior, and a complexity term $\mathcal{C}_i(2\lambda,\theta,N,\delta,\epsilon)$ that depends on the data and the prior. For bounded noise, $\mathcal{C}_i = O(1/N)$; for Lipschitz loss, $\mathcal{C}_i = O(\frac{\ln N)}{N})$; and for quadratic loss, $\mathcal{C}_i = O(\frac{(\ln N)^3}{N})$. By choosing $\lambda = \frac{\sqrt{N}}{L(\ln N)^r}$ with appropriately selected $L$ and $r$, and thus forcing $\frac{1}{\lambda}\KL(\rho \| \pi)$ to tend to zero as $N \to \infty$, we obtain an $O\left(\frac{(\ln N)^r}{\sqrt{N}} \right)$:
 \begin{corollary}[$O((\ln(N))^r/\sqrt{N})$ bound]
\label{thm:bounded:alt:col}
With the notation and assumptions of Theorem~\ref{thm:main}, 
\begin{equation}
\label{catoni1}
\begin{split}
  & \forall \rho \in\mathcal{M}_\pi: \epsilon E_{\theta \sim \rho } \mathcal{L}(\theta) \leq \epsilon E_{\theta\sim \rho } \hat{\mathcal{L}}_N(\theta)+  \frac{L(\ln(N))^{r}}{\sqrt{N}}\left(\KL(\rho ||\pi)+\ln\frac{C_{\delta}}{\delta} + 
  \mathcal{C}(\delta,\epsilon) \right)
\end{split}
\end{equation}
holds with probability at least $1-2\delta$ over data, 
for any prior $\pi$, $\delta \in (0,0.5)$, and for any $\epsilon \in \{-1,1\}$, and $N \ge \textcolor{black}{4}$, 
where 
$\mathcal{C}(\delta,\epsilon)$, $L$, $r$ are defined as follows.
\\
\textbf{Bounded noise:}
  If $\|e_g(t)\|_{\infty} \le c_{\rve}$, then 
      $r\triangleq 0$, $L \triangleq 1$, and 
\begin{equation*}
\begin{split}
   & \mathcal{C}(\delta,\epsilon) \triangleq \textcolor{black}{\frac{1}{2} \sum_{i=1}^2 \mathcal{C}_{m,b,i}(2,\bar{\alpha},G_e+2G_{e,1},L_b,c_{\rve})}, \quad 
    L_b \triangleq \left\{\begin{array}{rl} 2 G_ec_{\rve}\sqrt{n_{\rve}} & \mbox{ $\ell$ is the quadratic loss} \\
              L_{\ell} & \mbox{$\ell$ is $L_\ell$-Lipschitz }
         \end{array}\right. 
\end{split}
\end{equation*}
\textbf{Lipschitz loss:}
If $\ell$ is $L_{\ell}$-Lipschitz, then 
       $r \triangleq 1$, $L \triangleq 16L_{\ell}$, and 
\begin{equation*}
\begin{split}
   \mathcal{C}(\delta,& \epsilon)  \triangleq \mathcal{\bar{C}}_{ub}=\left(
        C_{u,1}+C_{u,2} 
   +  
    \frac{1}{4} \sum_{i=1}^2 \mathcal{C}_{m,b,i}\left(\frac{1}{2},\bar{\alpha},G_e+2G_{e,1},1,\frac{1}{\|\alpha_g\|_{\ell_1}K}\right)
      \right), 
\end{split}
\end{equation*}
%
\textbf{Quadaratic loss:}
If $\ell$ is quadratic, then 
$L \triangleq 192G_e \|\e_g\|_{\psi_2} \sqrt{2}\sqrt{1+\ln\left(\frac{3}{2\delta}\right)}$,  
$r \triangleq 3/2$, 
\begin{equation*}
\begin{split}
   \mathcal{C}(\delta,\epsilon) \triangleq & 
    \frac{L\max\{0,\epsilon\}}{128} +C_{u,1}+C_{u,2}+\frac{1}{2} \textcolor{black}{\sum_{i=1}^2 \mathcal{C}_{m,b,i}\left(\frac{1}{4},\bar{\alpha},G_e+2G_{e,1},1,\frac{1}{\|\alpha_g\|_{\ell_1}K}\right)} \\
\end{split}
\end{equation*}
\end{corollary}
The proof of Corollary~\ref{thm:bounded:alt:col} follows by substituting $\lambda=\frac{\sqrt{N}}{L(\ln N)^r}$ into Theorem~\ref{thm:main} and simplifying the resulting bounds (see Appendix~\ref{sec:proof_corollary} for the proof).
\textcolor{black}{For additional explanation of the statement of Corollary~\ref{thm:bounded:alt:col} see Appendix \ref{app:additional_intuition}.}
Note that while Corollary~\ref{thm:bounded:alt:col} makes the $O\left(\frac{(\ln N)^r}{\sqrt{N}}\right)$ rates explicit, Theorem~\ref{thm:main} may give a  tighter bound even for $\lambda=\frac{\sqrt{N}}{L(\ln N)^r}$ \cite{alquier2021userfriendly}.

The previous results bound the average generalization gap with respect to a posterior $\rho^{\train}$. For practical learning, however, we are interested in the generalization gap for a specific choice of $\theta_{\star}$ (e.g., MEP, MAP, or a single draw from $\rho^{\train}$). Below, we provide such bounds when $\rho^{\train}$ is the Gibbs posterior with $\lambda = \frac{\sqrt{N}}{L(\ln N)^r}$, where $L, r$ are as in Corollary~\ref{thm:bounded:alt:col}. This yields a generalization gap for $\theta_{\star}$ of order $O\left(\frac{(\ln N)^{r+1}}{\sqrt{N}}\right)$.
 We now introduce the necessary assumptions for these results.
\begin{assumption}
\label{as:oracle}
The following conditions are used in the sequel.

\begin{description}
    \item[\textbf{(Opt)}]
    An optimal model $\theta_{true} = \argmin_{\theta \in \Theta} \mathcal{L}(\theta)$ exists.

    \item[\textbf{(LipModel)}]
    The map $\theta \mapsto \alpha_{\theta}$ is $L_{\alpha}$-Lipschitz w.r.t.\
    to the $\|\cdot\|_{\ell_1}$ norm, i.e.,
    \[
        \|\alpha_{\theta_1}-\alpha_{\theta_2}\|_{\ell_1}
        \le L_{\alpha}\|\theta_1-\theta_2\|_2,
        \quad \forall \theta_1,\theta_2 \in \Theta.
    \]

    \item[\textbf{(Gauss)}]
    $\Theta=\mathbb{R}^{n_{\theta}}$ and $\pi$ is Gaussian
    $\mathcal{N}(\theta_m,P)$.

    \item[\textbf{(Uni)}]
    $\Theta$ is \textcolor{black}{a closed ball $\{\theta \in \mathbb{R}^{n_{\theta}} : \|\theta-\theta_m\|_2 \le R\}$ of radius $R$ and center $\theta_m$ for some $\theta_m \in \mathbb{R}^{n_{\theta}}$,} and $\pi$ is the uniform density on $\Theta$.

    \item[\textbf{(SConv)}]
    The true error $\mathcal{L}(\theta)$ is smooth and strongly convex
    in $\theta$ and its Hessian is bounded from below
    $m_{\Theta}I \preceq \nabla^2_{\theta}\mathcal{L}(\theta)$
    for all $\theta \in \Theta$.

    \item[\textbf{(Real)}]
    The predictor $\Sigma(\theta_{true})$ equals the predictor
    \eqref{eq:T0sys2} corresponding to the data generator.

    \item[\textbf{(PE)}]
    The spectral density $\Phi_{\w}$ of $\w$ satisfies
    $\Phi_{\w}(iz) \ge m_{\w} I$ for all $z \in [-\pi,\pi]$.
\end{description}
\end{assumption}
\begin{corollary}
\label{lem:bound:mean}
Assume the following
\\
\textbf{(1)} Assumptions \ref{as:generator}, \ref{as:parameterisation}, \ref{as:loss} and  \textbf{(Opt)} and \textbf{(LipModel)} of Assumption \ref{as:oracle}  hold, and
\\
\textbf{(2)}
either \textbf{(Gauss)} and \textcolor{black}{$\sigma > 0$ is arbitrary} or \textbf{(Uni)} of Assumption \ref{as:oracle} holds and \textcolor{black}{$\sigma < 2R$}, and \\
\textbf{(3)}
  $\theta_{\star}$ be selected from $\rho_{\mathrm{Gibbs}}$ using MEP, MAP or a single draw,
and $\rho_{\mathrm{Gibbs}}$ is computed using $\lambda=\lambda_N = \frac{\sqrt{N}}{\textcolor{black}{L}(\ln(N))^r}$
with $r,L$ are as defined in Corollary \ref{thm:bounded:alt:col} \textcolor{black}{and $N \ge 4$}.\\
\textbf{(4)}
Furthermore, if $\theta_{\star}$ is chosen using MEP, then assume that \textbf{(SConv)} of Assumption \ref{as:oracle} holds. \\
Consider the  constants $\mathcal{C}_{em}(\delta,\sigma)$,
$\mathcal{C}_{or}(\delta,\sigma),r_{\text{Lip},1},r_{\text{Lip},2}$ as in Table \ref{tab:constants:sysid}, and consider the 
inequalities
\begin{align}
& \textbf{Empirical bound:} \quad    \mathcal{L}(\theta_{\star}) - \hat{\mathcal{L}}_N(\theta_{\star})  \le  \frac{L(\ln(N))^{r+1}}{\sqrt{N}} C_{em}(\theta_{\star},\delta) 
  \label{bound:emp} \\
  & \textbf{Oracle bound:} \quad \mathcal{L}(\theta_{\star})  - \mathcal{L}(\theta_{true})  \le
     \frac{L(\ln(N))^{r+1}}{\sqrt{N}} \mathcal{C}_{or}(\delta,\sigma) 
     \label{bound:oracle} \\
  & \textbf{Parameter estimation error:} \quad \|\theta_{\star}-\theta_{true}\|^2_2  \le
    \frac{L(\ln(N))^{r+1}}{\sqrt{N}m_{\Theta}}
    \mathcal{C}_{or}(\delta,\sigma), 
  \label{cor:param:est1:eq1}  \\
  & \text{\textbf{$H_2$ error bound:}} \quad 
    \|H_{\theta_{true}}-H_{\theta_{\star}}\|_{H_2}^2 \le 
     \frac{L(\ln(N))^{r+1}}{m_{\rvw} \sqrt{N}} \mathcal{C}_{or}(\delta,\sigma)
   \label{col1:eq1}
\end{align}
Then the following holds: \textbf{(A)} \eqref{bound:emp} holds
with probability at least $1-(2+r_{\text{Lip},1})\delta$,  \\
\textbf{(B)} \eqref{bound:oracle} holds with probability at least $1-2\delta$ for single draw and MEP, and with probability at least $1-(2+r_{\text{Lip},1})\delta$ for MAP, \\
\textbf{(C)} \eqref{cor:param:est1:eq1} holds with probability at least $1-2\delta$ for single draw and MEP, and with probability at least $1-(2+r_{\text{Lip},1})\delta$ for MAP, 
if \textbf{(SConv)} of Assumption \ref{as:oracle} holds,  \\
\textbf{(D)} \eqref{col1:eq1} holds with probability at least $1-2\delta$ for single draw and MEP,
and with probability at least $1-(2+r_{\text{Lip},1})\delta$ for MAP,
if \textbf{(Real)} and \textbf{(PE)} of Assumption  \ref{as:oracle} holds, $\ell$ is the quadratic loss, and 
 $H_{\theta}$ the transfer function of the predictor  corresponding to $\theta$, and denote by $\|H\|_{H_2}$ the $H_2$ norm of the transfer function $H$. \\
%
%
 Here, the probability is taken over data for MEP and MAP,  and over the data and over all samples $\theta_{\star}$ drawn from $\rho_{\mathrm{Gibbs}}$ for the single draw.
 \footnote{The precise meaning of probability over data and samples drawn from $\rho_{\mathrm{Gibbs}}$ is explained in Remark \ref{rem:jointProb}}
\end{corollary}
\begin{table}[ht]
\centering
\caption{Constants for learning bounds under different algorithm choices (Corollary \ref{lem:bound:mean})}\label{tab:constants:sysid}
\begin{tabular}{|l|c|c|}
\hline
\textbf{Constant} & \textbf{MEP / Single Draw} & \textbf{MAP} \\
\hline
$L_{\Theta}(\delta)$ & \multicolumn{2}{c|}{$\begin{array}{ll}
  L_{\alpha}L_{\ell} \|\alpha_g\|_{\ell_1}\sqrt{n_\rve}c_{\rve} & \text{bounded noise} \\
  \frac{1}{16} L_{\alpha}\sqrt{n_\rve} \frac{1}{K} & \text{Lipschitz/quadratic}
  \end{array}$} \\
\hline
$C(\theta)$ & \multicolumn{2}{c|}{$\!\!\!\begin{array}{rl} 
  \frac{1}{2} \big (\sigma^2 \mathrm{trace} (P^{-1}) -n_{\theta}+ (\theta_m-\theta)^TP^{-1}(\theta_m-\theta)+
   \ln \frac{\det(P)}{\sigma^{2n_\theta}} + n_\theta \big) & \!\! \textbf{(Gauss)} \\
     \ln\frac{vol(\Theta)}{(\sqrt{\pi \textcolor{black}{\sigma^2/4}}^{n_{\theta}} \Gamma(n_{\theta}/2+1)}+ \frac{n_{\theta}}{2}  & \!\! \textbf{(Uni)} 
   \end{array}
   $\!\! } \\
\hline 
 $r_{\text{Lip},1}$, 
 $r_{\text{Lip},2}$
 & \multicolumn{2}{c|}{\(r_{\text{Lip},1}=\left\{\begin{array}{rl} 0 & \e_g\mbox{ is bounded }  \\
   1 & \mbox{ otherwise }
   \end{array}\right. \quad 
   r_{\text{Lip},2}=\left\{\begin{array}{rl} 2 & \ell \mbox{ is quadratic and} \\
    &   \mbox{and  $\e_g$ is unbounded }  \\
   0 & \mbox{ otherwise }
   \end{array}\right.
   \)
}\\
\hline
$\mathcal{C}_{L,emp}$ & \multicolumn{2}{c|}{$L_{\Theta}(\delta) \sigma+2r_{\text{Lip},1}(C_{u,1} + C_{u,2})$} \\
\hline
$\mathcal{C}_{L,true}$ & \multicolumn{2}{c|}{$\mathcal{C}_{L,emp}+2r_{\text{Lip},2}\frac{L}{512}$} \\
$\mathcal{C}_{L}(\delta,\epsilon,\textcolor{black}{\theta})$ & \multicolumn{2}{c|}{
  \textcolor{black}{$\mathcal{C}_{L,emp}+\ln\left(\frac{C_{\delta}}{\delta}\right)+C(\theta)+ \mathcal{C}(\delta,\epsilon)$}} \\
\hline
$C_{em}(\theta_{\star},\delta)$ & $\mathcal{C}_L(\delta,1,\theta_{\star})+2r_{\text{Lip},1}\ln\left(\frac{\textcolor{black}{C_{\delta}}}{\delta}\right)$ & \textcolor{black}{$\mathcal{C}_L(\delta,1,\textcolor{black}{\theta_{\star}})+\mathcal{C}_{L,true} 
  + 2r_{\text{Lip},1}\ln\left(\frac{\textcolor{black}{C_{\delta}}}{\delta}\right)$} \\
\hline
$\mathcal{C}_{or}(\delta,\sigma)$ & \textcolor{black}{$\begin{array}{l}\mathcal{C}_L(\frac{\delta}{2},1,\theta_{true})+\mathcal{C}_L(\frac{\delta}{2},-1,\theta_{true}) \\ + \mathcal{C}_{L,true}-2\mathcal{C}_{L,emp} \end{array}$} & 
\textcolor{black}{$\begin{array}{l}\mathcal{C}_L(\frac{\delta}{2},1,\theta_{\star})+\mathcal{C}_L(\frac{\delta}{2},-1,\theta_{true}) \\
  +4r_{\text{Lip},1}\ln\left(\frac{\textcolor{black}{2C_{\delta}}}{\delta}\right)+2\mathcal{C}_{L,true}
 \\
 +\ln\left(\frac{\pi(\theta_{true})}{\pi(\theta_{\star})}\right)
\end{array}$} \\
\hline
\end{tabular}
\end{table}
In summary, for any outcome $\theta_{\star}$ obtained via MAP, MEP, or a single draw from the Gibbs posterior, Corollary~\ref{lem:bound:mean} yields $O\left(\frac{(\ln N)^{r+1}}{\sqrt{N}}\right)$ bounds on: \textbf{(1)} the generalization gap \eqref{bound:emp}, \textbf{(2)} the excess true error over the optimal model \eqref{bound:oracle}, and \textbf{(3)} the parameter estimation error \eqref{cor:param:est1:eq1}--\eqref{col1:eq1}. The constants for the learning bounds, as given by Corollary \ref{lem:bound:mean}, are in Table \ref{tab:constants:sysid}. The influence of system properties and convergence rates is discussed in Section~\ref{sec:conclusions}.
\textcolor{black}{For additional intuition on the relationship between Corollary \ref{thm:bounded:alt:col} and Corollary \ref{lem:bound:mean} see Appendix \ref{app:additional_intuition}.}
\par
\textbf{Intuition behind the proof:}
The proof of Corollary \ref{lem:bound:mean} is in Appendix \ref{sec:proof_single_draw}.  For MAP, \eqref{bound:emp}--\eqref{bound:oracle} are derived from Corollary~\ref{thm:bounded:alt:col} by constructing suitable data-dependent posteriors and applying oracle inequality techniques from \cite[Theorem 4.1]{alquier2021userfriendly}. For single draw and MEP, similar arguments apply using \cite[Theorems 2.7 and Corollary 2.6]{alquier2021userfriendly} combined with the moment generating function bounds from  the proof of Theorem~\ref{thm:main}. Finally, the parameter estimation bounds \eqref{cor:param:est1:eq1}--\eqref{col1:eq1} follow from the oracle inequality \eqref{bound:oracle} by relating the parameter error to the loss difference through strong convexity (for \eqref{cor:param:est1:eq1}) or persistence of excitation (for \eqref{col1:eq1}).

\section{Discussion on the bound  and conclusions}
\label{sec:conclusions}
In this paper, we derived PAC-Bayesian error bounds for stochastic LTI systems with inputs and sub-Gaussian noise.
\par
The convergence rates are fastest for bounded noise, slower for Lipschitz losses, and slowest for quadratic losses with general sub-Gaussian noise (see Table~\ref{rates}).
\begin{color}{black}
For the unbounded case with Lipschitz loss it is $O((\ln N)/ \sqrt{N})$ and the logarithmic factor stems from using the truncation argument which replaces the unbounded noise of the data generator with a bounded noise of magnitude $O(\ln (N))$. For the case of unbounded data and quadratic loss, the bound is $O((\ln N)^{3/2} / \sqrt{N})$ and the extra $O((\ln N)^{1/2})$ factor comes by replacing the quadratic loss with the truncated quadratic loss whose Lipschitz constant is $O((\ln N)^{1/2})$. For the parameter estimation errors all these bounds are extended with an additional $O((\ln N))$ factor, which arises from the uniform PAC bound \eqref{bound:emp} of Corollary 3.4. In turn, the factor $O((\ln N))$ in the uniform PAC bound comes by instantiating the original PAC-Bayesian bound using uniform or Gaussian priors and choosing a suitable posterior which is centered around the estimated model, such that KL-divergence between prior and posterior is $O(\ln(N))$
\end{color}

All bounds in Theorem~\ref{thm:main}, Corollary~\ref{thm:bounded:alt:col}, and Corollary~\ref{lem:bound:mean} improve with increased system stability, lower sub-Gaussian noise norm, and depend logarithmically on the confidence level $\delta$ via $\ln(1/\delta)$.
Specifically, the bounds increase monotonically with the sub-Gaussian norm $\|\e_g\|_{\psi_2}$ of the noise, and with the $\ell_1$ norms and mixing coefficients of the Markov parameters of both the predictors and the data generator ($\|\alpha_\theta\|_{\ell_1}$, $\|\alpha_g\|_{\ell_1}$, $\theta_{\infty}(\alpha_\theta)$, $\theta_{\infty}(\alpha_g)$). Smaller $\ell_1$ norms and mixing coefficients correspond to more stable systems and weaker dependence on past inputs, making learning easier. Similarly, a smaller sub-Gaussian norm (i.e., lower noise variance) yields tighter bounds.
\par 
In summary, the bounds are tighter for less noisy and more stable systems, aligning with intuition: learning is easier when the system is less noisy and more stable.
\begin{color}{black}
The bound of Theorem~\ref{thm:main} and Corollary~\ref{thm:bounded:alt:col}
depend on the prior $\pi$, both via the 
KL-divergence between the prior and posterior and
via the terms $\ln E_{\theta\sim\pi} \exp\left(\mathcal{C}_i(2\lambda,\theta,N,\delta,\epsilon)\right)$. 
Strategies for choosing priors remains work for future research. Intuitively, one expects that a good prior is the one which is concentrated around the true parameter values, which renders the optimal (Gibbs posterior) close to the prior, which then renders the KL-divergence term small. 
\end{color}
Additionally, if the goal is to bound not only the generalization gap but also the parameter estimation error (in terms of the $H_2$ distance between the true and learned systems \eqref{col1:eq1}, or the Euclidean distance between parameters \eqref{cor:param:est1:eq1}), the bounds also depend on the \emph{persistence of excitation} of the input $\w$ ($m_{\rvw}$) and the \emph{identifiability} constant ($m_{\Theta}$). Larger values of $m_{\rvw}$ (richer input) and $m_{\Theta}$ (greater separation between models) yield tighter bounds, \textcolor{black}{see Appendix \ref{pers:excitation} for a more detailed discussion.}
The bound \eqref{col1:eq1} can also be used to bound the system matrices (up to similarity), see Remark~\ref{matrix:trans}, Appendix~\ref{app:indent}.




\begin{table}[ht]
\centering
\caption{Summary of convergence rates for different noise and loss function types \label{rates}}
\begin{tabular}{|c|c|c|}
\hline
\textbf{Scenario} & \textbf{Empirical PAC-Bayesian Bound} & \textbf{PAC/Oracle/parameter est./$H_2$ bounds} \\
\hline
\textbf{Bounded noise} & $O\left(\frac{1}{\sqrt{N}}\right)$ & $O\left(\frac{\ln(N)}{\sqrt{N}}\right)$ \\
\hline
\textbf{Lipschitz loss} & $O\left(\frac{\ln(N)}{\sqrt{N}}\right)$ & $O\left(\frac{(\ln(N))^2}{\sqrt{N}}\right)$
 \\
\hline
\textbf{Quadratic loss} & $O\left(\frac{(\ln(N))^{3/2}}{\sqrt{N}}\right)$ & $O\left(\frac{(\ln(N))^{5/2}}{\sqrt{N}}\right)$ 
\\
\hline
\end{tabular}
\end{table}

For parameter estimation error, our obtained rate $O((\ln N)^{r+1}/\sqrt{N})$ is more conservative than classical asymptotic theory \cite{LjungBook,CainesBook} (which yields  $O(1/\sqrt{N})$ asymptotic consistency) and existing finite-sample bounds \cite{oymak2021revisiting,lale2020logarithmic,Simchowitz_Foster_2020,Ziemann1}(which achieve $O(\ln N/\sqrt{N})$) . However, this is the expected price for generality: our bound applies to a broad class of learning algorithms, provides non-asymptotic guarantees with high probability, and simultaneously bounds the generalization gap (properties not offered by prior results). In addition, most existing finite-sample bounds are algorithm-specific (e.g., for least-squares only) and do not address the generalization gap.

\bibliographystyle{plain}
\bibliography{bib}

@article{alquier2013prediction,
  title={Prediction of time series by statistical learning: general losses and fast rates},
  author={P. Alquier and X Li and O. Wintenberger},
  journal={Dependence Modeling},
  volume={1},
  number={2013},
  pages={65--93},
  year={2013},
  publisher={De Gruyter}
}

@inproceedings{Simchowitz_Foster_2020, series={PMLR}, title={Naive Exploration is Optimal for Online LQR}, volume={119}, booktitle={Proceedings of the 37th ICML}, publisher={PMLR}, author={Simchowitz, Max and Foster, Dylan}, year={2020}, month={7}, pages={8937–8948}, collection={PMLR} }

@article{alquier2021userfriendly,
      title={User-friendly introduction to {PAC}-Bayes bounds}, 
      author={Pierre Alquier},
      year={2021},
      journal={arXiv:2110.11216},
      eprint={2110.11216},
      archivePrefix={arXiv},
      primaryClass={stat.ML}
}

@book{gikh96,
title = {Introduction to the Theory of Random Processes},
year = {1996},
author = {	I. I. Gikhman and A. V. Skorokhod},
publisher = {Dover},
edition = {1st}
}

@book{CainesBook,
        Author = {P. E. Caines},
        Publisher = {John Wiley and Sons},
        Title = {Linear Stochastic Systems},
        Year = {1988}}

@inproceedings{Dziugaite2017,
  author    = {G. K. Dziugaite and
               D. M. Roy},
  title     = {Computing Nonvacuous Generalization Bounds for Deep (Stochastic) Neural
               Networks with Many More Parameters than Training Data},
  booktitle = {{UAI}},
  publisher = {{AUAI} Press},
  year      = {2017}
}

@Book{LjungBook,
  Title                    = {System Identification: Theory for the user (2nd Ed.)},
  author                   = {L. Ljung},
  Publisher                = {PTR Prentice Hall., Upper Saddle River, USA},
  Year                     = {1999}
}

@article{guedj2019primer,
  title={{A Primer on {PAC}-Bayesian Learning}},
  author={B. Guedj},
  journal={arXiv:1901.05353},
  year={2019}
}

@inproceedings{nips-16,
	author    = {P. Germain and
	F. Bach and
	A. Lacoste and
	S. Lacoste-Julien},
	title     = {{PAC}-Bayesian Theory Meets Bayesian Inference},
	booktitle = {{NIPS}},
	pages     = {1876--1884},
	year      = {2016}
}

@book{LindquistBook, 
   title={Linear Stochastic Systems: A Geometric Approach to Modeling, Estimation and Identification},
    author={A. Lindquist and G. Picci},
    publisher={Springer},
    year={ 2015 }
}

@PHDTHESIS{RalfPeeters,
  author = { Peeters, R. L. M. },
  title = { System Identification Based on Riemannian Geometry: Theory and Algorithms},
  school = { Free University, Amsterdam},
  year = { 1994}
}

@book{Katayama:05,
 address={Berlin, London},
 series={Communications and control engineering},
 title={Subspace methods for system identification},
 ISBN={978-1-85233-981-4},
 publisher={Springer},
 author={Katayama, Tohru},
 year={2005},
 collection={Communications and control engineering},
 language={en} }

@book{shalev2014understanding,
  title={Understanding machine learning: From theory to algorithms},
  author={S. Shalev-Shwartz and S. Ben-David},
  year={2014},
  publisher={Cambridge university press}
}

@article{SarkarRD21,
  author    = {Tuhin Sarkar and
               Alexander Rakhlin and
               Munther A. Dahleh},
  title     = {Finite Time {LTI} System Identification},
  journal   = {J. Mach. Learn. Res.},
  volume    = {22},
  pages     = {26:1--26:61},
  year      = {2021},
}

@inproceedings{CDC21paper,
title = "{PAC-Bayesian theory for stochastic LTI systems}",
author = "Deividas Eringis and John Leth and Zheng-Hua Tan and Rafal Wisniewski and Alireza Fakhrizadeh Esfahani and Mihaly Petreczky",
year = "2021",
doi = "10.1109/CDC45484.2021.9682808",
isbn = "978-1-6654-3660-1",
pages = "6626--6633",
booktitle = "2021 60th IEEE Conference on Decision and Control (CDC)",
}

@article{eringis2021optimal,
title = "{Explicit construction of the minimum error variance estimator for stochastic LTI-ss systems}",
author = "Deividas Eringis and John Leth and Zheng Hua Tan and Rafal Wisniewski and Mihaly Petreczky",
year = "2023",
month = jul,
doi = "10.1016/j.automatica.2023.111018",
volume = "153",
journal = "Automatica",
issn = "0005-1098",
publisher = "Elsevier",
}

@article{alquier2012pred,
author = {Pierre Alquier and Olivier Wintenberger},
title = {{Model selection for weakly dependent time series forecasting}},
volume = {18},
journal = {Bernoulli},
number = {3},
publisher = {Bernoulli Society for Mathematical Statistics and Probability},
pages = {883 -- 913},
year = {2012}
}

@InProceedings{pmlr-v108-chen20d,
  title = 	 {On Generalization Bounds of a Family of Recurrent Neural Networks},
  author =       {Chen, Minshuo and Li, Xingguo and Zhao, Tuo},
  booktitle = 	 {Proceedings of AISTATS 2020},
  pages = 	 {1233--1243},
  year = 	 {2020},
  volume = 	 {108},
  series = 	 {PMLR},
  month = 	 {8},
}

@inproceedings{NEURIPS2018_d6288499,
 author = {Hazan, Elad and Lee, Holden and Singh, Karan and Zhang, Cyril and Zhang, Yi},
 booktitle = {Advances in Neural Information Processing Systems},
 pages = {},
 publisher = {Curran Associates, Inc.},
 title = {Spectral Filtering for General Linear Dynamical Systems},
 volume = {31},
 year = {2018}
}

@article{KOIRAN199863, title={Vapnik-Chervonenkis dimension of recurrent neural networks}, volume={86}, number={1}, journal={Discrete Applied Mathematics}, author={Koiran, Pascal and Sontag, Eduardo D.}, year={1998}, pages={63–79} }

@article{sontag1998learning,
  title={A learning result for continuous-time recurrent neural networks},
  author={Sontag, Eduardo D},
  journal={Systems \& control letters},
  volume={34},
  number={3},
  pages={151--158},
  year={1998},
  publisher={Elsevier}
}

@article{vidyasagar2006learning,
  title={A learning theory approach to system identification and stochastic adaptive control},
  author={Vidyasagar, Mathukumalli and Karandikar, Rajeeva L},
  journal={Probabilistic and randomized methods for design under uncertainty},
  pages={265--302},
  year={2006},
  publisher={Springer}
}

@article{campi2002finite,
  title={Finite sample properties of system identification methods},
  author={Campi, Marco C and Weyer, Erik},
  journal={IEEE Transactions on Automatic Control},
  volume={47},
  number={8},
  pages={1329--1334},
  year={2002},
  publisher={IEEE}
}

@inproceedings{simchowitz2019learning,
  title={Learning linear dynamical systems with semi-parametric least squares},
  author={Simchowitz, Max and Boczar, Ross and Recht, Benjamin},
  booktitle={Conference on Learning Theory},
  pages={2714--2802},
  year={2019},
  organization={PMLR}
}

@article{lale2020logarithmic,
  title={Logarithmic regret bound in partially observable linear dynamical systems},
  author={Lale, Sahin and Azizzadenesheli, Kamyar and Hassibi, Babak and Anandkumar, Anima},
  journal={Advances in Neural Information Processing Systems},
  volume={33},
  pages={20876--20888},
  year={2020}
}

@book{simchowitz2021statistical,
  title={Statistical Complexity and Regret in Linear Control},
  author={Simchowitz, Max},
  year={2021},
  publisher={University of California, Berkeley}
}

@article{oymak2021revisiting,
 title={Revisiting Ho–Kalman-Based System Identification: Robustness and Finite-Sample Analysis},
 volume={67},
 number={4},
 journal={IEEE Transactions on Automatic Control},
 author={Oymak, Samet and Ozay, Necmiye},
 year={2022},
 month={4},
 pages={1914–1928}}

@inproceedings{foster2020learning, series={PMLR}, title={Logarithmic Regret for Adversarial Online Control}, volume={119}, booktitle={Proceedings of the 37th ICML}, publisher={PMLR}, author={Foster, Dylan and Simchowitz, Max}, year={2020}, month={7}, pages={3211–3221}, collection={PMLR} }

@article{Ribarits1,
  author       = {Tomas McKelvey and
                  Anders Helmersson and
                  Thomas Ribarits},
  title        = {Data driven local coordinates for multivariable linear systems and
                  their application to system identification},
  journal      = {Autom.},
  volume       = {40},
  number       = {9},
  pages        = {1629--1635},
  year         = {2004},
  url          = {https://doi.org/10.1016/j.automatica.2004.04.015},
  doi          = {10.1016/J.AUTOMATICA.2004.04.015},
  bibsource    = {dblp computer science bibliography, https://dblp.org}
}

@article{eringis2023pacbayesian,
      title={{PAC}-Bayesian bounds for learning LTI-ss systems with input from empirical loss}, 
      author={Deividas Eringis and John Leth and Zheng-Hua Tan and Rafael Wisniewski and Mihaly Petreczky},
      year={2023},
      journal={arXiv:2303.16816},
      eprint={2303.16816},
      archivePrefix={arXiv},
      primaryClass={stat.ML}
}

@article{MASSUCCI2022110532,
title = {A statistical learning perspective on switched linear system identification},
journal = {Automatica},
volume = {145},
pages = {110532},
year = {2022},
issn = {0005-1098},
doi = {https://doi.org/10.1016/j.automatica.2022.110532},
url = {https://www.sciencedirect.com/science/article/pii/S0005109822003934},
author = {Louis Massucci and Fabien Lauer and Marion Gilson}
}

@article{METAKALARD2026100373,
title = {Uniform error bounds for quantized dynamical models},
journal = {IFAC Journal of Systems and Control},
volume = {35},
pages = {100373},
year = {2026},
issn = {2468-6018},
doi = {https://doi.org/10.1016/j.ifacsc.2026.100373},
url = {https://www.sciencedirect.com/science/article/pii/S2468601826000131},
author = {Abdelkader Metakalard and Fabien Lauer and Kevin Colin and Marion Gilson}
}

@book{pillonetto2022regularized,
  title={Regularized system identification: Learning dynamic models from data},
  author={Pillonetto, Gianluigi and Chen, Tianshi and Chiuso, Alessandro and De Nicolao, Giuseppe and Ljung, Lennart},
  year={2022},
  publisher={Springer Nature}
}

@article{HanzonStable1,
title = {On the differentiable manifold of fixed order stable linear systems},
journal = {Systems \& Control Letters},
volume = {13},
number = {4},
pages = {345-352},
year = {1989},
issn = {0167-6911},
author = {B. Hanzon},
}

@inproceedings{eringisRenyi,
  TITLE = {{PAC-Bayesian Error Bound, via R{\'e}nyi Divergence, for a Class of Linear Time-Invariant State-Space Models}},
  AUTHOR = {Eringis, Deividas and Leth, John and Tan, Zheng-Hua and Wisniewski, Rafal and Petreczky, Mih{\'a}ly},
  BOOKTITLE = {{41st International Conference on Machine Learning}},
  ADDRESS = {Vienne, Austria},
  PUBLISHER = {{PMLR}},
  VOLUME = {235},
  PAGES = {12560-12587},
  YEAR = {2024},
  MONTH = Jul,
}

@article{BallVolume,
author = {David J. Smith and Mavina K. Vamanamurthy},
title = {How Small Is a Unit Ball?},
journal = {Mathematics Magazine},
volume = {62},
number = {2},
pages = {101--107},
year = {1989},
publisher = {Taylor \& Francis},
}

@inproceedings{eringis2023pacbayesRNN,
title = "{PAC-Bayes Generalisation Bounds for Dynamical Systems Including Stable RNNs}",
author = "Deividas Eringis and John Leth and Zheng-Hua Tan and Rafal Wisniewski and Mihaly Petreczky",
year = "2024",
month = mar,
day = "25",
doi = "10.1609/aaai.v38i11.29076",
number = "11",
pages = "11901--11909",
booktitle = "Proceedings of the 38th AAAI Conference on Artificial Intelligence",
}

@misc{jin2019short,
      title={A Short Note on Concentration Inequalities for Random Vectors with SubGaussian Norm}, 
      author={Chi Jin and Praneeth Netrapalli and Rong Ge and Sham M. Kakade and Michael I. Jordan},
      year={2019},
      eprint={1902.03736},
      archivePrefix={arXiv},
      primaryClass={math.PR}
}

@book{Federer1996,
  author = {Herbert Federer},
  title = {Geometric Measure Theory},
  publisher = {Springer Berlin Heidelberg},
  address = {Berlin, Heidelberg},
  year = {1996},
  edition = {1},
  series = {Classics in Mathematics},
  doi = {10.1007/978-3-642-62010-2},
  isbn = {978-3-540-60656-7},
  note = {eBook ISBN: 978-3-642-62010-2. Originally published as volume 153 in the series Grundlehren der mathematischen Wissenschaften. IV, 677 pages. Springer-Verlag Berlin Heidelberg, 1996},
}

@book{vershynin2026highdimensional,
  title     = {High-Dimensional Probability: An Introduction with Applications in Data Science},
  author    = {Vershynin, Roman},
  edition   = {2},
  year      = {2026},
  publisher = {Cambridge University Press}
}

@book{NesterovBook,
  author    = {Yurii Nesterov},
  title     = {Introductory Lectures on Convex Optimization: A Basic Course},
  volume    = {87},
  publisher = {Springer Science \& Business Media},
  year      = {2004},
}

@ARTICLE{Ziemann1,
  author={Tsiamis, A. and Ziemann, I. and Matni, N. and Pappas, G. J.},
  journal={IEEE Control Systems Magazine},
  title={Statistical Learning Theory for Control: A Finite-Sample Perspective},
  year={2023},
  volume={43},
  number={6},
  pages={67-97},
  doi={10.1109/MCS.2023.3310345}}

@ARTICLE{Ziemann2,

  author={He, Jiabao and Ziemann, Ingvar and Rojas, Cristian R. and Qin, S. Joe and Hjalmarsson, Håkan},

  journal={IEEE Transactions on Automatic Control}, 

  title={Finite Sample Analysis of Open-loop Subspace Identification Methods}, 

  year={2026},

  volume={},

  number={},

  pages={1-16},

  doi={10.1109/TAC.2026.3671690}}

@article{NINNESS201040,
title = {Bayesian system identification via Markov chain Monte Carlo techniques},
journal = {Automatica},
volume = {46},
number = {1},
pages = {40-51},
year = {2010},
issn = {0005-1098},
doi = {https://doi.org/10.1016/j.automatica.2009.10.015},
url = {https://www.sciencedirect.com/science/article/pii/S0005109809004762},
author = {Brett Ninness and Soren Henriksen}
}
\appendix
\newpage

\begin{color}{black}
\section{Table of symbols}
\label{sect:symbols}
\begin{longtable}{>{\raggedright\arraybackslash}p{2.5cm} >{\raggedright\arraybackslash}p{5.5cm} >{\centering\arraybackslash}p{2.5cm}}
\caption{Notation, Descriptions, and Dimensions} \label{tab:notation} \\
\toprule
\textbf{Notation} & \textbf{Description} & \textbf{Dimension} \\
\midrule
\endfirsthead
\multicolumn{3}{c}{\tablename\ \thetable\ -- Continued} \\
\toprule
\textbf{Notation} & \textbf{Description} & \textbf{Dimension} \\
\midrule
\endhead
\midrule
\multicolumn{3}{r}{\footnotesize Continued on next page} \\
\endfoot
\bottomrule
\endlastfoot
\multicolumn{3}{l}{\textbf{Sets and Spaces}} \\
\midrule
$\Omega$ & Sample space & Set \\
$\F$ & $\sigma$-algebra on $\Omega$ & Collection of sets \\
$\bP$ & Probability measure on $\F$ & Function \\
$\Theta \subseteq \reals^{n_\theta}$ & Parameter space & Set \\
$B_\Theta$ & $\sigma$-algebra of Lebesgue-measurable subsets of $\Theta$ & Collection of sets \\
$\reals^n$ & $n$-dimensional Euclidean space & Set \\
$\mathcal{M}_\pi$ & Set of densities absolutely continuous w.r.t.\ $\pi$ & Set of functions \\
\midrule
\multicolumn{3}{l}{\textbf{Stochastic Processes and Random Variables}} \\
\midrule
$\y(t) \in \reals^{n_\y}$ & Output (label) stochastic process at time $t$ & $n_y$-dimensional vector \\
$\rvu(t) \in \reals^{n_\rvu}$ & Input (feature) stochastic process at time $t$ & $n_u$-dimensional vector \\
$\w(t) \in \reals^{n_\w}$ & Combined process: $[\y^T(t), \rvu^T(t)]^T$ or $\rvu(t)$ & $n_w$-dimensional vector \\
$\e_g(t) \in \reals^{n_e}$ & Noise/innovation process, i.i.d., zero-mean, sub-Gaussian & $n_e$-dimensional vector \\
$\e^s(t)$ & Innovation process in Kalman filter form & Vector \\
$\rvx(t) \in \reals^n$ & State of autonomous LTI data generator & $n$-dimensional vector \\
$\hat{\x}(t) \in \reals^{\hat{n}}$ & State of predictor LTI system & $\hat{n}$-dimensional vector \\
$\hat{\y}_\theta(t \mid t_0) \in \reals^{n_\y}$ & Predicted output based on past data from $t_0$ to $t$ & $n_y$-dimensional vector \\
$\hat{\y}_\theta(t)$ & Predicted output with infinite past ($t_0 \to -\infty$) & $n_y$-dimensional vector \\
$\hat{\rvs}_g(t) \in \reals^n$ & State of data generator in innovations form & $n$-dimensional vector \\
$\hat{\rvs}_g(t \mid t_0)$ & State of predictor in innovations form & $n$-dimensional vector \\
\midrule
\multicolumn{3}{l}{\textbf{System Matrices}} \\
\midrule
$\hat{A}_\theta \in \reals^{\hat{n} \times \hat{n}}$ & State matrix of predictor & $\hat{n} \times \hat{n}$ matrix \\
$\hat{B}_\theta \in \reals^{\hat{n} \times n_\w}$ & Input matrix of predictor & $\hat{n} \times n_w$ matrix \\
$\hat{C}_\theta \in \reals^{n_\y \times \hat{n}}$ & Output matrix of predictor & $n_y \times \hat{n}$ matrix \\
$\hat{D}_\theta \in \reals^{n_\y \times n_\w}$ & Direct feedthrough matrix of predictor & $n_y \times n_w$ matrix \\
$\hat{D}_{\rvu,\theta} \in \reals^{n_\y \times n_\rvu}$ & Sub-matrix of $\hat{D}_\theta$ when $\w = [\y^T, \rvu^T]^T$ & $n_y \times n_u$ matrix \\
$A_0, B_0, K_0, C_0, D_0$ & Matrices of data generator in innovations form & Various dimensions \\
$A_g \in \reals^{n \times n}$ & State matrix of autonomous LTI system & $n \times n$ matrix \\
$K_g \in \reals^{n \times n_e}$ & Kalman gain matrix & $n \times n_e$ matrix \\
$C_g \in \reals^{n_e \times n}$ & Output matrix of autonomous system & $n_e \times n$ matrix \\
$H_\theta$ & Transfer function of predictor & Matrix/rational function \\
\midrule
\multicolumn{3}{l}{\textbf{Parameters and Scalars}} \\
\midrule
$\theta \in \reals^{n_\theta}$ & Parameter vector & $n_\theta$-dimensional vector \\
$\theta_\star$ & Learned/estimated parameter & $n_\theta$-dimensional vector \\
$\theta_{\text{true}}$ & True parameter minimizing $\mathcal{L}(\theta)$ & $n_\theta$-dimensional vector \\
$n_\theta$ & Dimension of parameter space & Scalar \\
$n_\y$ & Output dimension & Scalar \\
$n_\rvu$ & Input dimension & Scalar \\
$n_\w$ & Dimension of $\w(t)$: $n_\y + n_\rvu$ or $n_\rvu$ & Scalar \\
$\hat{n}$ & State dimension of predictor & Scalar \\
$n$ & State dimension of data generator & Scalar \\
$n_e = n_\y + n_\rvu$ & Dimension of noise/innovation & Scalar \\
$t \in \sZ$ & Time index & Scalar \\
$t_0$ & Initial time for prediction & Scalar \\
$N$ & Sample size (number of data points) & Scalar \\
$\delta \in (0, 0.5)$ & Confidence level for PAC-Bayesian bound & Scalar \\
$\epsilon \in \{-1, 1\}$ & Sign parameter in Theorem~3.1 & Scalar \\
$\lambda > 0$ & PAC-Bayesian trade-off parameter & Scalar \\
$\gamma \in (0, 1)$ & Spectral radius bound: $\|\hat{A}_\theta^k\|_2 \leq M \gamma^k$ & Scalar \\
$M, C > 0$ & Constants bounding system matrices & Scalars \\
$c_{\rve}$ & Bound on $\|\e_g(t)\|_\infty$ (bounded noise case) & Scalar \\
$L_\ell$ & Lipschitz constant of loss function $\ell$ & Scalar \\
$L(\theta)$ & Loss-dependent constant & Scalar \\
$L_{qd}$ & Constant for quadratic loss case & Scalar \\
$m_\Theta$ & Identifiability constant (strong convexity lower bound) & Scalar \\
$m_{\w}$ & Persistence of excitation constant & Scalar \\
$K$ & $\min\{1, \|\bar{\alpha}\|_{\ell_1}\}$ & Scalar \\
\midrule
\multicolumn{3}{l}{\textbf{Functions and Operators}} \\
\midrule
$\bE[\cdot]$ & Expectation operator & Operator \\
$E_{\theta \sim \rho}[g(\theta)]$ & Expectation of $g$ under distribution $\rho$ & Scalar \\
$\|\cdot\|_2$ & Euclidean norm (vector) or induced matrix 2-norm & Scalar \\
$\|\cdot\|_\infty$ & Infinity norm & Scalar \\
$\|\cdot\|_{\psi_2}$ & Sub-Gaussian norm & Scalar \\
$\ell \colon \reals^{n_\y} \times \reals^{n_\y} \to [0, \infty)$ & Loss function & Function \\
$\ell_2$ loss & $\ell(y, \hat{y}) = \|y - \hat{y}\|_2$ & Function \\
Quadratic loss & $\ell(y, \hat{y}) = \|y - \hat{y}\|_2^2$ & Function \\
$\alpha(k)$ & Sequence of matrices (Markov parameters) & Matrix-valued sequence \\
$\bar{\alpha}(k)$ & $\sup_{\theta \in \Theta} \|\alpha_\theta(k)\|_2$ & Scalar-valued sequence \\
$\chi_A(\cdot)$ & Indicator function of set $A$ & Function \\
$\mathrm{vol}(A)$ & Lebesgue measure of set $A$ & Scalar \\
$\rho(\theta)$ & Probability density function on $\Theta$ & Function \\
$\pi(\theta)$ & Prior probability density on $\Theta$ & Function \\
$\mathcal{L}(\theta)$ & True error (expected loss) & Scalar \\
$\hat{\mathcal{L}}_N(\theta)$ & Empirical error (average loss on data) & Scalar \\
$V_N(\theta)$ & Auxiliary empirical error with infinite past & Scalar \\
$\KL(\rho \|\pi)$ & Kullback--Leibler divergence between $\rho$ and $\pi$ & Scalar \\
$\star$ & Convolution operator: $(\alpha \star \mathbf{r})(t) = \sum_{k=0}^\infty \alpha(k) \mathbf{r}(t-k)$ & Operator \\
$\triangleq$ & Defined by & Relation \\
$(\cdot)^T$ & Matrix/vector transpose & Operator \\
$\arg\min$ & Argument of minimum & Operator \\
\midrule
\multicolumn{3}{l}{\textbf{Norms and Constants for PAC-Bounds}} \\
\midrule
$\|\alpha\|_{\ell_1}$ & $\ell_1$ norm of sequence: $\sum_{k=0}^\infty \|\alpha(k)\|_2$ & Scalar \\
$\theta_\infty(\alpha)$ & Weakly dependent mixing coefficient: $\sum_{k=0}^\infty k \|\alpha(k)\|_2$ & Scalar \\
$\alpha_g(k)$ & Markov parameters of data generator & Matrix \\
$\alpha_\theta(k)$ & Markov parameters of predictor & Matrix \\
$G_e(\theta)$ & $(1 + \|\alpha_\theta\|_{\ell_1})\|\alpha_g\|_{\ell_1}$ & Scalar \\
$G_{e,1}(\theta)$ & $\theta_\infty(\alpha_\theta)\|\alpha_g\|_{\ell_1} + \theta_\infty(\alpha_g)(\|\alpha_\theta\|_{\ell_1} + 1)$ & Scalar \\
$G_e$ & $(1 + \|\bar{\alpha}\|_{\ell_1})\|\alpha_g\|_{\ell_1}$ & Scalar \\
$G_{e,1}$ & $\theta_\infty(\bar{\alpha})\|\alpha_g\|_{\ell_1} + \theta_\infty(\alpha_g)(\|\bar{\alpha}\|_{\ell_1} + 1)$ & Scalar \\
$C_{\delta}$ & Constant in PAC-Bayesian bound & Scalar \\
$K_{\delta}$ & Threshold for sample size $N$ in PAC-Bayesian bound & Scalar \\
$\mathcal{C}_i(\lambda, \theta, N, \delta, \epsilon)$ & Constants in PAC-Bayesian bound ($i=1,2$) & Scalar \\
$C_{u,1}, C_{u,2}$ & Constants for unbounded noise case & Scalars \\
$\bar{\Psi}_{\e_g}(s)$ & $e^{14^2 s^2 \|\e_g\|_{\psi_2}^2 + 3s \|\e_g\|_{\psi_2}}$ & Scalar \\
\midrule
\multicolumn{3}{l}{\textbf{Learning and Posterior Distributions}} \\
\midrule
$\rho$ & Posterior distribution over $\Theta$ & Density function \\
$\rho_{\text{Gibbs}}$ & Gibbs posterior minimizing PAC-Bayesian bound & Density function \\
$\train$ & Training dataset: $\{\rvy(t)(\omega), \rvu(t)(\omega)\}_{t=0}^{N-1}$ & Set of vectors \\
\bottomrule
\end{longtable}
\end{color}

\begin{color}{black}
\section{Additional discussion}

\subsection{Relationship between the learning problem considered in this paper and control theory/machine learning}
\label{extensions}

\textbf{Control, system identification and machine learning.}
 In control theory, a control problem involves designing inputs to achieve desired output/state properties using feedback laws. Learning a predictive model differs fundamentally: we aim to learn a model that predicts current outputs from past outputs and inputs -- a time-series prediction problem.

The link between them is that solving control problems requires predictive models, which often have to be learned from data. In control theory, the latter learning task is known as system identification. That is, system identification and machine learning address the same problem: find a predictive model from data. The key distinction between them is that system identification prioritizes model classes useful for control design. For this reason, it often focuses on parameter estimation error over generalization error. Small parameter estimation error ensures the learned model approximates the true physical system across all inputs, including those outside the training distribution. This property is essential for control design, since controllers typically generate inputs very different from those used for model learning.

This paper addresses learning state-space models from data, and its bounds are relevant for both system identification and machine learning, as it bounds both the generalization error and the parameter estimation error, and in fact it shows that the two are related. Specifically, we learn state-space models of the form (Assumption \ref{as:parameterisation}): $\hat{\mathbf{x}}(t+1)=\hat{A}_{\theta} \hat{\mathbf{x}}(t)+ \hat{B}_{\theta} \begin{bmatrix} \mathbf{y}(t) \\ \mathbf{u}(t) \end{bmatrix}$, $\hat{\mathbf{x}}(0)=0$, $\hat{\mathbf{y}}(t \mid 0)= \hat{C}_{\theta} \hat{\mathbf{x}}(t) + \hat{D}_{\theta,u} \mathbf{u}(t)$, where $\hat{\mathbf{y}}(t \mid 0)$ is the predicted output at time $t$ given all past outputs $\{\mathbf{y}(s)\}_{s=0}^{t-1}$ and past/present inputs $\{\mathbf{u}(s)\}_{s=0}^{t}$. 

\textbf{LTI systems as linear RNNs.}
 This model class corresponds exactly to linear RNNs with inputs $\mathbf{w}(t)=\begin{bmatrix} \mathbf{y}(t) \\ \mathbf{u}(t) \end{bmatrix}$ whose last layer is a completely linear layer $\hat{\mathbf{y}}(t \mid 0)=\hat{C}_{\theta} \hat{\mathbf{x}}(t) + \hat{D}_{\theta} \mathbf{w}(t)$, where $\hat{D}_{\theta}=\begin{bmatrix} 0 & \hat{D}_{\theta,u} \end{bmatrix}$ is such that only the component of $\mathbf{w}(t)$ corresponding to the input $\mathbf{u}(t)$ is used in the linear combination. A general RNN would then be of the form $\hat{\mathbf{x}}(t+1)=\stackrel{\rightarrow}{\sigma}(\hat{A}_{\theta} \hat{\mathbf{x}}(t) + \hat{B}_{\theta} \begin{bmatrix} \mathbf{y}(t) \ \mathbf{u}(t) \end{bmatrix}), ~ \hat{\mathbf{x}}(0)=0$ , $\hat{\mathbf{y}}(t \mid 0)=\hat{C}_{\theta} \hat{\mathbf{x}}(t) + \hat{D}_{\theta,u} \mathbf{u}(t)$ where $\stackrel{\rightarrow}{\sigma}$ is a component-wise application of nonlinear activation function $\sigma$, such as a sigmoid, ReLU or tanh function. Our case correspond to choosing $\sigma$ to be the identity activation function, i.e., $\sigma(x)=x$. That is, learning linear RNNs is exactly the learning problem considered in this paper, and the proven bounds apply directly to linear RNNs. 

\textbf{Differences with respect to classical machine learning problems and time-series prediction.}
 What perhaps distinguishes us from a more classical setting is that:
 
    Autoregressive modeling: We also use past values of the output $\mathbf{y}$ to predict current value of $\mathbf{y}$: if we choose the parametrization so that $\hat{B}_{\theta}=\begin{bmatrix} 0 & \hat{B}_{\theta,u} \end{bmatrix}$, i.e., $\hat{B}_{\theta}\mathbf{w}(t)=\hat{B}_{\theta,u}\mathbf{u}(t)$ then we get the classical setting of learning models to predict the current output based on past and current inputs.
    
    Learning from a single time series: We assume that the training data is a single sample $\{y(t),u(t)\}_{t=0}^{N}$ from the random variables $\{\mathbf{y}(t), \mathbf{u}(t)\}_{t=0}^{N}$, instead of having $N$ independently drawn samples $\{y_i(t),u_i(t)\}_{t=0,i=1}^{T,N}$ of length $T$.
    
    Realizability: We assume that $\mathbf{w}(t)=\begin{bmatrix} \mathbf{y}(t) \\ \mathbf{u}(t) \end{bmatrix}$ is generated by a LTI state-space system.

The first two features are not unique to learning for control, but it is present in machine learning. The last assumption is more control-oriented but is used only for proving the bound, not for the learning algorithms themselves. As it was explained before, realizability ensures that the output can be predicted from past outputs and inputs using a linear state-space model (linear RNN), making the learning goal attainable. From a control perspective, this is a natural sanity check: if the underlying system cannot be modeled as a noisy LTI state-space system, then no LTI model will be good enough for control. From a learning perspective, realizability ensures data mixing, which is a standard requirement for the empirical loss to converge to the true loss (a law of large numbers). Note that i.i.d.\ data is mixing with zero mixing coefficients. While one could assume mixing explicitly, estimating mixing coefficients for partially known systems is difficult. By contrast, assuming the data is generated by a stable LTI system allows us to replace mixing coefficients with quantities depending on estimates on matrix norms and Lyapunov exponents of the data generator -- information often available from physical considerations.

For this reason we believe that the problem considered in this paper is a particular case of time-series prediction/learning problem, and the results apply directly to learning linear RNNs from time-series data.

\textbf{Challenges in extending our results to nonlinear RNNs.}
 The results for  LTI case could be useful for RNN case, as LTIs represent a special case of RNNs.  However,  extending the results to nonlinear RNNs faces several technical obstacles. First, extending the truncation argument for unbounded noise with Lipschitz loss to nonlinear state-space models is nontrivial because the output is not a simple convolution of inputs with Markov parameters. While \cite{alquier2012pred} applied truncation to nonlinear autoregression, extending it to our setting would require bounding Lipschitz constants and decay rates of nonlinear predictors rather than relying on Markov parameters. Second, our truncation argument for quadratic loss exploits the fact that stable LTI systems driven by sub-Gaussian noise produce bounded responses with high probability. Whether this property extends to nonlinear models remains unclear. Third, our parameter estimation bounds (especially $H_2$ norm bounds) rely fundamentally on realization theory and identifiability of LTI systems -- tools absent for general nonlinear systems. Extensions may be possible for restricted classes such as bilinear or switched systems with developed realization theory.

 \subsection{Assumptions on the loss}
 \label{loss:assumption}
\subsection{Proof of Theorem \ref{thm:main}: quadratic loss}
  We assume the loss to be $L$-Lipschitz or quadratic For the bounded-noise case, $L$-Lipschitzness ensures that closeness in the predicted output translates to closeness in the loss: if $|\hat{\mathbf{y}}_{\theta}(t) - \mathbf{y}(t)|$ is small, then $\ell(\hat{\mathbf{y}}_{\theta}(t), \mathbf{y}(t))$ is also small. This is exploited in Lemmas \ref{lemma:bounded_mgf(V-hatL)} and \ref{lemma:bounded:mgf(L-V)}. For the unbounded-noise case, Lipschitzness is used to reduce the problem to the bounded case. Specifically, we upper bound the moment generating function of the loss difference between original and truncated data (Lemma \ref{vl-bl:lemma0}), where the truncation bounds the noise. The key insight is that although quadratic loss is not globally Lipschitz, for LTI systems the prediction-error norm is high-probability bounded. We therefore replace quadratic loss with a truncated version that coincides with it in the high-probability region and is Lipschitz globally (Appendix \ref{thm:main:quadratic}).
Lipschitzness is standard in PAC bounds, even in the i.i.d., setting, whereas convexity is typically not required. Convexity and smoothness matter for optimizing the empirical loss, not for deriving generalization bounds. 
In system identification, quadratic loss is preferred because it is both convex and smooth, facilitating optimization. But for unbounded data, quadratic loss is not globally Lipschitz, motivating the truncation approach above.

One may wonder what the role of convexity is in the presented bounds. Similarly to other PAC bounds,
convexity of the loss function does not play a role in our bounds, except in Corollary \ref{lem:bound:mean}, in the part on the bound of parameter estimation error, we assume strong convexity ($\mathbf{SConv}$, Assumption \ref{as:oracle}) of the true loss in the parameters, which is distinct from the loss function being convex. Strong convexity of the true loss depends on both the loss function and the predictor's dependence on parameters, and is not implied by convexity of the loss itself.

\subsection{Weak dependence, mixing coefficients, convolution}
\label{mixing:def}
Consider a bounded stationary process $\mathbf{z}$ taking values in $\mathbb{R}^{n_z}$ for some 
$n_z > 0$. Let $\mathcal{F}_{1-Lip,n}$ be the set of functions $f:(\mathbb{R}^{n_r})^n \rightarrow \mathbb{R}$ which are $1$-Lipschitz in the following sense:
$\|f(y_1,\ldots,y_n)-f(y_1',\ldots,y_n')\|_2 \le \sum_{i=1}^{n} \|y_i-y_i'\|_2$
for all $y_1,\ldots,y_n,y_1',\ldots,y'_n \in \mathbb{R}^{n}$. 
Let $\mathcal{F}_{t}$ be the $\sigma$-algebra generated by $\{\mathbf{z}(j)\}_{j \le t}$.
Define the $\gamma$-weak mixing coefficients $\theta_{\infty,n}(1) = \sup_{t \in \mathbb{Z}} \sup_{f\in \mathcal{F}_{1-Lip,n}} \left| \mathbb{E}[f(\mathbf{z}(t+1),\ldots,\mathbf{z}(t+n)) \mid \mathcal{F}_t] - \mathbb{E}[f(\mathbf{z}(t+1),\ldots,\mathbf{z}(t+n))] \right|$ of $\mathbf{z}$.
We say that $\mathbf{z}$ is weakly dependent, if $\theta_{\infty}(1)=\sup_{n \ge 1} \theta_{\infty,n}(1) < +\infty$.
It then follows that if $\mathbf{r}$ is a bounded i.i.d. process, and 
$\alpha=\{\alpha(k)\}_{k=0}^{\infty}$ is a square summable sequence of matrices,
then 
$(\alpha \star \mathbf{r})$ is stationary and weakly dependent with
$\theta_{\infty}(1) \le 2\theta_{\infty}(\alpha)\mathcal{C}$,
see \cite[Proposition 4.1]{alquier2012pred}.

\subsection{Persistence of excitation, identifiability}
\label{pers:excitation}
 The PE condition amounts to the input being persistently exciting of order $K$ for any $K$. Practically, it means that for any $k$, for large enough $N$, the following matrix formed from the observed sample path $w(t)=\begin{bmatrix} y^T(t) \ u^T(t) \end{bmatrix}^T$ of the process $\mathbf{w}(t)=\begin{bmatrix} \mathbf{y}^T(t) \ \mathbf{u}^T(t) \end{bmatrix}^T$ satisfies the following definiteness condition $m_{\mathbf{w}}I \le \frac{1}{N} W_{0, \mid k,N} W_{0, \mid k,N}^T$, where $W_{0, \mid k,N} =\begin{bmatrix} w(0) & w(1) & \cdots & w(N-1) \\ w(1) & w(2) & \cdots & w(N) \\ \vdots & \vdots & \cdots & \vdots \\ w(k) & w(k+1) & \cdots & w(N+k-1) \end{bmatrix}$. The existence of a constant $m_{\mathbf{w}}$ is ensured by the assumption that the data generator is minimum-phase, or, alternatively, that the joint input-output process $\mathbf{w}(t)=\begin{bmatrix} \mathbf{y}^T(t) \ \mathbf{u}^T(t) \end{bmatrix}^T$ is generated by an invertable ARMA system (see Lemma B.1 in [22]). This is a widespread assumption in system identification.

While the existence of $m_{\mathbf{w}}$ is ensured, its value indeed needs to be estimated. 
This can be done by taking $m_{\mathbf{w}}$ as a lower bound of the eigenvalues of $\frac{1}{N} W_{0\mid k,N} W_{0\mid k,N}^\top$ for large $k$ and $N$: if the estimate is much lower that the true value of $m_{\mathbf{w}}$, the bound will be more conservative, if it is too optimistic (higher than the true value), the bound may be invalid. 

The assumption (\textbf{SConv}) of Assumption \ref{as:oracle} implies that 
$\mathcal{L}(\theta) - \mathcal{L}(\theta_{\text{true}}) \ge m_{\Theta} \|\theta - \theta_{\text{true}}\|_2^2$.
In turn, the latter means \emph{structural identifiability} around $\theta_{\text{true}}$, i.e.,
that for any $\theta \ne \theta_{\text{true}}$, 
the input-output behavior of the corresponding models must be different.
Otherwise, their true losses would be the same leading to
$0 \ge m_{\Theta} \|\theta - \theta_{\text{true}}\|_2^2 > 0$, which is a contradiction. 
It also implies that $\rvw$ is \emph{persistently exciting} in the sense that if
the response of two models parameterized by $\theta \ne \theta_{\text{true}}$ is the same for $\rvw$, then these models are identical and hence there output response is 
the same for any input. Indeed, if the output response of the two models to $\rvw$
are identical, then $\rvy_{\theta}(t)=\rvy_{\theta_{\text{true}}}(t)$  and hence
$\mathcal{L}(\theta) = \mathcal{L}(\theta_{\text{true}})$.

That is, the value $m_{\Theta} > 0$ captures both identifiability and persistence of excitation,
while $m_{\rvw}$ captures only persistence of excitation. However, $m_{\rvw}$ is used to bound the $H_2$ distance between the true and estimated systems, and the latter distance corresponds to the distance
between the Markov-parameters. In turn, if we parametrize LTI systems by the infinite sequence of their Markov-parameters, the resulting parametrization is always structurally identifiable, as 
Markov-parameters determine the input-output behavior uniquely. 

\subsection{Additional explanation of the results  for bounded, unbounded and Lipschitz loss and unbounded and quadratic loss cases}
\label{app:additional_intuition}

In a nutshell, the statements of Theorem \ref{thm:main} and Corollary \ref{thm:bounded:alt:col} can be interpreted as follows

\textbf{Bounded data} This is the simplest case, the data is assumed to be generated by an LTI system driven by a noise whose norm is bounded by $c_{\rve}$, i.e.,
$\|\e_g(t)\|_ {\infty} \le c_{\rve}$. Then the learning data $\|\begin{bmatrix} \y(t) \\ \rvu(t) \end{bmatrix}\|_{2} \le c_{\rve}\sqrt{n_e}\|\alpha_g\|_2$
and $\max\{\|\y(t)-\hat{\y}_{\theta}(t)\|_2,\|\y(t)-\hat{\y}_\theta(t\mid 0)\|_2\} \le G_e(\theta)c_{\rve}\sqrt{n_e}$ and the statement of Theorem \ref{thm:main} is as follows:
  \begin{theorem}[Theorem \ref{thm:main}: bounded case]
    \label{thm:main:bounded}
Let Assumptions \ref{as:generator}, \ref{as:loss} and \ref{as:parameterisation} hold and
assume that the noise process of the data generator satisfies $\|\e_g(t)\|_ {\infty} \le c_{\rve}$.
    For any a prior density $\pi$ over $\Theta$, confidence level  $\delta\in(0,0.5)$, integer $\epsilon \in \{-1,1\}$, $N \ge 2$, and $\lambda  > 0$, it holds that 
    \begin{equation*}
    \begin{aligned}
         &\bP\Big(\forall \rho  \in \mathcal{M}_\pi: \epsilon E_{\theta\sim\rho} \mathcal{L}(\theta)
         \leq   \epsilon E_{\theta\sim\rho}\hat{\mathcal{L}}_N(\theta) +   \\
          & \frac{1}{\lambda}\Big(\KL(\rho\|\pi)  +\ln\dfrac{1}{\delta} + \frac{1}{2}\sum_{i=1}^2
            \ln E_{\theta\sim\pi} \exp\left(\mathcal{C}_{b,i}(2\lambda,\theta,L(\theta),c_{\rve},N)\right) \ Big)\Big) \ge 1-2\delta \,
    \end{aligned}
    \end{equation*}
    with
    $\KL(\rho\|\pi)=E_{\theta\sim\rho}\ln\frac{\rho(\theta)}{\pi(\theta)}$ the Kullback–Leibler divergence
    and 
    \begin{equation*}
    \begin{aligned}
         & 
         µ\mathcal{C}_{b,i}(\lambda,\theta,L(\theta),c_{\rve},N), ~
          \mathcal{C}_{b,i}(s,\theta,L,c,N) \triangleq 
         \frac{\mathcal{C}_{m,b,i}(s,\alpha_{\theta},G_e(\theta)+2G_{e,1}(\theta),L,c)}{N} \\
         & L(\theta)   \triangleq \left\{\begin{array}{rl} 2 G_e(\theta)c_{\rve}\sqrt{n_{\rve}} & \mbox{ $\ell$ is quadratic loss} \\
              L_{\ell} &  \mbox{ $\ell$ is $L_\ell$-Lipschitz }
         \end{array}\right.   \\
         &   \mathcal{C}_{m,b,i}(s,\alpha,X,L,c) 
         \triangleq \left\{ \begin{array}{rl}
            s Lc\sqrt{n_{\rve}}\theta_{\infty}(\alpha)\|\alpha_g\|_{\ell_1} & i=1 \\
             s^2 \frac{L^2}{2} X (\sqrt{n_{\rve}}c)^2 & i=2
          \end{array}\right. 
        \end{aligned}
      \end{equation*}          
    \end{theorem}
     In particular, for $\epsilon=1$ the theorem states:
     \begin{equation*}\begin{aligned}
        &\bP\Big(\forall \rho \in \mathcal{M}_\pi: E_{\theta\sim\rho} \mathcal{L}(\theta)
         \leq   E_{\theta\sim\rho}\hat{\mathcal{L}}_N(\theta) +   \\
          & \frac{1}{\lambda}\Big(\KL(\rho\|\pi)  +\ln\dfrac{1}{\delta} + \frac{1}{2}\sum_{i=1}^2
            \ln E_{\theta\sim\pi} \exp\left(\mathcal{C}_{b,i}(2\lambda,\theta,L(\theta),c_{\rve},N)\right) \ Big)\Big) \ge 1-2\delta \,
    \end{aligned}
    \end{equation*}
    and for $\epsilon=-1$ it states:
    \begin{equation*}\begin{aligned}
            & \bP\Big(\forall \rho  \in \mathcal{M}_\pi: 
              E_{\theta\sim\rho}\hat{\mathcal{L}}_N(\theta)  \leq E_{\theta\sim\rho} \mathcal{L}(\theta) +  \\
          & \frac{1}{\lambda}\Big(\KL(\rho\|\pi)  +\ln\dfrac{1}{\delta} + \frac{1}{2}\sum_{i=1}^2
            \ln E_{\theta\sim\pi} \exp\left(\mathcal{C}_{b,i}(2\lambda,\theta,L(\theta),c_{\rve},N)\right) \ Big)\Big) \ge 1-2\delta \,
    \end{aligned}
    \end{equation*}
   and by using  the union bound we get the following bound on the absolute value of the
   generalization gap:
   \begin{equation*}\begin{aligned}
            & \bP\Big(\forall \rho  \in \mathcal{M}_\pi: 
         \left|E_{\theta\sim\rho} \mathcal{L}(\theta) 
              -  E_{\theta\sim\rho}\hat{\mathcal{L}}_N(\theta) \right| \le  \\
          & \frac{1}{\lambda}\Big(\KL(\rho\|\pi)  +\ln\dfrac{1}{\delta} + \frac{1}{2}\sum_{i=1}^2
            \ln E_{\theta\sim\pi} \exp\left(\mathcal{C}_{b,i}(2\lambda,\theta,L(\theta),c_{\rve},N)\right) \ Big)\Big) \ge 1-4\delta \,
    \end{aligned}
    \end{equation*}
    That is, the constant $\mathcal{C}_i(\lambda, \theta,N,\delta,\epsilon)$, $i=1,2$ in Theorem \ref{thm:main} in the case of bounded data equals 
    $\mathcal{C}_{b,i}(\lambda,\theta,L(\theta),c_{\rve},N)$ and hence it
    does not depend on $\delta$ and $\epsilon$ for the bounded case.

Informally, the theorem says that for any prior $\pi$ on models, uniformly over all posteriors, the average (w.r.t. to a posterior $\rho$ on models) generalization gaps $E_{\theta \sim \rho}\mathcal L(\theta) -E_{\theta \sim \rho}\widehat{\mathcal L}_N(\theta)$ is bounded from above with probability $1-\delta$ over data, by the expression $\frac{1}{\lambda}(KL(\rho \| \pi)+\ln(\frac{1}{\delta}) + \frac{1}{2} \ln E_{\theta \sim \pi} \text{exp}(\mathcal{C}_{b,1}(2\lambda, \theta,c,N))+\frac{1}{2} \ln E_{\theta \sim \pi} \text{exp}(\mathcal{C}_{b,2}(2\lambda, \theta,c_{\rve},N))$, where $\mathcal{C}_{b,i}(2\lambda, \theta,c_{\rve},N)$ are constants $O(\lambda^2/N)$. That is, the bound includes a KL divergence $KL(\rho \|\pi)$ between the prior and posterior, the term $\ln(\frac{1}{\delta})$ describing the dependence on the confidence level $\delta$, and the average w.r.t. the prior $\pi$: $\frac{1}{2} \sum_{i=1}^{2} \ln E_{\theta \sim \pi} \text{exp}(\mathcal{C}_{b,i}(2\lambda, \theta,N))$. Note that the bound in Theorem \ref{thm:main} depends only on the prior and the distance of the posterior from the prior and on the temperature parameter $\lambda$. Intuitively, if $\lambda$ is small, we penalize the difference between the posterior and the prior, i.e., we would like the posterior to be close to the prior.

The constants $\mathcal{C}_{b,i}(2\lambda,\theta,c_{\rve},N)$ depend on  $L(\theta)$, the Markov parameters $\alpha_\theta$ of the predictor parametrized by $\theta$ and 
the Markov parameters $\alpha_g$ of the data generator. The dependence takes place via the constants
$G_{e}(\theta),G_{e,1}(\theta)$ and $\theta_{\infty}(\alpha_{\theta})$, $\|\alpha_g\|_{\ell_1}$ which are functions of the $\ell_1$ norms $\|\alpha_\theta\|_{\ell_1}$, 
$\|\alpha_g\|{\ell_1}$ and WDMC $\theta_{\infty}(\alpha_{\theta}),\theta_{\infty}(\alpha_g)$
of these Markov parameters. The constant $L(\theta)$ is the Lipschitz loss of the loss function $\ell$
applied to the outputs of the data generator and the predictions generated by the predictor corresponding to the parameter $\theta$. If the loss is $L_{\ell}$-Lipschitz, then $L(\theta)=L_{\ell}$.
If the loss $\ell$ is quadratic, then the restriction of $\ell$ to the set containing 
the output of the data generating system and the predictor corresponding to $\theta$ is $2 G_e(\theta)c_{\rve}\sqrt{n_{\rve}}$.

Corollary \ref{thm:bounded:alt:col} is obtained from Theorem \ref{thm:main} by taking $\lambda=\sqrt{N}$ and replacing 
\[ \theta_{\infty}(\alpha_{\theta}), ~ G_{e}(\theta), ~ L(\theta), ~ G_{1,\theta}
\] 
by their uniform upper bounds 
\[ \theta_{\infty}(\bar{\alpha}), ~ G_e, ~ G_{e,1}, ~ L_b=\left\{\begin{array}{rl} 2 G_ec_{\rve}\sqrt{n_{\rve}} & \mbox{ $\ell$ is the quadratic loss} \\
              L_{\ell} & \mbox{$\ell$ is $L_\ell$-Lipschitz }
         \end{array}\right.\] 
Then $\mathcal{C}_{i}(\lambda,\theta,N)$ is upper bounded by $\mathcal{C}_{b,m,i}(2\sqrt{N},\bar{\alpha},G_e+2G_{e,1},L_b,c_{\rve})$ which is independent of $\theta$, hence,  $\frac{1}{2\lambda}\ln E_{\theta\sim\pi} \exp\left(\mathcal{C}_i(2\lambda,\theta,N,\delta,\epsilon)\right) \le \frac{1}{2\lambda} \mathcal{C}_{b,m,i}(2\sqrt{N},\bar{\alpha},G_e+2G_{e,1},L_b,c_{\rve}) \le \frac{\mathcal{C}_{b,m,i}(2,\bar{\alpha},G_e+2G_{e,1},L_b,c_{\rve})}{\sqrt{N}}$, leading to the statement of Corollary
\ref{thm:bounded:alt:col}: 
\begin{corollary}[Corollary \ref{thm:bounded:alt:col}: bounded data case]
For every prior $\pi$, $\delta (0,0.5)$, $\epsilon \in \{-1,1\}$,
\begin{equation*}
\begin{split}
   &\bP\Big(\forall  \rho \in\mathcal{M}_\pi:  \epsilon E_{\theta \sim \rho } \mathcal{L}(\theta) \leq \epsilon E_{\theta\sim \rho } \hat{\mathcal{L}}_N(\theta)+  \\
  &  \frac{1}{\sqrt{N}}\left(\KL(\rho ||\pi)+\ln\frac{1}{\delta} + \frac{1}{2} \sum_{i=1}^2 \mathcal{C}_{m,b,i}(2,\bar{\alpha},G_e+2G_{e,1},L_b,c_{\rve})
   \right)\Big) \ge 1-2\delta.
\end{split}
\end{equation*}
\end{corollary}
That is, for $\epsilon=1$, it states that:
\begin{equation*}
\begin{split}
  & \bP\Big(\forall  \rho \in\mathcal{M}_\pi:  E_{\theta \sim \rho } \mathcal{L}(\theta) \leq  E_{\theta\sim \rho } \hat{\mathcal{L}}_N(\theta)+  \\
  &  \frac{1}{\sqrt{N}}\left(\KL(\rho ||\pi)+\ln\frac{1}{\delta} + \frac{1}{2} \sum_{i=1}^2 \mathcal{C}_{m,b,i}(2,\bar{\alpha},G_e+2G_{e,1},L_b,c_{\rve})
   \right)\Big) \ge 1-2\delta\,
\end{split}
\end{equation*}
and for $\epsilon=-1$:
\begin{equation*}
\begin{split}
   &\bP\Big(\forall  \rho \in\mathcal{M}_\pi: E_{\theta\sim \rho } \hat{\mathcal{L}}_N(\theta) \leq  E_{\theta \sim \rho } \mathcal{L}(\theta) + \\
  &  \frac{1}{\sqrt{N}}\left(\KL(\rho ||\pi)+\ln\frac{1}{\delta} + \frac{1}{2} \sum_{i=1}^2 \mathcal{C}_{m,b,i}(2,\bar{\alpha},G_e+2G_{e,1},L_b,c_{\rve})
   \right)\Big) \ge 1-2\delta\,
\end{split}
\end{equation*}
and with the union bound
\begin{equation*}
\begin{split}
   & \bP\Big(\forall \rho \in\mathcal{M}_\pi:   \left|E_{\theta\sim \rho } \hat{\mathcal{L}}_N(\theta) -  E_{\theta \sim \rho } \mathcal{L}(\theta)\right| \leq \\
  &  \frac{1}{\sqrt{N}}\left(\KL(\rho ||\pi)+\ln\frac{1}{\delta} + \frac{1}{2} \sum_{i=1}^2 \mathcal{C}_{m,b,i}(2,\bar{\alpha},G_e+2G_{e,1},L_b,c_{\rve})
   \right)\Big) \ge 1-4\delta.
\end{split}
\end{equation*}
That is, the upper bound of Corollary \ref{thm:bounded:alt:col}
is $\frac{1}{\sqrt{N}}(KL(\rho \| \pi)+\ln(\frac{1}{\delta}) + O(1))$ as $\mathcal{C}_{b,i}$ becomes $O(1)$.
 Bounding the generalization gap and its multiple by $-1$ allows us to bound \textbf{1)} the absolute value of the generalization gap using the union bound, \textbf{2)} to derive oracle inequalities, i.e., bound the gap between the true error of the learned model and the best possible true error. In turn, the latter is used for deriving bounds on the parameter estimation error.

The bound above assumes that the data is generated by an LTI system with bounded inputs and noise. However, in many applications the standard assumption is that the noise and inputs are (sub-)Gaussian, and hence potentially unbounded. This prompted us to extend the results to data generated by systems with sub-Gaussian noise.

\textbf{Unbounded data and $L$-Lipschitz loss functions} In this case, we approximate unbounded data with bounded data generated by the same LTI system, but with a truncated, and hence bounded noise: if $\e_g(t)$ is the noise of the data generating system, then we use the truncated noise $\e_{g,N}(t)=\e_g(t)$ if $\|\e_g(t)\|_2 < c_N$ and $\e_{g,N}(t)=0$ otherwise, where 
$c_N=2(\ln(N))/\max\{1,\|\bar{\alpha}\|_{\ell_1}\}$, and we consider the data 
$\bar{\y}(t)$, $\bar{\rvu}(t)$ generated by the data generator with $\e_g$ replaced by 
$\e_{g,N}$ and the predicted values $\bar{\hat{\y}}_{\theta}(t)$, $\bar{\hat{\y}}_{\theta}(t \mid 0)$ 
generated by the predictor corresponding to the parameter $\theta$, with $\y,\rvu$
replaced by the bounded values $\bar{\y},\bar{\rvu}$. 

Then we can reduce the unbounded case to the bounded one., by adding extra terms. The resulting bound is similar to the bounded case, leading to the following.
\begin{theorem}[Theorem \ref{thm:main}: unbounded data, Lipschitz loss case]
    \label{thm:main:unbounded}
    Let Assumptions \ref{as:generator}, \ref{as:loss} and \ref{as:parameterisation} hold, and assume
    that $\ell$ is $L_{\ell}$-Lipschitz.
    For any a prior density $\pi$ over $\Theta$, confidence level  $\delta\in(0,0.5)$, integer $\epsilon \in \{-1,1\}$, $N \ge 2$, and $0 < \lambda  < \frac{N}{32L_\ell}$, it holds that
    \begin{equation*}
    \begin{aligned}
         \bP\Big(\forall \rho & \in \mathcal{M}_\pi: \epsilon E_{\theta\sim\rho} \mathcal{L}(\theta)
         \leq   \epsilon E_{\theta\sim\rho}\hat{\mathcal{L}}_N(\theta) +   \\
          & \frac{1}{\lambda}\Big(\KL(\rho\|\pi)  +\ln\dfrac{1}{\delta} + \frac{1}{2}\sum_{i=1}^2
            \ln E_{\theta\sim\pi} \exp\left(\mathcal{C}_{ub,i}(2\lambda,\theta,L_{\ell},N)\right) \Big)\Big) \ge 1-2\delta \,
    \end{aligned}
    \end{equation*}
    with
    $\KL(\rho\|\pi)=E_{\theta\sim\rho}\ln\frac{\rho(\theta)}{\pi(\theta)}$ the Kullback–Leibler divergence, and 
    \begin{equation*}
    \begin{aligned}
        & 
         \mathcal{C}_{ub,i}(\lambda,\theta,L,N)
         \triangleq 
         \frac{1}{2} \mathcal{C}_{b,i}\left(2\lambda,\theta,L,\frac{2\ln(N)}{\|\alpha_g\|_{\ell_1}K},N\right) +
            \bar{\mathcal{C}}_{ub}(\lambda,L,N) \\
            &  \bar{\mathcal{C}}_{ub}(\lambda,L,N)  \triangleq \frac{\lambda^2 32 L^2 C_{u,1}}{N}+ \frac{4\lambda L C_{u,2} \ln(N)}{N}
        \end{aligned}
      \end{equation*}          
      where $\mathcal{C}_{b,i}(s,\theta,L,c)$ is defined as in Theorem \ref{thm:main:bounded}, and
      the constants, $K$,$C_{u,j}$, $j=1,2$ are  defined as follows:
    \begin{equation*}
     \begin{aligned}
        &  C_{u,1} \triangleq \bar{\Psi}_{\e_g}(\|\bar{\alpha}\|_{\ell_1}\|\alpha_g\|_{\ell_1})+\bar{\Psi}_{\e_g}(\|\alpha_g\|_{\ell_1}),  \quad \bar{\Psi}_{\rve_g}(s)=e^{14^2s^2 \|\rve_g\|_{\psi_2}^2+3s \|\rve_g\|_{\psi_2}} \\
        & C_{u,2} \triangleq \bar{\Psi}_{\rve}(\|\bar{\alpha}\|_{\ell_1}\|\alpha_g\|_{\ell_1})
        \frac{\textcolor{black}{\|\bar{\alpha}\|_{\ell_1}}}{K}+
        \bar{\Psi}_{\e_g}(\|\alpha_g\|_{\ell_1}) \frac{\textcolor{black}{1}}{K},  \quad
        \quad K=\min\{1,\|\bar{\alpha}\|_{\ell_1}\}
   \end{aligned}
    \end{equation*}
    \end{theorem}
     In particular, for $\epsilon=1$ the theorem states:
     \begin{equation*}\begin{aligned}
         & \bP\Big(\forall  \rho  \in \mathcal{M}_\pi: E_{\theta\sim\rho} \mathcal{L}(\theta)
         \leq   E_{\theta\sim\rho}\hat{\mathcal{L}}_N(\theta) +   \\
          & \frac{1}{\lambda}\Big(\KL(\rho\|\pi)  +\ln\dfrac{1}{\delta} + \frac{1}{2}\sum_{i=1}^2
            \ln E_{\theta\sim\pi} \exp\left(\mathcal{C}_{ub,i}(2\lambda,\theta,L_{\ell},N)\right) \Big)\Big) \ge 1-2\delta \,
    \end{aligned}
    \end{equation*}
    and for $\epsilon=-1$ it states:
    \begin{equation*}\begin{aligned}
            & \bP\Big(\forall  \rho \in \mathcal{M}_\pi: 
              E_{\theta\sim\rho}\hat{\mathcal{L}}_N(\theta)  \leq E_{\theta\sim\rho} \mathcal{L}(\theta) +  \\
          & \frac{1}{\lambda}\Big(\KL(\rho\|\pi)  +\ln\dfrac{1}{\delta} + \frac{1}{2}\sum_{i=1}^2
            \ln E_{\theta\sim\pi} \exp\left(\mathcal{C}_{ub,i}(2\lambda,\theta,L_{\ell},N)\right) \Big)\Big) \ge 1-2\delta \,
    \end{aligned}
    \end{equation*}
   and by using  the union bound we get the following bound on the absolute value of the
   generalization gap:
   \begin{equation*}\begin{aligned}
            &\bP\Big(\forall \rho  \in \mathcal{M}_\pi: 
         \left|E_{\theta\sim\rho} \mathcal{L}(\theta) 
              -  E_{\theta\sim\rho}\hat{\mathcal{L}}_N(\theta) \right| \le  \\
          & \frac{1}{\lambda}\Big(\KL(\rho\|\pi)  +\ln\dfrac{1}{\delta} + \frac{1}{2}\sum_{i=1}^2
            \ln E_{\theta\sim\pi} \exp\left(\mathcal{C}_{ub,i}(2\lambda,\theta,L_{\ell},N)\right) \Big)\Big) \ge 1-4\delta \,
    \end{aligned}
    \end{equation*}
   
    Note that the constant $\mathcal{C}_i(\lambda, \theta,N,\delta,\epsilon)=\mathcal{C}_{ub,i}(\lambda,\theta,L(\theta),N)$, $i=1,2$ of Theorem \ref{thm:main} 
    again does not depend on $\delta$ and $\epsilon$.

  Informally, the right-hand side of Theorem \ref{thm:main} is now $\frac{1}{\lambda}[KL(\rho \| \pi)+\ln \frac{1}{\delta} + \frac{1}{2} \ln E_{\theta \sim \rho} \exp(\frac{1}{2}C_{b,1}(4\lambda,\theta,c_N,N)+C_{ub}(2\lambda,L,N))+\frac{1}{2} \ln E_{\theta \sim \rho} \exp(\frac{1}{2} C_{b,2}(4\lambda, \theta,c_N, N)+C_{ub}(2\lambda,L,N))]$ where $C_{ub}(\lambda,N)=O(\lambda^2 L^2 \ln(N)/N)$ is an additional term accounting for approximating unbounded data with bounded one.
  This additional term represents an upper bound on the moment generating function for the difference between the empirical losses for the original and truncated data. 

The bound of Corollary \ref{thm:bounded:alt:col} can then be obtained from this bound by taking $\lambda=16 L_{\ell}\sqrt{N}/\ln(N)$, leading to the bound
\begin{corollary}[Corollary \ref{thm:bounded:alt:col}: unbounded, Lipschitz loss  case]
For every prior $\pi$, $\delta (0,0.5)$, $\epsilon \in \{-1,1\}$,
\begin{equation*}
\begin{split}
   & \bP\Big(\forall  \rho \in\mathcal{M}_\pi: \epsilon E_{\theta \sim \rho } \mathcal{L}(\theta) \leq \epsilon E_{\theta\sim \rho } \hat{\mathcal{L}}_N(\theta)+  
    \frac{16L_{\ell}\ln(N)}{\sqrt{N}}\Big(\KL(\rho ||\pi)+\ln\frac{1}{\delta} + \\
    & C_{u,1}+C_{u,2}+\frac{1}{2} \sum_{i=1}^2 \mathcal{C}_{m,b,i}(2,\bar{\alpha},G_e+2G_{e,1},1,\frac{1}{\|\alpha_g\|_{\ell_1}K})
   \Big)\Big) \ge 1-2\delta
\end{split}
\end{equation*}
\end{corollary}
That is, for $\epsilon=1$, it states that:
\begin{equation*}
\begin{split}
   & \bP\Big(\forall  \rho \in\mathcal{M}_\pi:  E_{\theta \sim \rho } \mathcal{L}(\theta) \leq  E_{\theta\sim \rho } \hat{\mathcal{L}}_N(\theta)+  
   \frac{16L_{\ell}\ln(N)}{\sqrt{N}}\Big(\KL(\rho ||\pi)+\ln\frac{1}{\delta} + \\
   &  C_{u,1}+C_{u,2}+\frac{1}{2} \sum_{i=1}^2 \mathcal{C}_{m,b,i}(\frac{1}{4},\bar{\alpha},G_e+2G_{e,1},1,\frac{1}{\|\alpha_g\|_{\ell_1}K})
   \Big)\Big) \ge 1-2\delta\,
\end{split}
\end{equation*}
and for $\epsilon=-1$:
\begin{equation*}
\begin{split}
   &\bP\Big(\forall  \rho \in\mathcal{M}_\pi:  E_{\theta\sim \rho } \hat{\mathcal{L}}_N(\theta) \leq  E_{\theta \sim \rho } \mathcal{L}(\theta) + 
  \frac{16L_{\ell}\ln(N)}{\sqrt{N}}\Big(\KL(\rho ||\pi)+\ln\frac{1}{\delta} + \\
  & C_{u,1}+C_{u,2}+\frac{1}{2} \sum_{i=1}^2 \mathcal{C}_{m,b,i}(\frac{1}{4},\bar{\alpha},G_e+2G_{e,1},1,\frac{1}{\|\alpha_g\|_{\ell_1}K})
  \Big)\Big) \ge 1-2\delta\,
\end{split}
\end{equation*}
and with the union bound
\begin{equation*}
\begin{split}
   &\bP\Big(\forall  \rho \in\mathcal{M}_\pi:  \left|E_{\theta\sim \rho } \hat{\mathcal{L}}_N(\theta) -  E_{\theta \sim \rho } \mathcal{L}(\theta)\right| \leq 
    \frac{16L_{\ell}\ln(N)}{\sqrt{N}}\Big(\KL(\rho ||\pi)+\ln\frac{1}{\delta} + \\
    &  C_{u,1}+C_{u,2}+\frac{1}{2} \sum_{i=1}^2 \mathcal{C}_{m,b,i}(\frac{1}{4},\bar{\alpha},G_e+2G_{e,1},1,\frac{1}{\|\alpha_g\|_{\ell_1}K})
  \Big)\Big) \ge 1-4\delta.
\end{split}
\end{equation*}
I.e., the upper bound on the generalization gap is of the form $O\left(\frac{\ln(N)}{\sqrt{N}}\right)(KL(\rho \| \pi)+\ln(\frac{1}{\delta}) + O(1))$. 

The additional terms in the bound suggest that the convergence guarantees for system identification are looser than in the bounded case. It also suggest that the learning problem for bounded data might be easier (require less data to achieve a good performance).

Unfortunately, the bound above has the drawback that it holds only for Lipschitz loss functions. However, for system identification, the preferred loss function is the quadratic loss $\|\hat{y}-y\|_2^2$, which is not Lipschitz. This motivates the last case:

\textbf{Unbounded data and quadratic loss function} We can show that we high probability (at least $1-\frac{2\delta}{3}$) the prediction errors $\|\hat{\mathbf{y}}_{\theta}(t \mid 0)-\mathbf{y}(t)\|_2$ and
$\|\hat{\mathbf{y}}_{\theta}(t)-\mathbf{y}(t)\|_2$ are smaller than the constant $C_N=6G_e\|\e_g\|_{\psi_2}\sqrt{\ln\left(\frac{3N^2}{2\delta}\right)}=O((\ln(N))^{1/2})$.  
When the prediction errors are smaller than $C_N$, 
we can replace the quadratic loss function in the empirical error with the truncated quadratic loss function $\ell_{C_N}(y',y)$, which is defined by $\ell_{C_N}(y',y)=\|y'-y\|_2^2$ if $\|y'-y\|_2 \le C_N$ and $\ell_{C_N}(y',y)=C_N$ otherwise. This truncated loss is Lipschitz with the Lipschitz constant $L_{qd}=12C_N$. We can then reuse the bound for $L_{qd}$-Lipschitz loss functions, stated for $\frac{\delta}{3}$ instead  of $\delta$ in order to get a bound for the true error of the loss function 
$\ell_{C_N}$. By using the union bound with the event that 
$\|\hat{\mathbf{y}}_{\theta}(t \mid 0)-\mathbf{y}(t)\|_2 \le C_N$ and
$\|\hat{\mathbf{y}}_{\theta}(t)-\mathbf{y}(t)\|_2 \le C_N$, we can obtain a 
high probability (at least $1-2\delta$) bound on the true error for the loss function $\ell_{C_N}$
in terms of the empirical error for the original quadratic loss function. 
Finally, we can show  that the true error for the quadratic loss functions is smaller than the true error for the loss function $\ell_{C_N}$ plus an additional terms which is $O(1/N)$. 
The bound of Theorem \ref{thm:main} then becomes as follows.
\begin{theorem}[Theorem \ref{thm:main}: unbounded data, quadratic loss case]
    \label{thm:main:unbounded:quadratic}
    Let Assumptions \ref{as:generator} and \ref{as:parameterisation} hold, and assume
    that $\ell$ is the quadratic loss function.
    For any a prior density $\pi$ over $\Theta$, confidence level  $\delta\in(0,0.5)$, integer $\epsilon \in \{-1,1\}$,  and $N \ge 3$, and $0 < \lambda  < \frac{N}{32L_{qd}}$,
    the inequality
    \begin{equation*}
    \begin{aligned}
         &\bP\Big(\forall \rho \in \mathcal{M}_\pi: \epsilon E_{\theta\sim\rho} \mathcal{L}(\theta)
         \leq   \epsilon E_{\theta\sim\rho}\hat{\mathcal{L}}_N(\theta) +   \\
          & \frac{1}{\lambda}\Big(\KL(\rho\|\pi)  +\ln\left(\frac{2}{3\delta}\right) + \frac{1}{2}\sum_{i=1}^2
            \ln E_{\theta\sim\pi} \exp\left(\mathcal{C}_i(2\lambda,\theta,N,\delta,\epsilon)\right) \Big)\Big) \ge 1-2\delta \,
    \end{aligned}
    \end{equation*}
    with
    $\KL(\rho\|\pi)=E_{\theta\sim\rho}\ln\frac{\rho(\theta)}{\pi(\theta)}$ the Kullback–Leibler divergence, and 
    $\mathcal{C}_i(\lambda,\theta,N,\delta,\epsilon)$ is defined as
    \begin{equation*}
  \begin{aligned}
     & \mathcal{C}_i(\lambda,\theta,N,\delta,1)  \triangleq  
     \frac{\lambda L^2_{qd}}{2N} +
       \mathcal{C}_{ub,i}\left(\lambda,\theta,L_{qd},
       N\right) \\ 
        & \mathcal{C}_i(\lambda,\theta,N,\delta,-1)  \triangleq  
        \mathcal{C}_{ub,i}\left(\lambda,\theta,L_{qd},N\right) \\ 
       &  L_{qd}=12G_e\|\e_g\|_{\psi_2}\sqrt{\ln\left(\frac{3N^2}{2\delta}\right)} 
    \end{aligned}
  \end{equation*}
  where $\mathcal{C}_{ub,i}(\lambda,\theta,L_{qd},N)$ is defined in 
  Theorem \ref{thm:main:unbounded}. 
    \end{theorem}
     In particular, for $\epsilon=1$ the theorem states:
     \begin{equation*}\begin{aligned}
         & \bP\Big(\forall  \rho  \in \mathcal{M}_\pi: E_{\theta\sim\rho} \mathcal{L}(\theta)
         \leq   E_{\theta\sim\rho}\hat{\mathcal{L}}_N(\theta) +   \\
          & \frac{1}{\lambda}\Big(\KL(\rho\|\pi)  +\ln\left(\frac{2}{3\delta}\right) + \frac{1}{2}\sum_{i=1}^2
            \ln E_{\theta\sim\pi} \exp\left(\mathcal{C}_{ub,i}(2\lambda,\theta,L_{qd},N)\right) 
            + \frac{\lambda L^2_{qd}}{2N}\Big)\Big) \ge 1-2\delta \,
    \end{aligned}
    \end{equation*}
    and for $\epsilon=-1$ it states:
    \begin{equation*}\begin{aligned}
            & \bP\Big(\forall \rho  \in \mathcal{M}_\pi: 
              E_{\theta\sim\rho}\hat{\mathcal{L}}_N(\theta)  \leq E_{\theta\sim\rho} \mathcal{L}(\theta) +  \\
          & \frac{1}{\lambda}\Big(\KL(\rho\|\pi)  +\ln\left(\frac{2}{3\delta}\right) + \frac{1}{2}\sum_{i=1}^2
            \ln E_{\theta\sim\pi} \exp\left(\mathcal{C}_{ub,i}(2\lambda,\theta,L_{qd},N)\right) \Big)\Big) \ge 1-2\delta \,
    \end{aligned}
    \end{equation*}
   and by using  the union bound we get the following bound on the absolute value of the
   generalization gap:
   \begin{equation*}\begin{aligned}
            & \bP\Big(\forall \rho  \in \mathcal{M}_\pi: 
         \left|E_{\theta\sim\rho} \mathcal{L}(\theta) 
              -  E_{\theta\sim\rho}\hat{\mathcal{L}}_N(\theta) \right| \le  \\
          & \frac{1}{\lambda}\Big(\KL(\rho\|\pi)  +\ln\left(\frac{2}{3\delta}\right) + \frac{1}{2}\sum_{i=1}^2
          \ln E_{\theta\sim\pi} \exp\left(\mathcal{C}_{ub,i}(2\lambda,\theta,L_{qd},N)\right)
          + \frac{\lambda L^2_{qd}}{2N} 
    \end{aligned}
    \end{equation*}
   That is, in this case the term $\mathcal{C}_i(\lambda,\theta,N,\delta,\epsilon)$ of Theorem \ref{thm:main} depends on $\delta$ via $L_gd$ and on $\epsilon$. 
   By right-hand side of the PAC-Bayesian inequality then becomes
\[ 
\begin{split}
& \frac{1}{\lambda}\Big[KL(\rho \| \pi)+\ln \left(\frac{2}{3\delta}\right) + \frac{1}{2} \ln E_{\theta \sim \rho} \exp(\frac{1}{2}C_{b,1}(4\lambda,\theta,c_N, N)+C_{ub}(2\lambda,L_{qd},N)+\frac{\lambda L_{qd}^2}{2N})+\\
& \frac{1}{2} \ln E_{\theta \sim \rho} \exp(\frac{1}{2} C_{b,2}(4\lambda, \theta,c_N, N)+C_{ub}(2\lambda,L_{qd},N)+\frac{\lambda L_{qd}^2}{2N})\Big]
\end{split}
\]
where $c_N=\frac{2\ln(N)}{\|\alpha_g\|_{\ell_1}K}$ as in the case of unbounded data and Lipschitz loss. If we choose $\lambda=\frac{\sqrt{N}}{16 \ln(N)L_{qd}})=\frac{\sqrt{N}}{384 \sqrt{2}G_e\|\e_g\|_{\psi}\sqrt{(1+\frac{2}{3}\delta)}\ln(N))^{3/2}}$  we get the bound
 Corollary \ref{thm:bounded:alt:col}:
 \begin{corollary}[Corollary \ref{thm:bounded:alt:col}: unbounded, quadratic loss  case]
For every prior $\pi$, $\delta (0,0.5)$, $\epsilon \in \{-1,1\}$,
$L=192G_e\|\e_g\|_{\psi}\sqrt{\ln\left(1  +\frac{2}{3\delta}\right)}$, 
\begin{equation*}
\begin{split}
   & \bP\Big(\forall  \rho \in\mathcal{M}_\pi: \epsilon E_{\theta \sim \rho } \mathcal{L}(\theta) \leq \epsilon E_{\theta\sim \rho } \hat{\mathcal{L}}_N(\theta)+  
  \frac{L(\ln(N))^{3/2}}{\sqrt{N}}\Big(\KL(\rho ||\pi)+\ln\left(\frac{2}{3\delta}\right)+\frac{\max\{0,\epsilon\}}{128}+ \\
  & C_{u,1}+C_{u,2}+\frac{1}{2} \sum_{i=1}^2 \mathcal{C}_{m,b,i}(2,\bar{\alpha},G_e+2G_{e,1},1,\frac{1}{\|\alpha_g\|_{\ell_1}K})
   \Big)\Big) \ge 1-2\delta
\end{split}
\end{equation*}
\end{corollary}
That is, for $\epsilon=1$, it states that:
\begin{equation*}
\begin{split}
   & \bP\Big(\forall  \rho \in\mathcal{M}_\pi:  E_{\theta \sim \rho } \mathcal{L}(\theta) \leq  E_{\theta\sim \rho } \hat{\mathcal{L}}_N(\theta)+  
    \frac{L(\ln(N))^{3/2}}{\sqrt{N}}\Big(\KL(\rho ||\pi)+\ln\left(\frac{2}{3\delta}\right) + \\
  & \frac{1}{128} +C_{u,1}+C_{u,2}+ 
   \frac{1}{2} \sum_{i=1}^2 \mathcal{C}_{m,b,i}\left(\frac{1}{4},\bar{\alpha},G_e+2G_{e,1},1,\frac{1}{\|\alpha_g\|_{\ell_1}K}\right)
   \Big)\Big) \ge 1-2\delta\,
\end{split}
\end{equation*}
and for $\epsilon=-1$:
\begin{equation*}
\begin{split}
   &\bP\Big(\forall  \rho \in\mathcal{M}_\pi:  E_{\theta\sim \rho } \hat{\mathcal{L}}_N(\theta) \leq   E_{\theta \sim \rho } \mathcal{L}(\theta) +
    \frac{L(\ln(N))^{3/2}}{\sqrt{N}}\Big(\KL(\rho ||\pi)+\ln\left(\frac{2}{3\delta}\right) +  \\
    & C_{u,1}+C_{u,2}+\frac{1}{2} \sum_{i=1}^2 \mathcal{C}_{m,b,i}(\frac{1}{4},\bar{\alpha},G_e+2G_{e,1},1,\frac{1}{\|\alpha_g\|_{\ell_1}K})
  \Big)\Big) \ge 1-2\delta\,
\end{split}
\end{equation*}
and with the union bound
\begin{equation*}
\begin{split}
   &\bP\Big(\forall  \rho \in\mathcal{M}_\pi:  \left|E_{\theta\sim \rho } \hat{\mathcal{L}}_N(\theta) -  E_{\theta \sim \rho } \mathcal{L}(\theta)\right| \leq \\
  &  \frac{L(\ln(N))^{3/2}}{\sqrt{N}}\Big(\KL(\rho ||\pi)+\ln\left(\frac{2}{3\delta}\right) + \frac{1}{128} \\
  & +C_{u,1}+C_{u,2}+\frac{1}{2} \sum_{i=1}^2 \mathcal{C}_{m,b,i}(\frac{1}{4},\bar{\alpha},G_e+2G_{e,1},1,\frac{1}{\|\alpha_g\|_{\ell_1}K}) \Big)\Big) \ge 1-4\delta.
\end{split}
\end{equation*}
That is the upper bound on the generalization gap is of the form 
\[ 
  O\left(\sqrt{\ln\left(1  +\frac{2}{3\delta}\right)}\right) O\left(\frac{\ln(N)^{3/2}}{\sqrt{N}}\right)\left(KL(\rho \| \pi)+\ln\left(\frac{2}{3\delta}\right) + O(1)\right)
\]
The presence of addition terms indicate that the bound for quadratic loss and unbounded data is larger than for the case of Lipschitz loss or for bounded data. 

To sum up, the bound for each case was derived from the previous case by adding an extra term. In all cases, the bound depends only on the prior and the distance of the posterior from the prior, and the bound of Corollary \ref{thm:bounded:alt:col} is an upper bound of the bound of Theorem \ref{thm:main} for a certain choice of $\lambda$ as an increasing function of $N$.

\textbf{Additional intuition behind Gibbs posteriors and Corollary \ref{lem:bound:mean}}
We can find the posterior which minimizes the bound of Theorem \ref{thm:main} or Corollary \ref{thm:bounded:alt:col}, which is the so called Gibbs posterior, and we can think of learning as taking the average (MEP), taking the most likely model (MAP) or taking a random model from this Gibb posterior. For all these cases we can derive a bound on the generalization gap by exploiting a closed-form formula for the KL divergence if the prior is uniform or Gaussian, and the posterior is the Gibbs posterior. Moreover, using the fact that the bounds of Theorem \ref{thm:main}/Corollary \ref{thm:bounded:alt:col} hold for both the generalization gap and its multiple by $-1$, using the union bound, we can derive a bound on the difference between the true error of the learned model and the smallest achievable true error. For quadratic loss, under the assumption that the data generator is included among the predictors, the latter can be related
to the parameter estimation error. These results are collected in Corollary \ref{lem:bound:mean}.
\end{color}

\section{Proofs of Section \ref{sect:Problem_formulation}}
\label{app:sect:Problem_formulation}
\begin{remark}[LTI in innovation form]
  \label{rem:innovation}
   Let us assume that $\e_g$ has zero mean. It follows
   that $\e_g(t)$ is the innovation process of $[\rvy^T(t)\;\rvu^T(t)]^T$
   and \eqref{eq:generator} is in the so called 
   (forward) innovation form, see
   \cite[Section 8.5.2]{Katayama:05}.
   Intuitively, it means that $\e_g$
   is the difference between $[\rvy^T(t)\;\rvu^T(t)]^T$ and its best linear prediction based on its own past values.
  \end{remark} 
\begin{remark}[Equivalence of \eqref{eq:generator} and \eqref{eq:T0sys1}]
  \label{rem:equiv_gen:proof}
Let us assume that $\e_g$ and $\e_s$ are both zero mean.
More precisely, if there is no feedback from $\rvy$ to $\rvu$ according to Definition 17.1.1. of \cite{LindquistBook},
then \eqref{eq:T0sys1} holds and by \cite{eringis2021optimal} the matrices $A_0,B_0,C_0,K_0,D_0$ can be computed from the matrices $A_g,K_g,C_g$. 

Conversely, if \eqref{eq:T0sys1} holds and $\rvu(t)$ is generated by an ARMA process, and there is no feedback from $\rvy(t)$ to $\rvu(t)$, then 
Assumption \ref{as:generator} holds and $A_g,K_g,C_g$  can be computed from those of \eqref{eq:T0sys1} and the ARMA representation of $\rvu(t)$, again by reversing the equations from
\cite{eringis2021optimal}. As a passing remark, if there is a feedback from $\rvy(t)$ to $\rvu(t)$ and  $\rvu(t)$ is generated by an LTI system driven by $\y(t)$, i.e.,  
\begin{equation}\label{eq:u_gen}
    \rvs_\rvu(t+1)=A_u\rvs_\rvu(t)+K_u\y(t), \quad 
    \rvu(t)=C_u\rvs_\rvu(t) 
\end{equation}
then again the generator described in Assumption \ref{as:generator} can be obtained by adding \eqref{eq:u_gen} to \eqref{eq:T0sys1}. 
\end{remark} 
\begin{remark}[Methods for stable parametrisations]
\label{rem:parameter:methods}
 The construction of \cite[Section 3.2.3]{RalfPeeters} relies on canonical forms based on structure indices, to which the diffeomorphism from \cite{HanzonStable1}
 between general and stable parametersations are applied. Intuitively, the latter transformation maps $(A,B,C,D)$ to $(\frac{A}{\mu+\epsilon},B,C,D)$, where $\mu$ is the maximum of moduli of the eigenvalues of $A$ and $\epsilon>0$. 
 The construction of \cite[Section  4.6]{Ribarits1} uses Ober's balanced canonical forms, and it relies on fixing the observability and controllability Grammians to being equal and diagonal. 
 That is, stability is assured by using diagonal quadratic Lyapunov functions. 
 In this case, by fixing upper bounds on the diagonals of the Grammians and on the norms of the system matrices occurring in the parameterization, we ensure existence of constants $M,C,\gamma$ mentioned in Assumption \ref{as:parameterisation}. 
\end{remark}
\begin{proof}[Proof of Lemma \ref{l:ihp}]
We have
$
\y_\theta(t\,|\,t_0)=\sum_{k=t_0}^{t-1} H_\theta(t-1-k)\,\w(k)+D_\theta \w(t)=\sum_{k=0}^{t-1-t_0} H_\theta(k)\,\w(t-1-k)+D_\theta \w(t)
$
with the impulse response $H_\theta(k)\triangleq C_\theta A_\theta^{k}B_\theta$ for $k\ge 0$. Because $A_\theta$ is Schur, there exist $M>0$ and $\rho\in(0,1)$ such that
$\|A_\theta^{k}\|_2\le M\rho^k$, hence $\sum_{k=s}^\infty \|H_\theta(k)\|_2\leq\|C_\theta\|_2\|B_\theta\|_2M\frac{\rho^{s}}{1-\rho} \rightarrow 0$ as $s \rightarrow \infty$. 
For $t_1<t_2\leq t$ we get, using stationarity of $\w$ ($\sigma_\w=(\bE\|\w(s)\|_2^2)^{1/2}$), that
\begin{subequations}\label{ytt2-ytt1}
\begin{align}
(\bE[\left\|\y_\theta(t\,|\,t_2)-\y_\theta(t\,|\,t_1)\right\|_2^2])^{1/2}
&\leq\sum_{m=t-t_2}^{t-1-t_1}(\bE\left[
\|H_\theta(m)\w(t-1-m)\|_2^2
\right])^{1/2}\\
&\leq\sigma_\w\sum_{m=t-t_2}^{t-1-t_1}\|H_\theta(m)\|_2\\
&\leq\sigma_\w\sum_{m=t-t_2}^{\infty}\|H_\theta(m)\|_2\to 0\quad\text{as $t_2\to-\infty$}
\end{align}
\end{subequations}
Therefore $\{\y_\theta(t\,|\,t_0)\}_{t_0\le t}$ is Cauchy (in $L^2$) and hence converges in mean square to
$
\y_\theta(t)= \sum_{k=-\infty}^{t-1} H_\theta(t-1-k)\,\w(k) \;+\; D_\theta \w(t)=\sum_{k=0}^{\infty} H_\theta(k)\,\w(t-1-k)+D_\theta \w(t).
$
This representation shows that $\y_\theta$ is the output of a causal, time-invariant, stable LTI filter
(with impulse response $\{H_\theta(k)\}$ and feed-through $D_\theta$) driven by the stationary input $\w$. Hence $\y_\theta$ is stationary, since stable LTI filtering preserves stationarity. Moreover, it follows that
\begin{align}\label{ytytt0}
\y_\theta(t) - \y_\theta(t\mid t_0) 
=  \sum_{k = t - t_0}^{\infty} H_\theta(k)\, \w(t-1-k).
\end{align}
so $\y_\theta(t)-\y_\theta(t\,|\,0)\to 0$ in mean square as $t\to\infty$ by arguments as in \eqref{ytt2-ytt1}.

Finally, consider
\begin{align}\label{LE}
\left|\bE\big[\ell\big(\y(t),\y_\theta(t)\big)\big]-\bE\big[\ell\big(\y(t),\y_\theta(t\,|\,t_0)\big)\big]\right|
\le \bE\big[\left|\ell\big(\y(t),\y_\theta(t)\big)-\ell\big(\y(t),\y_\theta(t\,|\,t_0)\big)\right|\big]
\end{align}
If $\ell$ is $L_\ell$-Lipschitz, then
the right hand side of \eqref{LE} is bounded by $L_\ell \bE[\|\y_\theta(t) - \y_\theta(t\mid t_0) \|_2]$, 
and hence, as $\y_\theta(t)-\y_\theta(t\,|\,0)\to 0$ in mean square as $t\to\infty$ or 
as $t_0 \to -\infty$, the right hand side of \eqref{LE} also tends to $0$, i.e., 
\[
   \bE\big[\left|\ell\big(\y(t),\y_\theta(t)\big)-\ell\big(\y(t),\y_\theta(t\,|\,t_0)\big)\right|\big] \to 0.
\]
as $t \to \infty$ or $t_0 \to -\infty$.

If $\ell$ is quadratic, then 
\begin{align*}
  & |\ell\big(\y(t),\y_\theta(t)\big)-\ell\big(\y(t),\y_\theta(t\,|\,t_0)\big)|=
    \left|\|\y(t)-\y_\theta(t)\|_2^2 - \|\y(t)-\y_\theta(t\,|\,t_0)\|_2^2\right|= \\
  & |2\y^T(t)(-\y_\theta(t)+\y_\theta(t\,|\,t_0))+\left(\|\y_\theta(t)\|_2^2-\|\y_\theta(t\,|\,t_0)\|_2^2\right)| \le \\
  & 2\|\y^T(t)\|_2 \|\y_\theta(t)-\y_\theta(t\,|\,t_0)\|_2 + \left|\|\y_\theta(t)\|_2-\|\y_\theta(t\,|\,t_0)\|_2\right| \left(\|\y_\theta(t)\|_2+\|\y_\theta(t\,|\,t_0)\|_2\right)\\
\end{align*}
Hence, using the Cauchy-Schwartz  inequality
\begin{align*} 
& \bE\left[|\ell\big(\y(t),\y_\theta(t)\big)-\ell\big(\y(t),\y_\theta(t\,|\,t_0)\big)|\right]  \le  \\
& 2\bE\left[\|\y^T(t)\|_2 \|\y_\theta(t)-\y_\theta(t\,|\,t_0)\|_2\right] 
+ \bE\left[\left|\|\y_\theta(t)\|_2-\|\y_\theta(t\,|\,t_0)\|_2\right|\big(\|\y_\theta(t)\|_2+\|\y_\theta(t\,|\,t_0)\|_2\big)\right] \le \\
& 
2\sqrt{\bE\left[\|\y_\theta(t) - \y_\theta(t\mid t_0) \|_2^2\right]}
\sqrt{\bE\left[\|\y(t)\|_2^2\right]}+ \\
&
\sqrt{\bE\left[\|\y_\theta(t)-\y_\theta(t\,|\,t_0)\|_2^2)\right]}
\sqrt{(\bE[\|\y_\theta(t)\|_2^2] +2\bE[\|\y_\theta(t)\|_2\|\y_\theta(t\,|\,t_0)\|_2]+\bE[\|\y_\theta(t\,|\,t_0)\|_2^2]\big)}
\end{align*}
Note that $\bE\left[\|\y_\theta(t) - \y_\theta(t\mid t_0) \|_2^2\right] \rightarrow 0$ as $t_0 \rightarrow -\infty$
or  $t \rightarrow \infty$, since $\y_{\theta}(t\mid t_0)- \y_{\theta}(t \mid t_0) \rightarrow 0$ in the mean-square
sense as $t_0 \to -\infty$ or $t \to \infty$. Moreover, 
\[ 2\bE[\|\y_\theta(t)\|_2\|\y_\theta(t\,|\,t_0)\|_2] \le  \bE[\|\y_\theta(t)\|_2^2] + \bE[\|\y_\theta(t\,|\,t_0)\|_2^2], \]
and  $\bE[\|\y_\theta(t)\|_2^2]$ is constant due to stationarity of $\y_{\theta}(t)$. 
Now for sufficiently large $t > T_{\epsilon}$ or $-t_0 > T_{\epsilon}$, $\bE\left[\|\y_\theta(t) - \y_\theta(t\mid t_0) \|_2^2\right] < \epsilon^2$, and it follows that
\begin{align*} 
\bE[\|\y_{\theta}(t \mid t_0)\|_2^2] &\le \bE[(\|\y_\theta(t)\|_2+\|\y_{\theta}(t\mid t_0)-\y_{\theta}(t)\|_2)^2] \\
&\le 2(\bE[\|\y_\theta(t)\|_2^2] + \bE[\|\y_\theta(t) - \y_\theta(t\mid t_0) \|_2^2]) < 2\bE[\|\y_\theta(t)\|_2^2] + 2\epsilon^2 
\end{align*}
Hence, it follows that for $t > T_{\epsilon}$ or $-t_0 > T_{\epsilon}$,
\[ \sqrt{(\bE[\|\y_\theta(t)\|_2^2] +2\bE[\|\y_\theta(t)\|_2]\|\y_\theta(t\,|\,t_0)\|_2]+\bE[\|\y_\theta(t\,|\,t_0)\|_2^2]\big)}
\le \sqrt{6\bE[\|\y_\theta(t)\|_2^2] + 4\epsilon^2} 
\]
and 
$\sqrt{6\bE[\|\y_\theta(t)\|_2^2] + 4\epsilon^2}=C$ does not depend on $t$.
Hence, if $t > T_{\epsilon}$ or $-t_0 > T_{\epsilon}$, then
\begin{align*} 
\bE\big[\left|\ell\big(\y(t),\y_\theta(t)\big)-\ell\big(\y(t),\y_\theta(t\,|\,t_0)\big)\right|\big]
\leq \epsilon (2\bE[\|\y(t)\|_2^2]+C)
\end{align*}
from which  it follows that 
$\bE\big[\left|\ell\big(\y(t),\y_\theta(t)\big)-\ell\big(\y(t),\y_\theta(t\,|\,t_0)\big)\right|\big] \to 0$
as either $t_0 \to -\infty$ or $t \to \infty$.

That is, in all cases, 
\begin{align}\label{Lcon}\displaystyle 
\lim_{t\to\infty} E\big[\ell\big(\y(t),\y_\theta(t\,|\,0)\big)\big]
= E\big[\ell\big(\y(t),\y_\theta(t)\big)\big]
=\lim_{t_0\to -\infty} E\big[\ell\big(\y(t),\y_\theta(t\,|\,t_0)\big)\big]
\end{align}
proving the lemma since $\mathcal L(\theta)$ is the right hand side of \eqref{Lcon} by definition.
\end{proof}

\section{Proofs of Section \ref{sec:asm}}
\label{app:constants}

\begin{remark}[$\alpha$ as Markov-parameters]
    \label{norms:rem}
   If $\alpha$ is the sequence of Markov parameters of a stable deterministic LTI system  determined
  by matrices $(A,B,C,D)$, i.e.,
  $\alpha(0)=D$, $\alpha(k)=CA^{k-1}B$, $k \ge 1$, then the $\ell_1$ norm $\|\alpha\|_{\ell_1}$ and the WDMC $\theta_{\infty}(\alpha)$ are both finite. In this case,
  $\|\alpha\|_{\ell_1}$ is the classical $\ell_1$ norm of the underlying LTI system.  
  Moreover, if $M >0$, $\gamma \in [0,1)$ are such that
  $\|A^k\|_2 \leq M \gamma^k$, then 
  $\|\alpha\|_{\ell_1} \le \|D\|_2+M\|B\|_2\|C\|_2/(1-\gamma)$ and 
  $\theta_{\infty}(\alpha) \le M\|B\|_2\|C\|_2/(1-\gamma)^2$.
  In turn, such $M$
  always exist if $A$ is a Schur matrix, and $\gamma$ can be chosen arbitrarily close to \(\rho(A)\) from above, while remaining below \(1\). 
  In particular, the upper bounds decrease as \(\gamma\) decreases, provided \(M\), \(\|B\|_2\), and \(\|C\|_2\) remain fixed.
  \end{remark}
  Finally, as promised before Lemma \ref{alpha:lemma2}, we can state that all the
  constants from Definition \ref{def:constants} are well-defined. 
\begin{lemma}
\label{alpha:lemma1.5}
  If Assumption \ref{as:generator} and \ref{as:parameterisation} hold, then
  the quantities $\|\alpha_{\theta}\|_{\ell_1}$ and $\theta_{\infty}(\alpha_{\theta})$, $\|\bar{\alpha}\|_{\ell_1}$,
  $\theta_{\infty}(\bar{\alpha})$, $G_e(\theta)$, $G_{e,1}(\theta)$, $G_e$, $G_{e,1}$ are well defined 
  real numbers. 
\end{lemma}
\begin{proof}
  From Remark \ref{norms:rem} and the assumption that $A_g$ is Schur it follows that 
  $\|\alpha_g\|_{\ell_1}$ and $\theta_{\infty}(\alpha_g)$ are well-defined. 
   By Assumption \ref{as:parameterisation} there exists constants $M,\gamma,C$ such that 
   for all $\theta \in \Theta$, $\|A_{\theta}^k\|_2 \le M \gamma^k$, $\|B_{\theta}\|_2 \le C$, and $\|C_{\theta}\|_2 \le C$. 
   Therefore, $\bar{\alpha}(k)=\sup_\theta \|\alpha_{\theta}(k)\|_2 \le M \gamma^{k-1} C^2$ and hence,
   by using the argument of Remark \ref{norms:rem},  
   $\|\alpha_{\theta}\|_{\ell_1} \le \sum_{k=0}^{\infty} \|\alpha\|_{\ell_1} \le  \sum_{k=0}^{\infty} \bar{\alpha}(k)=\|\bar{\alpha}\|_{\ell_1} \le C+\frac{MC^2}{1-\gamma} < \infty$ and
   $\theta_{\infty}(\alpha_{\theta})=\sum_{k=0}^{\infty} k \|\alpha_{\theta}(k)\|_2 \le \sum_{k=0}^{\infty} k \bar{\alpha}(k)=\theta_{\infty}(\bar{\alpha}) \le \frac{MC^2}{(1-\gamma)^2} < \infty$.
   Therefore, the quantities $\|\alpha_{\theta}\|_{\ell_1}$ and $\theta_{\infty}(\alpha_{\theta})$, $\|\bar{\alpha}\|_{\ell_1}$,
   $\theta_{\infty}(\bar{\alpha})$ defining $G_e(\theta)$, $G_{e,1}(\theta)$, $G_e$, $G_{e,1}$
   are well-defined, and hence so are $G_e(\theta)$, $G_{e,1}(\theta)$, $G_e$, $G_{e,1}$
\end{proof}
\begin{lemma}
\label{alpha:lemma2}
Let $\alpha_{g,y}(k)$ and $\alpha_{g,w}(k)$ be the matrix formed by the first $n_\rvy$ and the last $n_\rvw$ 
rows of $\alpha_g(k)$ respectively, where $n_{\rvw}$ is the number of entries of $\w(t)$. Let $c_{\theta,s}(k)=\left(\sum_{l=0}^{\min\{k,s\}} \alpha_{\theta}(l)\alpha_{g,w}(k-l)\right)$
and $c_{\theta}(k)=\sum_{l=0}^{k} \alpha_{\theta}(l)\alpha_{g,w}(k-l)$. Then
\begin{equation}
\label{alpha:lemma:eq1}
\begin{aligned}
& \begin{bmatrix} \y(t)
    \\ \hat{\rvy}_{\theta}(t)
\end{bmatrix}=\left(\begin{bmatrix} \alpha_{g,y} \\ c_{\theta} \end{bmatrix} \star \e_g\right)(t), \quad  
\begin{bmatrix} \y(t) \\ \hat{\rvy}_{\theta}(t \mid t_0) \end{bmatrix}=\left(\begin{bmatrix} \alpha_{g,y} \\ c_{\theta,t-t_0} \end{bmatrix} \star \e_g\right)(t),   \\
& \y(t)-\hat{\rvy}_{\theta}(t) = ((\alpha_{g,y} - c_{\theta}) \star \e_g)(t), \quad
\y(t)-\hat{\rvy}_{\theta}(t \mid t_0) = ((\alpha_{g,y} - c_{\theta,t-t_0}) \star \e_g)(t)
\end{aligned}
\end{equation}
Moreover,  the constants from Definition \ref{def:constants} satisfy:
\begin{align}
& \max\{\|c_{\theta}\|_{\ell_1}, \|c_{\theta,s}\|_{\ell_1}\} \le \|\alpha_{\theta}\|_{\ell_1}\|\alpha_g\|_{\ell_1}, \label{alpha:lemma:eq2a} \\ 
 & \max\{\left\|\begin{bmatrix} \alpha_{g,y} \\ c_{\theta} \end{bmatrix}\right \|_{\ell_1}, \left \|\begin{bmatrix} \alpha_{g,y} \\ c_{\theta,t-t_0} \end{bmatrix}\right\|_{\ell_1}, \|\alpha_{g,y}-c_{\theta}\|_{\ell_1}, \|\alpha_{g,y}-c_{\theta,t-t_0}\|_{\ell_1}\} \le G_e(\theta) \le G_e \label{alpha:lemma:eq2b} \\
& \max\{
  \theta_{\infty}\left(\begin{bmatrix} \alpha_{g,y} \\ c_{\theta} \end{bmatrix}\right),
  \theta_{\infty}\left(\begin{bmatrix} \alpha_{g,y} \\ c_{\theta,t-t_0} \end{bmatrix}\right),
\theta_{\infty}(\alpha_{g,y}-c_{\theta}), \theta_{\infty}(\alpha_{g,y}-c_{\theta,t-t_0})\} \le G_{e,1}(\theta) \label{alpha:lemma:eq2c} \\
& \sum_{t=t_0}^{N+t_0}\|c_{\theta,t-t_0}-c_{\theta}\|_{\ell_1} \le \theta_{\infty}(\alpha_{\theta})\|\alpha_g\|_{\ell_1} \label{alpha:lemma:eq2d}
  \end{align}
In particular, $[\y^T(t)~~\hat{\rvy}^T_{\theta}(t)]^T$ and $\y(t)-\hat{\rvy}_{\theta}(t)$ from \eqref{alpha:lemma:eq1} are stationary. 
\end{lemma}
\begin{proof}[Proof of Lemma  \ref{alpha:lemma2}]
  Let $C_w$ denote the matrix formed by the last $n_w$ rows of $C_g$. From Assumption \ref{as:generator} and the proof of Lemma~\ref{l:ihp} we get that
  \begin{align*}
  \w(t)&=\sum_{k=1}^{\infty} C_w A_g^{k-1} K_g \e_g(t-k)+ \alpha_{g,w}(0)\e_g(t)=\sum_{k=0}^{\infty} \alpha_{g,w}(k)\e_g(t-k)
  =(\alpha_{g,w}\star\e_g)(t)\\
  \y(t)&=\sum_{k=0}^{\infty} \alpha_{g,y}(k)\e_g(t-k)\\
  \hat{\y}_{\theta}(t \mid t_0)&=\sum_{k=1}^{t-t_0} \hat{C}_{\theta}\hat{A}_{\theta}^{k-1}\hat{B}_{\theta}\w(t-k)+\hat{D}_{\theta}\w(t)=\sum_{k=0}^{t-t_0} \alpha_{\theta}(k)\w(t-k)\\ 
  \hat{\y}_{\theta}(t)&=\lim_{t_0 \rightarrow -\infty} \hat{\y}_{\theta}(t \mid t_0)=\sum_{k=0}^{\infty} \alpha_{\theta}(k)\w(t-k) 
  \end{align*}
  It then follows using the new index $n=k+j\geq 0$ and the standard argument involving Cauchy products of series, 
  \begin{align*}
     \hat{\y}_{\theta}(t \mid t_0)
     &=\sum_{k=0}^{t-t_0} \alpha_{\theta}(k)\sum_{j=0}^{\infty} \alpha_{g,w}(j)\e_g(t-k-j)\\
      &= \sum_{n=0}^{\infty } \left(\sum_{k=0}^{\min(n,t-t_0)} \alpha_{\theta}(k)\alpha_{g,w}(n-k)\right)\e_g(t-n)\\
      &=\sum_{n=0}^{\infty} c_{\theta, t-t_0}(n)\e_g(t-n)=(c_{\theta, t-t_0} \star \e_g)(t)
  \end{align*}
  Notice that since $A_g$ and $\hat{A}_{\theta}$ are Schur, the series 
  $\sum_{k=0}^{\infty} \alpha_{\theta}(k)$ and $\sum_{k=0}^{\infty} \alpha_{g,w}(k)$ are absolutely convergent. 
  With a similar argument, using that absolutely convergent series can be rearranged, it follows that
  \begin{align*}
     \hat{\y}_{\theta}(t)
     &=\sum_{k=0}^{\infty} \alpha_{\theta}(k)\sum_{j=0}^{\infty} \alpha_{g,w}(j)\e_g(t-k-j)\\
      &= \sum_{n=0}^{\infty } \left(\sum_{k=0}^{n} \alpha_{\theta}(k)\alpha_{g,w}(n-k)\right)\e_g(t-n)\\
      &=
      \sum_{n=0}^{\infty} c_{\theta}(n)\e_g(t-n)=(c_{\theta} \star \e_g)(t)
  \end{align*}
  From this the first  claim of the lemma, i.e., \eqref{alpha:lemma:eq1}, follows.
  The first inequality of the second claim, i.e., 
  \eqref{alpha:lemma:eq2a}, follows
  by first noticing that 
  \begin{align*}
    & \|c_{\theta}(k) \|_2 \le \sum_{i=0}^{k} \|\alpha_{\theta}(i)\|_2 \|\alpha_{g,w}(k-i)\|_2, \\
     & \|c_{\theta,t-t_0}(k) \|_2 \le \sum_{i=0}^{\min(k,t-t_0)} \|\alpha_{\theta}(i)\|_2 \|\alpha_{g,w}(k-i)\|_2 
     \le \sum_{i=0}^{k} \|\alpha_{\theta}(i)\|_2 \|\alpha_{g,w}(k-i)\|_2
  \end{align*}
  Since $\sum_{i=0}^{k} \|\alpha_{\theta}(i)\|_2 \|\alpha_{g,w}(k-i)\|_2$ is the  $k$'th coefficient of the Cauchy-product of 
  $\sum_{k=0}^{\infty} \|\alpha_{\theta}(k)\|_2$ and $\sum_{k=0}^{\infty} \|\alpha_{g,w}(k)\|_2$, we obtain  \eqref{alpha:lemma:eq2a} by  
    \begin{align}\label{alpha:lemma2:pf2}
      \|c\|_{\ell_1}=\sum_{k=0}^{\infty}\|c(k)\|_2
  \le \sum_{i=0}^{\infty} \|\alpha_{\theta}(i)\|_2 \sum_{j=0}^{\infty} \|\alpha_{g,w}(j)\|_2    
      =\|\alpha_{\theta}\|_{\ell_1} \|\alpha_{g,w}\|_{\ell_1}
       \leq\|\alpha_{\theta}\|_{\ell_1} \|\alpha_{g}\|_{\ell_1}
  \end{align}
 with $c$ denoting either $c_{\theta}$ or $c_{\theta,t-t_0}$. Concerning the second inequality \eqref{alpha:lemma:eq2b}, notice that
  \begin{subequations}\label{alpha:lemma2:pf0}    
  \begin{align}
    & \!\!\!\!\!\left\|\begin{bmatrix} \alpha_{g,y}(k) \\ c_{\theta}(k) \end{bmatrix}\right\|_2 \le \|\alpha_{g,y}(k)\|_2 + \|c_{\theta}(k)\|_2, \quad
    \|\alpha_{g,y}(k) - c_{\theta}(k)\|_2 \le \|\alpha_{g,y}(k)\|_2 + \|c_{\theta}(k)\|_2 \\
   & \left\|\begin{bmatrix} \alpha_{g,y}(k) \\ c_{\theta,t-t_0}(k) \end{bmatrix}\right\|_2 \le \|\alpha_{g,y}(k)\|_2 + \|c_{\theta,t-t_0}(k)\|_2,\\
    &\|\alpha_{g,y}(k) - c_{\theta,t-t_0}(k)\|_2 \le \|\alpha_{g,y}(k)\|_2 + \|c_{\theta,t-t_0}(k)\|_2 
  \end{align}
  \end{subequations}
  and hence,  
\begin{subequations}\label{alpha:lemma2:pf1}    
  \begin{align}
  &  \max\left\{\left\|\begin{bmatrix} \alpha_{g,y} \\ c_{\theta} \end{bmatrix}\right\|_{\ell_1}, \|\alpha_{g,y}-c_{\theta}\|_{\ell_1}\right\}
    \le \|\alpha_{g,y}\|_{\ell_1} + \|c_{\theta}\|_{\ell_1} \\
   & \max\left\{\left\|\begin{bmatrix} \alpha_{g,y} \\ c_{\theta,t-t_0} \end{bmatrix}\right\|_{\ell_1}, \|\alpha_{g,y}-c_{\theta,t-t_0}\|_{\ell_1}\right\}
    \le \|\alpha_{g,y}\|_{\ell_1} + \|c_{\theta,t-t_0}\|_{\ell_1}
  \end{align}
\end{subequations}
  From \eqref{alpha:lemma2:pf1} and \eqref{alpha:lemma2:pf2}, the inequality of \eqref{alpha:lemma:eq2b} follows
  by using the definition of $G_{e}(\theta)$ and the fact that $\max\{\|\alpha_{g,y}\|_{\ell_1}, \|\alpha_{g,w}\|_{\ell_1}\} \le \|\alpha_{g}\|_{\ell_1}$.

 Concerning \eqref{alpha:lemma:eq2c}, notice that by definining $d_1(i)=i \|\alpha_{\theta}(i)\|_2$ and $d_2(i)=i \|\alpha_{g,w}(i)\|_2$,
  \begin{equation*}
  \begin{aligned}
     & \theta_{\infty}(c_{\theta,t-t_0})=\sum_{k=0}^{\infty} k \|c_{\theta,t-t_0}(k)\|_2 \le \sum_{k=0}^{\infty} k \sum_{i=0}^{\min(k,t-t_0)} \|\alpha_{\theta}(i)\|_2 \|\alpha_{g,w}(k-i)\|_2 = \\
     & \sum_{k=0}^{\infty} \sum_{i=0}^{\min(k,t-t_0)} i \|\alpha_{\theta}(i)\|_2 \|\alpha_{g,w}(k-i)\|_2  +  \sum_{k=0}^{\infty} \sum_{i=0}^{\min(k,t-t_0)} (k-i) \|\alpha_{\theta}(i)\|_2 \|\alpha_{g,w}(k-i)\|_2 =  \\
     & \sum_{k=0}^{\infty} \sum_{i=0}^{\min(k,t-t_0)} d_1(i) \|\alpha_{g,w}(k-i)\|_2 
     + \sum_{k=0}^{\infty} \sum_{i=0}^{\min(k,t-t_0)} d_2(k-i) \|\alpha_{\theta}(i)\|_2 \le \\
     & \sum_{k=0}^{\infty} \sum_{i=0}^{k} d_1(i) \|\alpha_{g,w}(k-i)\|_2 
     + \sum_{k=0}^{\infty} \sum_{i=0}^{k} d_2(k-i) \|\alpha_{\theta}(i)\|_2 
  \end{aligned}
\end{equation*}
Similarly,
\begin{equation*}
\begin{aligned}
  & \theta_{\infty}(c_{\theta})=\sum_{k=0}^{\infty} k \|c_{\theta}(k)\|_2 \le 
   \sum_{k=0}^{\infty} k \sum_{i=0}^{k} \|\alpha_{\theta}(i)\|_2 \|\alpha_{g,w}(k-i)\|_2 = \\
     & \sum_{k=0}^{\infty} \sum_{i=0}^{k} i \|\alpha_{\theta}(i)\|_2 \|\alpha_{g,w}(k-i)\|_2  +  \sum_{k=0}^{\infty} \sum_{i=0}^{k} (k-i) \|\alpha_{\theta}(i)\|_2 \|\alpha_{g,w}(k-i)\|_2 = \\
     & \sum_{k=0}^{\infty} \sum_{i=0}^{k} d_1(i) \|\alpha_{g,w}(k-i)\|_2 
     + \sum_{k=0}^{\infty} \sum_{i=0}^{k} d_2(k-i) \|\alpha_{\theta}(i)\|_2 \le \\
     & \left(\sum_{k=0}^{\infty} d_1(k)\right)
        \left(\sum_{k=0}^{\infty} \|\alpha_{g,w}(k)\|_2 \right)+
      \left(\sum_{k=0}^{\infty} d_2(k)\right)
        \left(\sum_{k=0}^{\infty} \|\alpha_{\theta}(k)\|_2 \right) = \\
     &   \theta_{\infty}(\alpha_{\theta})\|\alpha_{g,w}\|_{\ell_1}+
        \theta_{\infty}(\alpha_{g,w})\|\alpha_{\theta}\|_{\ell_1} \le  \theta_{\infty}(\alpha_{\theta})\|\alpha_{g}\|_{\ell_1}+
        \theta_{\infty}(\alpha_{g})\|\alpha_{\theta}\|_{\ell_1}
  \end{aligned}
\end{equation*}
In particular, 
\begin{equation}
    \label{alpha:lemma2:pf4}
\max\{\theta_{\infty}(c_{\theta}), \theta_{\infty}(c_{\theta,t-t_0})\} \le \theta_{\infty}(\alpha_{\theta})\|\alpha_g\|_{\ell_1} +\theta_{\infty}(\alpha_{g})\|\alpha_{\theta}\|_{\ell_1}
\end{equation}
Moreover, using \eqref{alpha:lemma2:pf0} it follows that 
\begin{align*}
     & \max\left\{\theta_{\infty}\left(\begin{bmatrix} \alpha_{g,y} \\ c_{\theta} \end{bmatrix}\right),
  \theta_{\infty}\left(\begin{bmatrix} \alpha_{g,y} \\ c_{\theta,t-t_0} \end{bmatrix}\right),
\theta_{\infty}(\alpha_{g,y}-c_{\theta}), \theta_{\infty}(\alpha_{g,y}-c_{\theta,t-t_0
})\right\}  \le \\
& \theta_{\infty}(\alpha_{g,y})+\max\{\theta_{\infty}(c_{\theta}), \theta_{\infty}(c_{\theta,t-t_0})\}
\end{align*}
from which inequality \eqref{alpha:lemma:eq2c} follows by using \eqref{alpha:lemma2:pf4}. Finally, notice that 
\[
   c_{\theta}(k)-c_{\theta,t-t_0}(k)=\left\{
    \begin{array}{ll}
       0 & k \le t-t_0 \\
       \sum_{i=t-t_0+1}^{k} \alpha_{\theta}(i)\alpha_{g,w}(k-i) & k > t-t_0
   \end{array}\right.
\] 
Hence, by setting $d_3(k)=\|\alpha_{\theta}(k+t-t_0+1)\|_2$
  \begin{align*}
        \|c_{\theta}-c_{\theta,t-t_0}\|_{\ell_1}
       &=
      \sum_{k=t-t_0+1}^{\infty} \|\sum_{i=t-t_0+1}^{k} \alpha_{\theta}(i)\alpha_{g,w}(k-i)\|_2\\
       &\le  \sum_{k=t-t_0+1}^{\infty} \sum_{i=t-t_0+1}^{k} \|\alpha_{\theta}(i)\|_2 \|\alpha_{g,w}(k-i)\|_2\\
       &= \sum_{k=t-t_0+1}^{\infty} \sum_{i=0}^{k-(t-t_0+1)} d_3(i)
       \|\alpha_{g,w}(k-i-(t-t_0+1))\|_2 \\
       &=\sum_{k=0}^{\infty}\sum_{i=0}^{k} d_3(i)
       \|\alpha_{g,w}(k-i)\|_2=
       \left(\sum_{k=t-t_0+1}^{\infty} \|\alpha_{\theta}(k)\|_2\right) \cdot
       \|\alpha_{g,w}\|_{\ell_1}
  \end{align*}
     By noticing that 
     \[
       \begin{aligned}
        & \sum_{t=t_0}^{N+t_0} \sum_{k=t-t_0+1}^{\infty} \|\alpha_{\theta}(k)\|_2 \le
        \sum_{k=1}^{\infty} \min\{k,N+1\} \|\alpha_{\theta}(k)\|_2
        \le 
        \sum_{k=1}^{\infty} k \|\alpha_{\theta}(k)\|_2=\theta_{\infty}(\alpha_{\theta})
       \end{aligned}
     \]
the last inequality \eqref{alpha:lemma:eq2d} follows. Finally, the stationarity is a consequence of the standard result that stable LTI filters maps stationary process to stationary process, see e.g., \cite[Section 5.6]{gikh96}.

\end{proof}

\section{Proof of Theorem \ref{thm:main}}
\label{sect:pf:thm_main}
  The proof of Theorem \ref{thm:main} is organized as follows. 
  First, we define an auxiliary quantity called \emph{the infinite past empirical error}
\[ V_N(\theta)=
\frac{1}{N}\sum_{t=0}^{N-1}\ell(\rvy(t),\hat{\rvy}_\theta(t)) \]
It follows that $V_N(\theta)$ is an unbiased estimate of $\mathcal{L}(\theta)$, i.e., $\bE[V_N(\theta)]=\mathcal{L}(\theta)$.

As the first step, we show the following general theorem.
\begin{theorem}
\label{thm:initial}
    Given a parameter set $\Theta$, a prior distribution $\pi$ over $\Theta$, $\delta\in(0,0.5)$, $\lambda>0$ and $\epsilon \in \{-1,1\}$. Then with probability at least $1-2\delta$, the following holds
    \begin{align}
        \forall \rho\in \mathcal{M}_\pi:\epsilon E_{\theta\sim\rho} \mathcal{L}(\theta)
        \leq \epsilon E_{\theta\sim\rho}\hat{\mathcal{L}}_N(\theta) +        
        \tilde{\mathbb{B}}_N(\rho,\pi,\lambda,\epsilon,\delta) \label{eq:initial:tildern} \\
    \tilde{\mathbb{B}}_N(\rho,\pi,\lambda,\epsilon,\delta)
        \triangleq \frac{1}{\lambda}\left ( \KL(\rho\|\pi) + \ln\dfrac{1}{\delta} + \Psi(\lambda,\pi,N,\epsilon)\right ) \label{eq:initial:def:tildern} \\
        \Psi(\lambda,\pi,N,\epsilon)\triangleq \frac{1}{2} \Big (\ln E_{\theta\sim\pi}\bE[e^{2\epsilon\lambda(V_N(\theta)-\hat{\mathcal{L}}_N(\theta))}] \nonumber \\
        +\ln E_{\theta\sim\pi}\bE[e^{2\epsilon\lambda (\mathcal{L}(\theta)-V_N(\theta))}] \Big ),\label{eq:Psi:initial}
    \end{align}
    with $\KL(\rho\|\pi)=E_{\theta\sim\rho}\ln\frac{\rho(\theta)}{\pi(\theta)}$ the Kullback–Leibler divergence. 
\end{theorem}
In the theorem above, the moment generating functions 
$\bE[e^{2\epsilon\lambda (\mathcal{L}(\theta)-V_N(\theta))}]$ and  $\bE[e^{2\epsilon\lambda(V_N(\theta)-\hat{\mathcal{L}}_N(\theta))}]$
in the right-hand side are interpreted as $\infty$ whenever $2\epsilon \lambda$ falls outside of their domain of definition. In this case, $\Psi(\lambda,\pi,N,\epsilon)$ is taken to mean $\infty$ and  the bounds are trivially true. 
\begin{proof}[Proof of Theorem \ref{thm:initial}]
    By \cite[Theorem 3]{nips-16}, for any measurable functions $X(\theta,\omega),Y(\theta,\omega)$,  $\tilde{\lambda}>0$, and $\epsilon \in \{-1,1\}$, with probability at least $1-\delta$ we have:
    \begin{multline}
        \forall \rho\in \mathcal{M}_{\pi}:\ E_{\theta\sim\rho}X(\theta,\omega)
        \leq E_{\theta\sim\rho}Y(\theta,\omega)+ \frac{1}{\tilde{\lambda}}\left (\KL(\rho\|\pi)+\ln\dfrac{1}{\delta}+ \Psi_{X,Y}(\tilde{\lambda}) \right ),
        \label{pf:thm:initial:eq1}
    \end{multline}
    with $\Psi_{X,Y}(\tilde{\lambda})=\ln E_{\theta\sim\pi}\bE[e^{\tilde{\lambda}(X(\theta,\omega)-Y(\theta,\omega))}]$.
    Now apply the above theorem for $X=\epsilon \mathcal{L}(\theta),Y=\epsilon V_N(\theta)$, and $X=\epsilon V_N(\theta),Y=\epsilon \hat{\mathcal{L}}_N(\theta)$. Then by applying a union bound on the two obtained probabilistic inequalities, we obtain with probability at least $1-2\delta$
    \begin{multline*}
        \forall \rho \in \mathcal{M}_{\pi}:\ \epsilon E_{\theta\sim\rho}\mathcal{L}(\theta)\leq \epsilon  E_{\theta\sim\rho}\hat{\mathcal{L}}_N(\theta)\\
        + \frac{1}{\tilde{\lambda}}\left (2\KL(\rho\|\pi)+2\ln\dfrac{1}{\delta}+ \Psi_{\epsilon V_N, \epsilon \hat{\mathcal{L}}_N}(\tilde{\lambda}) +\Psi_{\epsilon \mathcal{L},\epsilon V_N}(\tilde{\lambda})
        \right ) 
    \end{multline*}
     With some algebraic manipulation and $\tilde{\lambda}=2\lambda$, we obtain the statement of the theorem. 
\end{proof}
For the case of bounded innovation noise or Lipschitz loss function, we
then show that $\bE[e^{2\epsilon\lambda(V_N(\theta)-\hat{\mathcal{L}}_N(\theta))}]$
and $\bE[e^{2\epsilon\lambda (\mathcal{L}(\theta)-V_N(\theta))}]$ can be both upper bounded by
$e^{\mathcal{C}_i(\theta,2\lambda,N,\delta,\epsilon)}$, from which the statement of the theorem follows.
For the case of quadratic loss function, we will show that it can be reduced to the case of Lipschitz loss function
with probability at least  $1-2\delta/3$ and then we apply the result for Lipschitz loss with $2\delta/3$ instead of $\delta$.

\subsection{Proof of Theorem \ref{thm:main}: bounded case}
For the sake of readability, we repeat below the basic assumptions for the case of bounded data
\begin{assumption}[Bounded noise and Lipschitz loss]\label{as:bounded}
    The innovation noise $\rve_g(t)$ from \eqref{eq:generator} is essentially bounded , i.e.,
    \begin{align}
        \|\rve_g(t)\|_{\infty} \leq c_\rve,\; \forall t\in\sZ
    \end{align}
and the loss function $\ell$ is either $L_{\ell}$-Lipschitz or quadratic.
\end{assumption}
The following Lemma is obvious in the $L_\ell$-Lipschits case, and follow by applying Cauchy–Schwarz in the quadratic case.
\begin{lemma}
  \label{lemma:bounded:quadratic_loss}
   Let $B=\{(z,z') \in \mathbb{R}^{n_\rvy} \times \mathbb{R}^{n_\rvy} \mid \|z-z'\|_2 \le G_e(\theta)\sqrt{n_{\rve}}c_{\rve}\}$. Then 
   $(\y(t), \y_{\theta}(t)), (\y(t),\y_{\theta}(t \mid 0)) \in B$ 
   for all $t$, $t_0 \le t$, and for all
   $(z_1,z_2),(z_3,z_4) \in B$, 
   \[ \ell(z_1,z_2)-\ell(z_3,z_4)\le L(\theta)\left(\|z_1-z_3\|_2 +\|z_2-z_4\|_2 \right) \]
   where $L(\theta)$ is as in Theorem \ref{thm:main}, i.e., $L(\theta)=2G_e(\theta)\sqrt{n_{\rve}}c_{\rve}$ if $\ell$-quadratic, and $L(\theta)=L_{\ell}$ if $\ell$ is $L_{\ell}$-Lipschitz.
\end{lemma}
\begin{remark}[Norm loss $L_{\ell}$-Lipschitz]
\label{rem:lipschitz2}
     Note that if the loss $\ell$ is the $p$ norm of the difference, i.e., $\ell(y,\hat{y})=\|y-\hat{y}\|_p$, then 
    $\ell(y,\hat{y})\leq C_p\|y-\hat{y}\|_2$ with $C_p=\sqrt{n_{\rvy}}$ if $p=1$, and $C_p=1$
     if $p=\infty$ or $p=2$.   
\end{remark}

\begin{lemma}\label{lemma:bounded_mgf(V-hatL)}Let assumptions \ref{as:generator}, \ref{as:parameterisation} and \ref{as:bounded} hold. Then
    \begin{align*}
        \bE[e^{\left(\epsilon \lambda|V_N(\theta)-\hat{\mathcal{L}}_N(\theta)|\right)}] 
        \leq e^{\left(\frac{L(\theta)\lambda}{N} \theta_{\infty}(\alpha_{\theta})\|\alpha_g\|_{\ell_1} (c_\rve\sqrt{n_\rve})\right)}
        =e^{C_{b,1}(\lambda,\theta,L(\theta),c_\rve,N)}
    \end{align*}
\end{lemma}
\begin{proof}[Proof \ref{lemma:bounded_mgf(V-hatL)} ]
        Using the Lipschitz property of the loss function it follows that 
    \begin{equation}
        \label{lemma:bounded_mgf(V-hatL):eq1}
       \begin{aligned}
        & \epsilon (V_N(\theta)-\hat{\mathcal{L}}_N(\theta)) \le 
      |V_N(\theta) - \hat{\mathcal{L}}_N(\theta)| \le 
    \frac{1}{N} \sum_{t=0}^{N-1}|\ell(\y(t),\hat{\rvy}_{\theta}(t)) 
    - \ell(\rvy(t),\hat{\rvy}_{\theta}(t \mid 0))|
  \le \\
  & 
    \frac{L(\theta)}{N} \sum_{t=0}^{N-1} \|\hat{\rvy}_{\theta}(t)-\hat{\rvy}_{\theta}(t \mid 0)\|_2
       \end{aligned}
    \end{equation}
    From Lemma \ref{alpha:lemma2} it follows that 
    \[ 
       \hat{\rvy}_{\theta}(t)-\hat{\rvy}_{\theta}(t \mid 0) = 
       ((c_{\theta}-c_{\theta,t})\star\e_g)(t)
       =\sum_{k=0}^{\infty} 
       (c_{\theta}-c_{\theta,t})(k)\e_g(t-k)
    \]
    and hence 
    \[ 
       \sum_{t=0}^{N-1}\|\hat{\rvy}_{\theta}(t)-\hat{\rvy}_{\theta}(t \mid 0)\|_2 \le \sum_{t=0}^{N-1} \sum_{k=0}^{\infty} \|(c_{\theta}-c_{\theta,t})(k)\|_{2} \|\e_g(t-k)\|_2 \le 
       \sum_{t=0}^{N-1} \|c_{\theta}-c_{\theta,t}\|_{\ell_1} \sqrt{n_\rve} c_\rve
    \]
    From \eqref{alpha:lemma:eq2d} it follows that 
    \[
       \sum_{t=0}^{N-1} \|c_{\theta}-c_{\theta,t}\|_{\ell_1} \le \theta_{\infty}(\alpha_{\theta}) \|\alpha_g\|_{\ell_1}
    \]
    and hence, 
    \[
   \sum_{t=0}^{N-1} \|\hat{\rvy}_{\theta}(t)-\hat{\rvy}_{\theta}(t \mid 0)\|_2 \le \theta_{\infty}(\alpha_{\theta}) \|\alpha_g\|_{\ell_1}\sqrt{n_\rve} c_\rve
    \]
    from which the statement of the lemma follows using the definition of $\mathcal{C}_{b,1}(\lambda,\theta,L(\theta),c_\rve,N)$.
\end{proof}
\begin{lemma}\label{lemma:bounded:mgf(L-V)}
Let assumptions \ref{as:generator}, \ref{as:parameterisation}  and \ref{as:bounded} hold. Then
\begin{equation}
 \bE[e^{\epsilon \lambda(\mathcal{L}(\theta)-V_N(\theta))}] \leq 
  e^{\frac{(L(\theta))^2\lambda^2}{2N}(G_e(\theta)+2G_{e,1}(\theta))(c_{\rve}\sqrt{n_{\rve}})^2}
=e^{\mathcal{C}_{b,2}(\lambda,\theta,L(\theta),c_\rve,N)}
    \end{equation}
 \end{lemma}
\begin{proof}[Proof of Lemma \ref{lemma:bounded:mgf(L-V)} ]
 For each $\Sigma(\theta) \in \mathcal{F}$, consider
 $\rvz_t=\left[\y^T(t)~~\hat{\y}_\theta^T(t)\right]^T$.
 Then by Lemma \ref{alpha:lemma2},  \( \rvz_t=\left(\begin{bmatrix} \alpha_{g,y} \\ c_{\theta} \end{bmatrix} \star \e_g\right)(t) \)
 and $\rvz_t$ is stationary. Moreover, as $\left\|\begin{bmatrix}\alpha_{g,y} \\ c_{\theta} \end{bmatrix}\right\|_{\ell_1} \leq G_e(\theta) <\infty$ and $\theta_{\infty}\left(\begin{bmatrix} \alpha_{g,y} \\ c_{\theta} \end{bmatrix}\right) \leq G_{e,1}(\theta) <\infty$, it follows that $\rvz_t$ is a Bernoulli-shift. In particular,  $\|\rvz_t\|_\infty \leq G_e(\theta) c_\rve \sqrt{n_\rve}$, and by \cite[Proposition 4.2]{alquier2013prediction} 
 $\rvz_t$ is a weakly dependent process in the terminology of
 \cite{alquier2013prediction}, with $\|\rvz_0\|_\infty \leq G_e(\theta) c_\rve \sqrt{n_\rve}$, and the mixing coefficient $\theta_{\infty,N}(1)$ of $\rvz_t$ satisfies
 $\theta_{\infty,N}(1) < 2G_{e,1}(\theta)c_\rve \sqrt{n_\rve}$. 

If $\ell$ is $L_\ell$-Lipschitz, then let $B=\mathbb{R}^{n_{\rvy}} 
\times \mathbb{R}^{n_{\rvy}}$. If $\ell$ is quadratic, then consider the set $B$ from Lemma \ref{lemma:bounded:quadratic_loss}.
Define  the function $h(z_0,z_0',\ldots,z_{N-1},z_{N-1}')=\epsilon \frac{1}{L(\theta)} \sum_{i=0}^{N-1} \ell(z_i,z_i')$, $z_0,\ldots, z_{N-1}, z_0',\ldots, z_{N-1}' \in \mathbb{R}^{n_{\rvy}}$ on $B^N$.
We claim that 
 $h$ has Lipschitz constant $1$ in the following sense:
\[ |h(z_0,z_0',\ldots,z_{N-1}, z_{N-1}')-h(\tilde{z}_0, \tilde{z}_0',\ldots,\tilde{z}_{N-1}, \tilde{z}_{N-1}')| \le 
\sum_{i=0}^{N-1}  \|z_i-\tilde{z}_i\|_2 + \|z_i'-\tilde{z}_i'\|_2 \]
Indeed, using Lemma \ref{lemma:bounded:quadratic_loss}, 
\begin{align*}
    & |h(z_0,z_0',\ldots,z_{N-1}, z_{N-1}')-h(\tilde{z}_0, \tilde{z}_0',\ldots,\tilde{z}_{N-1}, \tilde{z}_{N-1}')|
    =\left |\frac{\epsilon}{L(\theta)}\sum_{i=0}^{N-1} \ell(z_i,z_i')-\ell(\tilde{z}_i,\tilde{z}_i') \right |\\
    &\leq \frac{1}{L(\theta)} \sum_{i=0}^{N-1} L(\theta)  (\|z_i-\tilde{z}_i\|_2 + \|z_i'-\tilde{z}_i'\|_2) =
     \sum_{i=0}^{N-1}  \|z_i-\tilde{z}_i\|_2 + \|z_i'-\tilde{z}_i'\|_2
\end{align*}

If $\ell$ is $L_{\ell}$-Lipschitz, then $h$ is already defined on  $(\mathbb{R}^{n_{\rvy}} \times \mathbb{R}^{n_{\rvy}})^N$. If $\ell$ is quadratic, 
we can use Kirszbraun's theorem for scalar valued functions
\cite[2.10.44,page 202]{Federer1996},
from which it then follows that $h$ can be extended to a $1$-Lipschitz map on $(\mathbb{R}^{n_{\rvy}} \times \mathbb{R}^{n_{\rvy}})^N$
without changing its values on $B^N$.  In fact, such an 
extension can be defined as 
$h_{ext}(z)=\inf_{x \in B^N} \{h(x) + \|z-x\|_2\}$.
By abuse of notation this extension will be denote by $h$ too. 
Notice that by Lemma \ref{lemma:bounded:quadratic_loss}, $(\y(t),\y_{\theta}(t)) \in B$ for all $t$, hence
$h(\y(0),\y_{\theta}(0),\ldots,\y(N-1),\y_{\theta}(N-1)) = \frac{\epsilon}{L(\theta)} \sum_{i=0}^{N-1} \ell(\y(i),\y_{\theta}(i))=\frac{\epsilon}{L(\theta)} N V_N(\theta)$.


Notice that, 
\[\epsilon \lambda(\bE[V_N(\theta)]- V_N(\theta))=\frac{\lambda L(\theta)}{N} \left ( \bE[h]-h \right ).\]

 Now we can apply \cite[Theorem 6.6]{alquier2013prediction}, from which it follows that for any $q\geq 0$
 \begin{align}
     \bE[e^{q(\bE[h]-h)}]\leq e^{\frac{q^2N}{2} (\|\rvz_t\|_\infty+\theta_{\infty,N}(1))^2}
 \end{align}
By  
choosing $q=\frac{\lambda L(\theta)}{N}$ and using the definition of $\mathcal{C}_{b,2}(\lambda,\theta,L(\theta),c_\rve,N)$, we obtain the statement of the lemma .
\end{proof}
\begin{proof}[Proof of Theorem \ref{thm:main}: bounded processes]
  The statement of the theorem follows from Theorem \ref{thm:initial}
  by replacing $\bE[e^{2\lambda \epsilon(V_N(\theta)-\hat{\mathcal{L}}_N(\theta))}]$ and $\bE[e^{2\lambda \epsilon (\mathcal{L}(\theta)-V_N(\theta))}]$ with their upper bound from
  Lemma \ref{lemma:bounded_mgf(V-hatL)}  and Lemma \ref{lemma:bounded:mgf(L-V)} respectively. 
\end{proof}

\subsection{Proof of Theorem \ref{thm:main}: unbounded processes and Lipschitz loss functions}
\label{lipschitz:unbounded}

 
Recall that according to our assumptions, the 2-norm of the noise process $\e_g$ is subgaussian. 
Bounded processes have this property, however, in this section we will not
assume that the process $\e_g$ is bounded. 

Recall that we also consider Lipschitz losses.
\begin{assumption}[Lipschitz loss]
  \label{as:unbounded_lipschitz}
We assume that the loss function $\ell$ is $L_{\ell}$-Lipschitz, i.e., 
for all $y,\hat{y},y',\hat{y}' \in \mathbb{R}^{n_\rvy}$,
$|\ell(y,\hat{y})-\ell(y',\hat{y}')| \le L_{\ell} (\|y-y'\|_2+\|\hat{y}-\hat{y}'\|_2)$.
\end{assumption}

The main idea is to use Theorem \ref{thm:initial} and to derive bounds for the
 moment generating functions 
$\bE[ e^{2\lambda \epsilon (V_N(\theta)-\hat{\mathcal{L}}_N(\theta))}]$ and $\bE[e^{2\lambda \epsilon (\mathcal{L}(\theta)-V_N(\theta))}]$. 
In order to achieve this, we will reduce the problem of bounding 
$\bE[ e^{2\lambda \epsilon (V_N(\theta)-\hat{\mathcal{L}}_N(\theta))}]$ and $\bE[e^{2\lambda \epsilon (\mathcal{L}(\theta)-V_N(\theta))}]$
to the same problem for bounded noises, 
by truncating the unbounded noise
$\e_g$  which drives the data generator. 
This problem was solved in Lemma \ref{lemma:bounded_mgf(V-hatL)} and Lemma \ref{lemma:bounded:mgf(L-V)}, which thus can be re-used. 

To this end, let 
$C > 0$ be any constant and let 
$\bar{\e}_g(t)$ be the truncated version of $\e_g(t)$, i.e., 
\begin{align}\label{e-}
\bar{\e}_g(t)=\e_g(t) \chi_{\{\|\e_g(t)\| \leq C\}}
:=\e_g(t) \chi(\|\e_g(t)\| \leq C). 
\end{align}
Let $\bar{\rvw}(t)=\begin{bmatrix} \bar{\rvy}^T(t) & \bar{\rvu}^T(t) \end{bmatrix}^T$ be the output of the data generator, if $\e_g(t)$ is replaced by $\bar{\e}_g(t)$, and let 
$\hat{\bar{\rvy}}_{\theta}(t \mid t_0)$ be the output of    the predictor 
\eqref{eq:predictor}
parametrized by $\theta$
with $\rvw$ replaced by $\bar{\rvw}$, and let 
$\hat{\bar{\rvy}}_{\theta}(t)$ 
be the limit 
$\lim_{t_0 \rightarrow -\infty} \hat{\bar{\rvy}}_{\theta}(t \mid t_0)$.
Note that Lemma \ref{l:ihp} applies, as $\bar{\e}_g(t)$ satisfies all the assumptions of
Assumption \ref{as:generator}.
Let $\bar{V}_N(\theta)=\frac{1}{N}\sum_{t=0}^{N-1} \ell(\bar{\rvy}(t),\hat{\bar{\rvy}}_{\theta}(t))$ be the corresponding infinite past empirical error, 
$\hat{\bar{\mathcal{L}}}_N(\theta)=\frac{1}{N}\sum_{t=0}^{N-1} \ell(\bar{\rvy}(t),\hat{\bar{\rvy}}_{\theta}(t \mid 0))$
the finite past empirical error, and 
$\bar{\mathcal{L}}(\theta)=\bE[\ell(\bar{\rvy}(t),\hat{\bar{\rvy}}_\theta(t))]$ the true error for the model corresponding to $\theta \in \Theta$, if it is fed with the truncated signal $\bar{\rvw}$
instead of $\rvw$. 

It turns out that we can bound the difference between these quantities and their non-truncated counterparts.
The basic steps are as follows:
First, using Lipschitz properties of the loss function, we bound the moment generating functions of
the differences $\sup_{\theta \in \Theta} |\hat{\mathcal{L}}_N(\theta) - \hat{\bar{\mathcal{L}}}_N(\theta)|$,
$\sup_{\theta \in \Theta}|V_N(\theta) - \bar{V}_N(\theta)|$ by the corresponding quantity for the $\ell_2$ loss, i.e.; by
the moment generating functions of 
$\sup_{\theta \in \Theta} \sum_{t=0}^{N-1} \|\hat{\rvy}_{\theta}(t \mid 0)-\hat{\bar{\rvy}}_{\theta}(t \mid 0)\|_2$
and $\sum_{t=0}^{N-1} \|\rvy(t)-\bar{\rvy}(t)\|_2$.
\begin{lemma}
\label{vl-bl:lemma0}
\begin{align}
    \bE[ & e^{\lambda \sup_{\theta \in \Theta}  |\hat{\mathcal{L}}_N(\theta) - \hat{\bar{\mathcal{L}}}_N(\theta)|}]
   \le  \nonumber \\
     & \sqrt{\bE[e^{\frac{2 \lambda L_{\ell}}{N} \sum_{t=0}^{N-1}
    \|\y(t)-\bar{\y}(t)\|_2}]}
    \sqrt{\bE[e^{\frac{2\lambda L_{\ell}}{N} \sup_{\theta \in \Theta} 
    \sum_{t=0}^{N-1} \|\hat{\rvy}_{\theta}(t \mid 0)-\hat{\bar{\rvy}}_{\theta}(t \mid 0)\|_2}]}
\label{vl-vbl:lemma1:pf:eq0} \\
    \bE[ & e^{\lambda \sup_{\theta \in \Theta}  |V_N(\theta) - \bar{V}_N(\theta)|}] 
   \le \nonumber \\
    & \sqrt{\bE[e^{\frac{2 \lambda L_{\ell}}{N} \sum_{t=0}^{N-1}
    \|\y(t)-\bar{\y}(t)\|_2}] }\sqrt{\bE[e^{\frac{2 \lambda L_{\ell}}{N} \sup_{\theta \in \Theta} \sum_{t=0}^{N-1} \|\hat{\rvy}_{\theta}(t)-\hat{\bar{\rvy}}_{\theta}(t)\|_2}]}
\label{vl-vbl:lemma1:pf:eq3}
\end{align}
\end{lemma}
Then, using the technique of \cite{alquier2012pred} we bound the moment generating functions of 
the truncation errors $\sup_{\theta \in \Theta} \sum_{t=0}^{N-1} \|\hat{\rvy}_{\theta}(t \mid 0)-\hat{\bar{\rvy}}_{\theta}(t \mid 0)\|_2$
and $\sum_{t=0}^{N-1} \|\rvy(t)-\bar{\rvy}(t)\|_2$. To this end, we use the following bound on the 
moment generating function of $\|\e_g(t)\|_2$.
\begin{lemma}[Subgaussian noise]
\label{lemma:subgauss:Cramer}
If
$\|\e_g(t)\|_2$ is sub-Gaussian  then
$\Psi_{\e_g}(\lambda)=\bE[e^{\lambda \|\e_g(0)\|_2}]$ 
 is well-defined for all $\lambda > 0$.
In fact, 
\begin{align*}
& \Psi_{\e_g}(\lambda) \le  \bar{\Psi}_{\rve_g}(\lambda)=e^{(\lambda^2 14^2 \|\rve_g\|_{\psi_2}^2+3\lambda \|\rve_g\|_{\psi_2})}
\end{align*}
\end{lemma}
Using this bound, the moment generating function of the  truncation error can be bounded as follows:
\begin{lemma}
 \label{basic_lemma:lipschitz2}
 For any $\lambda < 0.5N$ 
\begin{align*} 
  \bE[e^{\frac{\lambda}{N} \sum_{t=0}^{N-1}\|\y(t)-\bar{\y}(t)\|_2}] \le e^{\mathcal{C}(\textcolor{black}{\|\alpha_{g}\|_{\ell_1}},\lambda,N,C)} \\
  \bE[e^{\frac{\lambda}{N} \sup_{\theta \in \Theta}\sum_{t=0}^{N-1} \|\hat{\y}(t)-\hat{\bar{\y}}_\theta(t)\|_2}] \le e^{\mathcal{C}(\textcolor{black}{\|\alpha_{g}\|_{\ell_1}}\|\bar{\alpha}\|_{\ell_1},\lambda,N,C)} \\
\bE[e^{\frac{\lambda}{N} \sup_{\theta \in \Theta}\sum_{t=0}^{N-1} \|\hat{\y}(t \mid 0  )-\hat{\bar{\y}}_\theta(t \mid 0 )\|_2}] \le e^{\mathcal{C}(\textcolor{black}{\|\alpha_{g}\|_{\ell_1}}\|\bar{\alpha}\|_{\ell_1},\lambda,N,C)}
\end{align*}
where
\begin{equation}
\label{basic_lemma:lipschitz2:theo:eq1}
   \mathcal{C}(s,\lambda,N,C)=\lambda \frac{\bar{\Psi}_{\rve_g}(s)sC}{(e^{sC}-1)}+
     \frac{2\lambda^2}{N} \bar{\Psi}_{\rve_g}( s)
\end{equation}
\end{lemma}
From this we derive
\begin{lemma}
\label{vl-vbl:lemma1}
For any $\lambda < N/4L_{\ell}$ 
\begin{align*} 
& \max\{\bE[e^{\lambda \sup_{\theta \in \Theta} |\hat{\mathcal{L}}_N(\theta) - \hat{\bar{\mathcal{L}}}_N(\theta)|}],
\bE[e^{\lambda \sup_{\theta \in \Theta} |V_N(\theta) - \bar{V}_N(\theta)|}], 
e^{\lambda \sup_{\theta \in \Theta} |\mathcal{L}(\theta) - \mathcal{\bar{L}}(\theta)|}\} \le  \\
& e^{\frac{1}{2} \left( \mathcal{C}(\|\bar{\alpha}\|_{\ell_1} \|\textcolor{black}{\alpha_{g}}\|_{\ell_1}, 2L_{\ell}\lambda,N,C)+
\mathcal{C}(\|\textcolor{black}{\alpha_{g}}\|_{\ell_1}, 2L_{\ell}\lambda,N,C)\right)} 
\end{align*}
\end{lemma}
In order to go back to bounding the moment generating functions of 
$\epsilon (V_N(\theta)-\hat{\mathcal{L}}_N(\theta))$ and $\epsilon (\mathcal{L}(\theta)-V_N(\theta))$, 
we use the previous lemma, and the following inequality which we derive using the Cauchy-Schwarz inequality: 
\begin{lemma}
  \label{basic_lemma:lipschitz3.1}
  For $\epsilon\in\{-1,1\}$
   \begin{align}
      & \bE[e^{\lambda \epsilon (\mathcal{L}(\theta)-V_N(\theta))}]   \le \nonumber  \\
      & \sqrt{\bE[e^{2\lambda \epsilon (\bar{\mathcal{L}}(\theta)-\bar{V}_N(\theta))}]}
     \sqrt{\sqrt{\bE[e^{4\lambda \sup_{\theta \in \Theta} |\bar{V}_N(\theta)-V_N(\theta)|}]} 
     \sqrt{e^{4\lambda \sup_{\theta \in \Theta} |\bar{\mathcal{L}}(\theta)-\mathcal{L}(\theta)|}}}
     \label{eq:basic_lemma_lipschitz3:pf:eq0.1}  \\
        & \bE[e^{\lambda \epsilon (V_N(\theta)-\hat{\mathcal{L}}_N(\theta))}]   
      \le  \nonumber \\
      & \sqrt{\bE[e^{2\lambda \epsilon (\bar{V}_N(\theta)-\hat{\bar{\mathcal{L}}}_N(\theta))}]}
     \sqrt{\sqrt{\bE[e^{4\lambda \sup_{\theta \in \Theta} |\bar{V}_N(\theta)-V_N(\theta)|}]} \sqrt{\bE[e^{4\lambda \sup_{\theta \in \Theta} |\hat{\bar{\mathcal{L}}}_N(\theta)-\hat{\mathcal{L}}_N(\theta)|}}]}
     \label{eq:basic_lemma_lipschitz3:pf:eq0} 
    \end{align}
\end{lemma}
If we apply Lemma~\ref{vl-vbl:lemma1} 
to the right-hand sides of the latter inequalities, then we obtain.
\begin{lemma}
\label{basic_lemma:lipschitz3}
For $\epsilon \in \{1,-1\}$, for $\lambda< \frac{N}{16L_{\ell}}$,
\begin{align*}
 & \bE[e^{\lambda \epsilon (V_N(\theta)-\hat{\mathcal{L}}_N(\theta))}]   
 \le \\
 &e^{\frac{1}{4}(\mathcal{C}(\|\bar{\alpha}\|_{\ell_1}\|\textcolor{black}{\alpha_{g}}\|_{\ell_1}, 8L_{\ell}\lambda,N,C)+\mathcal{C}(\|\textcolor{black}{\alpha_{g}}\|_{\ell_1}, 8L_{\ell}\lambda,N,C))}
 \sqrt{\bE[e^{2\lambda \epsilon (\bar{V}_N(\theta)-\hat{\bar{\mathcal{L}}}_N(\theta))}]} 
 \\
 & \bE[e^{\lambda \epsilon (\mathcal{L}(\theta)-V_N(\theta))}]   
 \le e^{\frac{1}{4}(\mathcal{C}(\|\bar{\alpha}\|_{\ell_1}\|\textcolor{black}{\alpha_{g}}\|_{\ell_1}, 8L_{\ell}\lambda,N,C)+\mathcal{C}(\|\textcolor{black}{\alpha_{g}}\|_{\ell_1}, 8L_{\ell}\lambda,N,C))}
 \sqrt{\bE[e^{2\lambda \epsilon (\bar{\mathcal{L}}(\theta)-\bar{V}_N(\theta)})]}
\end{align*}
\end{lemma}
Finally, Lemma \ref{lemma:bounded_mgf(V-hatL)} and Lemma \ref{lemma:bounded:mgf(L-V)} 
derived for the bounded case can be applied to the moment generating functions of
$\epsilon (\bar{\mathcal{L}}(\theta)-\bar{V}_N(\theta))$ and $\epsilon (\bar{V}_N(\theta)-\hat{\bar{\mathcal{L}}}_N(\theta))$, for a suitable choice of $C$. 
The detailed proof is presented below. 
\begin{proof}[Proof of Theorem \ref{thm:main}: unbounded noise and Lipschitz loss]
 Let $C=\frac{2\ln(N)}{K\|\alpha_g\|_{\ell_1}}$ where $K=\min\{1, \|\bar{\alpha}\|_{\ell_1}\}$. Then
 \[ C\|\alpha_g\|_{\ell_1} \ge 2\ln(N), \quad C\|\bar{\alpha}\|_{\ell_1}\|\alpha_g\|_{\ell_1} \ge 2\ln(N) \]
and hence for $N \ge 2$,
 \begin{align*}
    e^{C\|\alpha_g\|_{\ell_1}}-1 \ge  e^{2\ln(N)} - 1 = N^2-1 \ge N, \quad
    e^{C\|\alpha_g\|_{\ell_1}\|\bar{\alpha}\|_{\ell_1}}-1 \ge  e^{2\ln(N)} - 1 = N^2 -1 \ge N
 \end{align*}
 and therefore
 \[
     \frac{1}{e^{C\|\alpha_g\|_{\ell_1}}-1} \le \frac{1}{N}, \quad 
     \frac{1}{e^{C\|\alpha_g\|_{\ell_1}\|\bar{\alpha}\|_{\ell_1}}-1} \le \frac{1}{N}
 \]
 Therefore, as $C\|\bar{\alpha}\|_{\ell_1}\|\alpha_g\|_{\ell_1} = 2\ln(N) \max\{1,\|\bar{\alpha}\|_{\ell_1}\}$,
 \begin{align*}
    & \mathcal{C}(\|\alpha_g\|_{\ell_1}, 8L_{\ell}\lambda,N,C) \le
      \frac{8L_{\ell}\lambda}{N}\bar{\Psi}_{\e_g}(\|\alpha_g\|_{\ell_1})\textcolor{black}{\|\alpha_g\|_{\ell_1}}\frac{2\ln(N)}{K\|\alpha_g\|_{\ell_1}} + \frac{ (8L_{\ell}\lambda)^2}{N}2\bar{\Psi}_{\e_g}(\|\alpha_g\|_{\ell_1}) \le \\
      & \frac{8L_{\ell}\lambda}{N}\bar{\Psi}_{\e_g}(\|\alpha_g\|_{\ell_1})\frac{\textcolor{black}{1}}{\min\{1,\|\bar{\alpha}\|_{\ell_1}\}}2\ln(N) + \frac{ 128L_{\ell}^2\lambda^2}{N}\bar{\Psi}_{\e_g}(\|\alpha_g\|_{\ell_1})
      \\
    & \mathcal{C}(\|\bar{\alpha}\|_{\ell_1}\|\alpha_g\|_{\ell_1}, 8L_{\ell}\lambda,N,C)
     \le \\
     &\frac{8L_{\ell}\lambda}{N}\bar{\Psi}_{\e_g}(\|\bar{\alpha}\|_{\ell_1}\|\alpha_g\|_{\ell_1})(\|\bar{\alpha}\|_{\ell_1}\|\alpha_g\|_{\ell_1})C + \frac{ (8L_{\ell}\lambda)^2}{N}2\bar{\Psi}_{\e_g}(\|\bar{\alpha}\|_{\ell_1}\|\alpha_g\|_{\ell_1}) = \\
    & \frac{8L_{\ell}\lambda}{N} \bar{\Psi}_{\e_g}(\|\bar{\alpha}\|_{\ell_1}\|\alpha_g\|_{\ell_1})\textcolor{black}{2\ln(N)\max\{1,\|\bar{\alpha}\|_{\ell_1}\}}+ \frac{128 L_{\ell}^2\lambda^2}{N}\bar{\Psi}_{\e_g}(\|\bar{\alpha}\|_{\ell_1}\|\alpha_g\|_{\ell_1})
 \end{align*}
 Hence, 
 \begin{equation}
  \label{thm:main:pf:unbounded:eq-1}
 \begin{aligned}
    & \frac{1}{4}(\mathcal{C}(\|\bar{\alpha}\|_{\ell_1}\|\alpha_g\|_{\ell_1}, 8L_{\ell}\lambda,N,C)+\mathcal{C}(\|\alpha_g\|_{\ell_1}, 8L_{\ell}\lambda,N,C))
    \le  \\
     & \frac{4\lambda \ln(N) L_{\ell}}{N} \left( 
      \underbrace{\bar{\Psi}_{\e_g}(\|\bar{\alpha}\|_{\ell_1}\|\alpha_g\|_{\ell_1})
      \textcolor{black}{\max\{1,\|\bar{\alpha}\|_{\ell_1}\}} + \bar{\Psi}_{\e_g}(\|\alpha_g\|_{\ell_1})\frac{\textcolor{black}{1}}{\min\{1,\|\bar{\alpha}\|_{\ell_1}\}}}_{C_{u,2}}\right)  \\
    & +\frac{32\lambda^2L_{\ell}^2}{N}\underbrace{\left(\bar{\Psi}_{\e_g}(\|\bar{\alpha}\|_{\ell_1}\|\alpha_g\|_{\ell_1})+ \bar{\Psi}_{\e_g}(\|\alpha_g\|_{\ell_1}) \right)}_{C_{u,1}}=\bar{\mathcal{C}}_{ub}(\lambda, L_{\ell},N)
 \end{aligned}
\end{equation}
  From Lemma \ref{lemma:bounded_mgf(V-hatL)} and Lemma \ref{lemma:bounded:mgf(L-V)} applied to 
  $\bar{V}_N(\theta)$, $\bar{\mathcal{L}}(\theta)$, $\hat{\bar{\mathcal{L}}}_N(\theta)$, with  $c_{\rve}=C=\frac{2\ln(N)}{K\|\alpha_g\|_{\ell_1}}$
  and $L(\theta)=L_{\ell}$,
  it follows that 
  \begin{align*}
    & \sqrt{\bE[e^{2\lambda \epsilon (\bar{V}_N(\theta)-\hat{\bar{\mathcal{L}}}_N(\theta))}]}
    \le \sqrt{e^{\mathcal{C}_{b,1}(2\lambda,\theta,L_{\ell},\frac{2\ln(N)}{K\|\alpha_g\|_{\ell_1}},N)}}=e^{\frac{1}{2}\mathcal{C}_{b,1}(2\lambda,\theta,L_{\ell},\frac{2\ln(N)}{K\|\alpha_g\|_{\ell_1}},N)}
    \\
    & \sqrt{\bE[e^{\epsilon 2\lambda(\bar{\mathcal{L}}(\theta)-\bar{V}_N(\theta))}]} \leq 
    \sqrt{e^{\mathcal{C}_{b,2}(2\lambda,\theta,L_{\ell},\frac{2\ln(N)}{K\|\alpha_g\|_{\ell_1}},N)}}=e^{\frac{1}{2}\mathcal{C}_{b,2}(2\lambda,\theta,L_{\ell},\frac{2\ln(N)}{K\|\alpha_g\|_{\ell_1}},N)}
  \end{align*} 
 \textcolor{black}{Using Lemma \ref{basic_lemma:lipschitz3} and the inequalities above} 
 for any $\textcolor{black}{\lambda  < \frac{N}{32L_{\ell}}}$, 
 \begin{equation}
  \label{thm:main:pf:unbounded:eq1}
 \begin{aligned}
    & \bE[e^{2\lambda \epsilon (V_N(\theta)-\hat{\mathcal{L}}_N(\theta))}] \le 
    e^{\bar{C}_{ub}(2\lambda,L_{\ell},N)+ \frac{1}{2}\mathcal{C}_{b,1}(4\lambda,\theta,L_{\ell},\frac{2\ln(N)}{K\|\alpha_g\|_{\ell_1}},N)}=
    e^{\mathcal{C}_1(2\lambda,\theta,N,\delta,\epsilon)}
    \\
    &  \bE[e^{2\lambda \epsilon (\mathcal{L}(\theta)-V_N(\theta))}]  \le e^{\bar{C}_{ub}(2\lambda,L_{\ell},N)+\frac{1}{2}\mathcal{C}_{b,2}(4\lambda,\theta,L_{\ell},\frac{2\ln(N)}{K\|\alpha_g\|_{\ell_1}},N)}=e^{\mathcal{C}_2(2\lambda,\theta,N,\delta,\epsilon)}
 \end{aligned}
\end{equation}
 The statement of the theorem follows from Theorem \ref{thm:initial}
  by replacing $\bE[e^{2\lambda \epsilon(V_N(\theta)-\hat{\mathcal{L}}_N(\theta))}]$ and $\bE[e^{2\lambda \epsilon (\mathcal{L}(\theta)-V_N(\theta))}]$ with their upper bound from \eqref{thm:main:pf:unbounded:eq1}. 

\end{proof}

The rest of the section is devoted to the proof of the lemmas used above. 

First, we prove Lemma \ref{lemma:subgauss:Cramer}
\begin{proof}[Proof of Lemma \ref{lemma:subgauss:Cramer}]

  First note that $\z(t):=\|\e_g(t)\|_2 - \bE[\|\e_g(t)\|_2]$ is zero mean sub-Gaussian,
  and using sub-additivity of the sub-Gaussian norm we get $\|\z(t)\|_{\psi_2} \le \|\|\e_g(t)\|_2\|_{\psi_2}+ \|\bE[\|\e_g(t)\|_2]\|_{\psi_2}=\|\rve_g\|_{\psi_2}+ \|\bE[\|\e_g(t)\|_2]\|_{\psi_2}$. Now from the proof of \cite[Proposition 2.6.1]{vershynin2026highdimensional}
  it follows that $\bE[\|\e_g(t)\|_2] \le 3\|\e_g\|_{\psi_2}$, and by \cite[Exercise 2.24]{vershynin2026highdimensional},
  $\|\bE[\|\e_g(t)\|_2]\|_{\psi_2} \le \bE[\|\e_g(t)\|_2]/\sqrt{\ln 2} \le 2 \bE[\|\e_g(t)\|_2]$. Hence $\|\z(t)\|_{\psi_2} \le 7\|\rve_g\|_{\psi_2}$.  
  
  Then from the proof of
  \cite[Proposition 2.6.1]{vershynin2026highdimensional} it follows that 
  $\bE[e^{\lambda (\|\e_g(t)\|_2 - \bE[\|\e_g(t)\|_2])}] \le e^{\lambda^2 K_z^2}$ 
  and  $K_z:=\sqrt{\frac{3}{2}} \|\z(t)\|_{\psi_2} \le 14 \|\rve_g\|_{\psi_2}$. It then follows that 
  $\Psi_{\e_g}(\lambda) \le e^{(\lambda^2 K_z^2+ \lambda \bE[\|\e_g(t)\|_2])}$ and
  the statement of the lemma follows. 
\end{proof}
In order to present the proof of Lemma \ref{vl-bl:lemma0}-\ref{basic_lemma:lipschitz3}, we need 
the following sequence of results.
\begin{lemma}
\label{noise:boundC-1}
  Let $\z$ be \textcolor{black}{nonnegative} valued random variable such that for all $0 < s \le \bar{\lambda}$
  its moment generating function is well-defined. Then for any $C > 0$,
  \[ \bE[\z \chi(\z > C)] \le \frac{C\bE[e^{\bar{\lambda} \z}]}{e^{\bar{\lambda} C}-1} \]
\end{lemma}
Note that, as in \eqref{e-}, we use $\chi(\z > C)$ as a shorthand for the indicator function $\chi_{\{\z > C\}}(\z)$. This practices will be used in the sequel.
\begin{proof}[Proof of Lemma \ref{noise:boundC-1}]
  Let us consider $\phi(x)=\frac{e^x-1}{x}$ \textcolor{black}{with the continuous extension \(\phi(0)=1\)}. It then follows
  that $\phi$ is a monotonically increasing function of $x$ for $x > 0$.
  Hence, 
  \[ 
    \frac{\phi(\bar{\lambda} \z)}{\phi(\bar{\lambda}C)} \ge  \chi(\z > C).
  \]
   Indeed, if $\z > C$, then 
   $\frac{\phi(\bar{\lambda} \z)}{\phi(\bar{\lambda}C)} \ge 1 = \chi(\z > C)$. If $\z \le  C$, then 
   $1 \geq \frac{\phi(\bar{\lambda} \z)}{\phi(\bar{\lambda}C)} \ge 0  = \chi(\z > C)$. 
   It then follows that 
   \[ 
      \z  \frac{\phi(\bar{\lambda} \z)}{\phi(\bar{\lambda}C)} \ge \z \chi(\z > C)
   \]
  and by taking expectations it follows that
  \[
     \begin{aligned}
       \bE[\z \chi(\z > C)] \le
       \frac{1}{\phi(\bar{\lambda}C)} \bE[\z \phi(\bar{\lambda} \z)]
     \end{aligned}  
  \]
  Note that
  \[
     \z \phi(\bar{\lambda} \z)=\frac{e^{\bar{\lambda} \z}-1}{\bar{\lambda}}
  \]
  and hence by using the inequality $e^x-1 \le e^{x}$ for $x \ge 0$ it follows that
  \[ 
     \bE[\z \phi(\bar{\lambda} \z)]=\frac{\bE[e^{\bar{\lambda} \z}]-1}{\bar{\lambda}} \le \frac{\bE[e^{\bar{\lambda} \z}]}{\bar{\lambda}}
  \]
  By noticing that
  \[
     \bar{\lambda}\phi(\bar{\lambda}C)=\frac{e^{\bar{\lambda}C}-1}{ C}
  \]
  the statement of the lemma follows. 
  \end{proof}
\begin{lemma}
\label{noise:boundC}
\begin{align*} 
     \bE[e^{s \|e_g(t)\|_2 \chi(\|\e_g(t)\|_2 > C)}]
     \le 
     \left(1+ s 
         \frac{\bar{\Psi}_{\rve_g}(\bar{\lambda})C}{e^{\bar{\lambda}C}-1} + s^2
         \frac{2\bar{\Psi}_{\rve_g}(\bar{\lambda})}{\bar{\lambda}^2}\right)
         :=\bar{\Psi}_{\rve_g}(s,\bar{\lambda},C)
 \end{align*}
 for all $s \in [0,0.5\bar{\lambda})$. 
\end{lemma}
\begin{proof}[Proof Lemma \ref{noise:boundC}]
   First note that 
   \begin{equation}
     \label{noise:boundC:eq1}
     \begin{aligned}
     & \bE[e^{s \|\e_g(t) \|_2 \chi(\|\e_g(t)\|_2 > C)}] =  1 + s  \bE[\|\e_g(t)\|
     _2\chi(\|\e_g(t)\|_2 > C)] + \\
      & \sum_{k=2}^{\infty} \frac{\bE[\|\e_g(t)\|_2^k \chi(\|\e_g(t)\|_2 > C)] s^k]}{k!}.
     \end{aligned}
    \end{equation}
and $\bE[\|\e_g(t)\|_2^k \chi(\|\e_g(t)\|_2 > C)] \le \bE[\|\e_g(t)\|_2^k]$.

As in Lemma~\ref{lemma:subgauss:Cramer} let $\Psi_{\rve_g}(\bar{\lambda})=
\bE[e^{\bar{\lambda}\|\e_g(0)\|_2}]=
\bE[e^{\bar{\lambda}\|\e_g(t)\|_2}]$ be the moment generating function for $\|\e_g(t)\|_2$. Then
\[ 
\Psi_{\rve_g}(\bar{\lambda})
=\sum_{k=0}^{\infty} \frac{\bE[\|\e_g(t)\|_2^k]}{k!} 
\bar{\lambda}^k
\]
and we get 
\[
    \frac{\bE[\|\e_g(t)\|_2^k]}{k!} \le \frac{\Psi_{\rve_g}(\bar{\lambda})}{\bar{\lambda}^k}.
\]
It follows that for all $k \ge 2$,
\[
   \frac{\bE[\|\e_g(t)\|_2^k \chi(\|\e_g(t)\|_2 > C)]s^k}{k!} \le 
   \frac{\bE[\|\e_g(t)\|_2^k]s^k}{k!} \le  
   \Psi_{\rve_g}(\bar{\lambda})\frac{s^k}{\bar{\lambda}^k}.
\]
Therefore, for all $s \in [0,0.5\bar{\lambda})$. 
\begin{equation}
\label{noise:boundC:eq2}
  \begin{aligned}
     & \sum_{k=2}^{\infty} \frac{\bE[\|e_g(t)\|_2^k \chi(\|\e_g(t)\|_2 > C)] s^k]}{k!} \le 
      \sum_{k=2}^{\infty}  \Psi_{\rve_g}(\bar{\lambda})\frac{s^k}{\bar{\lambda}^k}=
     \Psi_{\rve_g}(\bar{\lambda}) \frac{s^2}{\bar{\lambda}^2} \frac{1}{1-\frac{s}{\bar{\lambda}}} 
     = \\
     & \Psi_{\rve_g}(\bar{\lambda}) \frac{s^2}{\bar{\lambda} (\bar{\lambda}-s)} \le 
     \Psi_{\rve_g}(\bar{\lambda}) \frac{s^2}{0.5 \bar{\lambda}^2}=
     2\Psi_{\rve_g}(\bar{\lambda}) \frac{s^2}{\bar{\lambda}^2} 
  \end{aligned}
\end{equation}
where in the last inequality we used that as $s < 0.5\bar{\lambda}$,
$\bar{\lambda}-s \ge 0.5\bar{\lambda}$.
Finally, use Lemma \ref{noise:boundC-1} with $\z=\|\e_g(t)\|_2$ to obtain
 \begin{equation}
 \label{noise:boundC:eq3}
 \bE[\|\e_g(t)\|_2\chi(\|\e_g(t)\|_2 > C)]
 \le \frac{\Psi_{\rve_g}(\bar{\lambda}) C}{e^{\bar{\lambda}C}-1} 
  \end{equation}
The result now follows by combining \eqref{noise:boundC:eq1}, \eqref{noise:boundC:eq2} and 
\eqref{noise:boundC:eq3}, and then use Lemma~\ref{lemma:subgauss:Cramer}. 
\end{proof}
We will also need the following technical lemma
\begin{lemma}
\label{basic_lemma:lipschitz:lemma_aux}
  Let $a_k \ge 0$ be a sequence of nonnegative numbers such that the series $\sum_{k=0}^{\infty} a_k$ is convergent. For any $N$, define
  $c_{k,N}=\sum_{l=\max\{0,k-N+1\}}^k a_l$. Then for every $L\geq 0$,  $\sum_{k=0}^{L} c_{k,N} \le N \sum_{k=0}^{L} a_k$, 
  the series $\sum_{k=0}^{\infty} c_{k,N}$ is convergent, and $\sum_{k=0}^{\infty} c_{k,N} \le N \sum_{k=0}^{\infty} a_k$. 
\end{lemma}
\begin{proof}[Proof of Lemma \ref{basic_lemma:lipschitz:lemma_aux}]
     We will prove by induction on $L \ge N-1$ that
     \begin{equation} 
     \label{basic_lemma:lipschitz:lemma_aux:eq1}
       \sum_{k=0}^{L} c_{k,N} = \sum_{k=0}^{L-N+1} N a_k + \sum_{k=L-N+2}^{L} (L-k+1) a_k 
     \end{equation}
     For $L=N-1$, by induction on $N$, it can be shown that 
     \[ \sum_{k=0}^{N-1} c_{k,N}=\sum_{k=0}^{N-1} \sum_{l=0}^{k} a_l = \sum_{k=0}^{N-1} (N-k) a_k.  \]
     
     which yield \eqref{basic_lemma:lipschitz:lemma_aux:eq1} for $L=N-1$. 
     Assume \eqref{basic_lemma:lipschitz:lemma_aux:eq1} holds for $N-1\leq L'\leq L$. Notice that $c_{L+1,N}=\sum_{l=L-N+2}^{L+1} a_l$ and
     hence
     \[
       \begin{aligned}
         & \sum_{k=0}^{L+1} c_{k,N}=\sum_{k=0}^{L} c_{k,N} + c_{L+1,N}=  \sum_{k=0}^{L-N+1} N a_k + \sum_{k=L-N+2}^{L} (L-k+1) a_k + \sum_{l=L-N+2}^{L} a_l= \\
         & \sum_{k=0}^{L-N+1} N a_k + \sum_{k=L-N+2}^{L} (L-k+2) a_k=\sum_{k=0}^{L-N+1} N a_k + (L-(L-N+2)+2)a_{L-N+2} +  \\
         & \sum_{k=L-N+3}^{L} ((L+1)-k+1) a_k=  
          \sum_{k=0}^{L-N+2} N a_k +\sum_{k=L-N+3}^{L} ((L+1)-k+1) a_k 
       \end{aligned} 
     \]
     The last expression is the right-hand side of 
     \eqref{basic_lemma:lipschitz:lemma_aux:eq1} for $L$ being replaced by $L+1$, thus proving \eqref{basic_lemma:lipschitz:lemma_aux:eq1}.
    Note that for $k=L-N+2,\ldots, L$ we have $L-k+1 \le N$, and hence 
     \eqref{basic_lemma:lipschitz:lemma_aux:eq1} implies that $\sum_{k=0}^{L} c_{k,N} \le N \sum_{k=0}^{L} a_k$. Since 
     $a_k$ and $c_k$ are nonnegative, and $\sum_{k=0}^{\infty} a_k$ is convergent, the rest of the lemma follows.
\end{proof}

\begin{lemma}
\label{basic_lemma:lipschitz}
Let $\mathbf{I}$ be an index set, and $\{\beta_{i,t}(k)\}_{k=0}^{\infty}$, $i \in \mathbf{I}$, be a sequence of matrices such that \textcolor{black}{there exists a sequence $\{\bar{\beta}(k)\}_{k=0}^{\infty}$ with $\|\bar{\beta}\|_{\ell_1}< \infty$ and $\bar{\beta}(k) \ge \sup_{i \in \mathbf{I},t\geq 0} \|\beta_{i,t}(k)\|_2$}. Let $\rvz_{i}(t)=(\beta_{i,t} \star \e_g)(t)$ and
$\bar{\rvz}_i(t)=(\beta_{i,t} \star \bar{\e}_g)(t)$.
Then for $\lambda < 0.5N$
 \begin{align*}
 & \bE[e^{\frac{\lambda}{N} \sup_{i \in \mathbf{I}}\sum_{t=0}^{N-1} \|\rvz_i(t)-\bar{\rvz}_i(t)\|_2}]
  \le \prod_{k=0}^{\infty} 
  \bar{\Psi}_{\rve_g}\left(\frac{c(k) \lambda}{N},\|\bar{\beta}\|_{\ell_1},C\right) \le 
   e^{\mathcal{C}(\|\bar{\beta}\|_{\ell_1},\lambda,N,C)}  \\
   & \mathcal{C}(s,\lambda,N,C)=\lambda \frac{\bar{\Psi}_{\rve_g}(s)\textcolor{black}{s}C}{(e^{sC}-1)}+
     \frac{\lambda^2}{N} 2\bar{\Psi}_{\rve_g}( s)
 \end{align*} 
 where 
 $c(k)=\sum_{l=\max\{0,k-N+1\}}^{k} \|\bar{\beta}(l)\|_2$
 and the infinite product on the right-hand side is well-defined. 
\end{lemma}
\begin{proof}[Proof of Lemma \ref{basic_lemma:lipschitz}]
  Note that 
  \begin{align}
    &  c(k) 
    \le \|\bar{\beta}\|_{\ell_1}
      \label{basic_lemma:lipschitz:eq0.2} 
   \end{align}
    and by Lemma \ref{basic_lemma:lipschitz:lemma_aux}
    \begin{equation}
      \label{basic_lemma:lipschitz:eq1}
      \begin{aligned}
     & \sum_{k=0}^{\infty} c(k) \le N \sum_{k=0}^{\infty} \|\bar{\beta}(k)\|_2 = N \|\bar{\beta}\|_{\ell_1} \\
      \end{aligned}
    \end{equation}
    Note now that 
      \begin{align*}
      & \sum_{t=0}^{N-1}  \|\rvz_{i}(t)-\bar{\rvz}_{i}(t)\|_2  \\
      & \le
      \sum_{t=0}^{N-1}  \sum_{k=0}^{\infty} \|\beta_{i,t}(k)\|_2 \|\e_g(t-k)-\bar{\e}_g(t-k)\|_2 \\
      & \le \sum_{t=0}^{N-1}  \sum_{k=0}^{\infty} \|\bar{\beta}(k)\|_2 \|\e_g(t-k)\chi(\|\e_g(t-k)\|_2 > C)\|_2 \\
      & =\sum_{k=0}^{\infty} c(k) \|\e_g(N-k-1) \chi(\|\e_g(N-k-1)\|_2 > C)\|_2 
        \end{align*}    
        Hence,
        \begin{align*}
      &\bE[ e^{\frac{\lambda}{N} \sup_{i \in \mathbf{I}} \sum_{t=1}^{N}  \|\rvz_{i}(t)-\bar{\rvz}_{i}(t)\|_2}]
      \le \\
      & \le \bE[e^{\frac{\lambda}{N} \sum_{k=0}^{\infty} 
      c(k) \|\e_g(N-k-1) \chi(\|\e_g(N-k-1)\|_2 > C)\|_2}] \\
      & = \lim_{L \rightarrow \infty} 
      \prod_{k=0}^{L} \bar{\Psi}_{\rve_g}\left(\frac{c(k) \lambda}{N},\|\bar{\beta}\|_{\ell_1},C\right)
        \end{align*}
        where we applied  Lemma \ref{noise:boundC}
        to $\bE[e^{\frac{\lambda}{N} c(k)\|e_g(N-k-1) \chi(\|\e_g(N-k-1)\|_2 > C)\|_2}]$
        with $\bar{\lambda}=\|\bar{\beta}\|_{\ell_1}$ and $s=\frac{c(k)\lambda}{N} < 0.5\|\bar{\beta}\|_{\ell_1}$ since $\lambda<0.5N$ by assumption. \textcolor{black}{If \(\|\bar\beta\|_{\ell^1}=0\), the result is trivial. Hence assume \(\|\bar\beta\|_{\ell^1}>0\).}

        Note that 
        \[
        \begin{aligned}
        &  d_L=\ln\left(\prod_{k=0}^{L} \bar{\Psi}_{\rve_g}\left(\frac{c(k) \lambda}{N},\|\bar{\beta}\|_{\ell_1},C\right) \right)=\sum_{k=0}^{L} r_k \\
        & r_k:=\ln \left(1+\frac{c(k) \lambda}{N} \frac{\bar{\Psi}_{\rve_g}(\|\bar{\beta}\|_{\ell_1})
           \textcolor{black}{C}}{e^{\|\bar{\beta}\|_{\ell_1}C}-1} + \left(\frac{c(k) \lambda}{N}\right)^2 \frac{2\bar{\Psi}_{\rve_g}(\|\bar{\beta}\|_{\ell_1})}{\|\bar{\beta}\|_{\ell_1}^2} \right) \geq 0
        \end{aligned}
        \]
        We show that the series $\sum_{k=0}^{\infty} r_k$ 
        is bounded, more precisely, 
        \[
        \sum_{k=0}^{L} r_k \le 
      \|\bar{\beta}\|_{\ell_1} \lambda \frac{\bar{\Psi}_{\rve_g}(\|\bar{\beta}\|_{\ell_1})\|\textcolor{black}{C}}{e^{\|\bar{\beta}\|_{\ell_1}C}-1} + \frac{\lambda^2}{N} (\|\bar{\beta}\|_{\ell_1})^2 \frac{2\bar{\Psi}_{\rve_g}(\|\bar{\beta}\|_{\ell_1})}{\|\bar{\beta}\|_{\ell_1}^2} = \mathcal{C} (\|\bar{\beta}\|_{\ell_1},\lambda,N,C)
        \]
        From this it follows that $d_L$ is convergent, and from
        \[
      \prod_{k=0}^{ L} \bar{\Psi}_{\rve_g}\left(\frac{c(k) \lambda}{N},\|\bar{\beta}\|_{\ell_1},C\right) = e^{d_L}
        \]
        it follows that  
        \begin{equation}
       \label{basic_lemma:lipschitz:eq1.3}
       \begin{aligned}
        & \lim_{L \rightarrow \infty} 
        \prod_{k=0}^{ L} \bar{\Psi}_{\rve_g}\left(\frac{c(k) \lambda}{N},\|\bar{\beta}\|_{\ell_1},C\right) \le  e^{\mathcal{C}(\|\bar{\beta}\|_{\ell_1},\lambda,N,C)}
       \end{aligned}
        \end{equation}
        i.e., the infinite product is well-defined, and the statement of the lemma holds. 

        In order to show \eqref{basic_lemma:lipschitz:eq1.3}, note that using $\ln(1+x) \le x$, it follows that 
        \[
        \begin{aligned}
      & r_k \le \frac{c(k) \lambda}{N} \frac{\bar{\Psi}_{\rve_g}(\|\bar{\beta}\|_{\ell_1})\|\textcolor{black}{C}}{e^{\|\bar{\beta}\|_{\ell_1}C}-1} + \left(\frac{c(k) \lambda}{N}\right)^2 \frac{2\bar{\Psi}_{\rve_g}(\|\bar{\beta}\|_{\ell_1})}{\|\bar{\beta}\|_{\ell_1}^2}
        \end{aligned}
        \]
        Hence,
        \[
        \begin{aligned}
      & \sum_{k=0}^{L} r_k  \le \frac{\lambda \sum_{k=0}^{L}c(k)}{N} \frac{\bar{\Psi}_{\rve_g}(\|\bar{\beta}\|_{\ell_1})\|\textcolor{black}{C}}{e^{\|\bar{\beta}\|_{\ell_1}C}-1} +  \textcolor{black}{\lambda^2} \left(\frac{\sum_{k=0}^{L} (c(k))^2}{N^2}\right) \frac{2\bar{\Psi}_{\rve_g}(\|\bar{\beta}\|_{\ell_1})}{\|\bar{\beta}\|_{\ell_1}^2}
        \end{aligned}
        \]
        
        From \eqref{basic_lemma:lipschitz:eq0.2}
        it follows that
        $c(k) \le \|\bar{\beta}\|_{\ell_1}$.
        Hence
        \[
         \left(\frac{1}{N} c(k)\right)^2 \le \frac{1}{N^2} c(k)  \|\bar{\beta}\|_{\ell_1}
        \]
        and as from \eqref{basic_lemma:lipschitz:eq1} it follows that $\sum_{k=0}^{L} c(k) \le N \|\bar{\beta}\|_{\ell_1}$, we get that 
        %
        \begin{align*}
        & \sum_{k=0}^{L} \left(\frac{1}{N} c(k)\right)^2 \le 
        \frac{1}{N^2} \sum_{k=0}^{L}  c(k)
        \|\bar{\beta}\|_{\ell_1} \le 
        \frac{1}{N} \|\bar{\beta}\|_{\ell_1}^2
      \end{align*}
      \begin{color}{black}
        Hence, \[
        \begin{aligned}
      & \sum_{k=0}^{L} r_k  \le \frac{\lambda \sum_{k=0}^{L}c(k)}{N} \frac{\bar{\Psi}_{\rve_g}(\|\bar{\beta}\|_{\ell_1})\|\textcolor{black}{C}}{e^{\|\bar{\beta}\|_{\ell_1}C}-1} +  \lambda^2\left(\frac{\sum_{k=0}^{L} (c(k))^2}{N^2}\right) \frac{2\bar{\Psi}_{\rve_g}(\|\bar{\beta}\|_{\ell_1})}{\|\bar{\beta}\|_{\ell_1}^2} \le \\
      & \frac{\lambda N \|\bar{\beta}\|_{\ell_1}}{N} \frac{\bar{\Psi}_{\rve_g}(\|\bar{\beta}\|_{\ell_1})\|C}{e^{\|\bar{\beta}\|_{\ell_1}C}-1} +  \lambda^2 \frac{1}{N} \|\bar{\beta}\|_{\ell_1}^2 \frac{2\bar{\Psi}_{\rve_g}(\|\bar{\beta}\|_{\ell_1})}{\|\bar{\beta}\|_{\ell_1}^2} = 
      \lambda \|\bar{\beta}\|_{\ell_1} \frac{\bar{\Psi}_{\rve_g}(\|\bar{\beta}\|_{\ell_1})\|C}{e^{\|\bar{\beta}\|_{\ell_1}C}-1} +  \lambda^2 \frac{2\bar{\Psi}_{\rve_g}(\|\bar{\beta}\|_{\ell_1})}{N}
        \end{aligned}
        \]

      \end{color}
      This completes the proof.
    \end{proof}
Lemma \ref{basic_lemma:lipschitz} provides the basis for several
technical results. 

\begin{proof}[Proof of Lemma \ref{basic_lemma:lipschitz2}]
  From Lemma \ref{alpha:lemma2} it follows that 
  $\y(t)=(\alpha_{g,y} \star \e_g)(t)$ and 
  $\bar{\y}(t)=(\alpha_{g,y} \star \bar{\e}_g)(t)$.
  By setting $I=\{g\}$, $\beta_{i,t}=\alpha_{g,y}$, $i=g$ and \textcolor{black}{taking 
  $\bar{\beta}(k)=\|\alpha_g(k)\|_2 \ge \|\alpha_{g,y}(k)\|_2$, }
  and using Lemma \ref{basic_lemma:lipschitz}, the first inequality follows.

  Similarly, from Lemma \ref{alpha:lemma2} it follows that 
  $\hat{\y}(t)=(c_{\theta} \star \e_g)(t)$ and 
  $\hat{\bar{\y}}(t)=(c_{\theta} \star \bar{\e}_g)(t)$.
  Note that by definition 
  $c_{\theta}(k)=\sum_{l=0}^{k} \alpha_{\theta}(l) \alpha_{g,w}(k-l)$ and hence 
  \[
     \|c_{\theta}(k)\|_2 \le \sum_{l=0}^{k} \|\alpha_\theta(l)\|_2 \|\alpha_{g,w}(k-l)\|_2
    \textcolor{black}{\le \sum_{l=0}^{k} \bar{\alpha}(l) \|\alpha_{g}(k-l)\|_2}
  \]
  \begin{color}{black}
  Set
  \[
      \bar{\beta}(k)=\sum_{l=0}^{k} \bar{\alpha}(l) \textcolor{black}{\|\alpha_{g}(k-l)\|_2}
  \]
  from which by the well-known property of Cauchy products of series (the infinite sum of the Cauchy product of two absolutely convergent series is equal to the product of their infinite sums) it follows
  that
  \[ \|\bar{\beta}\|_{\ell_1}=\|\bar{\alpha}\|_{\ell_1} \textcolor{black}{\|\alpha_{g}\|_{\ell_1}}.
  \]
  Moreover, $\bar{\beta}(k) \ge \sup_{\theta \in \Theta}\|c_{\theta}(k)\|_2$. 
  By setting $I=\Theta$, $\beta_{i,t}=c_{\theta}$, $i=\theta \in \Theta$, and using
  $\bar{\beta}(k)=\sum_{l=0}^{k} \bar{\alpha}(l) \|\alpha_{g}(k-l)\|_2$, 
  from Lemma \ref{basic_lemma:lipschitz} it then  follows that 
  \[
   \bE[e^{\frac{\lambda}{N} \sup_{\theta \in \Theta}\sum_{t=0}^{N-1} \|\hat{\y}_\theta(t)-\hat{\bar{\y}}_\theta(t)\|_2}] \le e^{\mathcal{C}(\|\bar{\beta}\|_{\ell_1},\lambda,N,C)}=
       e^{\mathcal{C}(\|\bar{\alpha}\|_{\ell_1} \|\alpha_{g}\|_{\ell_1},\lambda,N,C)}
  \]
  \end{color}


 Finally, the proof of the third inequality is similar to the second one. 
 By Lemma \ref{alpha:lemma2} it follows that 
 $\hat{\y}_\theta(t\mid 0)=(c_{\theta,t} \star \e_g)(t)$ and 
 $\hat{\bar{\y}}_\theta(t\mid 0)=(c_{\theta,t} \star \bar{\e}_g)(t)$.
 \textcolor{black}{Notice that}
  \begin{align*}
     & \|c_{\theta,t}(k)\|_2=\|\sum_{l=0}^{\min\{k,t\}} \alpha_{\theta}(l)\alpha_{g,w}(k-l)\|_2 \le \sum_{l=0}^{\min\{k,t\}} \|\alpha_{\theta}(l)\|_2 \|\alpha_{g,w}(k-l)\|_2  \\
     & \le \sum_{l=0}^{k} \bar{\alpha}(l) \|\alpha_{g,w}(k-l)\|_2
     \le \textcolor{black}{\sum_{l=0}^{k} \bar{\alpha}(l) \|\alpha_{g}(k-l)\|_2}=\bar{\beta}(k)
  \end{align*}
  \begin{color}{black}
  Notice that 
  \[\bar{\beta}(k) \ge \sup_{\theta \in \Theta, t \ge 0} \|c_{\theta,t}(k)\|_2, \]  
  and $\|\bar{\beta}\|_{\ell_1}=\|\bar{\alpha}\|_{\ell_1} \|\alpha_{g}\|_{\ell_1}$.
   By applying Lemma \ref{basic_lemma:lipschitz} to 
 $I=\Theta$, $\beta_{i,t}=c_{\theta,t}$  and choosing $\bar{\beta}$ as above,  it follows that 
 \[ \bE[e^{\frac{\lambda}{N} \sup_{\theta \in \Theta}\sum_{t=0}^{N-1} \|\hat{\y}_\theta(t \mid 0)-\hat{\bar{\y}}_\theta(t \mid 0)\|_2}] \le e^{\mathcal{C}(\|\bar{\beta}\|_{\ell_1},\lambda,N,C)}=
   e^{\mathcal{C}(\|\bar{\alpha}\|_{\ell_1} \|\alpha_{g}\|_{\ell_1},\lambda,N,C)}
  \]
  \end{color}
\end{proof}

\begin{proof}[Proof of Lemma \ref{vl-bl:lemma0}]
   Notice that  
\begin{align*}
   &  |\hat{\mathcal{L}}_N(\theta) - \hat{\bar{\mathcal{L}}}_N(\theta)| \le 
    \frac{1}{N} \sum_{t=0}^{N-1}|\ell(\y(t),\hat{\rvy}_{\theta}(t \mid 0) )
    - \ell(\bar{\rvy}(t),\hat{\bar{\rvy}}_{\theta}(t \mid 0))|
  \le \\
  & \frac{L_{\ell}}{N} \sum_{t=0}^{N-1}
    \|\y(t)-\bar{\y}(t)\|_2+\frac{L_{\ell}}{N} \sum_{t=0}^{N-1} \|\hat{\rvy}_{\theta}(t \mid 0)-\hat{\bar{\rvy}}_{\theta}(t \mid 0)\|_2
\end{align*}
and hence 
\begin{align*}
\sup_{\theta \in \Theta} |\hat{\mathcal{L}}_N(\theta) - \hat{\bar{\mathcal{L}}}_N(\theta)| \le 
 \frac{L_{\ell}}{N} \sum_{t=0}^{N-1}
    (\|\y(t)-\bar{\y}(t)\|_2+\frac{L_{\ell}}{N}  \sup_{\theta \in \Theta} \sum_{t=0}^{N-1} \|\hat{\rvy}_{\theta}(t \mid 0)-\hat{\bar{\rvy}}_{\theta}(t \mid 0)\|_2)
\end{align*}
Then \begin{align*}
   & \bE[e^{\lambda \sup_{\theta \in \Theta} |\hat{\mathcal{L}}_N(\theta) - \hat{\bar{\mathcal{L}}}_N(\theta)|}]
   \le \bE[e^{\frac{\lambda L_{\ell}}{N} \sum_{t=0}^{N-1}
    \|\y(t)-\bar{\y}(t)\|_2} e^{\frac{\lambda L_{\ell}}{N} \sup_{\theta \in \Theta}\sum_{t=0}^{N-1} \|\hat{\rvy}_{\theta}(t \mid 0)-\hat{\bar{\rvy}}_{\theta}(t \mid 0)\|_2}]
    \le \\
    & \sqrt{\bE[e^{\frac{2 \lambda L_{\ell}}{N} \sum_{t=0}^{N-1}
    (\|\y(t)-\bar{\y}(t)\|_2}]}
    \sqrt{\bE[e^{\frac{2\lambda L_{\ell}}{N} \sup_{\theta \in \Theta} 
    \sum_{t=0}^{N-1} \|\hat{\rvy}_{\theta}(t \mid 0)-\hat{\bar{\rvy}}_{\theta}(t \mid 0)\|_2}]}
\end{align*}
from which \eqref{vl-vbl:lemma1:pf:eq0} follows.
Similarly, 
\begin{align*}
   &  |V_N(\theta) - \bar{V}_N(\theta)| \le 
    \frac{1}{N} \sum_{t=0}^{N-1}|\ell(\y(t),\hat{\rvy}_{\theta}(t)) 
    - \ell(\bar{\rvy}(t),\hat{\bar{\rvy}}_{\theta}(t))|
  \le \\
  & \frac{L_{\ell}}{N} \sum_{t=0}^{N-1}
    \|\y(t)-\bar{\y}(t)\|_2+\frac{L_{\ell}}{N} \sum_{t=0}^{N-1} \|\hat{\rvy}_{\theta}(t)-\hat{\bar{\rvy}}_{\theta}(t)\|_2
\end{align*}
and hence
\begin{align*}
  & \sup_{\theta \in \Theta} |V_N(\theta) - \bar{V}_N(\theta)| \le 
    \frac{L_{\ell}}{N} \sum_{t=0}^{N-1}
    (\|\y(t)-\bar{\y}(t)\|_2+\frac{L_{\ell}}{N} \sup_{\theta \in \Theta} \sum_{t=0}^{N-1} \|\hat{\rvy}_{\theta}(t)-\hat{\bar{\rvy}}_{\theta}(t)\|_2)
\end{align*}
so that
\begin{align*}
   & \bE[e^{\lambda \sup_{\theta \in \Theta}|V_N(\theta) - \bar{V}_N(\theta)|}]
   \le \bE[e^{\frac{\lambda L_{\ell}}{N} \sum_{t=0}^{N-1}
    \|\y(t)-\bar{\y}(t)\|_2}e^{\frac{\lambda L_{\ell}}{N} \sup_{\theta \in \Theta}\sum_{t=0}^{N-1} \|\hat{\rvy}_{\theta}(t)-\hat{\bar{\rvy}}_{\theta}(t)\|_2}]
    \le  \\
    & \sqrt{\bE[e^{\frac{2 \lambda L_{\ell}}{N} \sum_{t=0}^{N-1}
    (\|\y(t)-\bar{\y}(t)\|_2}]}\sqrt{\bE[e^{\frac{2 \lambda L_{\ell}}{N} \sup_{\theta \in \Theta} \sum_{t=0}^{N-1} \|\hat{\rvy}_{\theta}(t)-\hat{\bar{\rvy}}_{\theta}(t)\|_2}]}
\end{align*}
from which \eqref{vl-vbl:lemma1:pf:eq3} follows.
\end{proof}

\begin{proof}[Proof of Lemma \ref{vl-vbl:lemma1}]
From Lemma \ref{basic_lemma:lipschitz2} it follows that
\begin{equation} 
\label{vl-vbl:lemma1:pf:eq1}
\begin{aligned}
    & \sqrt{\bE[e^{\frac{2 \lambda L_{\ell}}{N} \sum_{t=0}^{N-1}
    \|\y(t)-\bar{\y}(t)\|_2}]} \le e^{\frac{1}{2} \mathcal{C}(\textcolor{black}{\|\textcolor{black}{\alpha_{g}}\|_{\ell_1}}, 2L_{\ell}\lambda,N,C)}\\
\end{aligned}    
\end{equation}
 and 
 \begin{equation}
\label{vl-vbl:lemma1:pf:eq2}
\begin{aligned}
& \sqrt{\bE[e^{\frac{2\lambda L_{\ell}}{N} \sup_{\theta \in \Theta}\sum_{t=0}^{N-1} \|\hat{\rvy}_{\theta}(t \mid 0)-\hat{\bar{\rvy}}_{\theta}(t \mid t_0)}]} \le  e^{\frac{1}{2} \mathcal{C}(\|\bar{\alpha}\|_{\ell_1} \|\textcolor{black}{\alpha_{g}}\|_{\ell_1}, 2L_{\ell}\lambda,N,C)}
\end{aligned}
 \end{equation}
 from which by using \eqref{vl-vbl:lemma1:pf:eq0} it follows that 
 \begin{align*}
    & \bE[e^{\lambda \sup_{\theta \in \Theta} |\hat{\mathcal{L}}_N(\theta) - \hat{\bar{\mathcal{L}}}_N(\theta)|}]
    \le   e^{\frac{1}{2} (\mathcal{C}(\|\textcolor{black}{\alpha_{g}}\|_{\ell_1}, 2L_{\ell}\lambda,N,C) + \mathcal{C}(\|\bar{\alpha}\|_{\ell_1} \|\textcolor{black}{\alpha_{g}}\|_{\ell_1}, 2L_{\ell}\lambda,N,C))}  
 \end{align*}   
Likewise,from Lemma \ref{basic_lemma:lipschitz2} it follows that 
\begin{align*}
    & \sqrt{\bE[e^{\frac{2 \lambda L_{\ell}}{N} \sup_{\theta \in \Theta}\sum_{t=0}^{N-1} \|\hat{\rvy}_{\theta}(t)-\hat{\bar{\rvy}}_{\theta}(t)\|_2}]} \le  e^{\frac{1}{2} \mathcal{C}(\|\bar{\alpha}\|_{\ell_1} \|\textcolor{black}{\alpha_{g}}\|_{\ell_1}, 2L_{\ell}\lambda,N,C)} 
\end{align*}
By combining this with \eqref{vl-vbl:lemma1:pf:eq1}, it follows by \eqref{vl-vbl:lemma1:pf:eq3} that 
\begin{equation}
\label{vl-vbl:lemma1:pf:eq4}
\begin{aligned}
& \bE[e^{\lambda \sup_{\theta \in \Theta}|V_N(\theta) - \bar{V}_N(\theta)|}] \le 
e^{\frac{1}{2} (\mathcal{C}(\|\textcolor{black}{\alpha_{g}}\|_{\ell_1}, 2L_{\ell}\lambda,N,C) + \mathcal{C}(\|\bar{\alpha}\|_{\ell_1}
 \|\textcolor{black}{\alpha_{g}}\|_{\ell_1}, 2L_{\ell}\lambda,N,C))}  
\end{aligned}
\end{equation}
Finally, as $\mathcal{L}(\theta)=\bE[V_N(\theta)]$ 
and $\bar{\mathcal{L}}(\theta)=\bE[\bar{V}_N(\theta)]$, it follows that
\begin{align*}
  \sup_{\theta \in \Theta} |\mathcal{L}(\theta) - \bar{\mathcal{L}}(\theta)| &= 
   \sup_{\theta \in \Theta} |\bE[V_N(\theta) - \bar{V}_N(\theta)]| \\
   &\le \sup_{\theta \in \Theta} \bE[|V_N(\theta) - \bar{V}_N(\theta)|] \le    \bE[\sup_{\theta \in \Theta} |V_N(\theta) - \bar{V}_N(\theta)|]
\end{align*}
and hence
\begin{equation}
  \label{vl-bl:lemma1:pf:eq5}
   e^{\lambda \sup_{\theta \in \Theta} |\mathcal{L}(\theta) - \bar{\mathcal{L}}(\theta)|} \le e^{\lambda \bE[\sup_{\theta \in \Theta} |V_N(\theta) - \bar{V}_N(\theta)|]}
\end{equation}
As the function 
$\phi:x \mapsto e^{\lambda x}$ is convex, by Jensen's inequality it follows that 
\begin{align*}
& e^{\lambda \bE[\sup_{\theta \in \Theta} |V_N(\theta) - \bar{V}_N(\theta)|]}=
\phi(\bE[\sup_{\theta \in \Theta} |V_N(\theta) - \bar{V}_N(\theta)|]) \le \\
&  \bE[\phi(\sup_{\theta \in \Theta} |V_N(\theta) - \bar{V}_N(\theta)|)]
   = \bE[e^{\lambda \sup_{\theta \in \Theta} |V_N(\theta) - \bar{V}_N(\theta)|}]
\end{align*}
and then the statement of the lemma follows from 
\eqref{vl-vbl:lemma1:pf:eq4} and \eqref{vl-bl:lemma1:pf:eq5}.
\end{proof}

\begin{proof}[Proof of Lemma \ref{basic_lemma:lipschitz3.1}]
   Note that
 \[
    \begin{aligned}
     & \epsilon (V_N(\theta)-\hat{\mathcal{L}}_N(\theta)) =
     \epsilon(\hat{\bar{\mathcal{L}}}_N(\theta)-\hat{\mathcal{L}}_N(\theta)) +\epsilon(\bar{V}_N(\theta)-\hat{\bar{\mathcal{L}}}_N(\theta))
     +\epsilon    (V_N(\theta)-\bar{V}_N(\theta)) \\
     & \le \sup_{\theta \in \Theta} |\hat{\mathcal{L}}_N(\theta)-\hat{\bar{\mathcal{L}}}_N(\theta)| 
     +\epsilon (\bar{V}_N(\theta)-\hat{\bar{\mathcal{L}}}_N(\theta))
     +\sup_{\theta \in \Theta} |\bar{V}_N(\theta)-V_N(\theta)| 
    \end{aligned}  
 \]
Hence, by using monotonicity of $x \mapsto e^{x}$ and 
Cauchy-Schwarz inequality for expectations,
    \begin{align*}
        & \bE[e^{\lambda\epsilon (V_N(\theta)-\hat{\mathcal{L}}_N(\theta))}]   
 \le 
      \bE[e^{\lambda \sup_{\theta \in \Theta} |\bar{V}_N(\theta)-V_N(\theta)|} e^{\lambda \sup_{\theta \in \Theta} |\hat{\bar{\mathcal{L}}}_N(\theta)-\hat{\mathcal{L}}_N(\theta)|} e^{\lambda \epsilon(\bar{V}_N(\theta)-\hat{\bar{\mathcal{L}}}_N(\theta))}] 
      \nonumber \\
     & \le \sqrt{\bE[e^{2\lambda\epsilon (\bar{V}_N(\theta)-\hat{\bar{\mathcal{L}}}_N(\theta))}]}
     \sqrt{\bE[e^{2\lambda \sup_{\theta \in \Theta} |\bar{V}_N(\theta)-V_N(\theta)|} e^{2\lambda \sup_{\theta \in \Theta}  |\hat{\bar{\mathcal{L}}}_N(\theta)-\hat{\mathcal{L}}_N(\theta)|}]} \nonumber \\
     & \le \sqrt{\bE[e^{2\lambda\epsilon (\bar{V}_N(\theta)-\hat{\bar{\mathcal{L}}}_N(\theta))}]}
     \sqrt{\sqrt{\bE[e^{4\lambda \sup_{\theta \in \Theta} |\bar{V}_N(\theta)-V_N(\theta)|}]} \sqrt{\bE[e^{4\lambda \sup_{\theta \in \Theta} |\hat{\bar{\mathcal{L}}}_N(\theta)-\hat{\mathcal{L}}_N(\theta)|}]}}
    \end{align*}
    proving \eqref{eq:basic_lemma_lipschitz3:pf:eq0}. A similar argument proves \eqref{eq:basic_lemma_lipschitz3:pf:eq0.1}.  
\end{proof}
\begin{proof}[Proof of Lemma \ref{basic_lemma:lipschitz3}]
 
 From Lemma~\ref{vl-vbl:lemma1} 
 it follows that 
 \[
   \begin{aligned}
      & \sqrt{\bE[e^{4\lambda \sup_{\theta \in \Theta} |\bar{V}_N(\theta)-V_N(\theta)|}]} \le e^{\frac{1}{4}(\mathcal{C}(\|\bar{\alpha}\|_{\ell_1}\|\textcolor{black}{\alpha_{g}}\|_{\ell_1}, 8L_{\ell}\lambda,N,C)+\mathcal{C}(\|\textcolor{black}{\alpha_{g}}\|_{\ell_1}, 8L_{\ell}\lambda,N,C))} \\
      &  \sqrt{\textcolor{black}{\bE}[ e^{4\lambda \sup_{\theta \in \Theta} |\hat{\bar{\mathcal{L}}}_N(\theta)-\hat{\mathcal{L}}_N(\theta)|}]} \le e^{\frac{1}{4}(\mathcal{C}(\|\bar{\alpha}\|_{\ell_1}\|\textcolor{black}{\alpha_{g}}\|_{\ell_1}, 8L_{\ell}\lambda,N,C)+\mathcal{C}(\|\textcolor{black}{\alpha_{g}}\|_{\ell_1}, 8L_{\ell}\lambda,N,C))}
    \end{aligned}    
 \]
 and hence 
\begin{multline*} 
\sqrt{\sqrt{\bE[e^{4\lambda\sup_{\theta \in \Theta} |\bar{V}_N(\theta)-V_N(\theta)|}]} \sqrt{\bE[e^{4\lambda\sup_{\theta \in \Theta} |\hat{\bar{\mathcal{L}}}_N(\theta)-\hat{\mathcal{L}}_N(\theta)|}]}}\\ 
\le e^{\frac{1}{4}(\mathcal{C}(\|\bar{\alpha}\|_{\ell_1}\|\textcolor{black}{\alpha_{g}}\|_{\ell_1}, 8L_{\ell}\lambda,N,C)+\mathcal{C}(\|\textcolor{black}{\alpha_{g}}\|_{\ell_1}, 8L_{\ell}\lambda,N,C))}
\end{multline*}
By using \eqref{eq:basic_lemma_lipschitz3:pf:eq0} the first statement of the lemma follows.
The proof of the second claim is similar. 
\end{proof}

\subsection{Proof of Theorem \ref{thm:main}: quadratic loss}
\label{thm:main:quadratic}
  The core of the proof is to show that the prediction errors
  $\|\y(t)-\hat{\y}_{\theta}(t)\|_2$ and $\|\y(t) - \hat{\y}_{\theta}(t\mid 0)\|_2$ 
  are sub-Gaussian and hence
  they are bounded with high probability as follows.
  \begin{lemma}
  \label{lem:quadratic1}
  For $N \ge \textcolor{black}{2}$ and $\delta \in (0,0.5)$
  \begin{equation}
\label{lem:quadratic1:eq1}
\begin{aligned}
     & \bP\left(\forall \theta \in \Theta,~  \forall t=0,\ldots, N-1: \quad 
     \right. \\
     & \left. 
          \max\{\|\y(t)-\hat{\y}_{\theta}(t)\|_2, \|\y(t)-\hat{\y}_{\theta}(t\mid 0)\|_2\} \le 6G_e (\|\e_g\|_{\psi_2})\sqrt{\ln\left(\frac{3N^2}{2\delta}\right)} \right) \ge 1-\frac{2\delta}{3}
\end{aligned}
  \end{equation}
  and 
  \begin{equation}
\label{lem:quadratic1:eq2}
   \bE[\|\y(t)-\hat{\y}_{\theta}(t)\|_2^2 \chi(\|\y(t)-\hat{\y}_{\theta}(t)\|_2^2 > (6G_e\|\e_g\|_{\psi_2})^2\ln\left(\frac{3N^2}{2\delta}\right))] \le 
   \frac{2 \cdot 6^2 G_e^2 \|\e_g\|_{\psi_2}^2\ln(\frac{3N^2}{2\delta})}{N}
\end{equation}
\end{lemma}
\textcolor{black}{We replace the quadratic loss by the globally Lipschitz truncated loss \(\ell_C\), which coincides with the quadratic loss whenever the prediction error is bounded by \(C\).} More precisely, let us define the loss function $\ell_C$ as follows:
\[ \ell_C(y, \hat{y}) = \min\{\|y-\hat{y}\|_2^2, C^2\} \]
\begin{lemma}
\label{lem:quadratic2}
For any $C > 0$, $\ell_C$ is $2C$-Lipschitz.
\end{lemma}
The proof then reduces the quadratic case to the case of Lipschitz loss $\ell_C$
for a suitable $C$ as follows.
\begin{proof}[Proof of Theorem \ref{thm:main}:quadratic case]
      Let us set $C=6G_e\|\e_g\|_{\psi_2}\sqrt{\ln(\frac{3N^2}{2\delta})}$ and let us consider the 
      true and empirical errores for the loss function $\ell_C$:
      \begin{equation}
        \label{eq:initial:tildern1:quadratic0} 
         {}^C\hat{\mathcal{L}}_{N}(\theta) = \frac{1}{N} \sum_{t=0}^{N-1} \ell_C(y(t), \hat{y}_{\theta}(t \mid 0)), \quad 
         {}^C \mathcal{L}(\theta) = \bE[\ell_C(y(t), \hat{y}_{\theta}(t))]
      \end{equation}
      By applying Theorem \ref{thm:main} for the case of Lipschitz loss function $\ell_C$  and
      with $\frac{2\delta}{3}$ instead of $\delta$, it follows
      that for \textcolor{black}{any fixed} $\epsilon \in \{-1,1\}$,
      for all $0 < \lambda  < \frac{N}{\textcolor{black}{32} \cdot 2C}$
        \begin{equation}
        \label{eq:initial:tildern1:quadratic} 
    \begin{aligned}
         \forall \rho & \in \mathcal{M}_\pi: \epsilon E_{\theta\sim\rho} {}^C\mathcal{L}(\theta)
         \leq   \epsilon E_{\theta\sim\rho}{}^C\hat{\mathcal{L}}_N(\theta) +  
          \frac{1}{\lambda}\left ( \KL(\rho\|\pi) + \ln\left(\dfrac{3}{2\delta}\right) + \right. \\
          &  \left. \frac{1}{2} 
          \left(\ln E_{\theta\sim\pi} \exp(\mathcal{\bar{C}}_1(2\lambda,\theta,N,\frac{2\delta}{3},\epsilon))+
          \ln E_{\theta \sim \pi} \exp(\mathcal{\bar{C}}_2(2\lambda,\theta,N,\frac{2\delta}{3},\epsilon))\right)
           \right )
    \end{aligned}
    \end{equation}
    holds with probability at least $1-\frac{4\delta}{3}$, where  for $i=1,2$.
    \begin{equation}
    \label{thm:main:quadratic:eq2}
       \bar{C}_i(\lambda,\theta,N,\textcolor{black}{\frac{2\delta}{3}},\epsilon)= 
       \mathcal{C}_{ub,i}(\lambda,\theta,12G_e\|\e_g\|_{\psi_2}\sqrt{\ln\left(\frac{3N^2}{2\delta}\right)},N)
    \end{equation}
    By using Lemma \ref{lem:quadratic1}, \eqref{eq:initial:tildern1:quadratic}
    and the union bound, it follows that 
   \begin{equation}
        \label{eq:initial:tildern1:quadratic22} 
    \begin{aligned}
         \forall \rho & \in \mathcal{M}_\pi: \epsilon E_{\theta\sim\rho} {}^C\mathcal{L}(\theta)
         \leq   \epsilon E_{\theta\sim\rho}{}^C\hat{\mathcal{L}}_N(\theta) +  
          \frac{1}{\lambda}\left ( \KL(\rho\|\pi) + \ln\left(\dfrac{3}{2\delta}\right) + \right. \\
          &  \left. \frac{1}{2} 
          \left(\ln E_{\theta\sim\pi} \exp(\mathcal{\bar{C}}_1(2\lambda,\theta,N,\frac{2\delta}{3},\epsilon))+
          \ln E_{\theta \sim \pi} \exp(\mathcal{\bar{C}}_2(2\lambda,\theta,N,\frac{2\delta}{3},\epsilon))\right) \right.\\
          \text{and}&\\
          &\forall \theta \in \Theta,~\forall t=0,\ldots,N-1: \quad
          \|\y(t)-\hat{\y}_{\theta}(t)\|_2 \le C, \|\y(t)-\hat{\y}_{\theta}(t \mid 0)\|_2 \le C
    \end{aligned}
    \end{equation}
    holds with probability at least  $1-(\frac{4\delta}{3}+\frac{2\delta}{3}) = 1-2\delta$ over data.

    Notice that if $\omega \in \Omega$ is such that
     \[ \forall \theta \in \Theta, t=0,\ldots,N-1: \quad
          \|\y(t)(\omega)-\hat{\y}_{\theta}(t)(\omega)\|_2 \le C, \|\y(t)(\omega  )-\hat{\y}_{\theta}(t \mid 0)(\omega)\|_2 \le C
    \]
    then $\ell_C(\y(t)(\omega),\hat{\y}_{\theta}(t \mid 0)(\omega)) = \|\y(t)(\omega)-\hat{\y}_{\theta}(t \mid 0)(\omega)\|_2^2$ and hence
    \[ \hat{\mathcal{L}}_N(\theta)(\omega) = {}^C\hat{\mathcal{L}}_N(\theta)(\omega) \]
    Therefore, the event defined by  \eqref{eq:initial:tildern1:quadratic22} equals the event defined by 
     \begin{equation}
        \label{eq:initial:tildern1:quadratic2} 
    \begin{aligned}
         \forall \rho & \in \mathcal{M}_\pi: \epsilon E_{\theta\sim\rho} {}^{C}\mathcal{L}(\theta)
         \leq   \epsilon E_{\theta\sim\rho}{} \hat{\mathcal{L}}_N(\theta) +  
          \frac{1}{\lambda}\left ( \KL(\rho\|\pi) + \ln\dfrac{3}{2\delta} + \right. \\
          &  \left. \frac{1}{2} 
          \left(\ln E_{\theta\sim\pi} \exp(\mathcal{\bar{C}}_1(2\lambda,\theta,N,\dfrac{2\delta}{3},\epsilon))+
          \ln E_{\theta \sim \pi} \exp(\mathcal{\bar{C}}_2(2\lambda,\theta,N,\dfrac{2\delta}{3},\epsilon)) \right)  \right. \\
          \text{and}&\\
          &\forall \theta \in \Theta,~\forall t=0,\ldots,N-1: \quad
          \|\y(t)-\hat{\y}_{\theta}(t)\|_2 \le C, \|\y(t)-\hat{\y}_{\theta}(t \mid 0)\|_2 \le C
    \end{aligned}
    \end{equation}

  For $\epsilon=-1$
  we will show \eqref{eq:initial:tildern1:quadratic2} implies
\begin{equation}
        \label{eq:initial:tildern1:quadratic3} 
    \begin{aligned}
         \forall \rho & \in \mathcal{M}_\pi: \epsilon E_{\theta\sim\rho} \mathcal{L}(\theta)
         \leq   \epsilon E_{\theta\sim\rho} \hat{\mathcal{L}}_N(\theta) +  
          \frac{1}{\lambda}\left ( \KL(\rho\|\pi) + \ln\left(\dfrac{3}{2\delta}\right) + \right. \\
          &  \left. \frac{1}{2} 
          \left(\ln E_{\theta\sim\pi} \exp(\mathcal{\bar{C}}_1(2\lambda,\theta,N,\frac{2\delta}{3},\epsilon))+
          \ln E_{\theta \sim \pi} \exp(\mathcal{\bar{C}}_2(2\lambda,\theta,N,\frac{2\delta}{3},\epsilon))\right) \right)
    \end{aligned}
    \end{equation}
   and as \eqref{eq:initial:tildern1:quadratic2} holds with probability at least $1-2\delta$, it follows that \eqref{eq:initial:tildern1:quadratic3} also holds with probability at least $1-2\delta$. 
   To this end,   notice that 
    \begin{multline}
    \label{thm:main:quadratic:eq24}
        \mathcal{L}(\theta) = \bE{[\|\y(t)-\hat{\y}_{\theta}(t)\|_2^2]}= 
       \bE[\|\y(t)-\hat{\y}_{\theta}(t)\|_2^2\chi(\|\y(t)-\hat{\y}_{\theta}(t)\|_2 \le C)]+ \\
        \bE[\|\y(t)-\hat{\y}_{\theta}(t)\|_2^2\chi(\|\y(t)-\hat{\y}_{\theta}(t)\|_2 > C)]
    \end{multline}
    and that $\bE[\|\y(t)-\hat{\y}_{\theta}(t)\|_2^2\chi(\|\y(t)-\hat{\y}_{\theta}(t)\|_2 > C)] \ge C^2\bP(\|\y(t)-\hat{\y}_{\theta}(t)\|_2 > C)$. Hence
    \begin{equation}
    \label{thm:main:quadratic:eq210}
    \begin{aligned}
     \mathcal{L}(\theta)& 
    \ge  
         \bE[\|\y(t)-\hat{\y}_{\theta}(t)\|_2^2\chi(\|\y(t)-\hat{\y}_{\theta}(t)\|_2 \le C)]+ 
          C^2\bP(\|\y(t)-\hat{\y}_{\theta}(t)\|_2 > C)\\
          &= 
    \bE[\ell_C(\y(t),\hat{\y}_{\theta}(t))]={}^C \mathcal{L}(\theta)
    \end{aligned}
    \end{equation}
  i.e., $\epsilon \mathcal{L}(\theta) \le \epsilon {}^C \mathcal{L}(\theta)$
  for $\epsilon=-1$, from which it follows that \eqref{eq:initial:tildern1:quadratic2}
  implies \eqref{eq:initial:tildern1:quadratic3}.

  For $\epsilon=1$, we will show that \eqref{eq:initial:tildern1:quadratic2} implies
   \begin{equation}
        \label{eq:initial:tildern1:quadratic4} 
    \begin{aligned}
         \forall \rho & \in \mathcal{M}_\pi: \epsilon E_{\theta\sim\rho} \mathcal{L}(\theta)
         \leq   \epsilon E_{\theta\sim\rho} \hat{\mathcal{L}}_N(\theta) +   \frac{72G_e^2\|\e_g\|_{\psi_2}^2\ln\left(\frac{3N^2}{2\delta}\right)}{N} + 
          \frac{1}{\lambda}\left ( \KL(\rho\|\pi) + \ln\dfrac{3}{2\delta} + 
          \right. \\
          &\left. 
          \frac{1}{2} 
          \left(\ln E_{\theta\sim\pi} \exp(\mathcal{\bar{C}}_1(2\lambda,\theta,N,\dfrac{2\delta}{3},\epsilon))+
          \ln E_{\theta \sim \pi} \exp(\mathcal{\bar{C}}_2(2\lambda,\theta,N,\dfrac{2\delta}{3},\epsilon))\right) 
           \right )
    \end{aligned}
    \end{equation}
   First note that 
   \[ \frac{72G_e^2\|\e_g\|_{\psi_2}^2\ln\left(\frac{3N^2}{2\delta}\right)}{N}=\frac{1}{\lambda}\ln\left(\exp\left(\lambda \frac{72G_e^2\|\e_g\|_{\psi_2}^2\ln\left(\frac{3N^2}{2\delta}\right)}{N}\right)\right) 
   \]
   and hence after algebraic manipulations 
   \begin{align*}
    & \frac{72G_e^2\|\e_g\|_{\psi_2}^2\ln\left(\frac{3N^2}{2\delta}\right)}{N} +
    \frac{1}{\lambda}\left ( \KL(\rho\|\pi) + \ln\dfrac{3}{2\delta} + 
            \right. \\
          &\left. 
          \frac{1}{2} 
          \left(\ln E_{\theta\sim\pi} \exp(\mathcal{\bar{C}}_1(2\lambda,\theta,N,\dfrac{2\delta}{3},\epsilon))+
          \ln E_{\theta \sim \pi} \exp(\mathcal{\bar{C}}_2(2\lambda,\theta,N,\dfrac{2\delta}{3},\epsilon))\right)\right) \textcolor{black}{\leq} \\
    & \frac{1}{\lambda}\left ( \KL(\rho\|\pi) + \ln\dfrac{3}{2\delta} + 
          \frac{1}{2} 
          \left(\ln E_{\theta\sim\pi} \exp(\mathcal{C}_1(2\lambda,\theta,N,\delta,\epsilon))+
          \ln E_{\theta \sim \pi} \exp(\mathcal{C}_2(2\lambda,\theta,N,\delta,\epsilon))\right) \right)
   \end{align*}
    where for $i=1,2$, 
   \[ \mathcal{C}_i(\lambda,\theta,N,\delta,\epsilon)=
      \lambda \frac{72G_e^2\|\e_g\|_{\psi_2}^2\ln\left(\frac{3N^2}{2\delta}\right)}{N}
      + \mathcal{C}_{ub,i}(\lambda,\theta,12G_e\|\e_g\|_{\psi_2}\sqrt{\ln\left(\frac{3N^2}{2\delta}\right)},N)
   \]   

   To show that \eqref{eq:initial:tildern1:quadratic2} implies   \eqref{eq:initial:tildern1:quadratic4}, notice that 
   \begin{equation}
    \label{thm:main:quadratic:eq22}
     \begin{aligned}
      & \mathcal{L}(\theta) = \bE{[\|\y(t)-\hat{\y}_{\theta}(t)\|_2^2]}= \\
      & \bE[\ell_C(\y(t),\hat{\y}_{\theta}(t))\chi(\|\y(t)-\hat{\y}_{\theta}(t)\|_2 \le C)]+\bE[\|\y(t)-\hat{\y}_{\theta}(t)\|_2^2\chi(\|\y(t)-\hat{\y}_{\theta}(t)\|_2 > C)] \le \\
      & \bE[\ell_C(\y(t),\hat{\y}_{\theta}(t))]+\bE[\|\y(t)-\hat{\y}_{\theta}(t)\|_2^2\chi(\|\y(t)-\hat{\y}_{\theta}(t)\|_2 > C)] \\
      &= {}^C \mathcal{L}(\theta)+ \bE[\|\y(t)-\hat{\y}_{\theta}(t)\|_2^2\chi(\|\y(t)-\hat{\y}_{\theta}(t)\|_2 > C)]
    \end{aligned}
    \end{equation}
    Using 
    \eqref{lem:quadratic1:eq2} it
    follows that 
    \begin{equation}
    \label{thm:main:quadratic:eq21}
     \begin{aligned}
       & \mathcal{L}(\theta) \le 
       {}^C \mathcal{L}(\theta)+ 72G_e^2\|\e_g\|_{\psi_2}^2\frac{\ln\left(\frac{3N^2}{2\delta}\right)}{N}
    \end{aligned}
    \end{equation}
    Hence, by adding  $72G_e^2\|\e_g\|_{\psi_2}^2\frac{\ln\left(\frac{3N^2}{2\delta}\right)}{N}$ to both sides of the inequality \eqref{eq:initial:tildern1:quadratic2}, it follows that \eqref{eq:initial:tildern1:quadratic4} is implied by \eqref{eq:initial:tildern1:quadratic2}. Combining the two case it follows that for quadratic loss and $\epsilon\in\{-1,1\}$, the inequality \eqref{eq:initial:tildern1} also holds with probability at least $1-2\delta$ over data.
  \end{proof}

We will conclude by presenting the proof of Lemma \ref{lem:quadratic1}-\ref{lem:quadratic2}. To this end, we need the following technical result. 
  \begin{lemma}
    \label{lem:subGaussian1}
    Let $\mathbf{I}$ be an index set, $\{\beta_{i,t}(k)\}_{k=0}^{\infty}$, $i \in \mathbf{I}$, be a sequence of matrices such that $\|\bar{\beta}\|_{\ell_1}< \infty$, where $\bar{\beta}(k)=\sup_{i \in \mathbf{I},t\geq 0} \|\beta_{i,t}(k)\|_2$, and let $\rvz_{i}(t)=(\beta_{i,t} \star \e_g)(t)$. 
Then $\|\z_i(t)\|_{2}$, $i \in \mathbb{I}$ and 
$\z(t)=\sup_{i \in \mathbf{I}} \|\rvz_i(t)\|_2$ are sub-Gaussian with 
sub-Gaussian norm 
\begin{equation}
\label{lem:subGaussian1:eq0}
    \|\|\z_i(t)\|_2\|_{\psi_2} \le \|\z(t)\|_{\psi_2} \le 6\|\bar{\beta}\|_{\ell_1}\|\e_g\|_{\psi_2}
\end{equation}
and the moments and the moment generating function of $\z^2(t)=\sup_{i \in \mathbf{I}} \|\rvz_i(t)\|_2^2$ satisfy
\begin{align}
   & \bE[\z^{2r}(t)] \le \|\bar{\beta}\|_{\ell_1}^{2r} (3\|\e_g\|_{\psi_2})^{2r} 2r!
      \label{lem:subGaussian1:eq1} \\
   & \bE[e^{\lambda \z^2(t)}] \le \frac{1+\lambda \|\bar{\beta}\|_{\ell_1}^2 (3\|\e_g\|_{\psi_2})^2}{1-\lambda \|\bar{\beta}\|_{\ell_1}^2 (3\|\e_g\|_{\psi_2})^2},  \quad |\lambda| < \frac{1}{\|\bar{\beta}\|_{\ell_1}^2 (3\|\e_g\|_{\psi_2})^2}
   \label{lem:subGaussian1:eq2}
\end{align}
Moreover, for any $C > 0$,
\begin{equation}
\label{lem:subGaussian1:eq4}
   \bE[\|\z_i(t)\|_2^2 \chi(\|\z_i(t)\|_2^2 > C^2)] \le 
   \frac{2C^2}{e^{\frac{C^2}{(6\|\bar{\beta}\|_{\ell_1} \|\e_g\|_{\psi_2})^2}}-1}
\end{equation}
\end{lemma}
\begin{proof}[Proof Lemma \ref{lem:subGaussian1}]
  Notice that 
  \begin{align*}
      \z(t) &= \sup_{i \in \mathbf{I}} \|\sum_{k=0}^{\infty} \beta_{i,t}(k) \e_g(t-k)\|_2\\ 
      &\le \sup_{i \in \mathbf{I}}\sum_{k=0}^{\infty} \|\beta_{i,t}(k)\|_2 \|\e_g(t-k)\|_2 \le \sum_{k=0}^{\infty} \|\bar{\beta}(k)\|_2 \|\e_g(t-k)\|_2 
  \end{align*}
  Hence, for any $r \ge 1$, by using the arithmetic-geometric mean inequality we get
  $a_1 \cdots a_k \le \frac{1}{k} \sum_{i=1}^k a_i^k$, $a_1,\ldots, a_k \ge 0$, 
  \begin{align*}
    \z^{2r}(t) \le \left( \sum_{k=0}^{\infty} \|\bar{\beta}(k)\|_2 \|\e_g(t-k)\|_2 \right)^{2r}=\sum_{k_1,\ldots,k_{2r}=0}^{\infty} \prod_{j=1}^{2r} \|\bar{\beta}(k_j)\|_2 \|\e_g(t-k_j)\|_2 \le \\
    \sum_{k_1,\ldots,k_{2r}=0}^{\infty} \prod_{j=1}^{2r} \|\bar{\beta}(k_j)\|_2  \frac{1}{2r} \sum_{j=1}^{2r}\|\e_g(t-k_j)\|_2^{2r}
  \end{align*}
  By taking expectations, using the fact that as $\|\e_g(t)\|_2$ is sub-Gaussian,
  and the proof of \cite[Proposition 2.6.1]{vershynin2026highdimensional} we obtain
  \begin{equation} 
   \label{lem:subGaussian1:pf:eq1} 
      \bE[\|\e_g(t-k_j)\|_2^{2r}]=\bE[\|\e_g(t)\|_2^{2r}] \le 2r(r-1)! (3\|\e_g\|_{\psi_2})^{2r}=2 r! (3\|\e_g\|_{\psi_2})^{2r}
  \end{equation} 
  it follows that
  \begin{align*}
      & \bE[\z^{2r}(t)] \le
     \left(\sum_{k=0}^{\infty} \|\bar{\beta}(k)\|_2\right)^{2r} 2 r! (3\|\e_g(t)\|_{\psi_2})^{2r}=
    \|\bar{\beta}\|_{\ell_1}^{2r} 2 r! (3\|\e_g\|_{\psi_2})^{2r}
  \end{align*}
  i.e., \eqref{lem:subGaussian1:eq1} holds.
  By noticing that
  \begin{align*}
     & \bE[e^{\lambda \z^2(t)}]=1+\sum_{r=1}^{\infty} \frac{\lambda^r \bE[\z^{2r}(t)]}{r!} \le 1+\sum_{r=1}^{\infty} \frac{\lambda^r \|\bar{\beta}\|_{\ell_1}^{2r} (3\|\e_g\|_{\psi_2})^{2r} 2r!}{r!} = \\
     &  1+ 2\lambda \|\bar{\beta}\|_{\ell_1}^2(3\|\e_g\|_{\psi_2})^2 \sum_{r=0}^{\infty} (\lambda \|\bar{\beta}\|_{\ell_1}^2 (3\|\e_g\|_{\psi_2})^2)^r =
      1 + \frac{2 \lambda \|\bar{\beta}\|_{\ell_1}^2(3\|\e_g\|_{\psi_2})^2}{1- \lambda \|\bar{\beta}\|_{\ell_1}^2 (3\|\e_g\|_{\psi_2})^2}= \\
      & \frac{1+\lambda \|\bar{\beta}\|_{\ell_1}^2(3\|\e_g\|_{\psi_2})^2}{1- \lambda \|\bar{\beta}\|_{\ell_1}^2 (3\|\e_g\|_{\psi_2})^2}
\end{align*}
it follows that \eqref{lem:subGaussian1:eq2} holds. 
Finally, for\footnote{If \(\|\bar\beta\|_{\ell^1}=0\), the result is trivial; hence assume \(\|\bar\beta\|_{\ell^1}>0\).} $\lambda = \frac{1}{4\|\bar{\beta}\|_{\ell_1}^2 (3\|\e_g\|_{\psi_2})^2}$, the bound becomes
\[
\bE[e^{\lambda \z^2(t)}] \le \frac{1+\frac{1}{4}}{1-\frac{1}{4}} = \frac{5/4}{3/4} = \frac{5}{3} \le 2
\]
from which by the definition of sub-Gaussian norm of $\z(t)$ it follows
\[
\|\z(t)\|_{\psi_2} \le \sqrt{\frac{1}{\lambda}} = 6\|\bar{\beta}\|_{\ell_1} \|\e_g\|_{\psi_2}
\]
i.e., \eqref{lem:subGaussian1:eq0} holds.
In particular, the sub-Gaussian norm of $\z(t)$ is bounded
and hence by \cite[Proposition 2.6.1]{vershynin2026highdimensional} 
$\z(t)$ is sub-Gaussian. Moreover, as 
$\z^2(t) \ge \|\z_i(t)\|_2^2$, it follows that the sub-Gaussian norm of
$\|\z_i(t)\|_2$ is also bounded and  it is also sub-Gaussian. 

Finally, from Lemma \ref{noise:boundC-1} applied to $\|\z_i(t)\|_2^2$
it follows that 
\begin{equation}
   \label{lem:subGaussian1:pf:eq01} 
\bE[\|\z_i(t)\|_2^2 \chi(\|\z_i(t)\|_2^2 > C^2)] \le 
\frac{\bE[e^{\bar{\lambda}\|\z_i(t)\|_2^2}]C^2}{e^{C^2\bar{\lambda}}-1}
\end{equation}
if $\bar{\lambda}$ is such that the moment generating function 
$\bE[e^{s\|\z_i(t)\|_2^2}]$ is well-defined on $[0,\bar{\lambda}]$.
Note that $\|\z_i(t)\|^2_2 \le \z^2(t)$ and hence 
$\bE[e^{s\|\z_i(t)\|^2}] \le \bE[e^{s\z^2(t)}]$. Now  $\bar{\lambda}$ can be chosen as $\bar{\lambda}=\frac{1}{4\|\bar{\beta}\|_{\ell_1}^2 (3\|\e_g\|_{\psi_2})^2}$
and hence $\bE[e^{\bar{\lambda}\|\z_i(t)\|^2}] \le \bE[e^{\bar{\lambda}\z^2(t)}] \le 2$.  Combining \eqref{lem:subGaussian1:pf:eq01} and the bound on $\bE[e^{\bar{\lambda}\|\z_i(t)\|^2}]$, we obtain \eqref{lem:subGaussian1:eq4}.
\end{proof}

\begin{proof}[Proof of Lemma \ref{lem:quadratic1}]
    Let $\mathbb{I}=\Theta \times \{1,2\}$,
    $\z_{(\theta,1)}(t)=\y(t)-\y_{\theta}(t)$, 
    $\z_{(\theta,2)}(t)=\y(t)-\y_{\theta}(t \mid 0)$,
    $\beta_{(\theta,1),t}=\alpha_{g,y}-c_{\theta}$, and $\beta_{(\theta,2),t}=\alpha_{g,y}-c_{\theta,t}$.
    Note that from the proof of Lemma \ref{alpha:lemma2} it follows
    that 
    \begin{align*}
     \|(\alpha_{g,y}-c_{\theta})(k)\|_2 
    &\le \|\alpha_{g,y}(k)\|_2+\sum_{l=0}^{k} \|\alpha_{\theta}(l)\|_2\|\alpha_{g,w}(k-l)\|_2 \\
    &\le\|\alpha_{g,y}(k)\|_2+ \sum_{l=0}^{k} \bar{\alpha}(l)\|\alpha_{g,w}(k-l)\|_2 \\
     \|(\alpha_{g,y}-c_{\theta,t})(k)\|_2 
     &\le \|\alpha_{g,y}(k)\|_2+\sum_{l=0}^{\min\{k,t\}} \|\alpha_{\theta}(l)\|_2\|\alpha_{g,w}(k-l)\|_2\\ 
     &\le \|\alpha_{g,y}(k)\|_2+\sum_{l=0}^{k} \bar{\alpha}(l)\|\alpha_{g,w}(k-l)\|_2.
    \end{align*}
    Hence $\bar{\beta}(k)=\sup_{i \in \mathbb{I},t\geq0} \|\beta_{i,t}(k)\|_2 \le \|\alpha_{g,y}(k)\|_2+\sum_{l=0}^{k} \bar{\alpha}(l)\|\alpha_{g,w}(k-l)\|_2$ and 
    \begin{align*}
    & \|\bar{\beta}\|_{\ell_1} \le \sum_{k=0}^{\infty} \|\alpha_{g,y}(k)\|_2 + \sum_{k=0}^{\infty} \sum_{l=0}^{k} \bar{\alpha}(l)\|\alpha_{g,w}(k-l)\|_2 \le 
\|\alpha_g\|_{\ell_1}+\|\bar{\alpha}\|_{\ell_1}\|\alpha_{g}\|_{\ell_1}=G_e
    \end{align*}
    From this and Lemma \ref{lem:subGaussian1} and Lemma \ref{alpha:lemma2}
    it follows that for any $t=0,\ldots,N-1$
    $\sup_{\theta \in \Theta} \max\{\|\y(t)-\y_{\theta}(t)\|_2, \|\y(t)-\y_{\theta}(t \mid 0)\|_2\}$ are sub-Gaussian random variable
    whose sub-Gaussian norms are upper bounded by 
    $6G_e \|\e_g\|_{\psi_2}$.
    Hence, using the sub-Gaussian tail bound (proof of (iii)$\implies$(i) in \cite[Proposition 2.6.1]{vershynin2026highdimensional}), we obtain the desired result
    \begin{multline*}
        \bP\left( 
          \sup_{\theta \in \Theta} \max\{\|\y(t)-\y_{\theta}(t)\|_2, \|\y(t)-\y_{\theta}(t \mid 0)\|_2\} \ge 6G_e \|\e_g\|_{\psi_2}\sqrt{\ln(\frac{3N^2}{2\delta})} \right) \le \\
           =2e^{-\ln(\frac{3N^2}{2\delta})} = \frac{2}{N}\frac{2\delta}{3N} \le \frac{2\delta}{3N} 
    \end{multline*}
    Therefore, by applying a union bound over $t=0,\ldots, N-1$, the first inequality
    \eqref{lem:quadratic1:eq1} of 
    the lemma follows. 

    In order to obtain the second inequality, we use 
    \eqref{lem:subGaussian1:eq4},
    and notice that 
    \begin{align*}
        \frac{2C^2}{e^{\frac{C^2}{(6G_e \|\e_g\|_{\psi_2})^2}}-1} &=
          \frac{2(6G_e \|\e_g\|_{\psi_2})^2\ln(\frac{3N^2}{2\delta})}
          {e^{\ln(\frac{3N^2}{2\delta})}-1} = 
         \frac{2 (6G_e\|\e_g\|_{\psi_2})^2\ln(\frac{3N^2}{2\delta}) 2\delta}{3N^2-2\delta}\\
         &\le 72 G_e^2\|\e_g\|_{\psi_2}^2\frac{\ln(\frac{3N^2}{2\delta})}{N}
    \end{align*}
  where we used that $2\delta \le 1$, $3N^2-2\delta \ge N^2 \ge N$, i.e.,
  \eqref{lem:quadratic1:eq2}.
\end{proof}

\begin{proof}[Proof of Lemma \ref{lem:quadratic2}]
  Notice that
\begin{equation}
  \label{lem:quadratic2:eq1}
  \begin{aligned}
   |\ell_C(y_1, \hat{y}_1) - & \ell_C(y_2, \hat{y}_2)|=  \\
   & \left \{  
    \begin{array}{ll}
    |\|y_1-\hat{y}_1\|_2^2 - \|y_2-\hat{y}_2\|_2^2| & \|y_1-\hat{y}_1\|_2^2\leq C^2, \quad \|y_2-\hat{y}_2\|_2^2 \le C^2 \\
    |\|y_2 - \hat{y}_2\|_2^2 - C^2| & \|y_1-\hat{y}_1\|_2^2 > C^2, \quad  \|y_2-\hat{y}_2\|_2^2 \le C^2 \\
    |\|y_1 - \hat{y}_1\|_2^2 - C^2| & \|y_2-\hat{y}_2\|_2^2 > C^2, \quad \|y_1-\hat{y}_1\|_2^2 \le C^2 \\
    0  & \|y_1-\hat{y}_1\|_2^2 > C^2, \quad  \|y_2-\hat{y}_2\|_2^2 > C^2 \\ 
    \end{array}
   \right.
  \end{aligned}
\end{equation}
 For each of the cases above we will show that the righ-hand side of 
 \eqref{lem:quadratic2:eq1} is bounded by $2C (\|y_1-y_2\|_2+\|\hat{y}_1-\hat{y}_2\|_2)$.

 First, consider the case when $\|y_1-\hat{y}_1\|_2^2\leq C^2$ and $\|y_2-\hat{y}_2\|_2^2 \le C^2$.
 Then by repeated use of the triangle inequality,
 \begin{align*}
 & |\|y_1-\hat{y}_1\|_2^2 - \|y_2-\hat{y}_2\|_2^2|= 
 |\|y_1-\hat{y}_1\|_2 - \|y_2-\hat{y}_2\|_2|(\|y_1-\hat{y}_1\|_2 + \\
 & \|y_2-\hat{y}_2\|_2) \le \|(y_1-y_2)-(\hat{y}_1-\hat{y}_2)\|_2 (2C) \le 
 2C (\|y_1-y_2\|_2+\|\hat{y}_1-\hat{y}_2\|_2)
 \end{align*}

  Consider the case when $\|y_1-\hat{y}_1\|_2^2 > C^2, \|y_2-\hat{y}_2\|_2^2 \le C^2$.
  Then by the definition of $\ell_C$,
   \begin{align*}
 &  |\ell_C(y_1, \hat{y}_1) - \ell_C(y_2, \hat{y}_2)| = |C^2 - \|y_2-\hat{y}_2\|_2^2| = \\
  & (C - \|y_2-\hat{y}_2\|_2 )(\|y_2-\hat{y}_2\|_2 + C) \le (C-\|y_2-\hat{y}_2\|_2) 2C  \le \\
  &  2C (\|y_1-\hat{y}_1\|_2-\|y_2-\hat{y}_2\|_2)\leq
  2C\|(y_2-y_1)-(\hat{y}_2-\hat{y}_1)\|_2 \le
  2C (\|y_1-y_2\|_2+\|\hat{y}_1-\hat{y}_2\|_2)
  \end{align*}
  The case when $\|y_2-\hat{y}_2\|_2^2 > C^2, \|y_1-\hat{y}_1\|_2^2 \le C^2$ is symmetric.

  Finally, consider the case when $\|y_1-\hat{y}_1\|_2^2 > C^2, \|y_2-\hat{y}_2\|_2^2 > C^2$.
  Then by the definition of $\ell_C$,
  \begin{align*}
  |\ell_C(y_1, \hat{y}_1) - \ell_C(y_2, \hat{y}_2)| &= |C^2 - C^2| = 0 \le 2C (\|y_1-y_2\|_2+\|\hat{y}_1-\hat{y}_2\|_2)
  \end{align*}
  This completes the proof.
  \end{proof}

  \subsection{Proof of Corollary \ref{thm:bounded:alt:col}}
  \label{sec:proof_corollary}
    The corollary follows from Theorem \ref{thm:main} by choosing $\lambda$ as follows:
    \begin{equation}
      \label{thm:bounded:alt:col:lambda}
      \lambda=\frac{\sqrt{N}}{L(\ln(N))^r} = \left\{\begin{array}{ll}
           \sqrt{N} & \e_g \mbox{ is bounded } \\
           \frac{\sqrt{N}}{\textcolor{black}{16 L_{\ell}\ln(N)}} & \ell \mbox{ is } L_{\ell} \mbox{-Lipschitz} \\
           \frac{\sqrt{N}}{192G_e\|\e_g\|_{\psi_2}\sqrt{1+\ln\left(\frac{3}{2\delta}\right)}(\ln(N))^{3/2}} & \ell \mbox{is quadratic}
      \end{array}\right. 
    \end{equation}
    and upper bounding the right-hand side of \eqref{eq:initial:tildern1}, i.e., by showing that
    \begin{equation}
    \label{thm:bounded:alt:col:pf:eq11}
     \begin{aligned}
          & \frac{1}{2} \left(   \ln E_{\theta \sim \pi} \exp(\mathcal{C}_1(2\lambda,\theta,N,\delta,\epsilon)) +
           \ln E_{\theta \sim \pi} \exp(\mathcal{C}_2(2\lambda,\theta,N,\delta,\epsilon)) 
          \right)  \le 
          \mathcal{C}(\delta,\epsilon)
      \end{aligned}
    \end{equation}

    \paragraph{Bounded case}
    In order to prove the bounded case, note that 
    from Lemma \ref{alpha:lemma2}
    \begin{equation}
      \label{thm:bounded:alt:col:pf:eq01}
      \begin{aligned}
       & \|\alpha_{\theta}\|_{\ell_1} \le \|\bar{\alpha}\|_{\ell_1}, \quad \theta_{\infty}(\alpha_{\theta}) \le \theta_{\infty}(\bar{\alpha}), \quad
        G_e(\theta) \le G_e, \quad  G_{e,1}(\theta) \le G_{e,1}
      \end{aligned}
    \end{equation} 
    In addition, notice that 
     the functions
    $\mathcal{C}_{m,b,i}(\lambda,\alpha,X,L,c)$ 
    have the following property;
    if $\lambda \le \bar{\lambda}\sqrt{N}$, $L_1 \le L_2$, $X_1 \le X_2$, $c_1 \le c_2$, and $\|\alpha(k)\|_2 \le \|\beta(k)\|_2$ for all $k \in \mathbb{N}$, then 
    \begin{equation}
      \label{thm:bounded:alt:col:pf:eq2}
      \begin{aligned}
       \mathcal{C}_{m,b,i}(\lambda,\alpha,X_1,L_1,c_1) \le N\mathcal{C}_{m,b,i}(\bar{\lambda},\beta,X_2,L_2,c_2)
      \end{aligned}
    \end{equation}

      Note that $L(\theta) \le L_b$, 
      and by using \eqref{thm:bounded:alt:col:pf:eq01}, it follows that 
       $\|\alpha_{\theta}(k)\|_2 \le \bar{\alpha}(k)$, $G_e(\theta) \le G_e$, $G_{e,1}(\theta) \le G_{e,1}$, 
      and hence,  by using \eqref{thm:bounded:alt:col:pf:eq2}, it follows that for all $i=1,2$
      \begin{align*}
      \mathcal{C}_i(2\lambda,\theta,N,\delta,\epsilon)=& \frac{\mathcal{C}_{m,b,i}(2\sqrt{N},\alpha_{\theta},G_e(\theta)+2G_{e,1}(\theta),L(\theta),c_{\e})}{N} \le \mathcal{C}_{m,b,i}(2,\bar{\alpha},G_e+2G_{e,1},L_b,c_{\e})
      \end{align*}
      Hence for $i=1,2$,
      \begin{align*}
        \ln E_{\theta \sim \pi} \exp(\mathcal{C}_i(2\lambda,\theta,N,\delta,\epsilon)) &\le 
        \ln E_{\theta \sim \pi} \exp(\mathcal{C}_{m,b,i}(2,\bar{\alpha},G_e+2G_{e,1},L_b,c_{\e})) 
        \\
        & = \mathcal{C}_{b,m,i}(2,\bar{\alpha},G_e+2G_{e,1},L_b,c_{\e})
      \end{align*}
      and hence \eqref{thm:bounded:alt:col:pf:eq11}.

      \paragraph{Lipschitz loss case:}
         First, notice that for $N \ge 4$, $\ln(N) \ge 1$ and hence  
         \begin{equation}
          \label{thm:bounded:alt:col:pf:eq3}
         \begin{aligned}
            \mathcal{\bar{C}}_{ub}(2\lambda,L_{\ell},N)= &\frac{1}{N}\left(\frac{2\sqrt{N}}{16 L_{\ell}\ln(N)} \right)^2 \left(32 L_{\ell}^2C_{u,1}\right)
            + 4\frac{1}{N} \frac{2\sqrt{N}}{16L_{\ell}\ln(N)} L_{\ell}C_{u,2} \ln(N) \\
            & = \frac{1}{2\ln(N)^2}C_{u,1}+\frac{1}{2} C_{u,2} \frac{1}{\sqrt{N}} \le C_{u,1}+C_{u,2}
         \end{aligned}
        \end{equation}
        Moreover, using \eqref{thm:bounded:alt:col:pf:eq2}, 
        \begin{equation}
          \label{thm:bounded:alt:col:pf:eq4}
         \begin{aligned}
           \mathcal{C}_{b,i}\Big(4\lambda, & \theta,L_{\ell},\frac{2\ln(N)}{\|\alpha_g\|_{\ell_1}K},N\Big)  =
           \mathcal{C}_{b,i}\left(\frac{4\sqrt{N}}{16L_{\ell}\ln(N)}, \theta,L_{\ell},\frac{2\ln(N)}{\|\alpha_g\|_{\ell_1}K},N\right)  \\
           &= \frac{\mathcal{C}_{m,b,i}\left(\frac{4}{16L_{\ell}}\frac{\sqrt{N}}{\ln(N)},\alpha_{\theta},G_e(\theta)+2G_{e,1}(\theta),L_{\ell},\frac{2\ln(N)}{\|\alpha_g\|_{\ell_1}K}\right)}{N} \\
           & \le \mathcal{C}_{m,b,i}\left(\frac{1}{4L_{\ell}\ln(N)},\bar{\alpha},G_e+2G_{e,1},L_{\ell}, \frac{2\ln(N)}{\|\alpha_g\|_{\ell_1}K}\right) \\
           & = \left\{ \begin{array}{ll}
             \frac{1}{4L_{\ell}\ln(N)} L_{\ell} \frac{2\ln(N)}{\|\alpha_g\|_{\ell_1}K} \sqrt{n_{\e}}\theta_\infty(\bar\alpha)\|\alpha_g\|_{\ell_1}  & i=1 \\
             \frac{1}{(4L_{\ell}\ln(N))^2} \frac{L_{\ell}^2}{2} (G_e+2G_{e,1}) \left(\frac{2\ln(N)}{\|\alpha_g\|_{\ell_1}K}\right)^2 n_{\e}  & i=2 \end{array} \right. \\
           & = \left\{ \begin{array}{ll}
             \frac{1}{2K} \sqrt{n_{\e}}\theta_\infty(\bar\alpha)& i=1 \\
             \frac{1}{8} \frac{(G_e+2G_{e,1})}{2} \left(\frac{1}{\|\alpha_g\|_2K}\right)^2 n_{\e} & i=2 \end{array} \right. \\
             & \le  \mathcal{C}_{m,b,i}\left(\frac{1}{2},\bar{\alpha}, G_e+2G_{e,1},1, \frac{1}{\|\alpha_g\|_{\ell_1}K}\right)
         \end{aligned}
        \end{equation}
     From this it then follows that 
     \begin{equation}
      \label{thm:bounded:alt:col:pf:eq5}
         \mathcal{C}_{ub,i}\left(\frac{2}{16L_{\ell}}\frac{\sqrt{N}}{\ln(N)},\theta,L_{\ell},N\right) \le 
           C_{u,1}+C_{u,2}+\frac{1}{2}\mathcal{C}_{m,b,i}\left(\frac{1}{2}, \bar{\alpha}, G_e+2G_{e,1},1, \frac{1}{\|\alpha_g\|_{\ell_1}K}\right)
     \end{equation}
     and hence
     \begin{align*}
          & \frac{1}{2} \left(   \ln E_{\theta \sim \pi} \exp(\mathcal{C}_1(2\lambda,\theta,N,\delta,\epsilon)) +
           \ln E_{\theta \sim \pi} \exp(\mathcal{C}_2(2\lambda,\theta,N,\delta,\epsilon)) 
          \right)  \\
          & \le \frac{1}{2} \sum_{i=1}^{2} \left(C_{u,1}+C_{u,2}+ \frac{1}{2}\mathcal{C}_{m,b,i}\left(\frac{1}{2}, \bar{\alpha}, G_e+2G_{e,1}, 1, \frac{1}{\|\alpha_g\|_{\ell_1}K}\right)\right) = \mathcal{C}(\delta,\epsilon)
      \end{align*}
      so \eqref{thm:bounded:alt:col:pf:eq11} holds.

 \paragraph{Quadratic loss case:}
            Notice that as $N^2 \ge 4$, it follows that $\ln(N^2) \ge 1$, and  hence
        \begin{equation} 
        \label{thm:bounded:alt:col:pf:eq6.0}
           \ln\left(\frac{3N^2}{2\delta}\right)=\left(\ln(N^2)+\ln\left(\frac{3}{2\delta}\right)\right)
            \le \ln(N^2)\left(1+\ln\left(\frac{3}{2\delta}\right)\right)
        \end{equation}
        By using \eqref{thm:bounded:alt:col:pf:eq6.0}
        and  applying \eqref{thm:bounded:alt:col:pf:eq5} with
      $L_{\ell}=\frac{L(\ln(N))^{1/2}}{16}$ and noticing that $\lambda=\frac{\sqrt{N}}{16 L_{\ell} \ln(N)}$
       it follows that
       \begin{equation}
        \label{thm:bounded:alt:col:pf:eq7}
        \begin{aligned}
           \mathcal{C}_{ub,i}(2\lambda, & \theta,12G_e\|\e_g\|_{\psi_2}\sqrt{\ln\left(\frac{3N^2}{2\delta}\right)},N)  \
           \le 
           \mathcal{C}_{ub,i}(2\lambda, \theta, \frac{L(\ln(N))^{1/2}}{16},N)  \\
            & \le C_{u,1}+C_{u,2}+\frac{1}{2}\mathcal{C}_{m,b,i}\left(1/2, \bar{\alpha}, G_e+2G_{e,1}, 1, \frac{1}{\|\alpha_g\|_{\ell_1}K}\right)
       \end{aligned}
       \end{equation}
       Moreover, from \eqref{thm:bounded:alt:col:pf:eq6.0} and the fact that 
       $\ln(N) \ge 1$ and $\ln\left(\frac{3N^2}{2\delta}\right) \ge 1$, it follows that
       \begin{equation}
        \label{thm:bounded:alt:col:pf:eq8}
       \begin{aligned}
         2\lambda & \frac{\ln\left(\frac{3N^2}{2\delta}\right)}{N} 72G^2_e\|\e_g\|_{\psi_2}^2\max\{0,\epsilon\}
          \\
         &= \frac{2\sqrt{N}}{192 G_e\|\e_g\|_{\psi_2} \sqrt{2}\ln(N) \sqrt{1+\ln\left(\frac{3}{2\delta}\right)}\sqrt{\ln(N)}}
            (72G^2_e\|\e_g\|_{\psi_2}^2\max\{0,\epsilon\})\frac{\ln\left(\frac{3N^2}{2\delta}\right)}{N}   \\
        & \leq  \frac{2\sqrt{N}}{192 G_e\|\e_g\|_{\psi_2}\ln(N) \sqrt{1+\ln\left(\frac{3}{2\delta}\right)}\sqrt{\ln(N^2)}}
            (72G^2_e\|\e_g\|_{\psi_2}^2\max\{0,\epsilon\}) \frac{(1+\ln\left(\frac{3}{2\delta}\right))\ln(N^2)}{N} \\
         & = \frac{2 \cdot 72G^2_e\|\e_g\|_{\psi_2}^2\max\{0,\epsilon\}\sqrt{1+\ln\left(\frac{3}{2\delta}\right)}\sqrt{\ln(N^2)}}{192 G_e\|\e_g\|_{\psi_2}\ln(N)\sqrt{N}} \\
         &= 
         \frac{6}{16} \frac{2G_e\|\e_g\|_{\psi_2}\sqrt{2}(\ln(N))^{1/2} \sqrt{1+\ln\left(\frac{3}{2\delta}\right)}\max\{0,\epsilon\}}{\ln(N) \sqrt{N}} \\
         & = \frac{6}{8 \cdot 192} \frac{L (\ln(N))^{1/2}}{\sqrt{N}\ln(N)} \max\{0,\epsilon\}\\ 
          & = \frac{1}{128} \frac{L}{\sqrt{N}(\ln(N))^{1/2}} \max\{0,\epsilon\}
       \end{aligned}
      \end{equation}
       Hence,  for $i=1,2$,
       \begin{align*}
            \mathcal{C}_i(2\lambda, & \theta,N,\delta,\epsilon) \le \frac{L}{128} \max\{0,\epsilon\} + \\
               & C_{u,1}+C_{u,2}+\frac{1}{2}\mathcal{C}_{b,m,i}\left(1/2, \bar{\alpha}, G_e+2G_{e,1}, 1, \frac{1}{\|\alpha_g\|_2K}\right)
       \end{align*}
       and thus 
      \begin{align*}
          & \frac{1}{2} \left(   \ln E_{\theta \sim \pi} \exp(\mathcal{C}_1(2\lambda,\theta,N,\delta,\epsilon)) +
           \ln E_{\theta \sim \pi} \exp(\mathcal{C}_2(2\lambda,\theta,N,\delta,\epsilon)) 
          \right)  \\
          & \le \frac{L}{128} \max\{0,\epsilon\} + \frac{1}{2} \sum_{i=1}^{2} \left(C_{u,1}+C_{u,2}+ \frac{1}{2}\mathcal{C}_{m,b,i}\left(1/2, \bar{\alpha}, G_e+2G_{e,1}, 1, \frac{1}{\|\alpha_g\|_2K}\right)\right)  \\
          & = \mathcal{C}(\delta,\epsilon)
      \end{align*}
   which completes the proof.

 \subsection{Uniform bounds derived from Corollary \ref{thm:bounded:alt:col}}
 \label{sec:uniform_bounds}
 In order to present the uniform bounds, we need the following assumptions from Assumption \ref{as:oracle}:
 \textbf{(LipModel)} and either 
\textbf{(Gauss)} or \textbf{(Uni)}.
 \begin{corollary}[Uniform [PAC bounds, Example 2.1,2.2 \cite{alquier2021userfriendly}]
 \label{lem:pac}
  With the Assumption  \textbf{(LipModel)} and either 
\textbf{(Gauss)} or \textbf{(Uni)} from Assumption \ref{as:oracle} the following holds.
  Let $L_{\Theta}(\delta)$, $\mathcal{C}_{L,emp}$, $\mathcal{C}_{L,true}$, $C(\theta)$,  $r_{\text{Lip},1}$, $r_{\text{Lip},2}$
   be defined as in 
   Table~\ref{tab:constants:sysid},
   and let 
   $L$, $C_{\delta}$, 
   $\mathcal{C}(\delta,\epsilon)$, $r$  be defined as in 
   Corollary \ref{thm:bounded:alt:col}
   For any $\delta \in (0,0.5)$, $\sigma  > 0
   $, $N\geq \textcolor{black}{4}$, 
   $\epsilon \in \{-1,1\}$, 
   \begin{equation}
   \label{lem:pac:eq1:uniform}
   \begin{aligned}
     & \forall  \theta \in \Theta: ~ 
    \epsilon \mathcal{L}(\theta)  \le 
    \epsilon \hat{\mathcal{L}}_{N}(\theta) + \\
     & \frac{L(\ln(N))^{r+1}}{\sqrt{N}} \left(\mathcal{C}_{L,emp}+\mathcal{C}_{L,true}+C(\theta)+\ln \textcolor{black}{\left(\frac{C_{\delta}}{\delta} \right)+2r_{\text{Lip},1}\ln\left(\frac{C_{\delta}}{\delta}\right)} + \mathcal{C}(\delta,\epsilon)\right)
    \end{aligned}
   \end{equation}
    holds with probability at least $1-(2+r_{\text{Lip},1})\delta$ over data.
 \end{corollary}
The proof of Corollary \ref{lem:pac} is based on a number of technical lemmas, which will also be used for
the proof of Corollary \ref{lem:bound:mean}.
\begin{lemma}[$\KL$ divergence]
  \label{pac:kl}
   Assume that either \textbf{(Gauss)} or \textbf{(Uni)} holds
   \textcolor{black}{and $\sigma > 0$ is arbitrary if \textbf{(Gauss)} holds, and $\sigma < 2R$ if \textbf{(Uni)} holds.}
   For every $\theta \in \Theta$, define the density $\rho_{\theta}$ as follows
  \[ \rho_{\theta}=\left\{ \begin{array}{rl}
                           \mathcal{N}(\theta,\frac{\sigma^2}{N}I) & \mbox{ if \textbf{(Gauss)} holds} \\
                            \mathcal{U}(\theta,\frac{\sigma}{\sqrt{N}})  & \mbox{ if \textbf{(Uni)} holds}
                         \end{array} \right. 
  \]
   where $\mathcal{N}(m,\sigma^2I)$ denotes the Gaussian distribution with mean $m$ and variance $\sigma^2I$, and
   $\mathcal{U}(m,r)$ denotes the uniform  distribution on the ball with center in $m$ and radius $r$ intersected with $\Theta$.
  Then for $N \ge 4$,
  \[ \KL(\rho_{\theta} \| \pi) \le C(\theta)\ln(N)\]
\end{lemma}
\begin{proof}[Proof of Lemma \ref{pac:kl}]
  If \textbf{(Gauss)} holds, then 
  using KL divergence for Gaussian densities (\cite[Lemma 11.3.4]{LindquistBook}) it follows
  that 
  \begin{align*}
     \KL&(\rho_{\theta} \| \pi)  = \\
     & \frac{1}{2} \left(\mathrm{trace}\left(\frac{\sigma^2}{N}P^{-1} \right) +
          (\theta_m-\theta)^TP^{-1}(\theta_m-\theta) - n_{\theta}  + \ln \frac{\det(P)N^{n_{\theta}}}{\sigma^{2n_{\theta}}} \right)  \\
      & =\frac{1}{2} \left(\mathrm{trace}\left(\frac{\sigma^2}{N} P^{-1} \right) +
          (\theta_m-\theta)^TP^{-1}(\theta_m-\theta) - n_{\theta}  + \ln \frac{\det(P)}{\sigma^{2n_{\theta}}}+ n_{\theta}\ln(N)\right)\\
        & \le \ln(N) C(\theta)  
  \end{align*}
  where we used the fact that $\ln(N) \ge 1$   and $\frac{1}{N} \le 1$ if $N \ge 4$.
  Likewise, if \textbf{(Uni)} holds,  \textcolor{black}{denote by}
  \textcolor{black}{$B(\theta)$ \textcolor{black}{the intersection of the ball of radius $\frac{\sigma}{\sqrt{N}}$ centered at $\theta$ with $\Theta$}.}
  \begin{color}{black}
    We claim that 
   $B(\theta)$ contains a ball of radius $\frac{\sigma}{2\sqrt{N}}$.
    Indeed, since $\sigma < 2R$, it follows that $\frac{\sigma}{2\sqrt{N}R} < 1$.
    Consider 
    \[ \theta_r=\frac{\sigma}{2\sqrt{N}R}\theta_m+\left(1-\frac{\sigma}{2\sqrt{N}R}\right) \theta. \]
    It then follows that 
    \[ \|\theta_r-\theta_m\|_2=\left(1-\frac{\sigma}{2\sqrt{N}R}\right)\|\theta-\theta_m\|_2 \le \left(1-\frac{\sigma}{2\sqrt{N}R}\right)R \le R, \] 
    hence $\theta_r$ belongs to $\Theta$.
    Likewise, 
    \[ \|\theta_r-\theta\|_2=\frac{\sigma}{2\sqrt{N}R} \|\theta_m-\theta\|_2 \le \frac{\sigma}{2\sqrt{N}R}R \le \frac{\sigma}{2\sqrt{N}} \le \frac{\sigma}{\sqrt{N}}, \] i.e., $\theta_r$ belongs to the ball of radius $\frac{\sigma}{\sqrt{N}}$ centered at $\theta$.
    Finally, we argue that the ball $B_{\frac{\sigma}{2\sqrt{N}}}(\theta_r)$ of radius $\frac{\sigma}{2\sqrt{N}}$ centered at $\theta_r$ is contained in $B(\theta)$ and $\Theta$.
    Let $\tilde{\theta} \in B_{\frac{\sigma}{2\sqrt{N}}}(\theta_r)$. Then
    \[
      \begin{split}
    & \|\tilde{\theta}-\theta_m\|_2=\|(\tilde{\theta}-\theta_r)+(\theta_r-\theta_m)\|_2 \le 
    \|\tilde{\theta}-\theta_r\|_2 + \|\theta_r-\theta_m\|_2 \\
    & \le \frac{\sigma}{2\sqrt{N}} +\left(1-\frac{\sigma}{2\sqrt{N}R}\right) \| \theta-\theta_m\|_2 \le \frac{\sigma}{2\sqrt{N}} + \left(1-\frac{\sigma}{2\sqrt{N}R}\right) R \le R
      \end{split}
    \]
    i.e., $\tilde{\theta} \in \Theta$.
    Similarly, 
    \[
    \begin{split}
    & \|\tilde{\theta}-\theta\|_2=\|(\tilde{\theta}-\theta_r)+(\theta_r-\theta)\|_2 \le 
    \|\tilde{\theta}-\theta_r\|_2 + \|\theta_r-\theta\|_2 \le \frac{\sigma}{2\sqrt{N}} + \frac{\sigma}{2\sqrt{N}} = \frac{\sigma}{\sqrt{N}}, 
    \end{split}
    \]
    i.e., $\tilde{\theta} \in B(\theta)$.
    Then $C_2=\mathrm{vol}(B(\theta))$ is larger than the volume  of $B_{\frac{\sigma}{2\sqrt{N}}}(\theta_r)$, i.e., 
    from the formula for the volume of a high-dimensional ball \cite{BallVolume} it follows that
   \[
    C_2 \ge \mathrm{vol}(B_{\frac{\sigma}{2\sqrt{N}}}(\theta_r))=\frac{\pi^{n_\theta/2}}{\Gamma(n_\theta/2+1)} \left(\frac{\sigma}{2\sqrt{N}}\right)^{n_\theta}
  \]
  Then
  \end{color}
\begin{align*}
    & KL(\rho_{\theta} \| \pi) = \int_{\|\tilde{\theta}-\theta\|_2 \le \frac{\sigma}{\sqrt{N}}} 
     \ln\left(\frac{\rho_{\theta}(\tilde{\theta})}{\pi(\tilde{\theta})}\right) \frac{1}{C_2} dm(\tilde{\theta})  \\
     & =\int_{\|\tilde{\theta}-\theta\|_2 \le \frac{\sigma}{\sqrt{N}}}  \frac{1}{C_2} \ln\left(\frac{\mathrm{vol}(\Theta)}{C_2}\right) dm(\tilde{\theta})
    = \ln\left(\frac{\mathrm{vol}(\Theta)}{C_2}\right)
\end{align*}
and hence 
\begin{align*}
  & \KL(\rho_{\theta} \| \pi) = \ln\left(\frac{\mathrm{vol}(\Theta)}{C_2}\right) \textcolor{black}{\le}
  \ln\left(\frac{\mathrm{vol}(\Theta)}{\frac{\pi^{n_\theta/2}}{\Gamma(n_\theta/2+1)} \left(\frac{\sigma}{2\sqrt{N}}\right)^{n_\theta}}\right) \\
  & =\ln\left(\frac{\mathrm{vol}(\Theta) \Gamma(n_\theta/2+1)}{\pi^{n_\theta/2} (\textcolor{black}{\sigma/2})^{n_\theta}}\right) + 
     \ln \left(\left(\sqrt{N}\right)^{n_\theta}\right) \\
     & = \ln\left(\frac{\mathrm{vol}(\Theta) \Gamma(n_\theta/2+1)}{\sqrt{\pi\textcolor{black}{(\sigma^2/4})}^{n_\theta}}\right) + 
       \frac{n_{\theta}}{2} \ln(N)
       \le C(\theta)\ln(N)
\end{align*}
%
\end{proof}
Another ingredient is the following:
\begin{lemma}[Lipschitz continuity of the true and empirical errors]
  \label{pac:lipsch:loss}
   Using the constants defined in Corollary \ref{thm:bounded:alt:col}-\ref{lem:pac}, 
   the following inequality holds
   \begin{equation}
   \label{pac:lipschitz:loss:eq31}
   \begin{aligned}
     & \forall \theta_1,\theta_2 \in \Theta:\\
          &|\hat{\mathcal{L}}_N(\theta_1) - \hat{\mathcal{L}}_N(\theta_2)| \le L_{\Theta}(\delta) L(\ln(N))^r \|\theta_1 - \theta_2\|_2
           +  \\
           & \qquad\qquad\qquad 2r_{\text{Lip,1}}\frac{L\ln(N)^{r}}{\sqrt{N}}\left(C_{u,1} + C_{u,2}+\ln\left(\frac{\textcolor{black}{C_{\delta}}}{\delta}\right)\right), 
   \end{aligned}
   \end{equation}  
   with probability $1-r_{\text{Lip},1}\delta$ over data, 
   where 
   $$ r_{\text{Lip},1}=\left\{\begin{array}{ll} 0 & \e_g \mbox{ is bounded } \\ 1 & \mbox{otherwise} \end{array}\right.$$ 
   and 
   \begin{equation}
   \label{pac:lipschitz:loss:eq32}
   \begin{aligned}
       & \forall ~ \theta_1,\theta_2 \in \Theta: \quad \\
          & |\mathcal{L}(\theta_1) - \mathcal{L}(\theta_2)| \le L_{\Theta}(\delta) L(\ln(N))^r \|\theta_1 - \theta_2\|_2
          + 2r_{\text{Lip},1}\frac{L\ln(N)^{r}}{\sqrt{N}}\left(C_{u,1} + C_{u,2}\right) + \\
           & \qquad\qquad\qquad 2r_{\text{Lip},2}  \frac{L}{512} \frac{L(\ln(N))^r}{\sqrt{N}}
   \end{aligned}
   \end{equation}
    where $$ r_{\text{Lip},2}=\left\{\begin{array}{ll} 2 & \ell \mbox{ is quadratic and $\e_g$ is unbounded} \\ 0 & \mbox{otherwise} \end{array}\right. .$$
   Moreover, the inequalities below 
   \begin{equation}
   \label{pac:lipschitz:loss:eq3}
   \begin{aligned}
     & \mbox{\eqref{catoni1} from Corollary \ref{thm:bounded:alt:col} holds, and }  \\
     & \forall \theta_1,\theta_2 \in \Theta: \\
         & |\hat{\mathcal{L}}_N(\theta_1) - \hat{\mathcal{L}}_N(\theta_2)| \le L_{\Theta}(\delta) L(\ln(N))^r \|\theta_1 - \theta_2\|_2
           +  \\
            & \qquad\qquad\qquad 2r_{\text{Lip},1}\frac{L\ln(N)^{r}}{\sqrt{N}}\left(C_{u,1} + C_{u,2}+\ln\left(\frac{\textcolor{black}{C_{\delta}}}{\delta}\right)\right), \\
           \text{ and } &\\
       & \forall ~ \theta_1,\theta_2 \in \Theta: \quad \\
          & |\mathcal{L}(\theta_1) - \mathcal{L}(\theta_2)| \le L_{\Theta}(\delta) L(\ln(N))^r \|\theta_1 - \theta_2\|_2
          + 2r_{\text{Lip},1}\frac{L\ln(N)^{r}}{\sqrt{N}}\left(C_{u,1} + C_{u,2}\right) + \\
           & \qquad\qquad\qquad 2r_{\text{Lip},2}  \frac{L}{512} \frac{L(\ln(N))^r}{\sqrt{N}}
   \end{aligned}
   \end{equation}
   hold simultaneously  with probability at least $1-(2+r_{\text{Lip},1})\delta$.
  
\end{lemma}
\begin{proof}[Proof of Lemma \ref{pac:lipsch:loss}]
  \textbf{Bounded noise case:}
  If $\ell$ is $L_{\ell}$-Lipschitz and $\e_g$ is bounded, then using (\textbf{LipModel}) and the argument of
  the proof from Lemma \ref{alpha:lemma2}, it follows that 
  \begin{align*}
    & |\ell(\y(t),\hat{\y}_{\theta_1}(t)) - \ell(\y(t),\hat{\y}_{\theta_2}(t))| \le L_{\ell} \|\hat{\y}_{\theta_1}(t) - \hat{\y}_{\theta_2}(t)\|_2 \\
    & \le L_{\ell}\sum_{k=0}^{\infty} \|c_{\theta_1}(k)-c_{\theta_2}(k)\|_{2} \|\e_g(t-k)\|_{2} 
    \le L_{\ell} \sum_{k=0}^{\infty} \sum_{l=0}^{k} \|\alpha_{\theta_1}(l)-\alpha_{\theta_2}(l)\|_2\|\alpha_{g,w}(k-l)\|_2 c_{\e}\sqrt{n_\e} \\
    & \le L_{\ell}\|\alpha_{\theta_1}-\alpha_{\theta_2}\|_{\ell_1}\|\alpha_{g}\|_{\ell_1}c_{\e}\sqrt{n_\e}=
          \underbrace{L_{\alpha}L_{\ell}\|\alpha_{g}\|_{\ell_1}c_{\e}\sqrt{n_\e}}_{L_{\Theta}(\delta)} \|\theta_1-\theta_2\|_2
  \end{align*}
  and 
  \begin{color}{black}
  \begin{align*}
    & |\ell(\y(t),\hat{\y}_{\theta_1}(t \mid 0)) - \ell(\y(t),\hat{\y}_{\theta_2}(t \mid 0))| \le L_{\ell} \|\hat{\y}_{\theta_1}(t \mid 0) - \hat{\y}_{\theta_2}(t \mid 0)\|_2 \\
    & \le L_{\ell}\sum_{k=0}^{\infty} \|c_{\theta_1,t}(k)-c_{\theta_2,t}(k)\|_{2} \|\e_g(t-k)\|_{2}  \\
    & \le L_{\ell} \sum_{k=0}^{\infty} \sum_{l=0}^{\min\{k,t\}} \|\alpha_{\theta_1}(l)-\alpha_{\theta_2}(l)\|_2\|\alpha_{g,w}(k-l)\|_2 c_{\e}\sqrt{n_\e} \\
    &  \le L_{\ell} \sum_{k=0}^{\infty} \sum_{l=0}^{\infty} \|\alpha_{\theta_1}(l)-\alpha_{\theta_2}(l)\|_2\|\alpha_{g,w}(k-l)\|_2 c_{\e}\sqrt{n_\e} \\
    & \le L_{\ell}\|\alpha_{\theta_1}-\alpha_{\theta_2}\|_{\ell_1}\|\alpha_{g}\|_{\ell_1}c_{\e}\sqrt{n_\e}=
          \underbrace{L_{\alpha}L_{\ell}\|\alpha_{g}\|_{\ell_1}c_{\e}\sqrt{n_\e}}_{L_{\Theta}(\delta)} \|\theta_1-\theta_2\|_2
  \end{align*}
  \end{color}
  From this it then follows that \textcolor{black}{for all $\theta_1,\theta_2 \in \Theta$:}
    \begin{align}
   & |\mathcal{L}(\theta_1) - \mathcal{L}(\theta_2)| =
     \textcolor{black}{\left|\bE[\ell(\y(t),\hat{\y}_{\theta_1}(t))-\ell(\y(t),\hat{\y}_{\theta_2}(t))]\right|} \nonumber \\
     & \textcolor{black}{\le \bE[\left|\ell(\y(t),\hat{\y}_{\theta_1}(t))-\ell(\y(t),\hat{\y}_{\theta_2}(t)\right|]} \le 
   L_{\Theta}(\delta) \|\theta_1 - \theta_2\| 
    \label{pac:lipschitz:loss:eq1} \\
   & |\hat{\mathcal{L}}_N(\theta_1) - \hat{\mathcal{L}}_N(\theta_2)| 
     \textcolor{black}{=\frac{1}{N} 
     \left|\sum_{t=0}^{N-1} \ell(\y(t),\hat{\y}_{\theta_1}(t\mid 0))-\ell(\y(t),\hat{\y}_{\theta_2}(t \mid 0))\right| \le } \nonumber  \\
    & \textcolor{black}{\frac{1}{N} \sum_{t=0}^{N-1} \left|\ell(\y(t),\hat{\y}_{\theta_1}(t\mid 0))-\ell(\y(t),\hat{\y}_{\theta_2}(t \mid 0))\right| \le }
   L_{\Theta}(\delta) \|\theta_1 - \theta_2\| 
    \label{pac:lipschitz:loss:eq2}
   \end{align}
   By noticing that \textcolor{black}{$r_{\text{Lip},1}=r=0$} and $L=1$ for the bounded case, 
   the claim of the lemma follows. 

  \textbf{Unbounded noise case:}
  Assume that the noise is not necessarily bounded and $\ell$  is $L_{\ell}$-Lipschitz. From Subsection \ref{lipschitz:unbounded},
  recall the definition of truncated empirical and true errors 
  $\bar{\mathcal{L}}_N(\theta)$ and $\bar{\mathcal{L}}(\theta)$, corresponding to the 
  output  $\bar{\y}(t)$ and $\bar{\y}_{\theta}(t)$ and $\bar{\y}_{\theta}(t \mid 0)$ generated for the truncated noise
  $\bar{\e}_g(t)=\e_g(t)\chi(\e_g(t) < C)$ for $C=\frac{\ln(N)}{K\|\alpha_g\|_{\ell_1}}$.
  Using Chernoff's inequality and Lemma \ref{vl-vbl:lemma1}, it follows that 
  \begin{equation}
    \label{pac:lipschitz:loss:eq05}
  \begin{aligned}
     \bP(\forall \theta \in \Theta: \left|\hat{\mathcal{L}}_N(\theta) - \bar{\mathcal{L}}_N(\theta)\right| > c) \le  
     \frac{e^{\lambda\frac{1}{2} \left(  \mathcal{C}(\|\bar{\alpha}\|_{\ell_1} \|\alpha_g\|_{\ell_1}, 2L_{\ell}\lambda,N)+
\mathcal{C}(\|\alpha_g\|_{\ell_1}, 2L_{\ell}\lambda,N)\right)}}{e^{\lambda c}}
  \end{aligned}
  \end{equation}
  \textcolor{black}{for all $0 < \lambda < \frac{N}{4L_{\ell}}$} and hence 
   \begin{equation}
    \label{pac:lipschitz:loss:eq5}
  \begin{aligned}
     & \bP\left(\forall \theta \in \Theta:  \left|\hat{\mathcal{L}}_N(\theta) - \bar{\mathcal{L}}_N(\theta)\right| \le 
      \frac{1}{2\lambda} \left(\mathcal{C}(\|\bar{\alpha}\|_{\ell_1} \|\alpha_g\|_{\ell_1}, 2L_{\ell}\lambda,N)+ \right. \right. \\
& \left. \left. \mathcal{C}(\|\alpha_g\|_{\ell_1}, 2L_{\ell}\lambda,N) + \ln\left(\frac{1}{\delta }\right)\right)\right)
      \ge 1-\delta
  \end{aligned}
  \end{equation}
   Similarly, from Lemma \ref{vl-vbl:lemma1} it follows that
   \begin{equation}
    \label{pac:lipschitz:loss:eq5.1}
    \forall \theta \in \Theta: \quad
    |\mathcal{L}(\theta) - \mathcal{\bar{L}}(\theta)| \le  
   \frac{1}{2\lambda} \left( \mathcal{C}(\|\bar{\alpha}\|_{\ell_1} \|\alpha_g\|_{\ell_1}, 2L_{\ell}\lambda,N)+
    \mathcal{C}(\|\alpha_g\|_{\ell_1}, 2L_{\ell}\lambda,N)\right) 
   \end{equation}
   From \eqref{thm:main:pf:unbounded:eq-1} in Subsection \ref{lipschitz:unbounded} for $C=\frac{\ln(N)}{K\|\alpha_g\|_{\ell_1}}$ it follows that
   \begin{align*}
   \frac{1}{2} \left( \mathcal{C}(\|\bar{\alpha}\|_{\ell_1} \|\alpha_g\|_{\ell_1}, 2L_{\ell}\lambda,N)+
    \mathcal{C}(\|\alpha_g\|_{\ell_1}, 2L_{\ell}\lambda,N)\right) = 2\bar{\mathcal{C}}_{ub}(\frac{\lambda}{8}, L_{\ell},N)
   \end{align*}
   and hence by  setting $\lambda=\frac{\sqrt{N}}{16 L_{\ell}\ln(N)}$  it follows that
   \begin{align}
     & \frac{1}{\lambda} (\frac{1}{2} \left(\mathcal{C}(\|\bar{\alpha}\|_{\ell_1} \|\alpha_g\|_{\ell_1}, 2L_{\ell}\lambda,N)+
\mathcal{C}(\|\alpha_g\|_{\ell_1}, 2L_{\ell}\lambda,N)\right) \nonumber \\
     & \le (\frac{1}{2} C_{u,1}+\frac{1}{8} C_{u,2}) \frac{16 L_{\ell}\ln(N)}{\sqrt{N}}
     \le (C_{u,1}+C_{u,2}) \frac{16L_{\ell}\ln(N)}{\sqrt{N}}
    \label{pac:lipschitz:loss:eq6}
   \end{align}
    and hence, appplying it to \eqref{pac:lipschitz:loss:eq5}-\eqref{pac:lipschitz:loss:eq5.1} and recalling that
    $L=16L_{\ell}$ results in
    \begin{equation}
    \label{pac:lipschitz:loss:eq07}
      \begin{aligned}
      & \bP\left(\forall \theta \in \Theta: \left|\hat{\mathcal{L}}_N(\theta) - \bar{\mathcal{L}}_N(\theta)\right| \le (C_{u,1}+C_{u,2}+\ln(\frac{1}{\delta})) \frac{16 L_{\ell}\ln(N)}{\sqrt{N}}\right) \ge 1-\delta \\
      &  \forall \theta_1,\theta_2 \in \Theta: |\mathcal{L}(\theta_1) - \mathcal{L}(\theta_2)| \le  (C_{u,1}+C_{u,2}) \frac{L\ln(N)}{\sqrt{N}}\\
      \end{aligned}
    \end{equation}
     Note that  \eqref{pac:lipschitz:loss:eq1}-\eqref{pac:lipschitz:loss:eq2} holds for $\bar{\mathcal{L}}_N(\theta)$ and 
     $\bar{\hat{\mathcal{L}}}_N(\theta)$ with $L_{\Theta}(\delta)=L_{\alpha}L_{\ell}\sqrt{n_\e}\frac{\ln(N)}{K\|\alpha_g\|_{\ell_1}}\textcolor{black}{\|\alpha_g\|_{\ell_1}=L_{\alpha}L_{\ell}\sqrt{n_\e}{\frac{\ln(N)}{K}}}$ and hence 
     by using \eqref{pac:lipschitz:loss:eq07} and the observation
     \begin{align*}
       & |\hat{\mathcal{L}}_N(\theta_1) - \hat{\mathcal{L}}_N(\theta_2)| \le
       |\hat{\mathcal{L}}_N(\theta_1) - \bar{\hat{\mathcal{L}}}_N(\theta_1)| + |\bar{\hat{\mathcal{L}}}_N(\theta_1) - \bar{\hat{\mathcal{L}}}_N(\theta_2)| + |\bar{\hat{\mathcal{L}}}_N(\theta_2) - \hat{\mathcal{L}}_N(\theta_2)| \\
       & |\mathcal{L}(\theta_1)-\mathcal{L}(\theta_2)| \le |\mathcal{L}(\theta_1)-\bar{\mathcal{L}}(\theta_1)| + |\bar{\mathcal{L}}(\theta_1)-\bar{\mathcal{L}}(\theta_2)| + |\bar{\mathcal{L}}(\theta_2)-\mathcal{L}(\theta_2)|
     \end{align*}
     it follows
     \begin{equation}
    \label{pac:lipschitz:loss:eq7}
      \begin{aligned}
      & \bP\left(\forall \theta_1, \theta_2 \in \Theta: \left|\hat{\mathcal{L}}_N(\theta_1) - \hat{\mathcal{L}}_N(\theta_1)\right| \le 2(C_{u,1}+C_{u,2}+\ln(\frac{1}{\delta})) \frac{16 L_{\ell}\ln(N)}{\sqrt{N}}) + \right. \\
      & \left. L_{\alpha}L_{\ell}\sqrt{n_\e}\frac{\ln(N)}{\textcolor{black}{K}} \|\theta_1 - \theta_2\|_2 \right) \ge 1-\delta \\
      & \forall \theta_1, \theta_2 \in \Theta:  \left|\mathcal{L}(\theta_1) - \mathcal{L}(\theta_2)\right| \\
      & \le  2(C_{u,1}+C_{u,2}+\ln(\frac{1}{\delta})) \frac{L \ln(N)}{\sqrt{N}}) +L_{\alpha}L_{\ell}\sqrt{n_\e}\frac{\ln(N)}{\textcolor{black}{K}} \|\theta_1 - \theta_2\|_2
      \end{aligned}
    \end{equation}
    Since $L_{\alpha}L_{\ell}\sqrt{n_\e}\frac{\ln(N)}{\textcolor{black}{K}}=L_{\Theta}L\ln(N)$ with the definition of $L_{\Theta}$
    as in Corollary \ref{lem:pac} \textcolor{black}{for the unbounded Lipschitz case} \eqref{pac:lipschitz:loss:eq7} can be rewritten as
\begin{equation}
    \label{pac:lipschitz:loss:eq8}
      \begin{aligned}
      & \bP\big(\forall \theta_1, \theta_2 \in \Theta: \left|\hat{\mathcal{L}}_N(\theta_1) - \hat{\mathcal{L}}_N(\theta_1)\right| \le 2(C_{u,1}+C_{u,2}+\ln(\frac{1}{\delta})) \frac{L\ln(N)}{\sqrt{N}}) +  \\
      & L_{\Theta}(\delta) L\ln(N)\|\theta_1 - \theta_2\|_2\big) \ge 1-\delta \\
      & \forall \theta_1, \theta_2 \in \Theta:  \left|\mathcal{L}(\theta_1) - \mathcal{L}(\theta_2)\right|  \\
      & \le  2(C_{u,1}+C_{u,2}+\ln(\frac{1}{\delta})) \frac{L\ln(N)}{\sqrt{N}}) +L_{\Theta}(\delta)L\ln(N) \|\theta_1 - \theta_2\|_2
      \end{aligned}
    \end{equation}
    hence \eqref{pac:lipschitz:loss:eq31}-\eqref{pac:lipschitz:loss:eq32}. 
    By using the union bound, 
    and that \eqref{catoni1} holds with probability $1-2\delta$, 
    \eqref{pac:lipschitz:loss:eq3} follows. 


\textbf{Quadratic case:}
Recall from \eqref{eq:initial:tildern1:quadratic0}, Subsection \ref{thm:main:quadratic} the definition of the loss
  function $\ell_C$ and the true and empirical errors
  ${}^C \mathcal{L}(\theta)$ and ${}^C \mathcal{L}_N(\theta)$.
  Recall that for $C=6G_e\|\e_g\|_{\psi_2}\sqrt{\ln(\frac{3N^2}{2\delta})}$. 

  Recall that  by Lemma \ref{lem:quadratic1} that 
   \begin{equation}
  \label{quadratic:lipschitz:eq1}
  \begin{aligned}
   & \sup_{t=0,\ldots,N-1, \theta \in \Theta} \max\{\|\y(t)-\y_{\theta}(t)\|_2, \|\y(t)-\y_{\theta}(t[0])\|_\infty\} \le C 
   \end{aligned}
   \end{equation}
   holds with probability at least $1-\frac{\delta}{3}$. 
   Also recall that $\ell_C$ is $L_{\ell}$-Lipschitz with   $L_{\ell} = 12G_e \|\e_g\|_{\psi_2} \sqrt{\ln\left(\frac{3N^2}{2\delta}\right)}$.

   By applying \eqref{pac:lipschitz:loss:eq8} to ${}^C \mathcal{L}_N(\theta)$ and ${}^C \mathcal{L}(\theta)$
   and replacing $\delta$ with $\frac{2\delta}{3}$ and using the union bound, it follows that
\begin{equation*}
  \label{quadratic:lipschitz:eq03}
  \begin{aligned}
   & \sup_{t=0,\ldots,N-1,\theta \in \Theta} \max\{\|\y(t)-\y_{\theta}(t)\|_2, \|\y(t)-\y_{\theta}(t[0])\|_\infty\} \le 6G_e (\|\e_g\|_{\psi_2})\sqrt{\ln\left(\frac{3N^2}{2\delta}\right)} \\
    &      \sup_{\theta_1,\theta_2 \in \Theta}\left|{}^C \mathcal{L}_N(\theta_1) - {}^C \mathcal{L}_N(\theta_2)\right| \\
         & \le L_{\alpha}L_{\ell}\frac{\sqrt{n_\e}}{\textcolor{black}{K}} \ln(N) \|\theta_1 - \theta_2\|
           +  2\frac{16L_{\ell}\ln(N)}{\sqrt{N}}\left(C_{u,1} + C_{u,2}+\ln\left(\textcolor{black}{\frac{3}{2\delta}}\right)\right), 
  \end{aligned}
\end{equation*}
holds with probability at least $1-\delta$, and
\begin{equation}
  \begin{aligned}
        &  \sup_{\theta_1,\theta_2 \in \Theta} \left|{}^C \mathcal{L}(\theta_1) - {}^C \mathcal{L}(\theta_2)\right| \\
        & \le L_{\alpha}L_{\ell}\frac{\sqrt{n_\e}}{\textcolor{black}{K}} \ln(N) \|\theta_1 - \theta_2\|
          + 2\frac{16L_{\ell}\ln(N)}{\sqrt{N}}\left(C_{u,1} + C_{u,2}\right)
   \end{aligned}
   \end{equation}
     where 
    $L_{\ell}=12G_e \|\e_g\|_{\psi_2} \sqrt{\ln\left(\frac{3N^2}{2\delta}\right)}$. 
\begin{color}{black}
Recall from Subsection \ref{thm:main:quadratic}, proof of Theorem \ref{thm:main} that 
    if 
    \[
         \forall: \quad   t=0,1,\ldots,N-1: \quad  \max\{\|\y(t)-\y_{\theta}(t)\|_2, \|\y(t)-\y_{\theta}(t \mid 0)\|_\infty\} \le C
    \]
   holds, then ${}^C \hat{\mathcal{L}}_N(\theta)=\hat{\mathcal{L}}_N(\theta)$, and hence
   \begin{equation}
  \label{quadratic:lipschitz:eq3}
  \begin{aligned}
   & \sup_{t=0,\ldots,N-1,\theta \in \Theta} \max\{\|\y(t)-\y_{\theta}(t)\|_2, \|\y(t)-\y_{\theta}(t[0])\|_\infty\} \le 6G_e (\|\e_g\|_{\psi_2})\sqrt{\ln\left(\frac{3N^2}{2\delta}\right)} \\
    &      \sup_{\theta_1,\theta_2 \in \Theta}\left|\mathcal{L}_N(\theta_1) - \mathcal{L}_N(\theta_2)\right| \\
         & \le L_{\alpha}L_{\ell}\frac{\sqrt{n_\e}}{\textcolor{black}{K}} \ln(N) \|\theta_1 - \theta_2\|
           +  2\frac{16L_{\ell}\ln(N)}{\sqrt{N}}\left(C_{u,1} + C_{u,2}+\ln\left(\textcolor{black}{\frac{3}{2\delta}}\right)\right), 
  \end{aligned}
\end{equation}
holds with probability at least $1-\delta$.
\end{color}
    By using 
    \[ \sqrt{\ln\left(\frac{3N^2}{2\delta}\right)} \le \sqrt{2}(\ln(N))^{1/2} \sqrt{1+\ln\left(\frac{3}{2\delta}\right)} \]
    it follows that 
    \begin{align*}
        L_{\alpha}L_{\ell}\frac{\sqrt{n_\e}}{K} \ln(N)  \le
        12 L_{\alpha}G_e \|\e_g\|_{\psi_2} \sqrt{1+\ln\left(\frac{3}{2\delta}\right)}\sqrt{2}\frac{\sqrt{n_\e}}{K}
        (\ln(N))^{3/2}
    \end{align*}

   Moreover, recall \eqref{thm:main:quadratic:eq22}-\eqref{thm:main:quadratic:eq21} holds, and thus
   using \eqref{thm:bounded:alt:col:pf:eq6.0}, it follows that
   \begin{equation}
  \label{quadratic:lipschitz:eq04}
   \begin{aligned}
       \left|\mathcal{L}(\theta)-{}^C\mathcal{L}(\theta)\right| & \le 72G_e^2\|\e_g\|_{\psi_2}^2\frac{\ln\left(\frac{3N^2}{2\delta}\right)}{N} \le 72 \frac{L^2}{192^2} \frac{\ln(N)}{N} = \\
       & \frac{6}{16 \cdot 192} L^2 \frac{\ln(N)}{N} \le  \frac{L}{512} \frac{L(\ln(N))^{\frac{3}{2}}}{\sqrt{N}}
   \end{aligned}
  \end{equation}
  In then follows that
  \begin{equation}
  \label{quadratic:lipschitz:eq4}
   \begin{aligned}
       & \left|\mathcal{L}(\theta_1)-\mathcal{L}(\theta_2)\right| \le 
       \left|\mathcal{L}(\theta_1)-{}^C\mathcal{L}(\theta_1)\right| + \left|{}^C\mathcal{L}(\theta_1)-{}^C\mathcal{L}(\theta_2)\right| + \left|{}^C\mathcal{L}(\theta_2)-\mathcal{L}(\theta_2)\right| \\
       & \le \frac{\sqrt{n_\e}}{K} L\ln(N) \|\theta_1 - \theta_2\|
          + 2\frac{16L_{\ell}\ln(N)}{\sqrt{N}}\left(C_{u,1} + C_{u,2}\right) + \frac{2L}{512} \frac{L(\ln(N))^{\frac{3}{2}}}{\sqrt{N}}
   \end{aligned}
  \end{equation}
   By repeating the argument of the proof of Theorem \ref{thm:main} for the quadratic case,
   and using \eqref{quadratic:lipschitz:eq4} and 
   \begin{align*}
     & \frac{1}{L} 12 L_{\alpha}G_{e} \|\e_g\|_{\psi_2} \sqrt{1+\ln\left(\frac{3}{2\delta}\right)}\sqrt{2}\frac{\sqrt{n_\e}}{K}= \\
       & \frac{12 L_{\alpha}G_{e} \|\e_g\|_{\psi_2} \sqrt{1+\ln\left(\frac{3}{2\delta}\right)}\sqrt{2}\frac{\sqrt{n_\e}}{K}}{192 G_e \|\e_g\|_{\psi_2}\sqrt{2}\sqrt{1+\ln\left(\frac{3}{2\delta}\right)}}=
        L_{\alpha}\frac{\sqrt{n_\e}}{16K}
   \end{align*}
   it then follows that
   \eqref{quadratic:lipschitz:eq3} implies that 
   \begin{align}
   & \mbox{\eqref{quadratic:lipschitz:eq1} holds} \\
   & \sup_{\theta_1,\theta_2 \in \Theta} |\mathcal{L}_N(\theta_1) - \mathcal{L}_N(\theta_2)| \\
   & \le L_{\Theta}(\delta) L(\ln(N))^{3/2} \|\theta_1 - \theta_2\|
           +  \frac{2L\ln(N)^{r}}{\sqrt{N}}\left(C_{u,1} + C_{u,2}+\ln\left (\textcolor{black}{\frac{3}{2\delta}}\right)\right)
  \label{quadratic:lipschitz:eq41}
   \end{align}
   which is implies \eqref{pac:lipschitz:loss:eq31}, and 
   \begin{align}
           \\
       & \sup_{\theta_1,\theta_2 \in \Theta}
          |\mathcal{L}(\theta_1) - \mathcal{L}(\theta_2)| \\
          &  \le L_{\Theta} L(\ln(N))^{3/2} \|\theta_1 - \theta_2\|
          + \frac{2L\ln(N)^{3/2}}{\sqrt{N}}\left(C_{u,1} + C_{u,2}\right) + \\
          & +  2\frac{L}{512} \frac{L (\ln(N))^{3/2}}{\sqrt{N}}
  \label{quadratic:lipschitz:eq42}
   \end{align}
  which is equivalent to
  \eqref{pac:lipschitz:loss:eq32}. 
   Since \eqref{quadratic:lipschitz:eq3} holds with probability at least $1-\delta$, it then follows that
   \eqref{pac:lipschitz:loss:eq31} also holds with probability at least $1-\delta$.

   Finally, from Corollary \ref{thm:bounded:alt:col} it follows that \eqref{catoni1} holds with probability at least $1-\delta$, and hence by the union bound \textcolor{black}{\eqref{pac:lipschitz:loss:eq3}} follows.

\end{proof}
\begin{proof}[Proof of Corollary \ref{lem:pac}]
 Consider the inequality \eqref{pac:lipschitz:loss:eq3} which holds with probability at least $1-(2+r_{\mathrm{Lip,1}})\delta$.
If \eqref{pac:lipschitz:loss:eq3} holds, then for all $\theta \in \Theta$,
\[ \epsilon (\mathcal{L}(\theta)) - (\mathcal{L}(\tilde{\theta}))  \le |\epsilon (\mathcal{L}(\theta)-\mathcal{L}(\tilde{\theta}))|=|\mathcal{L}(\theta)-\mathcal{L}(\tilde{\theta})| \]
and hence 
\[
   \epsilon \mathcal{L}(\theta) \le \epsilon \mathcal{L}(\tilde{\theta}) 
      + L_{\Theta}(\delta) L(\ln(N))^r \|\theta - \tilde{\theta}\|_2
          + 2r_{\text{Lip},1}\frac{L\ln(N)^{r}}{\sqrt{N}}\left(C_{u,1} + C_{u,2}\right) + \\
            2r_{\text{Lip},2}  \frac{L}{512} \frac{L(\ln(N))^r}{\sqrt{N}}
\]
Then $E_{\tilde{\theta} \sim \rho_{\theta}} \|\tilde{\theta}-\theta\|^2_2 \le \frac{\sigma^2}{N}$ and hence
 \[ E_{\tilde{\theta} \sim \rho_{\theta}} \|\tilde{\theta}-\theta\|_2 \le  \sqrt{E_{\tilde{\theta} \sim \rho_{\theta}} \|\tilde{\theta}-\theta\|^2_2} \le \frac{\sigma}{\sqrt{N}} \]
Therefore
\begin{equation}
  \label{lem:pac:eq1}
\begin{aligned}
    & \mathcal{L}(\theta) \le E_{\tilde{\theta} \sim \rho_{\theta}}\mathcal{L}(\tilde{\theta}) + 
     L_{\Theta}(\delta) L(\ln(N))^r   \|\theta - \tilde{\theta}\|_2 \\
          & + 2r_{\text{Lip},1}\frac{L\ln(N)^{r}}{\sqrt{N}}\left(C_{u,1} + C_{u,2}\right) + 
          2r_{\text{Lip},2}  \frac{L}{512} \frac{L(\ln(N))^r}{\sqrt{N}} \\
          & \le \frac{\sigma L_{\Theta}(\delta)}{\sqrt{N}}+2r_{\text{Lip},1}\frac{L\ln(N)^{r}}{\sqrt{N}}\left(C_{u,1} + C_{u,2}\right) + 
          2r_{\text{Lip},2}  \frac{L}{512} \frac{L(\ln(N))^r}{\sqrt{N}}
\end{aligned}
\end{equation}
Similary, 
\begin{align*}
& \epsilon \hat{\mathcal{L}}_N(\tilde{\theta}) \le \epsilon \hat{\mathcal{L}}_N(\theta) + 
L_{\Theta}(\delta) L(\ln(N))^r \|\theta - \tilde{\theta}\|_2
           +  2r_{\text{Lip},1}\frac{L\ln(N)^{r}}{\sqrt{N}}\left(C_{u,1} + C_{u,2}+\ln\left(\frac{\textcolor{black}{C_{\delta}}}{\delta}\right)\right)
\end{align*}
and hence 
\begin{equation}
  \label{lem:pac:eq2}
\begin{aligned}
    & \epsilon E_{\tilde{\theta} \sim \rho_{\theta}}\hat{\mathcal{L}}_N(\tilde{\theta})  \le  \epsilon \hat{\mathcal{L}}_N(\theta)+
     L_{\Theta}(\delta) L(\ln(N))^r E_{\tilde{\theta} \sim \rho_{\theta}}\|\theta - \tilde{\theta}\|_2 \\
          & + 2r_{\text{Lip},1}\frac{L\ln(N)^{r}}{\sqrt{N}}\left(C_{u,1} + C_{u,2}+\ln\left(\frac{\textcolor{black}C_{\delta}}{\delta}\right)\right) \\
          & \le \frac{\sigma L_{\Theta}(\delta)}{\sqrt{N}}+2r_{\text{Lip},1}\frac{L\ln(N)^{r}}{\sqrt{N}}\left(C_{u,1} + C_{u,2}
          \ln\left(\frac{\textcolor{black}{C_{\delta}}}{\delta}\right)
          \right)
\end{aligned}
\end{equation}
In addition, since \eqref{catoni1} also holds, it follows that
\begin{equation}
  \label{lem:pac:eq3}
    \epsilon E_{\tilde{\theta} \sim \rho_{\theta}}\mathcal{L}(\tilde{\theta}) \le \epsilon E_{\tilde{\theta} \sim \rho_{\theta}}\hat{\mathcal{L}}_N(\tilde{\theta}) + \frac{L (\ln(N))^r}{\sqrt{N}} \left( \KL(\rho_{\theta} \| \pi) + \ln\left(\frac{C_{\delta}}{\delta}\right)+\textcolor{black}{\mathcal{C}(\delta,\epsilon)}\right)
\end{equation}
By combining \eqref{lem:pac:eq1}, \eqref{lem:pac:eq2}, and \eqref{lem:pac:eq3}, it follows that
\begin{align*}
   \epsilon \mathcal{L}(\theta) \le & \epsilon E_{\tilde{\theta} \sim \rho_{\theta}}\mathcal{L}(\tilde{\theta}) + \frac{\sigma L_{\Theta}(\delta)}{\sqrt{N}}+2r_{\text{Lip},1}\frac{L\ln(N)^{r}}{\sqrt{N}}\left(C_{u,1} + C_{u,2}\right)  \\
          & + 2r_{\text{Lip},2}  \frac{L}{512} \frac{L(\ln(N))^r}{\sqrt{N}} \\
          & \le \epsilon E_{\tilde{\theta} \sim \rho_{\theta}}\hat{\mathcal{L}}_N(\tilde{\theta}) + \frac{L (\ln(N))^r}{\sqrt{N}} \left( \KL(\rho_{\theta} \| \pi) + \ln\left(\frac{C_{\delta}}{\delta}\right)+\textcolor{black}{\mathcal{C}(\delta,\epsilon)}\right) + \\
          & + \frac{\sigma L_{\Theta}(\delta)}{\sqrt{N}}+2r_{\text{Lip},1}\frac{L\ln(N)^{r}}{\sqrt{N}}\left(C_{u,1} + C_{u,2}\right) + 2r_{\text{Lip},2}  \frac{L}{512} \frac{L(\ln(N))^r}{\sqrt{N}} \\
          & \le \epsilon \hat{\mathcal{L}}_N(\theta) + 
\frac{\sigma L_{\Theta}(\delta)}{\sqrt{N}}
           +  2r_{\text{Lip},1}\frac{L\ln(N)^{r}}{\sqrt{N}}\left(C_{u,1} + C_{u,2}+\ln\left(\frac{C_{\delta}}{\delta}\right)\right) \\
           & + \frac{L (\ln(N))^r}{\sqrt{N}} \left( \KL(\rho_{\theta} \| \pi) + \ln\left(\frac{C_{\delta}}{\delta}\right)+\textcolor{black}{\mathcal{C}(\delta,\epsilon)} \right) + \\
           & + \frac{\sigma L_{\Theta}(\delta)}{\sqrt{N}}+2r_{\text{Lip},1}\frac{L\ln(N)^{r}}{\sqrt{N}}\left(C_{u,1} + C_{u,2}\right) + 2r_{\text{Lip},2}  \frac{L}{512} \frac{L(\ln(N))^r}{\sqrt{N}} \\
           & = \epsilon \hat{\mathcal{L}}_N(\theta)+\frac{2\sigma L_{\Theta}(\delta)}{\sqrt{N}} 
            + 4r_{\text{Lip},1}\frac{L\ln(N)^{r}}{\sqrt{N}}\left(C_{u,1} + C_{u,2}\right) + 2r_{\text{Lip},2}  \frac{L}{512} \frac{L(\ln(N))^r}{\sqrt{N}}  \\
           & + \frac{L(\ln(N))^r}{\sqrt{N}} \left(C(\theta) \ln(N) + \textcolor{black}{\ln\left(\frac{C_{\delta}}{\delta}\right)+ 2r_{\text{Lip},1} \ln\left(\frac{C_{\delta}}{\delta}\right)} + \mathcal{C}(\epsilon,\delta) \right) \\                                             
           & \le \epsilon \hat{\mathcal{L}}_N(\theta) + \frac{L(\ln(N))^{r+1}}{\sqrt{N}} \left(2\sigma L_{\Theta}(\delta)+4r_{\text{Lip},1}\left(C_{u,1} + C_{u,2}\right)  \right. \\
           & \left. + 2r_{\text{Lip},2}  \frac{L}{512}+ \left(C(\theta) + \textcolor{black}{\ln\left(\frac{C_{\delta}}{\delta}\right)+ 2r_{\text{Lip},1} \ln\left(\frac{C_{\delta}}{\delta}\right)} + \mathcal{C}(\epsilon,\delta) \right)
           \right)
\end{align*}
where in the last step we used that $r \ge 0$ and $\ln(N) \ge 1$ for $N \ge 4$.
\end{proof}

\section{Proof of Corollary \ref{lem:bound:mean}}
\label{sec:proof_single_draw}

We start with the following remark explaining the notion of random models drawn
from a data-dependent posterior. 
\begin{remark}[Joint probability over data and parameters]
  \label{rem:jointProb}
Let $\rho^{\train}$ be a function on $\Omega \times \Theta$
which is measurable w.r.t. the direct product $\F \times \B$ of $\sigma$-algebras, and such that
for any $\omega \in \Omega$,
$\rho^{\train}(\omega): \Theta \ni \theta \mapsto \rho^{\train}(\omega,\theta)$ is a 
probability density on $\Theta$.
We refer to $\rho^{\train}$ as a \emph{random density on \textcolor{black}{$\Theta$}}. Let $B$ be a subset of $\F \times \B$ and denote by
$P_{\theta \sim \rho^{\train}}(B)$ the random variable
$\omega \mapsto P_{\theta \sim \rho^{\train}(\omega)} (\{ \theta \mid (\omega, \theta) \in B\})$, where 
$P_{\theta \sim \rho^{\train}(\omega)}$ is the probability measure induced by the density $\rho^{\train}(\omega)$.
Define the probability measure
$\bP \times P_{\theta \sim \rho^{\train}}$
on $\F \times \B$ as follows:
\[
   \left(\bP \times P_{\theta \sim \rho^{\train}}\right)(B) \triangleq
   \bE[P_{\theta \sim \rho^{\train}}(B)]
\]
Similarly to the previous definition of relationship holding with probability over data, if $\diamond$ is a binary operator and 
$\{X_i,Y_i\}_{i \in I}$ is family of functions taking values in $\mathbb{R}$ and which are measurable w.r.t. $(\Omega \times \Theta,\F \times \B)$, 
then we say that $\forall i \in I: X_i \diamond  Y_i$, holds with probability  at least $1-\delta$ over data and over samples
from $\rho^{\train}$, if the set
$B=\{ (\omega,\theta) \in \Omega \times \Theta \mid \forall i \in I: X_i(\omega,\theta) \diamond  Y_i(\omega,\theta)\}$ is
measurable in $(\Omega \times \Theta,\F \times \B)$ and
$\left(\bP \times P_{\theta \sim \rho^{\train}}\right)(B) \geq 1-\delta$. 

Note that if a relationship holds with a probability at least $1-\delta$ over data, then it also holds with a probability at least $1-\delta$ over data and sampled from $\rho^{\train}$: if $F$ is an event from $\mathcal{F}$, then 
$
\left(\bP \times P_{\theta \sim \rho^{\train}}\right)(F) = \bP(F) \geq 1-\delta.
$

\end{remark}

\subsection{Proof of the \eqref{bound:emp} on empirical error}
First, we present a useful property of Gibbs posteriors.  
 \begin{lemma}[Property of Gibbs distribution]
  \label{lem:gibbs_property}
  \begin{align}
    \hat{\mathcal{L}}_N(\theta)+\frac{1}{\lambda}\ln\left(\frac{\rho^{\text{Gibbs}}(\theta)}{\pi(\theta)}\right) 
       = &  -\frac{1}{\lambda}\ln E_{\theta \sim \pi} e^{-\lambda \hat{\mathcal{L}}_N(\theta)} \nonumber \\
    & =\inf_{\rho \in \mathcal{M}_{\pi}} \left(\bE_{\theta \sim \rho} \hat{\mathcal{L}}_N(\theta) + \frac{1}{\lambda} \KL(\rho\|\pi)\right)
  \label{lem:gibbs_property:eq1} \\
   E_{\theta \sim \rho^{\text{Gibbs}}}\hat{\mathcal{L}}_N(\theta)+\frac{1}{\lambda} 
      \KL(\rho^{\text{Gibbs}}\|\pi) & =  -\frac{1}{\lambda} \ln   \bE_{\theta \sim \pi} e^{-\lambda \hat{\mathcal{L}}_N(\theta)} = \nonumber \\
    & =\inf_{\rho \in \mathcal{M}_{\pi}} \left(\bE_{\theta \sim \rho}(\hat{\mathcal{L}}_N(\theta)) + \frac{1}{\lambda} \KL(\rho\|\pi)\right)
  \label{lem:gibbs_property:eq2}
  \end{align}
 \end{lemma}
 \begin{proof}[Proof of Lemma \ref{lem:gibbs_property}]
  Using the definition of the Gibbs distribution $\rho^{\text{Gibbs}}(\theta)=\frac{\pi(\theta)e^{-\lambda \hat{\mathcal{L}}_N(\theta)}}{Z}$, $Z=E_{\theta \sim \pi} e^{-\lambda \hat{\mathcal{L}}_N(\theta)}$, it follows that
  \[ \ln\left(\frac{\rho^{\text{Gibbs}}(\theta)}{\pi(\theta)}\right) = -\lambda \hat{\mathcal{L}}_N(\theta) - \ln Z \]
  and hence 
  \[ \hat{\mathcal{L}}_N(\theta)+\frac{1}{\lambda}\ln\left(\frac{\rho^{\text{Gibbs}}(\theta)}{\pi(\theta)}\right) = \hat{\mathcal{L}}_N(\theta) - \hat{\mathcal{L}}_N(\theta) - \frac{1}{\lambda} \ln Z = - \frac{1}{\lambda} \ln Z \]
This gives the first equality of \eqref{lem:gibbs_property:eq1}.

As to  the first equality of \eqref{lem:gibbs_property:eq2}, it follows directly from the first equality of  \eqref{lem:gibbs_property:eq1} by taking the expectation with respect to $\theta \sim \rho^{\text{Gibbs}}$ and recalling that $\KL(\rho^{\text{Gibbs}}\|\pi) = \bE_{\theta \sim \rho^{\text{Gibbs}}} \left[\ln \frac{\rho^{\text{Gibbs}}(\theta)}{\pi(\theta)}\right]$.

The second equalities of \eqref{lem:gibbs_property:eq1}-\eqref{lem:gibbs_property:eq2} are 
Donsker-Varadhan's variational formula \cite[Lemma 2.2]{alquier2021userfriendly}

 \end{proof}
 Next, we present a version of Theorem \ref{thm:main} where the inequality
holds not only  over data but over data and 
 any model randomly drawn from a data-dependent posterior.
 \begin{lemma}[Single draw bound]
\label{lem:single_draw}
 Let $\rho^{\train}$ be a \textcolor{black}{a random density on $\Theta$}. 
 For any prior $\pi$,  any $\delta \in (0,0.5)$
 with a probability at least  $(1-2\delta)$ over the data and over all samples $\theta_{\star}$ drawn from $\rho_{\train}$,
 for any $\lambda$ satisfying the condition of Theorem \ref{thm:main}
  \begin{equation}
\label{RenyiBound:single_draw}
 \epsilon \mathcal{L}(\theta_{\star}) \le  \epsilon  \hat{\mathcal{L}}_N(\theta_{\star})+ 
 \frac{1}{\lambda} \left(\ln \left(\frac{\rho^{\train}(\theta_{\star})}{\pi(\theta_{\star})}\right)+ \ln\left(\frac{C_{\delta}}{\delta}\right) + \frac{1}{2}\sum_{i=1}^{2} \ln \bE_{\theta \sim \pi} \exp(\mathcal{C}_i(2\lambda,\theta_{\star},N,\delta,\epsilon))\right)
  \end{equation}
  In particular, with probability at least $(1-2\delta)$ over data and over all samples  $\theta_{\star}$ drawn from $\rho_{\train}$, 
  \begin{equation}
  \label{RenyiBound:single_draw1}
  \epsilon \mathcal{L}(\theta_{\star}) \le \hat{\mathcal{L}}_N(\theta_{\star})+\frac{L(\ln(N))^{r}}{\sqrt{N}}\left (\ln\left(\frac{\rho^{\train}(\theta_{\star})}{\pi(\theta_{\star})}\right)+\ln\left(\frac{C_{\delta}}{\delta}\right) + 
  \mathcal{C}(\delta,\epsilon)\right)
  \end{equation}
  for all $N \ge K_{\delta}$, where the constants $K_{\delta},C_{\delta},L,r$ are defined as in Corollary \ref{thm:bounded:alt:col}.
 Moreover, if  $\epsilon=1$ and 
 $\rho^{\train}=\rho_{\text{Gibbs}}$, then with probability at  least  $(1-2\delta)$
  over the data and over $\theta_{\star}$
  sampled from $\rho_{\text{Gibbs}}$,
  \begin{equation}
\label{RenyiBound:single_draw:giibs}
\begin{aligned}
 \mathcal{L}(\theta_{\star}) \le  
 \textcolor{black}{-\frac{1}{\lambda} \ln\left(E_{\theta \sim \pi} e^{-\lambda \hat{\mathcal{L}}_N(\theta)}\right)} +
 \frac{1}{\lambda} \left(\ln\left(\frac{C_{\delta}}{\delta}\right) + \frac{1}{2}\sum_{i=1}^{2} \ln \bE_{\theta \sim \pi} \exp(\mathcal{C}_i(2\lambda,\theta,N,\delta,\epsilon))\right) \\
\end{aligned}
  \end{equation}
  In particular, if the Gibbs posterior $\rho^{\text{Gibbs}}$ is defined for 
  $\lambda=\lambda_N=\frac{\sqrt{N}}{L(\ln(N))^r}$ then for $N \ge K_{\delta}$
\begin{equation}
\label{RenyiBound:single_draw:giibs1}
\begin{aligned}
 \mathcal{L}(\theta_{\star}) \le  
  -\frac{1}{\lambda} \ln E_{\theta \sim \pi} e^{-\lambda \hat{\mathcal{L}}_N(\theta)} + 
 \frac{L(\ln(N))^{r}}{\sqrt{N}}\left(\ln\left(\frac{C_{\delta}}{\delta}\right) +  \mathcal{C}(\delta,1)\right)
\end{aligned}
\end{equation}
  \end{lemma}
\begin{proof}[Proof of Lemma \ref{lem:single_draw}]
 Following the argument of the proof of \cite[Theorem 2.7]{alquier2021userfriendly}, it follows
 that for any measurable function
 $X$ on $\Omega \times \Theta$,
 $E_{\theta \sim \pi} e^{2\lambda X(\omega,\theta)} = E_{\theta \sim \rho^{\train}} e^{2\lambda X(\omega,\theta)-\ln \frac{\rho^{\train}(\omega,\theta)}{\pi(\theta)}}$ and hence 
 $\bE E_{\theta \sim \pi} e^{2\lambda X(\theta)}=\bE E_{\theta \sim \rho^{\train}}
 e^{2\lambda X(\theta)-\ln \frac{\rho^{\train}(\theta)}{\pi(\theta)}}$
 with the random variables $X(\theta): \omega \ni \Omega \mapsto X(\omega,\theta)$.
 Notice that $\bE E_{\theta \sim \rho^{\train}}$ is the expectation operator
 corresponding to the probability measure
 $\bP \times P_{\theta \sim \rho^{\train}}$, and that
$\bE[E_{\theta \sim \pi} X]=E_{\theta \sim \pi} \bE[X]$ 
 , and hence, by Chernoff's bound:
 \begin{align*}
     & \bP \times P_{\theta \sim \rho^{\train}}\big (X < \frac{1}{2\lambda}\big(\ln \left(\frac{1}{\delta}\right) +  \ln \left(\frac{\rho^{\train}(\theta)}{\pi(\theta)}\right)+ \ln \bE E_{\theta \sim \pi} e^{2\lambda X}\big)\big) \\
     & = \bP \times P_{\theta \sim \rho^{\train}}\Big (\left(X - \frac{1}{2\lambda} \ln \left(\frac{\rho^{\train}(\theta)}{\pi(\theta)}\right)\right) <  \frac{1}{2\lambda}\left(\ln \left( \frac{1}{\delta}\right) + \ln \bE E_{\theta \sim \pi} e^{2\lambda X}\right)\Big)
     \le \frac{\bE E_{\theta \sim \rho^{\train}} e^{2\lambda X - \ln \frac{\rho^{\train}(\theta)}{\pi(\theta)}}}{e^{\ln \frac{1}{\delta} + \ln \bE E_{\theta \sim \pi} e^{2\lambda X}}} \\
     & =\frac{\bE E_{\theta \sim \pi} e^{2\lambda X}}{\frac{1}{\delta} \bE E_{\theta \sim \pi} e^{2\lambda X}}=\delta
 \end{align*}
 Hence, 
 \begin{equation}
  \label{lem:single_draw:quadratic:pf0}
   \begin{split}
    X \le  \frac{1}{2\lambda}\big (\ln \left(\frac{1}{\delta}\right) +\ln \left(\frac{\rho^{\train}(\theta)}{\pi(\theta)}\right) + \ln E_{\theta \sim \pi} \bE e^{2\lambda X(\theta)}\big)
   \end{split} 
 \end{equation}
 with probability at least $1-\delta$ over data and samples from $\rho^{\train}$.
 Hence, by taking $X=\epsilon (\mathcal{L}(\theta)-V_N(\theta))$ and then $X=\epsilon (V_N(\theta))-\hat{\mathcal{L}}_N(\theta))$, and 
 by taking the union bound as in the proof of Theorem \ref{thm:initial}, it follows that
 \begin{equation}
\label{lem:single_draw:quadratic:pf1}
 \begin{split}
  \textcolor{black}{\epsilon}\mathcal{L}(\theta) \le & \textcolor{black}{\epsilon} \hat{\mathcal{L}}_N(\theta) + \frac{1}{\lambda} \Big(\ln\left(\frac{1}{\delta}\right) +\ln \left(\frac{\rho^{\train}(\theta)}{\pi(\theta)}\right)  \\
  &+\frac{1}{2} \Big (\ln E_{\theta\sim\pi}\bE[e^{2\epsilon\lambda(V_N(\theta)-\hat{\mathcal{L}}_N(\theta))}]
        +\ln E_{\theta\sim\pi}\bE[e^{2\epsilon\lambda (\mathcal{L}(\theta)-V_N(\theta))}] \Big )\Big)
 \end{split}
 \end{equation}
\textcolor{black}{holds with probability at least $1-2\delta$ over data and samples from $\rho^{\train}$}.

\textbf{Proof of \eqref{RenyiBound:single_draw} for bounded noise}
 Let us  apply the  inequality \eqref{lem:single_draw:quadratic:pf1} 
 and use
 the bound on $\bE[e^{2\epsilon\lambda (V_N(\theta)-\hat{\mathcal{L}}_N(\theta))}]$ and $\bE[e^{2\epsilon\lambda (\mathcal{L}(\theta)-V_N(\theta))}]$ from
  Lemma \ref{lemma:bounded:mgf(L-V)}-\ref{lemma:bounded_mgf(V-hatL)}, from which \eqref{RenyiBound:single_draw} follows.

\textbf{Proof of \eqref{RenyiBound:single_draw} for Lipschitz loss}
 Let us apply  \eqref{thm:main:pf:unbounded:eq1} from Subsection \ref{lipschitz:unbounded} to bound
 the bound on $\bE[e^{2\epsilon\lambda (V_N(\theta)-\hat{\mathcal{L}}_N(\theta))}]$ and $\bE[e^{2\epsilon\lambda (\mathcal{L}(\theta)-V_N(\theta))}]$, from which \eqref{RenyiBound:single_draw} follows.
   
\textbf{Proof of \eqref{RenyiBound:single_draw} quadratic loss}
Recall from Subsection \ref{thm:main:quadratic} the definition of the losses ${}^C \mathcal{L}(\theta)$, ${}^C \hat{\mathcal{L}}_N(\theta)$.
corresponding to the loss $\ell_C$. Note that by Lemma \ref{lem:quadratic2}  $\ell_C$ is $2C$-Lipschitz.
Notice that from the statement of the present lemma  proven for the case of Lipschitz losses
for $\lambda \le \frac{\sqrt{N}}{\textcolor{black}{32} \cdot 2C\ln(N)}$,  with 
$C=6G_e\|\e_g\|_{\psi_2}\sqrt{\ln(\frac{3N^2}{2\delta})}$ probability $1-\frac{4}{3\delta}$ over data and parameters sampled
from $\rho^{\train}$, the following holds:
\begin{equation}
  \label{lem:single_draw:quadratic}
  \begin{aligned}
\epsilon {}^C\mathcal{L}(\theta)
         \leq  &  \epsilon {}^C \hat{\mathcal{L}}_N(\theta) +  
          \frac{1}{\lambda}\left( \ln\left(\frac{\rho(\theta)}{\pi(\theta)}\right) + \ln\left( \frac{3}{2\delta}\right) + \right. \\
          &  \left. \frac{1}{2} \sum_{i=1}^2 \ln \bE_{\theta \sim \pi} \exp\left(\bar{C}_i(2\lambda,\theta,N,\frac{2\delta}{3},\epsilon)\right)\right)
  \end{aligned}
\end{equation}
where $\bar{C}_i(2\lambda,\theta,N,\frac{2\delta}{3},\epsilon)$ satisfies \eqref{thm:main:quadratic:eq2}.
From Lemma \ref{lem:quadratic1} it follows that \eqref{lem:quadratic1:eq1} holds with probability at least
$1-\frac{2\delta}{3}$ 
over data and hence \eqref{lem:quadratic1:eq1} holds with probability at least $1-\frac{2\delta}{3}$ over data and parameters sampled from $\rho ^{\theta}$. Hence, by using the union bound, it follows that
\eqref{lem:single_draw:quadratic} and \eqref{lem:quadratic1:eq1} hold
simultaneously 
with probability at least  $1-2\delta$ over data and parameters sampled from $\rho^{\train}$.

Moreover, if 
\eqref{lem:quadratic1:eq2} holds, then 
it holds that 
${}^C \hat{\mathcal{L}}_N(\theta)=\hat{\mathcal{L}}_N(\theta)$ 
and by and \eqref{thm:main:quadratic:eq21} -- \eqref{thm:main:quadratic:eq22} 
\[ 
    {}^C \mathcal{L}(\theta) \le \mathcal{L}(\theta) \le  {}^C \mathcal{L}(\theta)+ 72G_e^2\|\e_g\|_{\psi_2}^2 \frac{\ln\left(\frac{3N^2}{2\delta}\right)}{N}
\]
Hence, 
\begin{equation} 
  \label{lem:single_draw:quadratic1}
    \epsilon \mathcal{L}(\theta) \le \epsilon {}^C \mathcal{L}(\theta)+ 72G_e^2\|\e_g\|_{\psi_2}^2\frac{\ln\left(\frac{3N^2}{2\delta}\right)}{N} \max\{\epsilon,0\}
\end{equation}
 Hence, by using \eqref{lem:single_draw:quadratic} and ${}^C \hat{\mathcal{L}}_N(\theta)=\hat{\mathcal{L}}_N(\theta)$,
 it follows that if  \eqref{lem:single_draw:quadratic1} and \eqref{lem:quadratic1:eq1} both hold, then 
 \eqref{RenyiBound:single_draw} holds. Since \eqref{lem:single_draw:quadratic1} and \eqref{lem:quadratic1:eq1} holds
 with probability at least $1-2\delta$ over data and samples from $\rho^{\train}$, so does  \eqref{RenyiBound:single_draw}.

 \textbf{Proof of \eqref{RenyiBound:single_draw} implies \eqref{RenyiBound:single_draw:giibs}}
   It follows by applying Lemma \ref{lem:gibbs_property}. 

  \textbf{Proof of \eqref{RenyiBound:single_draw1}}
     From the proof of Corollary \ref{thm:bounded:alt:col} it follows that
     for $\lambda=\frac{\sqrt{N}}{L(\ln(N))^{r}}$  where
     $r$ is as in Corollary \ref{thm:bounded:alt:col},
     \[
     \frac{1}{2\lambda } \left(   \ln E_{\theta \sim \pi} \exp(\mathcal{C}_1(2\lambda,\theta,N,\delta,\epsilon)) +
           \ln E_{\theta \sim \pi} \exp(\mathcal{C}_2(2\lambda,\theta,N,\delta,\epsilon)) \right)
           \le \mathcal{C}(\delta,\epsilon)
     \]
     Using this inequality, it follows that \eqref{RenyiBound:single_draw} implies \eqref{RenyiBound:single_draw1}.
\end{proof}
\begin{lemma}
\label{lem:bound:mean1}
 If $\mathcal{L}(\theta)$ is convex, $\rho$ is a probability density over $\Theta$, and 
 $\theta_{\star}=E_{\theta \sim \rho} \theta$ is the mean 
 of the density $\rho$, then
 \begin{equation}
  \label{lem:bound:mean:eq1}
     \mathcal{L}(\theta_{\star}) \le E_{\theta \sim \rho}[\mathcal{L}(\theta)].
\end{equation}
\end{lemma}
\begin{proof}[Proof of Lemma \ref{lem:bound:mean1}]
 If $\mathcal{L}(\theta)$ is convex in $\theta$, then 
$\mathcal{L}(\theta_{\star})=\mathcal{L}(E_{\theta \sim \rho^{\train}} \theta)   \le E_{\theta \sim \rho^{\train}} \mathcal{L}(\theta)$ by Jensen's inequality.
\end{proof}
\begin{proof}[Proof of \eqref{bound:emp}]
    For the single draw case, from \eqref{RenyiBound:single_draw:giibs} of  Lemma \ref{lem:single_draw}, it follows
    that with probability at least $1-2\delta$ over data and samples from $\rho^{\text{Gibbs}}$,
    \begin{equation}
    \label{bound:emp:single_draw:pf:eq1}
        \mathcal{L}(\theta_{\star}) \le -\ln E_{\theta \sim \pi} e^{-\lambda \hat{\mathcal{L}}_N(\theta)} +
          \frac{L(\ln(N))^{r}}{\sqrt{N}}\left(\ln\left(\frac{C_{\delta}}{\delta}\right) +  \mathcal{C}(\delta,\textcolor{black}{1})\right)
    \end{equation}
    for $\lambda=\sqrt{N}/(L(\ln(N))^{r})$
     For the MEP case,  using Lemma \ref{lem:bound:mean1} and Lemma \ref{lem:gibbs_property} and Corollary \ref{thm:bounded:alt:col}it follows that with probability at least $1-2\delta$ over data, 
     it holds that
     \begin{equation}
    \label{bound:emp:single_draw:pf:eq2}
    \begin{aligned}
        & \mathcal{L}(\theta_{\star}) \le E_{\theta \sim \rho^{\text{Gibbs}}} \mathcal{L}(\theta) \le E_{\theta \sim \rho^{\text{Gibbs}}} \hat{\mathcal{L}}_N(\theta)  + 
          \underbrace{\frac{L(\ln(N))^{r}}{\sqrt{N}}}_{1/\lambda}
          \left(\KL(\rho^{\text{Gibbs}} \mid \pi)+\ln\left(\frac{C_{\delta}}{\delta}\right) +  \mathcal{C}(\delta,\textcolor{black}{1})\right) \\
       & = \left( E_{\theta \sim \rho^{\text{Gibbs}}} \hat{\mathcal{L}}_N(\theta)  + \frac{1}{\lambda} \KL(\rho^{\text{Gibbs}} \mid \pi)\right) + \frac{L(\ln(N))^{r}}{\sqrt{N}}\textcolor{black}{\left(\ln\left(\frac{C_{\delta}}{\delta}\right) +  \mathcal{C}(\delta,\textcolor{black}{1})\right)} \\
       &  = -\ln E_{\theta \sim \pi} e^{-\lambda \hat{\mathcal{L}}_N(\theta)} +
         \frac{L(\ln(N))^{r}}{\sqrt{N}}\textcolor{black}{\left(\ln\left(\frac{C_{\delta}}{\delta}\right) +  \mathcal{C}(\delta,\textcolor{black}{1})\right)}
    \end{aligned}      
    \end{equation}
    with $\lambda=\sqrt{N}/(L(\ln(N))^{r})$.
    Using the second equality of \eqref{lem:gibbs_property:eq2}, it follows that
    for any posterior $\rho \in \mathcal{M}_{\pi}$,
    \begin{equation}
    \label{bound:emp:single_draw:pf:eq3}
          -\frac{1}{\lambda}\ln E_{\theta \sim \pi} e^{-\lambda \hat{\mathcal{L}}_N(\theta)} \le
            E_{\theta \sim \rho} \hat{\mathcal{L}}_N(\theta) + \frac{1}{\lambda}\KL(\rho \| \pi)
    \end{equation}
    Choose $\rho_{\theta_{\star}}$ as in Lemma \ref{pac:kl}, i.e., 
    \[ \rho_{\theta_{\star}}=\left\{ \begin{array}{rl}
                           \mathcal{N}(\theta_{\star},\frac{\sigma^2}{N}) & \mbox{ if \textbf{(Gauss)} holds} \\
                            \mathcal{U}(\theta_{\star},\frac{\sigma}{\sqrt{N}})  & \mbox{ if \textbf{(Uni)} holds}
                         \end{array} \right. 
   \]
      From Lemma \ref{pac:lipsch:loss} it follows that \eqref{pac:lipschitz:loss:eq31} holds with probability at least 
  $1-r_{\text{Lip},1}\delta$ over data, and hence over data and samples from $\rho^{\text{Gibbs}}$. 
    In particular, as \eqref{pac:lipschitz:loss:eq31} implies
   \begin{equation}
    \label{bound:emp:single_draw:pf:eq4}
    \begin{aligned}
    E_{\theta \sim \rho_{\theta_{\star}}} \hat{\mathcal{L}}_N(\theta) & \le \hat{\mathcal{L}}_N(\theta_{\star})+
     L_{\Theta}(\delta) L(\ln(N))^{r+1} \underbrace{E_{\theta \sim \rho_{\theta_{\star}}}\|\theta - \theta_{\star}\|_2}_{\le \frac{\sigma}{\sqrt{N}}}
           +  \\
           & 2r_{\text{Lip},1}\frac{L(\ln(N))^{r}}{\sqrt{N}}\left(C_{u,1} + C_{u,2}+\ln\left(\frac{\textcolor{black}{C_{\delta}}}{\delta}\right)\right), 
    \end{aligned}
    \end{equation}
   It then follows that \eqref{bound:emp:single_draw:pf:eq4}
   holds with probability at least $1-r_{\text{Lip},1}\delta$
   over data and samples from $\rho^{\text{Gibbs}}$.
   Then using the union bound on \eqref{bound:emp:single_draw:pf:eq1},\eqref{bound:emp:single_draw:pf:eq4},  and \eqref{bound:emp:single_draw:pf:eq3} and Lemma \ref{pac:kl} it follows that
   \begin{equation}
   \label{bound:emp:single_draw:pf:eq5}
   \begin{aligned}
         & \mathcal{L}(\theta_{\star}) \le  E_{\theta \sim \rho^{\text{Gibbs}}} \hat{\mathcal{L}}_N(\theta) 
       \le - \ln E_{\theta \sim \pi} e^{-\lambda} \hat{\mathcal{L}}_N(\theta) + 
         \frac{L(\ln(N))^{r}}{\sqrt{N}}\left(\ln\frac{C_{\delta}}{\delta} +  \mathcal{C}(\delta,\textcolor{black}{1})\right) \\
        & \le E_{\theta\sim \rho_{\theta_{\star}}}\hat{\mathcal{L}}_N(\theta) +
         \frac{1}{\lambda} \ KL(\rho_{\theta_{\star}} \| \pi) + 
         \frac{L(\ln(N))^{r}}{\sqrt{N}}\left(\ln\left(\frac{C_{\delta}}{\delta}\right) +  \mathcal{C}(\delta,\textcolor{black}{1})\right) \\
       & \le  \hat{\mathcal{L}}_N(\theta_{\star})+
     L_{\Theta}(\delta) L(\ln(N))^{r+1} \frac{\sigma}{\sqrt{N}}
           +  
           2r_{\text{Lip},1}\frac{L\ln(N)^{r}}{\sqrt{N}}\left(C_{u,1} + C_{u,2}+\ln\left(\frac{\textcolor{black}{C_{\delta}}}{\delta}\right)\right) \\
           & + \frac{L(\ln(N))^{r+1}}{\sqrt{N}}  C(\theta_{\star}) +  \frac{L(\ln(N))^{r}}{\sqrt{N}}\left(\ln\left(\frac{C_{\delta}}{\delta}\right) +  \mathcal{C}(\delta,\textcolor{black}{1})\right) \\
           & =\hat{\mathcal{L}}_N(\theta_{\star}) +
           \frac{L(\ln(N))^{r+1} \sigma}{\sqrt{N}}
           \Big(\underbrace{L_{\Theta}(\delta)\sigma + 2r_{\text{Lip},1}(C_{u,1}+C_{u,2})}_{\mathcal{C}_{L,emp}} +  2r_{\text{Lin},1}\ln\left(\frac{\textcolor{black}{C_{\delta}}}{\sigma}\right)\\
           & +\textcolor{black}{\underbrace{\left(C(\theta_{\star}) + \ln\left(\frac{C_{\delta}}{\delta}\right) +C(\delta,\textcolor{black}{1})\right)}_{\mathcal{C}_L(\sigma,1,\theta_{\star})}}
           \Big)\le \hat{\mathcal{L}}_N(\theta_{\star}) + \frac{L(\ln(N))^{r+1} \sigma}{\sqrt{N}} \mathcal{C}_{em}(\delta,\sigma)
    \end{aligned}   
  \end{equation}
  with probability at least $1-(2+r_{\text{Lip},1})\delta$ over data and $\theta_{\star}$ sampled from $\rho^{\text{Gibbs}}$.
  Similarly, by using  that \eqref{bound:emp:single_draw:pf:eq2} holds with probability at least $1-2\delta$, 
  the fact that \eqref{pac:lipschitz:loss:eq31} holds with 
  probability as least $1-r_{\text{Lip},1}\delta$ over data, and that \eqref{pac:lipschitz:loss:eq31} implies 
  \eqref{bound:emp:single_draw:pf:eq3}, it follows,
  using the union bound that for MEP,
  \begin{equation}
   \label{bound:emp:single_draw:pf:eq6}
   \begin{aligned}
         \mathcal{L}(\theta_{\star}) \le & E_{\theta \sim \rho^{\text{Gibbs}}} \hat{\mathcal{L}}_N(\theta) 
          \le -\ln e^{-\lambda E_{\theta \sim \pi}\hat{\mathcal{L}}_N(\theta)}  + 
         \frac{L(\ln(N))^{r}}{\sqrt{N}}\left(\ln\left(\frac{C_{\delta}}{\delta}\right) +  \mathcal{C}(\delta,\textcolor{black}{1})\right)  \\
        & \le E_{\theta\sim \rho_{\theta_{\star}}}\hat{\mathcal{L}}_N(\theta) +\frac{1}{\lambda} \ KL(\rho_{\theta_{\star}} \| \pi) +  \frac{L(\ln(N))^{r}}{\sqrt{N}}\left(\ln\left(\frac{C_{\delta}}{\delta}\right) +  \mathcal{C}(\delta,\textcolor{black}{1})\right) \\
        & \le \hat{\mathcal{L}}_N(\theta_{\star})+
     L_{\Theta}(\delta) L(\ln(N))^{r+1} \frac{\sigma}{\sqrt{N}} 
            + \frac{L(\ln(N))^{r+1}}{\sqrt{N}} \mathcal{C}(\theta_{\star})  \\
            & + 2r_{\text{Lip},1}\frac{L\ln(N)^{r}}{\sqrt{N}}\left(C_{u,1} + C_{u,2}+\ln\left(\frac{\textcolor{black}{C_{\delta}}}{\delta}\right)\right) \\
           & +  \frac{L(\ln(N))^{r}}{\sqrt{N}}\left(\ln\left(\frac{C_{\delta}}{\delta}\right) +  \mathcal{C}(\delta,\textcolor{black}{1})\right)  \\
        &\textcolor{black}{ = \hat{\mathcal{L}}_N(\theta_{\star}) + \frac{L(\ln(N))^{r}}{\sqrt{N}} \Big(
            \underbrace{\left(L_{\Theta}(\delta)\sigma+ 2r_{\text{Lip},1} \left(C_{u,1}+C_{u,2}\right)\right)}_{\mathcal{C}_{L,emp}} +  2r_{\text{Lip},1}\ln\left(\frac{C_{\delta}}{\delta}\right) +} \\
            & \textcolor{black}{+ \underbrace{\left(\mathcal{C}(\theta_{\star})+\ln\left(\frac{C_{\delta}}{\delta}\right)
            + \mathcal{C}(\delta,1)\right)}_{\mathcal{C}_L(\delta,1,\theta_{\star})}
        \Big)}              
            \le \hat{\mathcal{L}}_N(\theta_{\star}) + \frac{L(\ln(N))^{r+1} \sigma}{\sqrt{N}} \mathcal{C}_{em}(\delta,\sigma)
    \end{aligned}   
  \end{equation}
  holds with probability at least $1-(2+r_{\text{Lip},1})\delta$ over data.
   
    Finally, the case of MAP follows from Lemma \ref{lem:pac} by applying it to the random model $\theta_{\star}$:
    as  
    \[
    \begin{aligned}
     \forall \theta \in \Theta: &
    \mathcal{L}(\theta) \le 
    \hat{\mathcal{L}}_{N}(\theta) +  \\
     & 
     \frac{L(\ln(N))^{r+1}}{\sqrt{N}} \left(
    \textcolor{black}{\mathcal{C}_{L,true}+\mathcal{C}_{L,emp}+
     C(\theta)
     +\ln\left(\frac{C_{\delta}}{\delta}\right)+ 2r_{\text{Lip},1}\ln\left(\frac{\textcolor{black}{C_{\delta}}}{\delta}\right) + \mathcal{C}(\delta,\textcolor{black}{1})}\right)
    \end{aligned}
    \]
    holds with probability at least $1-(2+r_{\text{Lip},1})\delta$ over data, it then follows that 
    for $\theta=\theta_{\star}$, 
    \[
    \begin{aligned}
    \mathcal{L}(\theta_{\star}) & \le 
    \hat{\mathcal{L}}_{N}(\theta_{\star}) + 
     \frac{L(\ln(N))^{r+1}}{\sqrt{N}} \left(
       \textcolor{black}{\mathcal{C}_{L,true}}+ \right. \\
       & \left. +\underbrace{\left(\mathcal{C}_{L,emp}+C(\theta_{\star})+\textcolor{black}{\ln\left(\frac{C_{\delta}}{\delta}\right)+\mathcal{C}(\delta,1)}\right)}_{\mathcal{C}_{L}(\sigma,1,\theta_{\star})} +
      2r_{\text{Lip},1}\ln\left(\frac{\textcolor{black}{C_{\delta}}}{\delta}\right) \right) \\
     & = \hat{\mathcal{L}}_{N}(\theta_{\star}) + 
     \frac{L(\ln(N))^{r+1}}{\sqrt{N}} \mathcal{C}_{em}(\delta,\sigma)
    \end{aligned}
    \]
\end{proof}

\subsection{Proof of \eqref{bound:oracle}}
  From \eqref{catoni1} it follows that  each of the inequalities below hold with probability at least $1-2\delta$
  over data
  \begin{align}
      \forall \rho \in \mathcal{M}_\pi: E_{\theta \sim \rho} \mathcal{L}(\theta) \leq & E_{\theta\sim \rho} \hat{\mathcal{L}}_N(\theta)+  \frac{L(\ln(N))^{r}}{\sqrt{N}}(\KL(\rho ||\pi)+\ln\left(\frac{C_{\delta}}{\delta}\right) +  \mathcal{C}(\delta,1))
      \label{bound:oracle:eq1} \\
      \forall \rho \in \mathcal{M}_\pi: -E_{\theta \sim \rho} \mathcal{L}(\theta) \leq & -E_{\theta\sim \rho} \hat{\mathcal{L}}_N(\theta)+  \frac{L(\ln(N))^{r}}{\sqrt{N}}(\KL(\rho ||\pi)+\ln\left(\frac{C_{\delta}}{\delta}\right) + 
  \mathcal{C}(\delta,-1))
   \label{bound:oracle:eq2} 
  \end{align}
  In particular, by using the union bound it follows that \eqref{bound:oracle:eq1} and \eqref{bound:oracle:eq2} hold simultaneously with probability at least $1-4\delta$.  Using the fact that 
  for any posterior $\rho \in \mathcal{M}_{\pi}$, 
  \[ 
    E_{\theta\sim \rho^{\text{Gibbs}}} \hat{\mathcal{L}}_N(\theta)+  \frac{L(\ln(N))^{r}}{\sqrt{N}}\KL(\rho^{\text{Gibbs}} ||\pi) \le 
     E_{\theta \sim \rho} \hat{\mathcal{L}}_N(\theta) + \frac{L(\ln(N))^{r}}{\sqrt{N}}(\KL(\rho ||\pi)
  \]
 it  
  then follows that with probability at least $1-4\delta$
  over data 
  \begin{align}
      \forall \rho \in \mathcal{M}_{\pi}: & E_{\theta \sim \rho^{\text{Gibbs}}} \mathcal{L}(\theta) \leq  E_{\theta\sim \rho^{\text{Gibbs}}} \hat{\mathcal{L}}_N(\theta)+  \frac{L(\ln(N))^{r}}{\sqrt{N}}(\KL(\rho^{\text{Gibbs}} ||\pi)+\ln\left(\frac{C_{\delta}}{\delta}\right) +  \mathcal{C}(\delta,1)) \nonumber \\
       & \le E_{\theta \sim \rho} \hat{\mathcal{L}}_N(\theta) + \frac{L(\ln(N))^{r}}{\sqrt{N}}\KL(\rho ||\pi)
       + \frac{L(\ln(N))^{r}}{\sqrt{N}}\left(\ln\left(\frac{C_{\delta}}{\delta}\right)+\mathcal{C}(\delta,1)\right) \nonumber \\
       & \le E_{\theta \sim \rho} \mathcal{L}(\theta) + \frac{L(\ln(N))^{r}}{\sqrt{N}}\left(\KL(\rho ||\pi)+\ln\left(\frac{C_{\delta}}{\delta}\right)+\mathcal{C}(\delta,-1)\right)+ \nonumber \\
       & + \frac{L(\ln(N))^{r}}{\sqrt{N}}\left(\KL(\rho ||\pi)+\ln\left(\frac{C_{\delta}}{\delta}\right)+\mathcal{C}(\delta,1)\right)
       \nonumber \\
       & =  E_{\theta \sim \rho} \mathcal{L}(\theta) + \frac{L(\ln(N))^{r}}{\sqrt{N}}\left(2\KL(\rho ||\pi)+2\ln\left(\frac{C_{\delta}}{\delta}\right)+\mathcal{C}(\delta,-1)+\mathcal{C}(\delta,1)\right)
      \label{bound:oracle:eq3}
  \end{align}
  For the posterior $\rho_{\theta_{true}}$ from Lemma \ref{pac:kl}, using
  \eqref{pac:lipschitz:loss:eq32} from Lemma \ref{pac:lipsch:loss}, it follows that
  \begin{align}
    & E_{\theta \sim \rho_{\theta_{true}}}  \mathcal{L}(\theta)  \le \mathcal{L}(\theta_{true})  \nonumber \\
    & +L_{\Theta}(\delta)L(\ln(N))^r E_{\theta \sim \rho_{\theta_{true}}}\|\theta - \textcolor{black}{\theta_{true}}\|_2 \nonumber \\
          & + 2r_{\text{Lip},1}\frac{L\ln(N)^{r}}{\sqrt{N}}\left(C_{u,1} + C_{u,2}\right) +
            2r_{\text{Lip},2}  \frac{L}{512} \frac{L(\ln(N))^r}{\sqrt{N}} \nonumber \\
            & \le \mathcal{L}(\theta_{true})+ L_{\Theta}(\delta)\sigma \frac{L(\ln(N))^r}{\sqrt{N}}
          + 2r_{\text{Lip},1}\frac{L\ln(N)^{r}}{\sqrt{N}}\left(C_{u,1} + C_{u,2}\right) + 2r_{\text{Lip},2}  \frac{L}{512} \frac{L(\ln(N))^r}{\sqrt{N}}  \nonumber \\
        & \le \mathcal{L}(\theta_{true})+\frac{L\ln(N)^{r}}{\sqrt{N}} \mathcal{C}_{L,true}
      \label{bound:oracle:eq4}
  \end{align}
  By applying  \eqref{bound:oracle:eq3}  with $\rho=\rho_{\theta_{true}}$ and using that by Lemma \ref{pac:kl} 
  \[ \KL(\rho_{\theta_{true}} || \pi) \le \mathcal{C}(\theta_{true})\ln(N) \]
  and the fact that $\ln(N) \ge 1$ it follows that 
 \begin{align} 
    E_{\theta \sim \rho^{\text{Gibbs}}} \mathcal{L}(\theta) \le & \mathcal{L}(\theta_{true})+ \nonumber \\
      & \frac{L\ln(N)^{r+1}}{\sqrt{N}} \left( \mathcal{C}_{L,true} + \textcolor{black}{\underbrace{2\mathcal{C}(\theta_{true})
       +2\ln\left(\frac{C_{\delta}}{\delta}\right)+\mathcal{C}(\delta,-1)+\mathcal{C}(\delta,1)}_{\mathcal{C}_L(\delta,1,\theta_{true})+\mathcal{C}_L(\delta,-1,\theta_{true})-2\mathcal{C}_{L,emp}}} \right)
      \label{bound:oracle:eq41}
 \end{align}
 with probability at least $1-4\delta$.
 If $\theta_{\star}$ is MEP, then by Lemma \ref{lem:bound:mean1}
 $\mathcal{L}(\theta_{\star}) \le E_{\theta \sim \rho^{\text{Gibbs}}} \mathcal{L}(\theta)$. 
 Combining this with \eqref{bound:oracle:eq41}, \eqref{bound:oracle} for MEP follows.

 For single draw, from Lemma \ref{lem:single_draw}, \eqref{RenyiBound:single_draw:giibs1}, it follows that
 with probability at least $1-2\delta$ overa data and samples $\theta_{\star}$ from $\rho^{\text{Gibbs}}$
 it holds that 
 \begin{align}
  \mathcal{L}(\theta_{\star}) \le   -\frac{1}{\lambda} \ln E_{\theta \sim \pi} e^{-\lambda \hat{\mathcal{L}}_N(\theta)} + 
 \frac{L(\ln(N))^{r}}{\sqrt{N}}\left(\ln\left(\frac{C_{\delta}}{\delta}\right) +  \mathcal{C}(\delta,1)\right)
      \label{bound:oracle:eq5}
 \end{align}
 with $\lambda=\lambda_N=\frac{\sqrt{N}}{L(\ln{N})^r}$. 
 By using\eqref{lem:gibbs_property:eq1} from  Lemma \ref{lem:gibbs_property}, 
 it follows that for $\rho_{\theta_{true}}$ from Lemma \ref{pac:kl} 
 \begin{align}
      -\frac{1}{\lambda} \ln E_{\theta \sim \pi} e^{-\lambda \hat{\mathcal{L}}_N(\theta)} \le
       E_{\theta \sim \rho_{\theta_{true}}} \mathcal{\mathcal{L}}_N(\theta) + \frac{L(\ln(N))^{r}}{\sqrt{N}} \KL(\rho_{\theta_{true}} ||\pi) 
      \label{bound:oracle:e61}
\end{align}
Using that \ref{bound:oracle:eq2} holds with probability at least $1-2\delta$ over data, 
 and using \eqref{bound:oracle:eq4},  and $\ln(N) \ge 1$,
it follows that
\begin{align}
-\frac{1}{\lambda} & \ln E_{\theta \sim \pi} e^{-\lambda \hat{\mathcal{L}}_N(\theta)} \le 
       E_{\theta \sim \rho_{\theta_{true}}} \mathcal{\mathcal{L}}_N(\theta) + \frac{L(\ln(N))^{r}}{\sqrt{N}} \KL(\rho_{\theta_{true}} ||\pi)  \nonumber \\
       & \le E_{\theta \sim \rho_{\theta_{true}}} \mathcal{\mathcal{L}}(\theta) + \frac{L(\ln(N))^{r}}{\sqrt{N}} \ KL(\rho_{\theta_{true}} ||\pi) \nonumber \\
       &+ \frac{L(\ln(N))^{r}}{\sqrt{N}}\left(\KL(\rho_{\theta_{true}} ||\pi)+\ln\left(\frac{C_{\delta}}{\delta}\right)+\mathcal{C}(\delta,-1)\right) \nonumber \\
       &  
        \le \mathcal{L}(\theta_{true})+\frac{L(\ln(N))^{r+1}}{\sqrt{N}} \left(\mathcal{C}_{L,true} + 2\mathcal{C}(\theta_{true}) + \ln\left(\frac{C_{\delta}}{\delta}\right) + \mathcal{C}(\delta,-1)\right)
       \label{bound:oracle:eq07}
\end{align}
with probability at least $1-2\delta$ over data and hence over data and samples from $\rho^{\text{Gibbs}}$.
By combining \eqref{bound:oracle:eq5} and \eqref{bound:oracle:eq07} using the union bounds, it follows that
with probability $1-4\delta$ over data and samples $\theta_{\star}$ from $\rho^{\text{Gibbs}}$,
\begin{align}
  \mathcal{L}(\theta_{\star}) \le & \mathcal{L}(\theta_{true})+\frac{L(\ln(N))^{r+1}}{\sqrt{N}} \left( \mathcal{C}_{L,true} + 2\mathcal{C}(\theta_{true}) + \ln\left(\frac{C_{\delta}}{\delta}\right) + \mathcal{C}(\delta,-1)) \right) +  \nonumber  \\
 &+  \frac{L(\ln(N))^{r}}{\sqrt{N}}\left(\ln\left(\frac{C_{\delta}}{\delta}\right) +  \mathcal{C}(\delta,1)\right) \nonumber \\
 & \le \mathcal{L}(\theta_{true})+\frac{L(\ln(N))^{r+1}}{\sqrt{N}} \left( \mathcal{C}_{L,true} + \textcolor{black}{\underbrace{2\mathcal{C}(\theta_{true}) + 2\ln\left(\frac{C_{\delta}}{\delta}\right) +  \mathcal{C}(\delta,1) + \mathcal{C}(\delta,-1)}_{\mathcal{C}_L(\delta,1,\theta_{true})+\mathcal{C}_L(\delta,-1,\theta_{true})-2\mathcal{C}_{L,emp}}}\right)
 \nonumber
\end{align}
from which the statement of \eqref{bound:oracle} for single draw follows, by replacing $\delta$ by $\frac{\delta}{2}$.

Finally, we consider the case of MAP. Since \eqref{bound:emp}, i.e.,
\begin{equation}
    \label{bound:oracle:eq7}
\begin{aligned}
  \mathcal{L}(\theta_{\star}) & \le  \hat{\mathcal{L}}_N(\theta_{\star}) + \\
   & \frac{\textcolor{black}{L}(\ln(N))^{r+1}}{\sqrt{N}} 
  \Big(\textcolor{black}{\mathcal{C}_{L,true}+\mathcal{C}_{L}(\delta,1,\theta_{\star})+
     2r_{\text{Lip},1}\ln\left(\frac{C_{\delta}}{\delta}\right)} \Big)
\end{aligned}
\end{equation}
holds with probability at least $1-(2+r_{\text{Lip},1})\delta$.
From \eqref{lem:pac:eq1:uniform} applied to $\theta=\theta_{true}$ and $\epsilon=-1$ it also follows that 
\begin{equation}
    \label{bound:oracle:eq8}
\begin{aligned}
  \hat{\mathcal{L}}_N(\theta_{true})  &\le \mathcal{L}(\theta_{true}) + \\
    & \frac{\textcolor{black}{L}(\ln(N))^{r+1}}{\sqrt{N}} 
  \Big(
     \textcolor{black}{\mathcal{C}_{L,true}+\mathcal{C}_{L}(\delta,-1,\theta_{true})+
     2r_{\text{Lip},1}\ln\left(\frac{C_{\delta}}{\delta}\right)} \Big)
\end{aligned}
\end{equation}
holds with probability at least $1-(2+\mathrm{r_{\text{Lip},1}})\delta$ over data.
If $\theta_{\star}$ is MAP, then $\theta_{\star}$
minimizes $\hat{\mathcal{L}}_N(\theta)-\frac{1}{\lambda} \ln \pi(\theta)$ and hence 
$\hat{\mathcal{L}}_N(\theta_{\star})-\frac{1}{\lambda} \ln \pi(\theta_{\star}) \le 
\hat{\mathcal{L}}_N(\theta_{true})-\frac{1}{\lambda} \ln \pi(\theta_{true})$.
Hence, 
\[
\hat{\mathcal{L}}_N(\theta_{\star}) \le 
\hat{\mathcal{L}}_N(\theta_{true})+\frac{1}{\lambda} \ln \left(\frac{\pi(\theta_{\star})}{\pi(\theta_{true})}\right)
\]
and then by using \eqref{bound:oracle:eq7}-\eqref{bound:oracle:eq8} and the fact that $\lambda=\frac{\sqrt{N}}{L(\ln(N))^r}$ and
hence $\frac{1}{\lambda} \le  \frac{L(\ln(N))^{r+1}}{\sqrt{N}}$,
it follows that with probability $1-(4+2r_{\text{Lip},1})\delta$ over data,
\begin{equation}
\label{pf:col1:eq6}
\begin{split}
 & \hat{\mathcal{L}}_N(\theta_{\star})
         \leq  
        \mathcal{L}(\theta_{true})  \\
        & +\frac{\textcolor{black}{L}(\ln(N))^{r+1}}{\sqrt{N}} 
  \Big(
   \textcolor{black}{2\mathcal{C}_{L,true}+\mathcal{C}_{L}(\delta,1,\theta_{\star})+\mathcal{C}_{L}(\delta,-1,\theta_{true})+
     4r_{\text{Lip},1}\ln\left(\frac{C_{\delta}}{\delta}\right)} \Big)  \\
     & 
        + \frac{L(\ln(N))^{r+1}}{\sqrt{N}} \ln \left(\frac{\pi(\theta_{\star})}{\pi(\theta_{true})}\right)
\end{split}
\end{equation}
holds, hence, by replacing $\delta$ with $\delta/2$, it follows that  \eqref{bound:oracle} holds with probability
at least $1-(2+r_{\text{Lip},1})\delta$.

\subsection{Proof of \eqref{cor:param:est1:eq1}}
   If \textbf{(SConv)} holds, then by \cite[Theorem 2.1.11, Theorem 2.1.8]{NesterovBook}
   \[ m_{\Theta} \|\theta_{\star}-\theta_{true}\|^2 \le \mathcal{L}(\theta_{\star}) - \mathcal{L}(\theta_{true}) \]
   from which using \eqref{bound:oracle} the inequality \eqref{cor:param:est1:eq1} follows.

\subsection{Proof of \eqref{col1:eq1}}
From
\cite[Chapter 4, page 87]{LjungBook} 
it follows that 
$\sigma^2_{\e^s}=\mathcal{L}(\theta_{\mathrm{true}})=\inf_{\theta \in \Theta} \mathcal{L}(\theta)$, 
where $\sigma^2_{\e^s}=\bE[\|\e^s(t)\|^2_2]$, 
if $\ell$  is the quadratic loss function, 
i.e. the true system yields the best possible prediction in mean square sense.
\begin{equation}
\label{param_set:true_eq}
 \forall \theta \in \Theta:  \sigma^2_{\e^s}=\mathcal{L}(\theta_{true}) \le \mathcal{L}(\theta) 
\end{equation}
where $\sigma^2_{\e^s}$ is the variance of $\e^s(t)$.
That is, the parameter $\theta_{true}$ corresponding to the data generator results in the smallest possible generalisation loss for quadratic loss function.
Let $H_{\theta}$ be the transfer function of
the predictor parametrized by $\theta$

It can be shown (see \cite[Appendix, (82)]{eringisRenyi}) and \cite[Chapter 4, page 87]{LjungBook}) that
\begin{equation}
\label{rem:lqg:eq0}
\|H_{\theta_{\mathrm{true}}}-H_{\theta} \|_{H_2}^2 \le \frac{1}{m_{\w}} (\mathcal{L}(\theta) - \mathcal{L}(\theta_{true})), 
\end{equation}
Then \eqref{col1:eq1} follows from \eqref{bound:oracle} and \eqref{rem:lqg:eq0}.

\subsection{Remark on relationship with system identification and learning}
\label{app:indent}

\begin{remark}[Identification of FIR models]
    \label{rem:fir}
    Consider the SISO case ($n_y=n_u=1$), fix 
    an integer $T_0$ and let
    $\Theta=\mathbb{R}^{T_0}$ and 
    for any $\theta=(\theta_1,\ldots, \theta_{T_0})^T \in \Theta$, let
    $D_{\theta}=\theta_1$,
    \[ 
    A=\begin{bmatrix} 0 & 0 & \cdots & 0 & 0 \\ 
        1 & 0 & \cdots & 0 &0 \\
        \vdots & \vdots & \cdots & \vdots & \vdots \\
        0 & 0 & \cdots & 1 & 0
    \end{bmatrix}, ~ 
    B=\begin{bmatrix} 1 \\ 0 \\ \vdots \\ 0 \end{bmatrix}, ~
    C_\theta=\begin{bmatrix} \theta_2 \\ \vdots \\ \theta_{T_0} \\ 1 \end{bmatrix}^T
    \]
    Then $\hat{\rvy}_{\theta}(t)=u(t-T_0)+\sum_{i=0}^{T_0-1} \theta_{i+1} \rvu(t-i)$. Assume that there exists 
    $\theta_{true} \in \Theta$, such that 
    $\rvy(t)=\hat{\rvy}_{\theta_{true}}(t)+\rve(t)$,
    where $\rve(t)$ is i.i.d. with zero mean and variance $\sigma^2$, and independent of $\rvu(s)$, $s \in \mathbb{Z}$.
    Let $\hat{\rvu}(t)=\rvu(t)$ for $t > 0$ and
    let $\hat{\rvu}(t)=0$ for $t < 0$.
    Let $\Phi$ be the regressor matrix
    constructed from $\tilde{\rvu}$ as in
    \cite[(3.45), page 48]{pillonetto2022regularized} and
    let 
    $Y=\begin{bmatrix} \rvy(0) & \cdots &\rvy(N-1) \end{bmatrix}^T$.
    Then
    $\hat{\mathcal{L}}_N(\theta)=\frac{1}{N} \|Y-\Phi \theta\|_2^2$.
    Assume that the prior $\pi$ is the Gaussian density $\mathcal{N}(\tilde{\theta},P)$, $P=P^T > 0$. 
    Then a lengthy but straightforward calculation
    reveal that
    the Gibbs posterior is the  Gaussian density
    $\mathcal{N}(\tilde{\theta}+P\Phi^T(\Phi P \Phi^T + \frac{N}{2\lambda} I)^{-1}(Y-\Phi \tilde{\theta}), P-P\Phi^T(\Phi P\Phi^T + \frac{N}{2\lambda} I)^{-1}\Phi P)$, 
    which is the conditional density of
    $\theta$ w.r.t. $Y,\Phi$ presented in
    \cite[(4.6), page 101]{pillonetto2022regularized}, if
    $Y=\Phi\theta+E$ is assumed with 
    $E \sim \mathcal{N}(0,\frac{N}{2\lambda} I_{T_0})$.
    Then the MAP and MEP estimates are both
    $\theta_{\star}=\tilde{\theta}+P\Phi^T(\Phi P \Phi^T + \frac{N}{2\lambda} I)^{-1}(Y-\Phi \theta_{true})$
    and they solve  the regularized
    optimization problem
    $\argmin_{\theta \in \Theta} \lambda \hat{\mathcal{L}}_N(\theta) +  (\theta-\tilde{\theta})P^{-1}(\theta-\tilde{\theta})$. Moreover, they are also
    the maximum likelihood
    of the posterior $P(\theta \mid Y,\Phi)$. 

 Then Corollary \ref{lem:bound:mean} applies to $\theta_{\star}$,
  yielding  a 
  $O(\frac{(\ln(N))^r}{\sqrt{N}})$ bound on
  the generalization gap of $\theta_{\star}$
  and a $O(\frac{(\ln(N))^{r+1}}{\sqrt{N}})$ bound
  on the square parameter estimation error
  $\|\theta_{\star}-\theta_{true}\|_2^2$.
  \end{remark}
  %
 %
 %
 \begin{remark}(\emph{Maximum likelihood estimates   of state-space models}). 
  \label{rem:ml}
  Consider the parametrization of LTI systems
  \begin{align}
        \hat{\x}(t+1)&=(\hat{A}_\theta + \hat{L}_{\theta} \hat{C}_{\theta}) \hat{\x}(t)
         +(\hat{B}_\theta+L_{\theta}D_{\theta})\rvu(t)+\hat{L}_\theta\e_{\theta}(t),  \nonumber \\
        \y_\theta(t)&=\hat{C}_\theta\hat{\x}(t)+\hat{D}_\theta\rvu(t) + \e_{\theta}(t)
  \label{eq:predictor_gen}
 \end{align}
 where 
 $A_{\theta}$ and $A_{\theta}+L_{\theta}C_{\theta}$ are Schur, 
 $\e_{\theta}(t)$ is i.i.d., zero mean Gaussian with
 variance $\frac{1}{\lambda}$, $(\rvu,\e_{\theta})$ is jointly Gaussian, $\e_{\theta}(t)$ is independent of
 $\rvu(s)$ for $s < t$. 
 We can associate with \eqref{eq:predictor_gen} the predictor 
\begin{align}
        \hat{\x}(t+1)&=\hat{A}_\theta \hat{\x}(t)+\hat{B}_\theta \rvu(t)+\hat{L}_\theta\y(t),  \nonumber \\
        \hat{\y}_\theta(t \mid 0)&=\hat{C}_\theta\hat{\x}(t)+\hat{D}_\theta\rvu(t),
  \label{eq:predictor_gen1}
 \end{align}
 Assume that
 $\y$ is generated by an LTI system of the form \eqref{eq:predictor_gen} which satisfies  \cite[(1), INP B, page 427]{CainesBook}.

 In this case, by \cite[(2.7), page 422, Chapter 7]{CainesBook}
 the conditional likelihood function
 $p_{\theta}(\{y(t)\}_{t=-1}^{N-1} \mid \{u(t)\}_{t=0}^{N-1})$ 
 $\{\y_{\theta}(t)\}_{t=0}^{N-1}$ w.r.t
 $\{\rvu(t)\}_{t=0}^{N-1}$, evaluated at 
 $y(t)=\rvy(t)$ and $u(t)=\rvu(t)$, $t=-1,\ldots,N-1$
 equals 
 $Ze^{-\frac{N\lambda}{2} \hat{\mathcal{L}}_N(\theta)}$
 for a suitable constant $Z$, where the empirical error 
 $\hat{\mathcal{L}}_N(\theta)$ is defined for the quadratic loss. 
 That is, 
 i.e., up to an additive constant, its logarithm equals $\ln \rho^{\text{Gibbs}}(\theta)$.
 In particular, the maximum likelihood estimate 
 $\mathrm{argmax}_{\theta \in \Theta} p_{\theta}(\{y(t)\}_{t=0}^{N-1} \mid \{u(t)\}_{t=0}^{N-1})$, see
 \cite[Chapter 7,page 428]{CainesBook} 
 equals the MAP 
 $\mathrm{argmax}_{\theta \in \Theta} \rho^{\text{Gibbs}}(\theta)$ of the Gibbs posterior corresponding to the uniform distribution $\pi=\frac{1}{vol(\Theta)}$ as a prior. Here we assume that  $vol(\Theta)$  is finite.

 We can then use the bounds \eqref{bound:emp}, \eqref{bound:oracle}, \eqref{cor:param:est1:eq1}, \eqref{col1:eq1}, 
 for the generalization gap, the gap between the true error of the learned model and the best true error and the parameter estimation error respectively.
 \end{remark} 
 \begin{remark}[Bayesian methods for estimating LTI systems]
  \label{rem:bayesian}
    The Bayesian method for system identification of LTI systems \cite{NINNESS201040} corresponds to taking a single draw, MEP or MAP from the conditional distribution of the parameter given the past observations. If the parametrization is in innovation form, i.e., of the form \eqref{eq:predictor_gen} and all the signals are jointly Gaussian,  then by
    Remark \ref{rem:ml}
    the joint density of  $\{\y(t)\}_{t=-1}^{N-1}$ given 
    $\{\rvu(t)\}_{t=0}^{N-1}$ and 
    $\{\y(s),\rvu(s)\}_{s < 0}$ and a parameter 
    $\theta$ is $Ze^{-N\lambda \hat{\mathcal{L}}_N(\theta)}$ for a suitable constant, 
    and hence by \cite[Section 3, eq. (12)]{NINNESS201040}
    the conditional density of $\theta$ with respect to
    $\{\y(t),\rvu(t)\}_{t=0}^{N-1}$ is
    the Gibbs posterior 
    $\frac{e^{-N\lambda \hat{\mathcal{L}}_N(\theta)}p_{\theta}(\y(-1))\pi(\theta)}{E_{\theta \in \pi} e^{-\lambda \hat{\mathcal{L}}_N(\theta)}}$ for the quadratic loss and a suitable prior, 
    and hence 
    Gibbs posterior for the predictors \eqref{eq:predictor_gen1}, and hence taking a random draw, MEP or MAP from 
    this conditional density, as suggested by \cite{NINNESS201040}, corresponds to taking a random draw, MEP or MAP from
    the Gibbs posterior. In particular, the bounds  \eqref{bound:emp}, \eqref{bound:oracle}, \eqref{cor:param:est1:eq1}, \eqref{col1:eq1} apply.
 \end{remark}
  \begin{remark}[PEM methods]
    \label{rem:pem}
    To begin with if a PEM algorithm can be reduced to optimizing a cost function
    \( \hat{\mathcal{L}}_N(\theta)  + g(\theta),\)
    then the algorithm is equivalent to taking the MAP estimate of the Gibbs posterior with a prior $\pi$ such that 
    $-\frac{1}{\lambda}\ln \pi(\theta)=\frac{g(\theta)}{\lambda}$. In particular, \eqref{bound:emp}-\eqref{col1:eq1} from Corollary \ref{lem:bound:mean} applies in this case. For example, the quadratic regularization term $g(\theta)=
    (\theta - \theta_m)^T P^{-1} (\theta - \theta_m) $
    corresponds to $\pi \sim \mathcal{N}(\theta_m, P/\lambda)$. If $\pi$ can be interpreted as a Gaussian density, then MAP is MEP and the corresponding sharper
    bounds can be used. 

    Furthermore, many  PEM algorithms rely on choosing randomly the initial parameter estimate. The outcomes of such algorithms can be viewed as drawing a sample from a suitable posterior distribution. More precisely, consoider an algorithm
    $A_N$. For every initial value $\theta_0 \in \Theta$ and data $\mathcal{D}$, the algorithm
    generates an estimate $A_N(\theta_0,\mathcal{D})$.
    Most optimization algorithms used for PEM methods
    (Gradient, Gauss-Newton), etc. are of this form, if we fix the
    number of steps.
    Assume that the initial parameter values $\theta_{init}$
    are i.i.d., samples from a prior density $\pi$. Then the resulting parameter values
    $A_N(\theta_{init},\mathcal{D})$ are i.i.d. samples from
    a distribution $P_{A_N}(B)=\int_{A_N^{-1}(B,\mathcal{D})} \pi(\theta)dm(\theta)$ for any $B \in B_{\theta}$.
    Assume that the latter distribution has a density
    $\hat{\rho}$ which is in $\mathcal{M}_{\pi}$. 
    Then Lemma \ref{lem:single_draw} provides 
    an error bound for the true error of such an algorithm, when applied to an initial parameter value
    randomly sampled from $\pi$.
  \end{remark}
\begin{remark}[Strong convexity of the true error]
  \label{app:strong convexity}
 Strong convexity implies that there is a unique minimizer of the true error, which is
 the true system. It  can be seen as an \emph{identifiability and persistence of excitation assumption}, as it implies that there is no parameter value which gives the same
 input-output behavior as the true system (otherwise their true errores would be the same),
 and $\w$ is rich enough (the average response to $\w$ allows us to find the true system).
 In fact, strong convexity of the true error (at least locally), is a standard
 assumption  for studying the asympotic distribution of the parameter estimation error
 \cite[Theorem 9.1]{LjungBook}. 
\end{remark}
\begin{remark}[Bounds on the difference between  system matrices]
  \label{matrix:trans}
Assuming that the LTI systems $\Sigma(\theta)$ are all minimal,
we can use \eqref{col1:eq1}
and \cite[Theorem V.2]{oymak2021revisiting}
to derive
an error bound (in high-probability) 
on the difference between 
the matrices of the learned and true system.
Alternatively, if we consider parametrizations based on structure indices \cite[Chapter 3]{RalfPeeters} 
we could use the fact that $H_2$ norms can be used to approximate locally the Riemannian distance between parameter values  \cite[Chapter 4, page 129]{RalfPeeters}. 
\end{remark}
%
%
%
\begin{remark}[Relationship with \cite{eringis2023pacbayesRNN,eringis2023pacbayesian}]
  \label{rem:deividas}
 For the case of bounded noise in the data generator, Theorem \ref{thm:main} provides a PAC-Bayesian bound on the
 generalization gap which is similar to \cite[Theorem 2, Corollary 1]{eringis2023pacbayesRNN} and \cite[Theorem 5.2]{eringis2023pacbayesian}. However, \cite{eringis2023pacbayesRNN} considers RNNs, and the corresponding 
 constants $L_{\rvv}(\theta),L_{g,\rvs}(\theta),L_{g,\rvv}$ correspond to $\|\hat{B}_{\theta}\|_2$, $\|\hat{C}_{\theta}\|_2$, and $\|\hat{D}_{\theta}\|_2$, respectively, and $\tau(\theta)$ corresponds to an upper bound on the spectrum of $\hat{A}_{\theta}$
 and $C(\theta)$ is such that $\|\hat{A}_{\theta}^k\|_{2} \le \tau(\theta)^k C(\theta)$.
   In our setting, $L_{\rvv}(\theta),L_{g,\rvs}(\theta),L_{g,\rvv}$ can be taken to be $C$ from Assumption \ref{as:parameterisation}, $C(\theta)$ can be taken to be $M$ and $\tau(\theta)$ can be taken to be $\gamma$ from Assumption \ref{as:parameterisation}.
    Then the constants $G_{\theta}$, $H_{\theta}$ from \cite{eringis2023pacbayesRNN}  
    are upper bounds on $\|\alpha_{\theta}\|_{\ell_1}$ and $\theta_{\infty}(\alpha_{\theta})$ respectively.
    Hence the term $\hat{\Psi}_{\pi}(\lambda,N)$ of 
    \cite[Theorem 2, Corollary 1]{eringis2023pacbayesRNN}  is an upper bound on the term
    $\frac{1}{2}\sum_{i=1}^2 \ln E_{\theta\sim\pi} \exp\left(\mathcal{C}_i(2\lambda,\theta,N,\delta,\epsilon)\right)$ from \eqref{eq:initial:tildern1}, i.e., the bound of \cite[Theorem 2, Corollary 1]{eringis2023pacbayesRNN} is
    more conservative than the one of this paper.

  As to \cite[Theorem 5.2]{eringis2023pacbayesian}, it uses $\|\alpha_{g,y}-c_{\theta}\|_{\ell_1} \le G_e(\theta)$ and $\theta_{\infty}(\alpha_{g,y}-c_{\theta}) \le G_{e,1}(\theta)$ in the bounds, where $c_{\theta}$ and $\alpha_{g,y}$ are defined in Lemma \ref{alpha:lemma2}. While this may appear less conservative, the use of $c_{\theta}$ makes the bound difficult to evaluate in practice, as $c_{\theta}$ depends on the parameters of the unknown data generator. In contrast, our bound uses only $\|\alpha_g\|_{\ell_1}$ and $\theta_{\infty}(\alpha_g)$, which, although they also depend on the data generator, are easier to estimate or upper bound in practice.

  Note that \cite{eringis2023pacbayesRNN,eringis2023pacbayesian} do not address the case of unbounded signals or parameter estimation errors. They also do not provide generalization bounds for single draw, MEP, or MAP estimators.
\end{remark}
\begin{remark}[Relationship with finite-sample bounds \cite{oymak2021revisiting,lale2020logarithmic,Simchowitz_Foster_2020,Ziemann1,Ziemann2}]
  \label{rem:finite-sample}
Recent works such as \cite{oymak2021revisiting,lale2020logarithmic,Simchowitz_Foster_2020,Ziemann1,Ziemann2} provide finite-sample parameter estimation bounds for specific algorithms (e.g., least-squares for ARX models, subspace methods), typically assuming the data generator starts from a deterministic initial state and often neglecting the stochastic component. However, assuming a deterministic initial state can be problematic, since process noise generally drives the system to a random initial state determined by past disturbances.

In contrast, our results: \textbf{(1)} apply to a broad class of learning algorithms (including PEM, maximum likelihood, Bayesian, and regularized FIR methods), \textbf{(2)} assume the data generator is in a stationary regime, \textbf{(3)} explicitly account for both deterministic and stochastic components, and \textbf{(4)} provide generalization gap bounds. The trade-off is that our bound on the square of the parameter estimation error is $O\left(\frac{(\ln N)^{r+1}}{\sqrt{N}}\right)$, which is more conservative than the $O\left(\frac{\ln N}{N}\right)$ rates in the cited works.
\end{remark}

\section{Numerical Experiments}
\textcolor{black}{The code for the numerical example can be found at \url{https://github.com/MOHAMMADZAHD93/PACneurips2026}.}

We evaluate the right-hand sides of \eqref{eq:initial:tildern1} (Theorem~\ref{thm:main}) and \eqref{catoni1} (Corollary~\ref{thm:bounded:alt:col}), compare them with the absolute generalization gap $\left|\bE_{\theta\sim\rho_D}\mathcal L(\theta)-\bE_{\theta\sim\rho_D}\hat{\mathcal L}_N(\theta)\right|$, and report how its median and 0.9 quantile ($q_{.90}$) change with $N$, using $\mathcal L$ from \eqref{eq:GenLoss} and $\hat{\mathcal L}_N$ from Section~\ref{sect:Problem_formulation}. At $N=1000$ and $N=64000$, we compare the bounds directly with the observed gap. We also test non-vacuity in the following sense: whether there exists a sample size $N_1$ for which a bound is below the 90th-percentile error observed at $N=4$. For the Gibbs posterior, the error compared with the right-hand side of \eqref{cor:param:est1:eq1} is the squared parameter error $\|\theta_\star-\theta_{true}\|_2^2$.

We use the data generator in Assumption~\ref{as:generator} and the predictor class in Assumption~\ref{as:parameterisation}.

\subsection{Data generator and predictors}

The generator has three states, two inputs, and two outputs. In the notation of \eqref{eq:generator}, it is
\begin{equation}
\label{eq:numerical-generator}
\rvx(t+1)=A_g\rvx(t)+K_g\rve_g(t),\qquad \begin{bmatrix}\rvy(t)\\\rvu(t)\end{bmatrix}
=C_g\rvx(t)+\rve_g(t),
\end{equation}

where $A_g=\begin{bmatrix}0.25&0&0\\0&0.20&0\\0&0&0.125\end{bmatrix}$, $K_g=\begin{bmatrix}0&0&0.02&0\\0&0&0&0.02\\0&0&0.70&0.63\end{bmatrix}$, and $C_g=\begin{bmatrix}0.70&0&0\\0&0.70&0\\0&0&0\\0&0&0\end{bmatrix}$.

The innovation is $\rve_g(t)=[\bv^{\mathsf T}(t),\rvu^{\mathsf T}(t)]^{\mathsf T}$, where the coordinates of $\bv(t)$ and $\rvu(t)$ are independent across coordinates and time. The first two rows of the second equation in \eqref{eq:numerical-generator} generate $\rvy(t)$, and its last two rows generate $\rvu(t)$. The input to the predictors is $\w(t)=[\rvy^{\mathsf T}(t),\rvu^{\mathsf T}(t)]^{\mathsf T}$.

The predictor matrices in \eqref{eq:predictor} are $\hat A_\theta=A_g+(\theta_1-1)D_1+(\theta_2-1)D_2$, $D_1=\begin{bmatrix}0&0&3.700\\0&0&0\\0&0&0\end{bmatrix}$, $D_2=\begin{bmatrix}0&0&0\\0&0&3.404\\0&0&0\end{bmatrix}$, $\hat B_\theta=K_g$, $\hat C_\theta=\begin{bmatrix}0.70&0&0\\0&0.70&0\end{bmatrix}$, and $\hat D_\theta=0_{2\times4}$. Thus, $\hat{\x}(t+1)=\hat A_\theta\hat{\x}(t)+K_g\w(t)$, $\hat{\y}_\theta(t\mid0)=\hat C_\theta\hat{\x}(t)$, and $\hat{\x}(0)=0$, with $\hat A_\theta=\begin{bmatrix}0.25&0&3.700(\theta_1-1)\\0&0.20&3.404(\theta_2-1)\\0&0&0.125\end{bmatrix}$.

The parameter set is $\Theta=[0.6,1.4]^2$ and $\theta_{true}=(1,1)$. At this parameter, $\hat A_{\theta_{true}}=A_g$. For every $\theta\in\Theta$, the eigenvalues of $\hat A_\theta$ are $0.25$, $0.20$, and $0.125$, so the predictors satisfy the Schur condition in Assumption~\ref{as:parameterisation}.

In the bounded case, $u_i(t)\sim U[-0.05,0.05]$ and $v_i(t)\sim U[-0.005,0.005]$ for $i\in\{1,2\}$. In the unbounded case, $u_i(t)=0.0125R_{u,i}(t)+0.0000625Z_{u,i}(t)$ and $v_i(t)=0.25\bigl(0.0125R_{v,i}(t)+0.0000625Z_{v,i}(t)\bigr)$ for $i\in\{1,2\}$, where the $R$ variables are Rademacher, the $Z$ variables are standard Gaussian, and all these variables are independent across coordinates and time.

For each data realization, we initialize the generator state at zero, simulate the generator for 160 initial time steps, and discard those input-output pairs so that the generator state approaches its stationary regime. We retain the next $N$ input-output pairs, initialize the predictor state at zero, and update it on those pairs as in \eqref{eq:predictor}. These $N$ pairs are used to compute $\hat{\mathcal L}_N$.

\subsection{Data-dependent uniform posterior}

We first use the prior $\pi=U([0.6,1.4]^2)$. The posterior $\rho_D$ is uniform on a square of side length $0.1$ centred at a constrained minimizer of $\hat{\mathcal L}_N$. Its centre is constrained so that the square is contained in $\Theta$. Hence, $\KL(\rho_D\|\pi)=\log(0.8^2/0.1^2)=\log 64$.

We consider the $\ell_2$ loss, $\ell(\y,\hat{\y})=\|\y-\hat{\y}\|_2$, and the quadratic loss, $\ell(\y,\hat{\y})=\|\y-\hat{\y}\|_2^2$, from Assumption~\ref{as:loss}. For each realization and sample size $N$, we compute $\left|\bE_{\theta\sim\rho_D}\mathcal L(\theta)-\bE_{\theta\sim\rho_D}\hat{\mathcal L}_N(\theta)\right|$. The sample sizes are $N\in\{4,500,1000,2000,4000,8000,16000,32000,64000\}$. Table~\ref{tab:numerical-uniform-gap} reports the median and $q_{.90}$ of this absolute gap over 40 independent data realizations.

The posterior integrals are evaluated deterministically. For the quadratic loss, $\mathcal L(\theta)$ is computed from stationary second moments. For the $\ell_2$ loss, it is estimated using a fixed stationary Sobol reference. We use $\delta=0.025$ in \eqref{eq:initial:tildern1} and \eqref{catoni1}. Applying the one-sided inequalities to both signs $\epsilon\in\{-1,1\}$ gives a 90\% probability statement for the absolute gap. We evaluate the prior expectations in \eqref{eq:initial:tildern1} by $61\times61$ tensor Gauss--Legendre quadrature. For each sample size, innovation law, and loss, we minimize the right-hand side of \eqref{eq:initial:tildern1} over the admissible $\lambda$ range in Theorem~\ref{thm:main}. We evaluate 161 equally spaced points in $\log\lambda$, then apply bounded scalar minimization between the adjacent grid points around the smallest grid value, with absolute tolerance $10^{-10}$. For the unbounded absolute-gap bound, we minimize separately for both signs in \eqref{eq:initial:tildern1} and report the larger value. The last column of Table~\ref{tab:numerical-uniform-gap} reports the right-hand sides of \eqref{eq:initial:tildern1} and \eqref{catoni1}, in that order.

\begin{table}[t]
\centering
\small
\caption{Posterior-averaged absolute generalization gap and the right-hand sides of \eqref{eq:initial:tildern1} and \eqref{catoni1}.}
\label{tab:numerical-uniform-gap}
\begin{tabular}{llrrrr}
\toprule
Innovations & Loss & $N$ & Median & $q_{.90}$ & \eqref{eq:initial:tildern1} / \eqref{catoni1} \\
\midrule
Bounded & $\ell_2$ & $1{,}000$ & $3.53\times10^{-5}$ & $7.27\times10^{-5}$ & $0.0564\;/\;0.2678$ \\
Bounded & $\ell_2$ & $64{,}000$ & $3.62\times10^{-6}$ & $6.64\times10^{-6}$ & $0.00703\;/\;0.0335$ \\
Bounded & Quadratic & $1{,}000$ & $3.44\times10^{-7}$ & $9.04\times10^{-7}$ & $0.0321\;/\;0.2591$ \\
Bounded & Quadratic & $64{,}000$ & $3.64\times10^{-8}$ & $7.64\times10^{-8}$ & $0.00400\;/\;0.0324$ \\
Unbounded & $\ell_2$ & $1{,}000$ & $7.52\times10^{-6}$ & $2.00\times10^{-5}$ & $22.30\;/\;72.06$ \\
Unbounded & $\ell_2$ & $64{,}000$ & $3.56\times10^{-7}$ & $7.59\times10^{-7}$ & $4.35\;/\;14.43$ \\
Unbounded & Quadratic & $1{,}000$ & $6.71\times10^{-8}$ & $1.39\times10^{-7}$ & $73.20\;/\;480.68$ \\
Unbounded & Quadratic & $64{,}000$ & $5.34\times10^{-9}$ & $1.43\times10^{-8}$ & $17.26\;/\;121.84$ \\
\bottomrule
\end{tabular}
\end{table}

At both displayed sample sizes, the median and $q_{.90}$ of the gap are below both right-hand sides in every row. Relative to the corresponding $q_{.90}$, the two bounds range from $7.8\times10^2$ to $4.2\times10^5$ for bounded innovations and from $1.1\times10^6$ to $8.5\times10^9$ for unbounded innovations. In each row, the right-hand side of \eqref{eq:initial:tildern1} is smaller than that of \eqref{catoni1}. These comparisons show that the bounds are loose at the displayed sample sizes and that \eqref{catoni1} is looser than \eqref{eq:initial:tildern1} in this experiment.

To assess whether the observed sample-size dependence of the gap is consistent with the upper-bound dependence in \eqref{eq:initial:tildern1} and \eqref{catoni1}, we compute its median and $q_{.90}$ from 40 independent data realizations at each $N\in\{500,1000,2000,4000,8000,16000,32000,64000\}$. For each innovation law and loss, we fit separate least-squares lines to the logarithms of these two summaries against $\ln N$. Each entry in Table~\ref{tab:numerical-uniform-slopes} is the fitted exponent for one summary of the posterior-averaged absolute generalization gap.

\begin{table}[t]
\centering
\small
\caption{Fitted exponents from least-squares fits of the logarithms of the median and $q_{.90}$ of the posterior-averaged absolute generalization gap against $\ln N$, over $N\in\{500,1000,2000,4000,8000,16000,32000,64000\}$.}
\label{tab:numerical-uniform-slopes}
\begin{tabular}{llrr}
\toprule
Innovation law & Loss & Median slope & $q_{.90}$ slope \\
\midrule
Bounded & $\ell_2$ & $-0.567$ & $-0.505$ \\
Bounded & Quadratic & $-0.583$ & $-0.534$ \\
Unbounded & $\ell_2$ & $-0.789$ & $-0.810$ \\
Unbounded & Quadratic & $-0.648$ & $-0.578$ \\
\bottomrule
\end{tabular}
\end{table}

Across all four innovation-law and loss combinations, the median and $q_{.90}$ of the posterior-averaged absolute generalization gap decrease over the sample sizes considered. The bounded-innovation slopes, from $-0.583$ to $-0.505$, are close to the $-1/2$ exponent in the upper bounds. The unbounded-innovation slopes, from $-0.810$ to $-0.578$, are more negative over this finite range. Thus, the observed gaps approximately follow the stated dependence in the bounded cases and decrease faster over this range in the unbounded cases. This is consistent with \eqref{eq:initial:tildern1} and \eqref{catoni1}, whose rates are upper bounds.

To test non-vacuity in the sense stated above, we estimate the $q_{.90}$ of the posterior-averaged absolute generalization gap at $N=4$ from 400 independent data realizations. For each innovation law and loss, we ask whether the right-hand side of \eqref{eq:initial:tildern1} or \eqref{catoni1}, evaluated at some $N_1$, is below this fixed 90th-percentile gap for $N=4$. We locate a crossing by doubling and integer bisection in the final bracket, multiply the resulting integer by $1.01$, round up to obtain $N_1$, and verify directly that the right-hand side at $N_1$ is below the fixed $N=4$ quantile. The reported $N_1$ values are verified witnesses with a 1\% margin, not claims that they are the smallest possible sample sizes. Table~\ref{tab:numerical-uniform-crossings} reports the $N=4$ quantiles and these witnesses.

\begin{table}[t]
\centering
\small
\caption{The $q_{.90}$ values at $N=4$ and verified sample-size witnesses $N_1$ for \eqref{eq:initial:tildern1} and \eqref{catoni1}.}
\label{tab:numerical-uniform-crossings}
\begin{tabular}{llrrr}
\toprule
Innovations & Loss & $q_{.90}$ at $N=4$ & $N_1$ for \eqref{eq:initial:tildern1} & $N_1$ for \eqref{catoni1} \\
\midrule
Bounded & $\ell_2$ & $5.9934\times10^{-3}$ & $8.8835\times10^4$ & $2.0171\times10^6$ \\
Bounded & Quadratic & $8.7284\times10^{-5}$ & $1.3551\times10^8$ & $8.9030\times10^9$ \\
Unbounded & $\ell_2$ & $9.4409\times10^{-3}$ & $6.8737\times10^{10}$ & $9.3633\times10^{11}$ \\
Unbounded & Quadratic & $2.4156\times10^{-4}$ & $1.0744\times10^{16}$ & $8.5340\times10^{17}$ \\
\bottomrule
\end{tabular}
\end{table}

For every innovation-law and loss combination, Table~\ref{tab:numerical-uniform-crossings} gives an $N_1$ at which each bound is below the 90th-percentile generalization gap observed at $N=4$. Thus, the bounds are not vacuous under the stated criterion. The $N_1$ values for Theorem~\ref{thm:main} range from $8.8835\times10^4$ to $1.0744\times10^{16}$. The corresponding values for Corollary~\ref{thm:bounded:alt:col} are larger by factors of approximately 14 to 79. These large sample sizes show that non-vacuity in this sense does not mean the bounds are tight at the displayed sample sizes. The comparison establishes that each bound can fall below an observed generalization gap, while Table~\ref{tab:numerical-uniform-gap} shows that both bounds remain much larger than the observed gaps at $N=1{,}000$ and $N=64{,}000$.

\subsection{Gibbs posterior}

We next evaluate \eqref{cor:param:est1:eq1} for the Gibbs posterior defined in Corollary~\ref{lem:bound:mean}, using the quadratic loss. We retain the generator in \eqref{eq:numerical-generator}, the predictors in \eqref{eq:predictor}, $\Theta=[0.6,1.4]^2$, $\pi=U([0.6,1.4]^2)$, and the innovation laws. We change only the posterior from $\rho_D$ to $\rho_{\mathrm{Gibbs}}$. The quantity compared with the right-hand side of \eqref{cor:param:est1:eq1} is $\|\theta_\star-\theta_{true}\|_2^2$, where $\theta_\star=\bE_{\theta\sim\rho_{\mathrm{Gibbs}}}\theta$.

We compute $\|\theta_\star-\theta_{true}\|_2^2$ by evaluating the normalizing constant and posterior mean with 1089 fixed quadrature nodes. Refining the integration to 2601 nodes changes the posterior mean by less than $2.9\times10^{-15}$. We use $\delta=0.05$ in \eqref{cor:param:est1:eq1}, which gives a 90\% probability statement for the bound. We use 40 independent data trajectories for $N\in\{500,1000,2000,4000,8000,16000,32000,64000\}$ and 1000 additional independent trajectories at $N=4$ to estimate the $q_{.90}$ used in Table~\ref{tab:numerical-gibbs-crossings}.

\begin{table}[t]
\centering
\small
\caption{Constants used in \eqref{cor:param:est1:eq1}.}
\label{tab:numerical-gibbs-constants}
\begin{tabular}{lrrrr}
\toprule
Innovations & $L$ & $r$ & $m_\Theta$ & $\mathcal C_{or}(\delta,\sigma)$ \\
\midrule
Bounded & $1$ & $0$ & $0.00933654$ & $13.38869$ \\
Unbounded & $37.90410$ & $1.5$ & $0.00175065$ & $59.65784$ \\
\bottomrule
\end{tabular}
\end{table}

\begin{table}[t]
\centering
\small
\caption{Squared parameter error $\|\theta_\star-\theta_{true}\|_2^2$ and the right-hand side of \eqref{cor:param:est1:eq1}.}
\label{tab:numerical-gibbs-errors}
\begin{tabular}{lrrrr}
\toprule
Innovations & $N$ & Median & $q_{.90}$ & \eqref{cor:param:est1:eq1} \\
\midrule
Bounded & $1{,}000$ & $1.20\times10^{-10}$ & $3.97\times10^{-10}$ & $313.25$ \\
Bounded & $64{,}000$ & $1.11\times10^{-10}$ & $5.28\times10^{-10}$ & $62.73$ \\
Unbounded & $1{,}000$ & $6.32\times10^{-17}$ & $2.41\times10^{-16}$ & $5.12\times10^6$ \\
Unbounded & $64{,}000$ & $1.23\times10^{-17}$ & $5.87\times10^{-17}$ & $2.08\times10^6$ \\
\bottomrule
\end{tabular}
\end{table}

At $N=4$, the $q_{.90}$ of $\|\theta_\star-\theta_{true}\|_2^2$ over 1000 independent data realizations is $3.0956784\times10^{-10}$ for bounded innovations and $1.5926441\times10^{-14}$ for unbounded innovations. We apply the same non-vacuity criterion: we find the smallest integer $N_1>4$ for which the right-hand side of \eqref{cor:param:est1:eq1}, evaluated at $N_1$, falls below the corresponding 90th-percentile squared parameter error observed at $N=4$. Table~\ref{tab:numerical-gibbs-crossings} reports the fixed $N=4$ quantiles and $\log_{10}N_1$, avoiding the long decimal representations of the exact integers.

\begin{table}[th!]
\centering
\small
\caption{The $q_{.90}$ values at $N=4$ and the logarithms of the first integer sample sizes $N_1$ at which the right-hand side of \eqref{cor:param:est1:eq1} falls below them.}
\label{tab:numerical-gibbs-crossings}
\begin{tabular}{lrr}
\toprule
Innovations & $q_{.90}$ at $N=4$ & $\log_{10}N_1$ \\
\midrule
Bounded & $3.0956784\times10^{-10}$ & $28.9802$ \\
Unbounded & $1.5926441\times10^{-14}$ & $50.1296$ \\
\bottomrule
\end{tabular}
\end{table}

At $N=1000$ and $N=64000$, the right-hand side of \eqref{cor:param:est1:eq1} exceeds the median and $q_{.90}$ of $\|\theta_\star-\theta_{true}\|_2^2$ in Table~\ref{tab:numerical-gibbs-errors} by about 11-12 orders for bounded innovations and 22-23 orders for unbounded innovations. Nevertheless, the bound is not vacuous in the sense above: it falls below the 90th-percentile squared parameter error observed at $N=4$ when $\log_{10}N_1=28.9802$ for bounded innovations and $\log_{10}N_1=50.1296$ for unbounded innovations. These very large $N_1$ values show that non-vacuity under this criterion coexists with substantial looseness at the sample sizes in Table~\ref{tab:numerical-gibbs-errors}. The common prior is symmetric about $\theta_{true}$, so the very small posterior-mean errors, especially for unbounded innovations, must not be interpreted as evidence of strong learning from the data. 

\end{document}